\documentclass[10pt]{article}
\usepackage[margin=1in]{geometry}
\usepackage{amsmath,amssymb,amsthm}
\usepackage{graphicx}
\usepackage{booktabs}
\usepackage{array}
\usepackage{authblk}
\usepackage{natbib}
\usepackage{xcolor}
\usepackage{hyperref}

\graphicspath{{figures/}}

\newtheorem{theorem}{Theorem}
\newtheorem{proposition}[theorem]{Proposition}
\newtheorem{corollary}[theorem]{Corollary}

\newcommand{\starG}{\star_G}
\newcommand{\calA}{\mathcal{A}}
\newcommand{\calB}{\mathcal{B}}
\newcommand{\calS}{\mathcal{S}}
\newcommand{\calU}{\mathcal{U}}
\newcommand{\calV}{\mathcal{V}}
\newcommand{\bbR}{\mathbb{R}}
\newcommand{\bbC}{\mathbb{C}}
\newcommand{\bbZ}{\mathbb{Z}}
\newcommand{\Rsq}{R\textsuperscript{2}}

\hypersetup{colorlinks=true,linkcolor=blue,citecolor=blue,urlcolor=blue}

\usepackage{algorithm}
\usepackage{algpseudocode}
\newtheorem{lemma}[theorem]{Lemma}
\newtheorem{definition}[theorem]{Definition}
\newtheorem{remark}[theorem]{Remark}
\newcommand{\starM}{\star_M}
\newcommand{\starSig}{\star_{\Sigma}}
\newcommand{\calC}{\mathcal{C}}
\newcommand{\calI}{\mathcal{I}}
\newcommand{\calT}{\mathcal{T}}
\newcommand{\calX}{\mathcal{X}}
\newcommand{\calZ}{\mathcal{Z}}
\newcommand{\bbH}{\mathbb{H}}
\newcommand{\bbK}{\mathbb{K}}

\begin{document}

\title{\textbf{Exact Symmetry as Algebra: Group-Structured Tensors that Enforce Physical Selection Rules}}

\author[4]{Paulina Hoyos}
\author[1]{Shashanka Ubaru}
\author[5]{Dongsung Huh}
\author[1]{Vasileios Kalantzis}
\author[1]{{Kenneth L. Clarkson}}
\author[3]{Misha Kilmer}
\author[2]{Haim Avron}
\author[1]{Lior Horesh}

\affil[1]{IBM Research}
\affil[2]{Tel-Aviv University}
\affil[3]{Tufts University}
\affil[4]{University of Texas at Austin}
\affil[5]{Independent Researcher}

\date{}
\maketitle

\noindent\textbf{Significance Statement.}
Symmetry determines which physical responses of molecules and crystals are allowed. Yet many machine-learned models learn symmetry only approximately from data, leading to physically inconsistent predictions. We show that these violations can be measured, bounded, and corrected without retraining. By associating a tensor algebra with a symmetry group, we decompose prediction error into symmetry-breaking and symmetry-preserving components, obtain reference-free lower bounds on model error, and enforce exact selection rules on model outputs. This framework transforms symmetry from an approximate statistical property learned from data into an exact physical constraint enforced after training, improving the reliability of scientific machine learning across chemistry, materials science, and physics.

\begin{abstract} 
Symmetry determines which physical responses of molecules and crystals are allowed and underlies many models used throughout the physical sciences. Yet several widely used machine-learned models learn symmetry only approximately from data, raising the question of whether approximate symmetry is sufficient for scientific prediction. Using the Orb-v3 force field, we show that it is often not: the relaxed space group of a crystal depends on the orientation of the input for $47\%$ of a $150$-crystal subsample, and relaxation lowers the space group for half of all $399$ crystals studied. To characterize, quantify, and correct such failures, we introduce a tensor algebra, $\starG$, defined by convolution over a finite symmetry group and applicable to tensors with noncommuting entries. We prove that the symmetry group determines the algebra's block-diagonal transform and use this structure to analyze model errors on symmetric targets. The resulting projections decompose the average squared error into orthogonal components arising from selection-rule violations, inconsistencies across symmetry-related inputs, and residual error. The removable components reveal how much error arises from broken symmetry, provide reference-free lower bounds on prediction error, and can be eliminated without sacrificing accuracy. For operators, projection followed by blockwise truncation yields the optimal symmetric low-rank approximation. Applying these projections to force predictions reduces both rates of space-group change to at most $2\%$. Exact selection rules, reference-free error bounds, and symmetry-preserving corrections are therefore available for any trained model without retraining. 
\end{abstract}

\section{Introduction} 
\label{sec:intro} 

Symmetry organizes the natural sciences and, increasingly, artificial intelligence. In physics, continuous symmetries yield conservation laws~\citep{noether1918}, and symmetries of every kind yield selection rules that determine which responses of a system can be nonzero~\citep{wigner1959group}. In machine learning, a model that respects a task's symmetry transfers what it learns about one input to every transformed copy of it, as convolutional networks do for translated images and equivariant networks for rotated objects~\citep{cohen2016group,bronstein2021geometric}. The two traditions meet in molecular and materials modeling. Rotating a molecule or exchanging identical atoms leaves its energy unchanged, and the operations of a crystal's space group map the crystal to itself. These symmetries constrain atomic forces, make some components of elastic, dielectric, and piezoelectric tensors vanish and tie others together, and enforce degeneracies among vibrational modes~\citep{cotton1990chemical}. We call the components that a selection rule requires to vanish \emph{forbidden}. A model respects a symmetry when it is \emph{equivariant}, meaning that transforming the input transforms the output in the corresponding way.

Machine-learned models increasingly stand in for quantum-mechanical calculations of molecules and materials. Such models acquire equivariance in one of two ways. Equivariant neural networks build symmetry into every layer~\citep{thomas2018tensor,fuchs2020se3,schutt2017schnet,batzner2022e3,batatia2022mace}. Several leading foundation models, pretrained on large datasets for atomistic simulation and structure prediction, instead learn the symmetry from augmented training data while leaving their architecture unconstrained~\citep{neumann2024orb,rhodes2025orbv3,abramson2024alphafold3}. Learning symmetry from examples leaves an equivariance error that varies from input to input. This paper asks two questions. Is approximate symmetry sufficient for scientific prediction? Where it is not, can exact symmetry be restored to an already trained model without sacrificing accuracy?

Our analysis shows that approximate symmetry can fail in both structural relaxation and tensor prediction. Molecular dynamics, structural relaxation, and self-consistent iteration apply the same model thousands to millions of times, allowing small symmetry violations to accumulate. When we relaxed $399$ crystals from the Materials Project using a widely deployed foundation model, the space group of the relaxed crystal depended on the orientation of the input crystal for $47\%$ of a $150$-crystal subsample, even though orientation carries no physical meaning (Section~\ref{sec:audit}). Relaxation also lowered the space group of half of all crystals studied. A lower space group can signal a genuine instability, whereas a result that depends on orientation reflects a model error. We further show that symmetry violations can grow linearly with the number of steps along stable trajectories and exponentially along unstable ones, where increasing numerical precision merely delays the eventual loss of symmetry (Section~\ref{sec:depth}). Even a single model evaluation can be misleading: a regressor trained to predict piezoelectric tensors from composition identified $773$ of $788$ centrosymmetric crystals as piezoelectric candidates despite symmetry forbidding piezoelectricity in every case (Section~\ref{sec:crystals}).

The standard remedy projects a prediction onto the responses allowed by symmetry by averaging over symmetry-related configurations. Atomistic simulation codes symmetrize forces and stresses during relaxation~\citep{larsen2017ase}, materials software symmetrizes predicted tensors~\citep{ong2013pymatgen}, and machine-learning methods use frame averaging, symmetrized fine-tuning, and canonicalization to improve invariance and equivariance~\citep{puny2022frame,basu2023equituning,kaba2023canonicalization}. Such projections are orthogonal, so Pythagoras's theorem guarantees that they cannot increase the squared error when the target itself is symmetric. These observations motivate three further questions. Which part of a model's error can be removed by enforcing selection rules on each output, and which part requires averaging predictions over symmetry-related inputs? Can the removable component be quantified without reference calculations? How do projection, averaging, and other operations used in scientific machine learning, including low-rank compression of operators and neural-network layers, fit together for symmetry groups whose operations do not commute?

Answering these questions requires a framework in which symmetry is not an external constraint imposed after computation but an intrinsic part of the computation itself. We build such a framework by placing the symmetry group in the product of a tensor algebra. Kilmer and collaborators showed that a fixed invertible transform turns multiway arrays~\citep{kolda2009tensor} into an algebra that behaves much like the algebra of matrices, admitting analogues of the singular value decomposition and optimal low-rank approximation~\citep{kilmer2021tensor,kernfeld2015tensor}. However, their construction assumes commutative multiplication and therefore cannot directly represent noncommutative symmetry groups, including three-dimensional rotations~\citep{dunbar2025tensor}. We define a product, $\starG$, whose multiplication is convolution over an arbitrary finite group $G$, allowing tensor entries to lie in the group algebra $\bbR[G]$~\citep{serre1977linear,diaconis1988group}. Requiring the product itself to be equivariant determines its transform, thereby making the symmetry group itself determine the coordinates in which prediction, correction, and compression are analyzed. Up to basis choices within irreducible blocks that preserve ranks and the set of optimal low-rank approximations, the transform is uniquely the group-Fourier transform (Theorem~\ref{thm:canon}). The symmetry group therefore determines the algebra used to analyze, compress, and correct predictions. Through exact finite reductions, the construction also covers band-limited signals on the rotation group and the crystallographic space groups (SI Sections~\ref{sec:si_compact} and~\ref{sec:spacegroups}).

The resulting block structure provides a direct description of symmetry violations. The group-Fourier transform decomposes every element of the algebra into matrix blocks, one for each irreducible representation of $G$. Physical responses with known symmetry, including atomic forces and tensor-valued properties, decompose in the same way under the symmetry group of the underlying structure. Selection rules permit only the invariant block, corresponding to the trivial representation, to remain nonzero. Projecting a prediction onto that block therefore enforces the selection rules and cannot increase the squared error when the target is symmetric (Proposition~\ref{prop:repair}).

Averaging the predictions for all transformed copies of an input, each transformed back to the original frame, makes a model equivariant at that input. Projection and averaging are nested operations, and for symmetric targets they decompose the model's squared error, averaged over symmetry-related inputs, into three orthogonal components (Theorem~\ref{thm:three_part}). One component arises from violations of selection rules and is removed by projection alone. A second reflects inconsistency across transformed copies of an input that remains after projection and is removed by averaging. A third remains after both operations. Because the first two components are computed entirely from the model outputs, each provides a reference-free lower bound on prediction error. The same representation-theoretic blocks govern compression. Under a rank budget, the optimal symmetric approximation of an operator is obtained by projecting onto the symmetric subspace and truncating each block independently, both in the algebra (Theorem~\ref{thm:ey}) and in the feature spaces of equivariant neural networks (Theorem~\ref{thm:eymod} and SI Corollary~\ref{cor:joint_opt}). Table~\ref{tab:lean} summarizes the results that we formally verified in Lean~4.

Equivariant architectures, projection, and averaging provide complementary approaches to enforcing symmetry in scientific machine learning. Equivariant architectures build symmetry into a model from the outset. Projection and averaging instead operate on models that already exist, including foundation models that cannot easily be modified or retrained. Projection enforces the selection rules of a known symmetry using a single model evaluation followed by linear algebra on the output, whereas averaging enforces equivariance through $|G|$ model evaluations. In our crystal-relaxation experiments, projecting each force evaluation reduced both rates of space-group change to at most $2\%$. Across $30{,}280$ symmetrized structures, the forbidden force component alone certified a force error of at least $10$~meV/\AA{} for $48\%$ of them (Section~\ref{sec:audit}). Projecting predicted tensors reduced the error for more than $99\%$ of crystals whose point group constrains the tensor without forcing it to vanish (Section~\ref{sec:crystals}). The framework provides reference-free error bounds, exact selection rules, and symmetry-preserving corrections for any trained model. Projection leaves the symmetry-allowed component of the error unchanged. On pooled molecular-property prediction accuracy, equivariant graph networks nonetheless remain ahead of every $\starG$ model we trained (SI Section~\ref{sec:qm9_accuracy}).

\begin{figure}[!t]
\centering
\includegraphics[width=\textwidth]{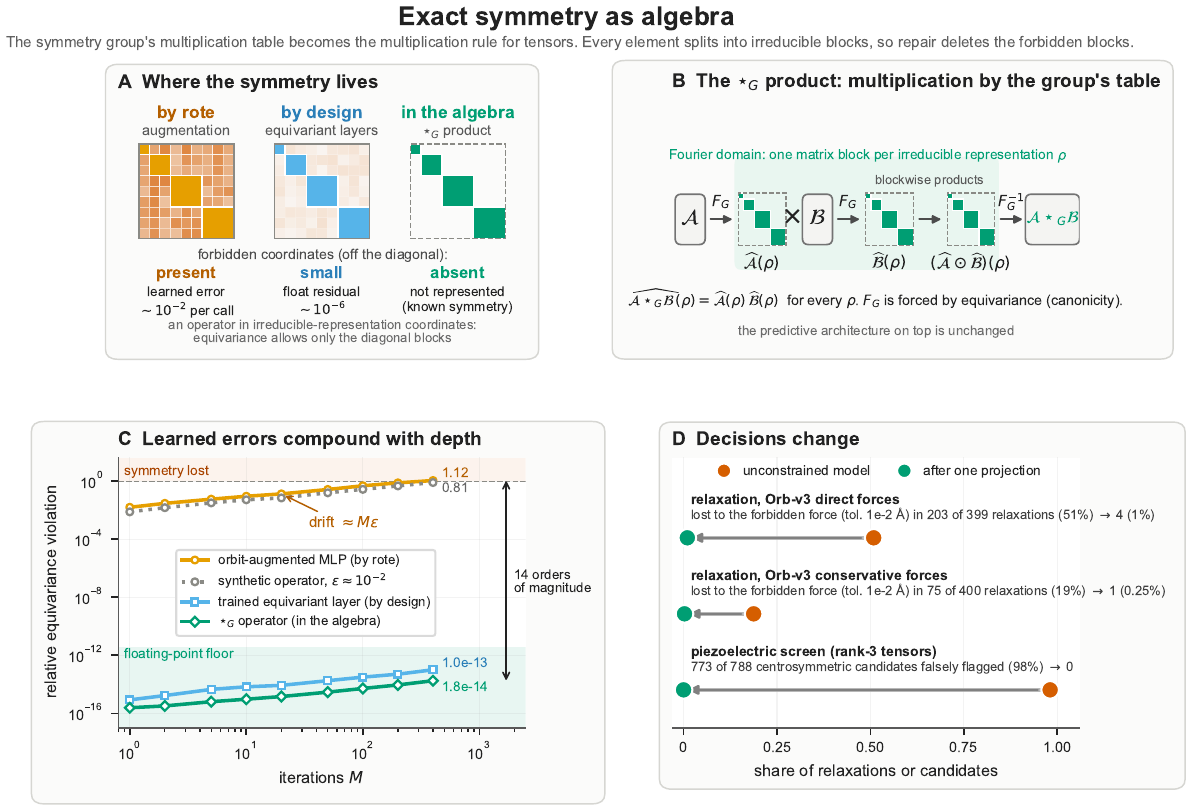}
\caption{\textbf{Exact symmetry as algebra.}
(\textbf{A}) An operator written in irreducible-representation coordinates of a known symmetry group is equivariant only if it is block diagonal, with each diagonal block acting as $I\otimes A_\rho$, so the off-diagonal blocks are symmetry-forbidden coordinates. A model that learned its symmetry from augmented data leaves forbidden components of order $10^{-2}$ per call, and an equivariant architecture leaves the floating-point residual of its implementation. In the $\starG$ block representation the forbidden coordinates are absent from the operator.
(\textbf{B}) The $\starG$ product. The group-Fourier transform $F_G$ turns group convolution into independent matrix products, one per irreducible representation $\rho$, and the inverse transform returns the result. Equivariance forces the choice of $F_G$ (Theorem~\ref{thm:canon}).
(\textbf{C}) Relative equivariance violation of the $M$-fold iterate of one-step operators on the octahedral group. An orbit-augmented perceptron and a synthetic operator with per-call error $\varepsilon\approx10^{-2}$ drift as $M\varepsilon$ and lose the symmetry by $M=400$. A trained equivariant layer and the exactly equivariant operator in ambient coordinates remain at the floating-point floor, fourteen orders of magnitude lower (SI~Section~\ref{sec:si_compounding}).
(\textbf{D}) Decisions before and after one projection per output. At a spglib tolerance of $10^{-2}$~\AA{}, the pymatgen default, the space group of $203$ of $399$ Materials Project structures under Orb-v3 with direct forces and $75$ of $400$ with conservative forces is lowered by an unconstrained relaxation and kept when each force call is projected, and $4$ and $1$ are lowered with projection (Section~\ref{sec:audit}). A composition-only regressor flags $773$ of $788$ centrosymmetric crystals as piezoelectric candidates, and none after projection (Section~\ref{sec:crystals}).}
\label{fig:framework}
\end{figure}

\section{The \texorpdfstring{$\starG$}{starG} algebra}
\label{sec:theory}

\paragraph{The product.}
Let $G$ be a finite group of order $n$. The group algebra $\mathbb{K}_G=\bbR[G]$ consists of real functions on $G$, multiplied by convolution, $(x\ast y)(c)=\sum_{a\in G}x(a)\,y(a^{-1}c)$. We store an $\ell\times m$ array of such functions as a tensor $\calA\in\bbR^{\ell\times m\times n}$ and define the $\starG$ product of $\calA$ and $\calB\in\bbR^{m\times p\times n}$ by matrix multiplication with convolution in place of scalar multiplication,
\begin{equation}\label{eq:product}
(\calA\starG\calB)_{ij}=\sum_{k=1}^{m}\calA_{ik}\ast\calB_{kj}.
\end{equation}
For a direct product $G=G_1\times\cdots\times G_d$ the group index is multi-dimensional and $\calA$ is a tensor of order $2+d$. The product commutes with left translation of the group index, $(g\cdot\calA)\starG\calB=g\cdot(\calA\starG\calB)$ for every $g\in G$, so the map $\calA\mapsto\calA\starG\calB$ with a fixed $\calB$ is equivariant, and so is every composition of such maps (SI~Section~\ref{sec:si_gft}). For an abelian group the product reduces to the $\star_M$ product of Kernfeld, Kilmer and Aeron with $M$ the discrete Fourier transform~\citep{kernfeld2015tensor}, whose Eckart--Young theory in the commutative tubal setting was developed by Mor and Avron~\citep{mor2025quasitubaltensoralgebraseparable,mor2026sufficientnecessaryconditionseckartyoung}. No commutative product of that family represents a non-abelian group~\citep{dunbar2025tensor}.

Network layers already use several ingredients of the construction. Equivariant linear maps of compact groups are group convolutions~\citep{kondor2018generalization,cohen2019general}, Clebsch--Gordan networks store activations as matrices in irreducible-representation channels~\citep{kondor2018clebsch}, steerable networks compose direct-product symmetries~\citep{cesa2022program}, and a symmetry-breaking input can select a subgroup inside one network~\citep{goel2026anysubgroup}. The $\starG$ construction places the group in the multiplication rule of the algebra that these layers compute in, so a change of group is a change of transform.

\paragraph{The block structure.}
Let $\widehat G$ denote the irreducible unitary representations $\rho$ of $G$, with dimensions $d_\rho$ satisfying $\sum_\rho d_\rho^2=n$. The group-Fourier transform $F_G$ sends a function $x$ to its blocks $\hat x(\rho)=\sum_{g}x(g)\rho(g)\in\bbC^{d_\rho\times d_\rho}$, and it turns convolution into blockwise matrix multiplication (SI~Section~\ref{sec:si_convtensor}). Applied entrywise to $\calA$, it produces one $\ell d_\rho\times m d_\rho$ block $\widehat\calA(\rho)$ per irreducible representation, and
\begin{equation}\label{eq:blocks}
\widehat{\calA\starG\calB}(\rho)=\widehat\calA(\rho)\,\widehat\calB(\rho)\qquad\text{for every }\rho\in\widehat G .
\end{equation}
For square arrays the algebra is therefore the matrix algebra $M_m(\bbR[G])$, and Equation~\eqref{eq:blocks} is the Artin--Wedderburn decomposition of $\bbC[G]$ into matrix blocks, the finite-group case of the Peter--Weyl theorem~\citep{serre1977linear,peterweil1927}. The matrix toolbox transfers block by block. The $\starG$-transpose of $\calA$ has the Hermitian transposes $\widehat\calA(\rho)^H$ as its blocks, and the singular value, QR, symmetric eigenvalue and least-squares decompositions, together with matrix functions, are computed by applying the matrix algorithm to each block. When the factors of complex-conjugate and quaternionic blocks are chosen compatibly, as in Theorem~\ref{thm:ey}, each resulting factor is again an element of the real algebra and inherits its equivariance. The dense transform costs $O(\ell m n^2)$ operations, and generalized fast Fourier transforms reduce that cost for the symmetric group and many other finite groups~\citep{clausen1989fast,maslen1997generalized}. The blockwise products cost $O(\ell m p\sum_\rho d_\rho^3)$ operations.

\paragraph{The group determines the transform.}
In the $\star_M$ framework the transform $M$ is chosen by the modeler~\citep{kilmer2021tensor,avron2025demystifying}. For data with a known symmetry, the group makes the choice.

A transform-based product in this framework is given by an invertible linear map $M$ from $\bbC[G]$ onto a direct sum of matrix algebras, with $a\star_M b=M^{-1}(Ma\cdot Mb)$ and the blockwise matrix product in the middle.

\begin{theorem}[Canonicity]\label{thm:canon}
Let $\star_M$ be a product on $\bbC[G]$ defined by an invertible transform $M$ onto a direct sum of matrix algebras. If $\star_M$ is associative with the group identity as its unit, and if $(g\cdot a)\star_M b=g\cdot(a\star_M b)$ for every $g\in G$, where $(g\cdot a)(h)=a(g^{-1}h)$ is left translation, then $\star_M$ is group convolution, its block sizes are the irreducible dimensions $d_\rho$, and $M$ equals $F_G$ up to a change of basis within each block and a relabeling of blocks of equal size.
\end{theorem}

\noindent Theorem~\ref{thm:canon} characterizes the group-Fourier transform among transform-based products by translation equivariance. Its abelian case is a finite-group form of Jaming's characterization of the Fourier transform as the map that carries convolution to a pointwise product~\citep{jaming2009characterization}. Dunbar and Newman showed that a symmetry-adapted transform block-diagonalizes the product~\citep{dunbar2025tensor}, and Theorem~\ref{thm:canon} shows that equivariance requires that transform. The proof (SI~Theorem~\ref{thm:necessity}) shows that equivariance and the unit make left multiplication by the indicator of $g$ a translation by $g$, and then fixes the transform by the uniqueness of the Artin--Wedderburn decomposition up to the automorphisms of the Skolem--Noether theorem. The theorem matters for transform-based tensor algebras, in which the transform is a free invertible matrix and the rank of a tensor and its optimal truncation depend on that choice~\citep{kilmer2021tensor}. For data with a known symmetry the transform is determined up to a change of basis within each block, and that change of basis preserves the rank of every block. The low-rank problem, which measures the error in the original coordinates subject to a rank bound for each irreducible representation, therefore depends on the group alone, and so does its set of minimizers, which Theorem~\ref{thm:ey} computes in a basis of unitary representations. A real transform onto full real matrix blocks satisfies the theorem's equivariance condition only when every irreducible representation of $G$ is of real type, which for an abelian group requires every element to square to the identity (SI~Corollary~\ref{cor:real_canon}). The $\star_M$ products built on the discrete cosine transform or another real transform with real blocks are therefore not equivariant under cyclic shifts of length three or more.

The theorem has a definite scope. Its first part, that the product is group convolution, needs only a bilinear product that commutes with left translation and has the group identity as its unit, so it holds for every such product, transform-based or not (SI~Remark~\ref{rem:partI_strength}). Its second part, which fixes the transform, uses the transform-based form of the product. The unit condition is a normalization. A product with any other two-sided unit $u$ is the twisted convolution $a\starG u^{-1}\starG b$, which is isomorphic to group convolution and has the same blocks (SI~Lemma~\ref{lem:twisted} and Remark~\ref{rem:unit_normalization}). The group acts exactly on a model's inputs and outputs even when the model has learned its equivariance only approximately, so the algebra applies to the outputs of every model. The projection below uses the exact symmetry group of the structure, which a symmetry finder determines from the atomic positions, and the model's approximate equivariance affects only the size of the part that the projection removes.

\paragraph{Optimal truncation.}
Let $\sigma_1(\rho)\ge\sigma_2(\rho)\ge\cdots$ be the singular values of the block $\widehat\calA(\rho)$. The Frobenius norm of $\calA$ obeys the Plancherel identity $\|\calA\|_F^2=\sum_\rho (d_\rho/n)\,\|\widehat\calA(\rho)\|_F^2$, and this weighting is what makes blockwise truncation optimal.

\begin{theorem}[Eckart--Young for $\starG$]\label{thm:ey}
For any rank profile $(k_\rho)_{\rho\in\widehat G}$ with equal ranks on each pair of complex-conjugate representations and even ranks on representations of quaternionic type, truncating each block $\widehat\calA(\rho)$ to its leading $k_\rho$ singular values, with singular vectors chosen compatibly on conjugate and quaternionic blocks, minimizes $\|\calA-\calB\|_F^2$ over all real $\calB$ whose block at $\rho$ has rank at most $k_\rho$, and the minimum equals $\sum_\rho (d_\rho/n)\sum_{j>k_\rho}\sigma_j(\rho)^2$. In the $\starG$ singular value decomposition $\calA=\calU\starG\calS\starG\calV^H$, the diagonal entries of $\calS$ are singular elements of $\mathbb{K}_G$, and keeping the leading $k$ of them is the case $k_\rho=k\,d_\rho$.
\end{theorem}

\noindent Theorem~\ref{thm:ey} extends the transform-domain Eckart--Young theorem of the $\star_M$ framework~\citep{kilmer2021tensor,kilmer2011factorization}, whose admissible transforms Mor characterized in the commutative setting~\citep{mor2025eckart}, to non-commuting entries, where each block is a matrix and the Plancherel weights $d_\rho/n$ enter the optimum. For a real tensor the blocks of a complex-conjugate pair are complex conjugates of each other, and a block of quaternionic type carries a quaternionic structure whose singular values come in pairs, so any other profile reduces to the largest admissible one below it with the same minimum (SI~Corollary~\ref{cor:real_profiles}). Every profile is admissible when all irreducible representations are of real type, as for the octahedral group and for $\mathrm{SO}(3)$. Truncation in a symmetry-adapted block basis is standard practice in non-abelian density-matrix renormalization and in symmetric tensor networks~\citep{mcculloch2002nonabelian,singh2010tensor,singh2012tensor,weichselbaum2012nonabelian}, and Theorem~\ref{thm:ey} supplies the optimality statement and the error certificate for that operation.

The feature spaces of equivariant networks do not have the multiplicities of the algebra above. A space $V=\bigoplus_\rho V_\rho\otimes\bbR^{m_\rho}$ carries the irreducible representation $\rho$ on $V_\rho$ with an arbitrary multiplicity $m_\rho$, and an equivariant linear map on $V$ acts as $\mathrm{id}_{V_\rho}\otimes A_\rho$ on each summand by Schur's lemma, with $A_\rho$ an $m_\rho\times m_\rho$ matrix, for representations of real type such as those of $\mathrm{SO}(3)$ and the octahedral group (complex and quaternionic types enlarge the scalars, SI~Remark~\ref{rem:field}). The linear channel-mixing maps of a MACE interaction layer with $96$ channels at each angular momentum $\ell\le3$ are such maps, each with four blocks $A_\ell\in\bbR^{96\times96}$ (SI~Remark~\ref{rem:enn_multiplicity}). Compressing such a layer should preserve what the network computes, so the error is measured in an inner product that the modeler chooses to reflect it. Let $\Sigma$ and $\Gamma$ be positive definite operators that commute with the group action, for example the second moment of the layer's inputs and of the gradient of the network's energy with respect to the layer's outputs, each averaged over the group. By Schur's lemma they have blocks $\mathrm{id}_{V_\rho}\otimes\Sigma_\rho$ and $\mathrm{id}_{V_\rho}\otimes\Gamma_\rho$.

\begin{theorem}[Optimal truncation on a representation space]\label{thm:eymod}
For a representation space whose irreducible constituents are of real type and any rank profile $(k_\rho)$, the equivariant map $B$ with blocks $B_\rho=\Gamma_\rho^{-1/2}\big[\Gamma_\rho^{1/2}A_\rho\Sigma_\rho^{1/2}\big]_{k_\rho}\Sigma_\rho^{-1/2}$, where $[\cdot]_{k}$ is the rank-$k$ truncated singular value decomposition, minimizes $\|\Gamma^{1/2}(T-B)\Sigma^{1/2}\|_F^2$ over equivariant maps whose block at $\rho$ has rank at most $k_\rho$. The minimum equals $\sum_\rho d_\rho\sum_{j>k_\rho}\sigma_j\big(\Gamma_\rho^{1/2}A_\rho\Sigma_\rho^{1/2}\big)^2$.
\end{theorem}

\noindent The proof is in SI~Corollary~\ref{cor:ey_module_weighted}. With $\Sigma$ and $\Gamma$ the identity the statement is SI~Theorem~\ref{thm:ey_module}, and on the free module $m_\rho=m\,d_\rho$ it is Theorem~\ref{thm:ey}. With $\Gamma$ the identity the quantity minimized is the expected squared output error of the layer, and with $\Gamma$ the energy-gradient second moment it is the Kronecker-factored estimate of the expected squared change in energy~\citep{martens2015kfac}. Weighted truncations of this kind are the basis of activation-aware and Fisher-weighted compression of language models~\citep{hsu2022weighted,yuan2023asvd,wang2024svdllm}, and Theorem~\ref{thm:eymod} states when they preserve equivariance and are optimal. The block structure of $\Sigma$ and $\Gamma$ reduces the problem to independent weighted matrix approximations, so the proof is the classical one applied per block. When every block has the same size, as in MACE, and a truncated block is stored as a pair of factors, each retained rank costs the same number of parameters in every block, so a fixed budget is spent optimally by keeping the globally largest values of $d_\rho\,\sigma_j(\Gamma_\rho^{1/2}A_\rho\Sigma_\rho^{1/2})^2$. Projection and truncation also form a single optimization. The best equivariant approximation of any operator under a rank profile is the truncation of its projection onto the equivariant operators, and its error is the forbidden energy of the operator plus the discarded part of the spectrum (SI~Corollary~\ref{cor:joint_opt}).

\paragraph{Projection.}
Projecting a prediction of a quantity whose symmetry is known onto its invariant block enforces the selection rules. Let $S$ be the subspace of tensors that the symmetry allows, for example the elastic tensors permitted by a crystal's point group, and let $P$ be the orthogonal projector onto $S$. For a group of operations acting on the tensor, $P$ is the group average of those operations (the Reynolds operator), and in the language of the group algebra it is the action of the idempotent $e=|G|^{-1}\sum_{g}g$, which keeps the invariant component of the tensor's representation and removes the others (SI~Section~\ref{sec:spacegroups}). The projection enforces the selection rules of a structure whose symmetry is given. It does not make the model equivariant at other inputs, which requires averaging the model over the group, at $|G|$ evaluations per call, or an equivariant architecture. In dynamics the projection applies along a trajectory that keeps the symmetry, where the symmetry group of the configuration stays fixed. Selection rules of this kind are classical results of computation in symmetry-adapted bases~\citep{wigner1959group,cotton1990chemical}.

\begin{proposition}[Projection cannot increase the error]\label{prop:repair}
Write $Q=I-P$, with norms in the Euclidean inner product in which $P$ is orthogonal. For a prediction $C$ of a target $C^\star$,
\begin{equation}\label{eq:repair}
\|PC-C^\star\|^2=\|C-C^\star\|^2-\|Q(C-C^\star)\|^2+\|QC^\star\|^2 .
\end{equation}
If the target satisfies the symmetry, $QC^\star=0$, then projection lowers the squared error by the forbidden energy $\|QC\|^2$ of the prediction. Otherwise projection lowers the error whenever the forbidden energy of the residual exceeds that of the target.
\end{proposition}

\noindent The identity is Pythagoras's theorem in the orthogonal decomposition $S\oplus S^\perp$ (SI~Proposition~\ref{prop:enforce_guarantee}). Its second case covers a measured tensor that is symmetric only to within the accuracy of the calculation that produced it.

\paragraph{Projection of the model.}
The projection acts on one prediction for one structure. Averaging the model itself over a finite group, $(\mathcal{E}f)(x)=|G|^{-1}\sum_g\rho(g)^{-1}f(gx)$ at $|G|$ evaluations per input, makes it equivariant at every input. Write $(\mathcal{R}f)(x)$ for the projection of $f(x)$ by the stabilizer of $x$, the operation of Proposition~\ref{prop:repair}.

\begin{theorem}[Three parts of a model's error]\label{thm:three_part}
For inputs drawn from a distribution invariant under $G$, $\mathcal{E}$ and $\mathcal{R}$ are orthogonal projectors on the square-integrable models, and every equivariant model lies in the range of $\mathcal{R}$. For every model $f$ and every equivariant target $f^\star$,
\begin{equation}\label{eq:three_part_main}
\|f-f^\star\|^2=\|f-\mathcal{R}f\|^2+\|\mathcal{R}f-\mathcal{E}f\|^2+\|\mathcal{E}f-f^\star\|^2 .
\end{equation}
At each input $x$, the variance of the back-rotated predictions $\rho(g)^{-1}f(gx)$ over $g\in G$ bounds from below the mean squared error of $f$ over the orbit of $x$.
\end{theorem}

\noindent The first term is the selection-rule error, which one projection per output removes. The second is the inconsistency of the model across transformed copies of each input that remains after projection, which averaging removes, and the third is the error that both operations leave unchanged. At a structure with no symmetry of its own the first term vanishes and the second carries the whole violation. The forbidden energy of Proposition~\ref{prop:repair} and the orbit variance are therefore two error floors that need no reference calculation, one from a single evaluation and one from the evaluations on an orbit (SI~Theorem~\ref{thm:two_level}). Each statement is the Hilbert-space form of the Reynolds average used to analyze invariant learning~\citep{elesedy2021provably}, and the theorem adds their nesting, which assigns each part of the error to the intervention that removes it.

\paragraph{Accumulation.}
A step map that is equivariant only to within $\varepsilon$ violates the symmetry by an amount that can grow with the number of steps. Let $T$ be an exactly equivariant step with Lipschitz constant $L$, which maps symmetric states to symmetric states, and let $T_\varepsilon$ differ from $T$ by at most $\varepsilon$ at every state of a trajectory that starts from a symmetric state. The distance $d_k$ of the $k$-th state from the symmetric states then obeys $d_{k+1}\le L\,d_k+\varepsilon$, for nonlinear steps as for linear ones (SI~Proposition~\ref{prop:depth_law_nonlinear}). For $L\le1$ the violation after $M$ steps is at most $M\varepsilon$, and for linear steps the bound is attained when the per-step errors add coherently (SI~Proposition~\ref{prop:depth_law}). For $L=e^{\lambda\Delta t}>1$, as on a dynamically unstable trajectory with time step $\Delta t$, the bound becomes $M\varepsilon e^{\lambda t}$ at time $t=M\Delta t$, so the bound reaches a fixed size after a time of order $\lambda^{-1}\log(1/\varepsilon)$ (SI~Proposition~\ref{prop:depth_law_growth}). An exactly equivariant step has $\varepsilon=0$ and violates the symmetry at no depth.

\paragraph{Composition and compact groups.}
Several symmetries enter through one group.

\begin{theorem}[Product groups]\label{thm:product}
Let $\mathcal{T}_G$ be the multiplication table of $G$, the tensor with $\mathcal{T}_G(a,b,c)=1$ when $ab=c$ and $0$ otherwise. For $G=G_1\times\cdots\times G_d$, the multiplication table and the group-Fourier transform factorize as Kronecker products, $\mathcal{T}_G=\mathcal{T}_{G_1}\otimes\cdots\otimes\mathcal{T}_{G_d}$ and $F_G=F_{G_1}\otimes\cdots\otimes F_{G_d}$.
\end{theorem}

\noindent Adding an independent symmetry, one that enters as a direct-product factor, therefore changes the transform by one Kronecker factor and leaves every algorithm unchanged (SI~Section~\ref{sec:si_product_proof}). For a compact group such as $\mathrm{SO}(3)$, a signal band-limited to finitely many irreducible representations has finitely many blocks, and the band-limited signals form a finite-dimensional algebra to which every result above applies, for linear operations that keep the band limit (SI~Section~\ref{sec:si_compact}). For $\mathrm{SO}(3)$ the blocks are the angular-momentum channels $\ell\le\ell_{\max}$, which are the channels of e3nn, NequIP and MACE. A finite subgroup can then reproduce the continuous average exactly.

\begin{proposition}[Finite subgroups that suffice]\label{prop:subgroup}
Let $H\subset\mathrm{SO}(3)$ be finite and $L\ge1$. On signals band-limited to $\ell\le L$, the group average over $H$ equals the group average over $\mathrm{SO}(3)$ if and only if no representation $\rho_\ell$ with $1\le\ell\le L$ contains a nonzero $H$-invariant vector, that is, $\sum_{h\in H}\chi_\ell(h)=0$ for $1\le\ell\le L$. The chiral tetrahedral, octahedral and icosahedral groups satisfy this up to $L=2$, $3$ and $5$, and no further.
\end{proposition}

\noindent The proof restricts each average to the $\ell$-isotypic component, where the $H$-average is the projector onto the $H$-invariants of $\rho_\ell$ and the $\mathrm{SO}(3)$-average vanishes for $\ell\ge1$ (SI~Proposition~\ref{prop:subgroup_si}, verified numerically to $10^{-15}$). The statement concerns averages and invariant projections. Equivariance of a map under $\mathrm{SO}(3)$ needs the corresponding condition on the intertwiners between its input and output channels. The proposition explains which continuous constraints a crystal's finite point group imposes. A symmetric rank-2 tensor carries $\ell=0,2$, so any cubic point group makes it isotropic, as the cubic dielectric tensors of Section~\ref{sec:crystals} are. The elastic tensor carries $\ell=0,2,4$, and the octahedral group retains an invariant at $\ell=4$, which is the cubic anisotropy channel.

\paragraph{Formal verification.}
We formalized the algebra in Lean~4 with Mathlib~\citep{demoura2021lean4,mathlib2020}, in $15$ modules and $3{,}608$ lines with no unresolved proof obligations. Table~\ref{tab:lean} gives the status of each main result. A proved result with no entry in the last column depends only on Lean's three standard axioms. The Eckart--Young theorem and the QR reduction rest in addition on classical results that we state as axioms, because Mathlib lacks them in the form we need, and the matrix Eckart--Young theorem is derived from them. The audit file prints the axiom dependencies of $59$ results (SI~Section~\ref{sec:lean}).

\begin{table}[!t]
\centering
\caption{\textbf{Lean~4 status of the main results.} Proved means proved from Lean's three standard axioms together with the classical results, if any, that the last column lists.}
\label{tab:lean}
\small
\begin{tabular}{@{}lll@{}}
\toprule
\textbf{Result} & \textbf{Status} & \textbf{Assumed in addition} \\
\midrule
Algebra laws and equivariance of $\starG$ & Proved & None \\
Theorem~\ref{thm:canon}, product is group convolution & Proved & None \\
Theorem~\ref{thm:canon}, transform is $F_G$ & Not formalized & \\
Proposition~\ref{prop:repair}, projection identity & Proved & None \\
Theorem~\ref{thm:three_part}, three parts of the error & Not formalized & \\
Accumulation laws, linear and nonlinear & Proved & None \\
Theorem~\ref{thm:product}, multiplication table & Proved & None \\
Theorem~\ref{thm:ey}, Eckart--Young & Proved & Plancherel, rank capture, Ky Fan \\
QR, eigenvalue and least-squares reductions & Proved & Matrix QR, for QR only \\
Theorem~\ref{thm:eymod}, representation spaces & Not formalized & \\
Octahedral selection rules & Proved by computation & Lean's compiled evaluation \\
\bottomrule
\end{tabular}
\end{table}

\section{Results}

\subsection{Selection-rule violations of foundation force fields across crystals}
\label{sec:audit}

The trajectories above follow single systems. To test whether such violations are a regularity of non-equivariant foundation models, we audited five pretrained models on $40$ crystals of the de Jong dataset, chosen across five crystal systems with primitive cells of at most ten atoms (Figure~\ref{fig:fm_audit} and SI~Section~\ref{sec:si_fm_audit}). The models were Orb-v2, two variants of its successor Orb-v3~\citep{rhodes2025orbv3}, one with a direct force head and one whose forces are the gradient of its energy, and MACE-MP-0 in single and double precision. At each density-functional structure the space-group Reynolds operator splits a model's forces into an allowed and a forbidden part. The input is the spglib-standardized cell, whose positions carry the symmetry to floating-point precision, so the forbidden part is a violation by the model.

The forbidden force separates the models that learned their symmetry from the equivariant one, although the models also differ in training data and force heads. Its median over crystals was $4\times10^{-4}$ to $7\times10^{-3}$~eV/\AA{} for the three Orb models, $7\times10^{-7}$~eV/\AA{} for MACE-MP-0 in single precision and $2\times10^{-15}$~eV/\AA{} in double precision (Figure~\ref{fig:fm_audit}A). An unconstrained relaxation from the density-functional structure changed the space group that spglib detects at a tolerance of $10^{-3}$~\AA{} in $20$, $34$ and $15$ of the $40$ crystals under Orb-v2, Orb-v3 with the direct head and Orb-v3 with gradient forces, and in none under MACE-MP-0. With each force call projected, every model kept every space group. On these well-relaxed crystals the atoms moved off their symmetric positions by a median of $2\times10^{-4}$ to $10^{-3}$~\AA{} under the three Orb models, so at the $10^{-2}$~\AA{} tolerance that pymatgen uses by default the space group changed in only $1$ to $3$ of the $40$ crystals, and the consequence lay in symmetry-sensitive quantities such as the vibrational multiplets.

Vibrational degeneracies show the same ordering of the models. For the $31$ crystals whose point groups protect multiplets at the $\Gamma$ point, we computed the finite-difference Hessian at the standard step of $0.01$~\AA. The largest splitting of a protected multiplet had a median of $0.7$~cm$^{-1}$ (maximum $3.0$) for Orb-v3 with the direct head and $0.14$~cm$^{-1}$ (maximum $0.46$) for Orb-v3 with gradient forces, below typical Raman linewidths, while MACE-MP-0 stayed below $0.02$~cm$^{-1}$, the floor of the finite-difference stencil (Figure~\ref{fig:fm_audit}B). Group-averaging the Hessian restored every protected degeneracy to within $2\times10^{-12}$~cm$^{-1}$ for every model. Because $g$ fixes the symmetric configuration, the averaged force-constant matrix is the Jacobian of the group-averaged force law $\bar F(x)=|G|^{-1}\sum_g g^{-1}F(gx)$, an equivariant force law that costs $|G|$ force calls per evaluation.

The audit also shows the limit of the projection. The symmetric structure of SrSbSe$_2$F is a saddle point of every model, MACE-MP-0 included, with four imaginary modes. Both Orb-v3 models relaxed away from it to space group $P1$, whereas MACE-MP-0 stayed at $P4/nmm$, because its forces carry no forbidden component to start the escape. A projected relaxation holds the saddle under every model, so the unprojected phonon calculation of Section~\ref{sec:depth} remains the test of stability.

We repeated the audit at screening scale on Materials Project structures~\citep{jain2013commentary} for the two Orb-v3 models and MACE-MP-0 in single precision (Figure~\ref{fig:fm_audit}D and SI~Section~\ref{sec:si_fm_screening}). At a symmetric structure the reference forces of any symmetric method lie in the allowed subspace, so Proposition~\ref{prop:repair} bounds a model's force error from below by its forbidden force, without a reference calculation. Over $30{,}280$ symmetrized structures the bound certifies a root-mean-square force error of at least $9.7$~meV/\AA{} at the median structure under Orb-v3 with the direct head, and of at least $10$~meV/\AA{} in $48\%$ of structures. The floor does not come from trivially constrained structures. The $13\%$ of structures in which the symmetry allows no force at all carry $0.2$ to $1.8\%$ of the forbidden energy, and the structures whose atoms all sit on general positions, where the floor comes only from the relations between symmetry-equivalent atoms, carry about $20\%$ (SI~Table~\ref{tab:cert_strata}). Under Orb-v3 with gradient forces the certified median is $2.0$~meV/\AA{}, and under MACE-MP-0 no structure reaches $1$~meV/\AA{}. The second floor of Theorem~\ref{thm:three_part}, the variance of the back-rotated forces over the $24$ rotations of the cube, needs no symmetry of the structure. On $45$ of these structures its median was $8.0$~meV/\AA{} under Orb-v3 with the direct head and $2.1$~meV/\AA{} with gradient forces, close to the forbidden force at the same structures. On randomly displaced copies, which have no symmetry and therefore no forbidden force, it kept the value of the parent structure to within $14\%$ (SI~Table~\ref{tab:cert_orbit}).

In $400$ relaxations of structures with at most $20$ atoms, sampled in proportion to the crystal systems, the unconstrained relaxation lowered the space group detected at a tolerance of $10^{-2}$~\AA{}, the pymatgen default, and the projected relaxation kept it in $75$ of $400$ structures ($19\%$) under Orb-v3 with gradient forces, $203$ of $399$ ($51\%$, one run failed) under the direct head and $7$ of $400$ ($2\%$) under MACE-MP-0. Counting only relaxations in which both arms reached the force tolerance, these shares were $12\%$ of $346$, $48\%$ of $250$ and $0.3\%$ of $380$, so unconverged runs account for part of the share under gradient forces and little of the share under the direct head. Under each model $27$ to $29\%$ of the structures are unstable under its own forces. Under the direct head $109$ of the $203$ lowerings started from a structure with no unstable symmetry-breaking mode, so in these runs the forbidden force, without a detected initial instability, moved the atoms off it, while $70$ of the $75$ lowerings under Orb-v3 with gradient forces started from an unstable structure. We tested these lowerings with density-functional calculations on a prespecified sample (SI~Section~\ref{sec:si_fm_screening}). In $14$ of the $15$ tested lowerings under the direct head that started with no unstable mode under any of the three models, the symmetric structure was a stable density-functional minimum, and the calculations found no lower minimum for the remaining one. These lowerings are therefore errors of the model. Where all three models found an unstable mode, the calculations confirmed a lower-symmetry minimum in $6$ of $10$ structures. Projected relaxations lowered no space group at tolerances up to $10^{-3}$~\AA{} and at most $1\%$ at $10^{-2}$~\AA{}. At the $0.1$~\AA{} tolerance of Materials Project labels, the unconstrained relaxations kept the reported label in $74$ to $97\%$ of structures and the projected ones in $98$ to $99\%$ (SI~Section~\ref{sec:si_fm_screening}).

The reported label also depended on the orientation of the input crystal, a choice that leaves the physical input unchanged. We relaxed $150$, $120$ and $80$ of these structures in four orientations related by rigid rotations. Without projection, the space group at $10^{-2}$~\AA{} differed between orientations in $47\%$ of structures under the direct head, $13\%$ under Orb-v3 with gradient forces and $3\%$ under MACE-MP-0, against $1$ to $5\%$ when the same orientation was run twice. At $0.1$~\AA{} it differed in $23\%$, $8\%$ and $1\%$. Under the direct head the dependence was as frequent for structures that are stable under the model ($49\%$ of $106$). With projection the label differed in at most $2\%$ of structures.

\begin{figure}[!t]
\centering
\includegraphics[width=\textwidth]{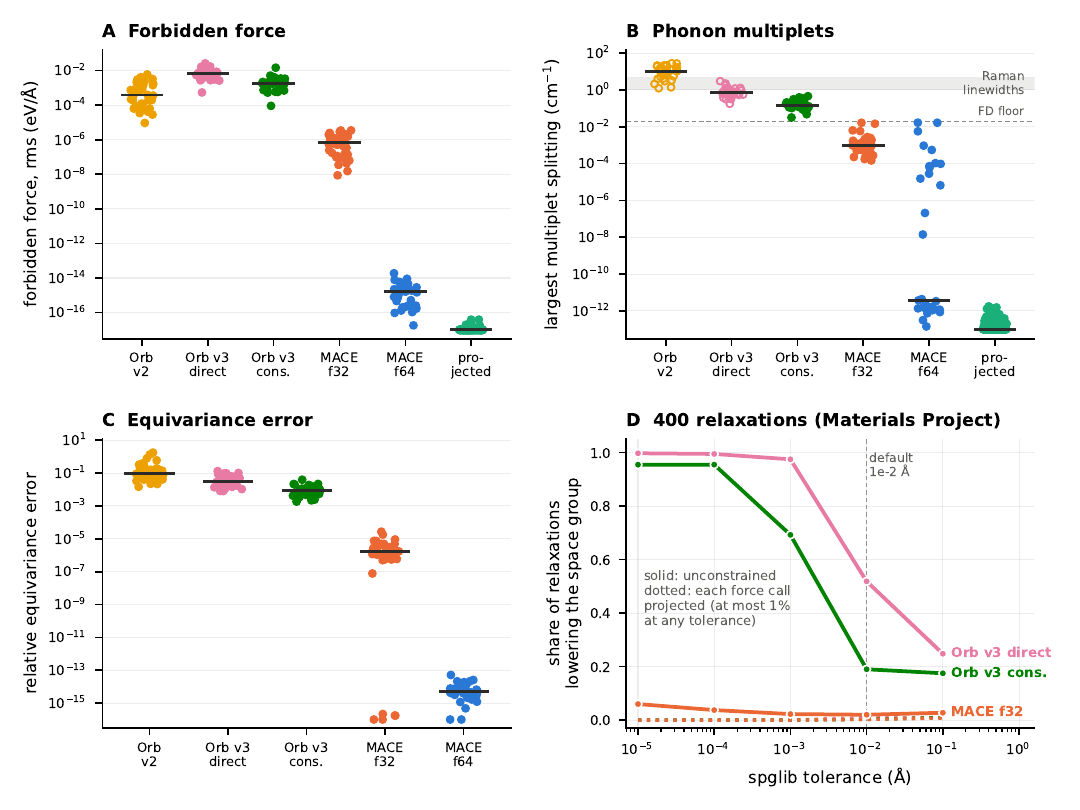}
\caption{\textbf{Selection-rule violations of pretrained force fields.} In A to C each point is one of $40$ crystals at its density-functional structure, and bars mark medians. (\textbf{A}) Root-mean-square forbidden force, the component of the predicted forces outside the symmetry-allowed subspace of the space group. The last column shows every model after the Reynolds projection. (\textbf{B}) Largest splitting of a symmetry-protected $\Gamma$-point multiplet at the standard finite-difference step of $0.01$~\AA, for the $31$ crystals with protected multiplets. Open markers denote direct-force models, whose spectra depend on the step. The shaded band marks typical Raman linewidths ($1$ to $5$~cm$^{-1}$) and the dashed line the floor of the finite-difference stencil. (\textbf{C}) Relative equivariance error of the forces under the point-group operations, at a configuration rattled by $0.03$~\AA. (\textbf{D}) Share of $400$ Materials Project relaxations whose space group is lowered, against the spglib tolerance, without projection (solid) and with each force call projected (dotted), for Orb-v3 with direct and with conservative forces and single-precision MACE-MP-0. The dashed line marks the tolerance of $10^{-2}$~\AA{}, the pymatgen default (SI~Section~\ref{sec:si_fm_screening}).}
\label{fig:fm_audit}
\end{figure}

\subsection{Per-call violations compound on production force fields}
\label{sec:depth}

Molecular dynamics applies a force field once per time step, which is where the accumulation laws decide a physical outcome. We evaluated pretrained force fields at their published settings: the equivariant MACE architecture, as MACE-OFF for organic molecules~\citep{batatia2022mace,kovacs2023maceoff} and MACE-MP-0 for materials~\citep{batatia2023foundation}, and Orb-v2, a non-equivariant model trained with rotational augmentation~\citep{neumann2024orb}. At a configuration with a site symmetry, the forces belong to the totally symmetric channel, and any force component in another irreducible representation violates a selection rule. Decomposing predicted forces into the isotypic components of the site-symmetry group, we measured a relative equivariance error of about $4\times10^{-2}$ per call for Orb-v2, $10^{-6}$ for MACE in single precision and $10^{-15}$ in double precision. A controlled iteration on the octahedral group, at per-step errors spanning those reported for trained approximately equivariant models~\citep{gruver2023lie,wang2022approximately}, reproduces the linear accumulation law (Figure~\ref{fig:framework}C and SI~Section~\ref{sec:si_compounding}). The qualitative contrast between exact and approximate rollouts has also been observed in concurrent work~\citep{wang2026exactequivariance}, and the measurements below quantify it on production models.

\begin{figure}[!t]
\centering
\includegraphics[width=\textwidth]{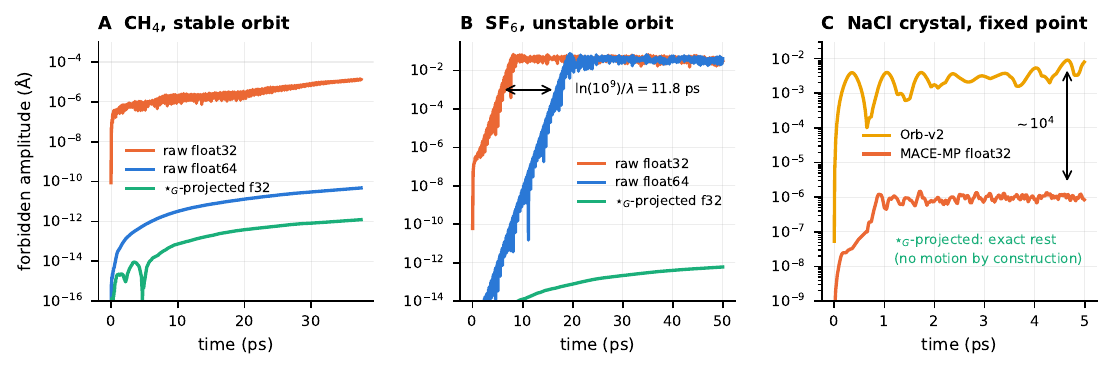}
\caption{\textbf{Accumulation of symmetry violations on production force fields.}
Symmetry-forbidden amplitude under constant-energy dynamics started from an exactly symmetric configuration, with rigid translations and rotations excluded. The projected arm applies the Reynolds projection to each force call, with the same model and precision. The projection adds one small matrix product per call (SI~Section~\ref{sec:foundation_protocol}).
(\textbf{A})~Methane under MACE-OFF, whose C--H bond length lies within $0.04\%$ of experiment. In single precision the model develops about $10^{-5}$~\AA{} of forbidden amplitude over $37.5$~ps, consistent with linear accumulation, whereas double precision and the projected model remain at the numerical floor. Energy drift is comparable across arms.
(\textbf{B})~SF$_6$ on a dynamically unstable orbit. The internal E and T$_2$ forbidden channels grow exponentially at the same rate in single and double precision, and the nine additional digits of double precision delay the breakdown by $\ln(10^9)/\lambda=11.8$~ps.
(\textbf{C})~The ideal rocksalt crystal, a fixed point of the exact dynamics. After $5$~ps with the center of mass removed, Orb-v2 displaces the atoms by about $0.5\%$ of the interatomic spacing, four orders of magnitude more than single-precision MACE-MP. In every panel the projection holds the forbidden components at the numerical floor.}
\label{fig:depth_law}
\end{figure}

On an unstable trajectory the violation grows exponentially, as Proposition~\ref{prop:depth_law_growth} predicts (Figure~\ref{fig:depth_law}B). We started SF$_6$ at the model's symmetric equilibrium, an orbit that the exact dynamics never leaves but on which perturbations grow. The internal forbidden channels grew at the same exponential rate in single and double precision, so the nine additional digits of double precision delayed the breakdown by $11.2$~ps, close to the predicted $\ln(10^9)/\lambda=11.8$~ps for the growth rate $\lambda\approx1.75$~ps$^{-1}$ fitted to the forbidden channels and a factor of about three in time. Every departure from this orbit is numerical in origin, so the measured quantity is an escape time, and the breakdown time grows only as $\log(1/\varepsilon)$ in the per-call error.

On the stable trajectory we measured, the violation grew linearly. Under MACE-OFF in single precision, the internal forbidden amplitude of methane reached about $10^{-5}$~\AA{} after $37.5$~ps (Figure~\ref{fig:depth_law}A). Fitting the slope of the amplitude against the number of steps in successive windows gave a mean exponent of $0.96$ over three decades, the coherent law of Proposition~\ref{prop:depth_law} and above the exponent $1/2$ of independent errors. Extrapolated to one nanosecond, the amplitude stays near $10^{-4}$~\AA{}, two orders of magnitude below the error of the underlying electronic-structure method. The ideal rocksalt crystal shows the learned case (Figure~\ref{fig:depth_law}C). Its atoms are at rest under the exact dynamics, yet after $5$~ps Orb-v2 displaced them by $8\times10^{-3}$~\AA{}, about $0.5\%$ of the interatomic distance, against $8\times10^{-7}$~\AA{} for MACE-MP in single precision. The displacement ratio of about $10^4$ matches the ratio of $7\times10^3$ between the two models' per-call forces at the ideal geometry.

Projecting each force call onto its allowed channels kept every forbidden channel below $10^{-12}$ throughout these trajectories and left the allowed component of each force call unchanged. Force projection is established practice in symmetry-constrained relaxation~\citep{larsen2017ase}, and Proposition~\ref{prop:repair} adds the guarantee that it cannot increase the force error on the symmetric orbit, where the exact force is symmetric. The prescription applies to workloads in which the symmetric configuration is the physical target, such as relaxation, phonon calculations and elastic response. It does not apply to thermal sampling, where departures from symmetry are physical. A projected relaxation also holds a symmetric stationary point that the physical system would leave, so an unprojected phonon calculation remains the test of whether the symmetric structure is stable.

\subsection{Selection rules on crystal tensors of three ranks}
\label{sec:crystals}

A crystal's point group fixes which components of its response tensors vanish and which are tied. We decomposed the $1{,}181$ elastic tensors of the de Jong dataset~\citep{dejong2015charting} into their $\ell=0,2,4$ rotational channels without using symmetry labels. Among the $452$ cubic crystals the median $\ell=2$ energy fraction is $1.4\times10^{-10}$ and $99\%$ lie below $10^{-3}$, as the cubic selection rule requires, so the density-functional tensors obey their selection rules to within the accuracy of the relaxations that produced them.

Trained regressors do not obey them. We built the projector $P$ of Proposition~\ref{prop:repair} for every crystal from its spglib-determined point group, for all seven crystal systems, and applied it to the predictions of two kinds of model. An unconstrained gradient-boosted regressor placed a mean fraction $5\times10^{-3}$ of its predicted elastic energy in forbidden channels, and a trained crystal graph network placed $0.9$ to $2.6\times10^{-2}$ across training-set sizes. Projection reduced both to machine precision ($1.5\times10^{-32}$ for the graph network at every training size) and left the allowed channels unchanged. Symmetry-aware tensor predictors use architectural constraints, space-group projections, canonicalization or inference-time corrections~\citep{wen2024matten,yan2024gmtnet,hua2024goectp,jin2026ceitnet}, and for architectures that already enforce the point group the projector acts as an exact-zero certificate.

\begin{figure}[!t]
\centering
\includegraphics[width=\textwidth]{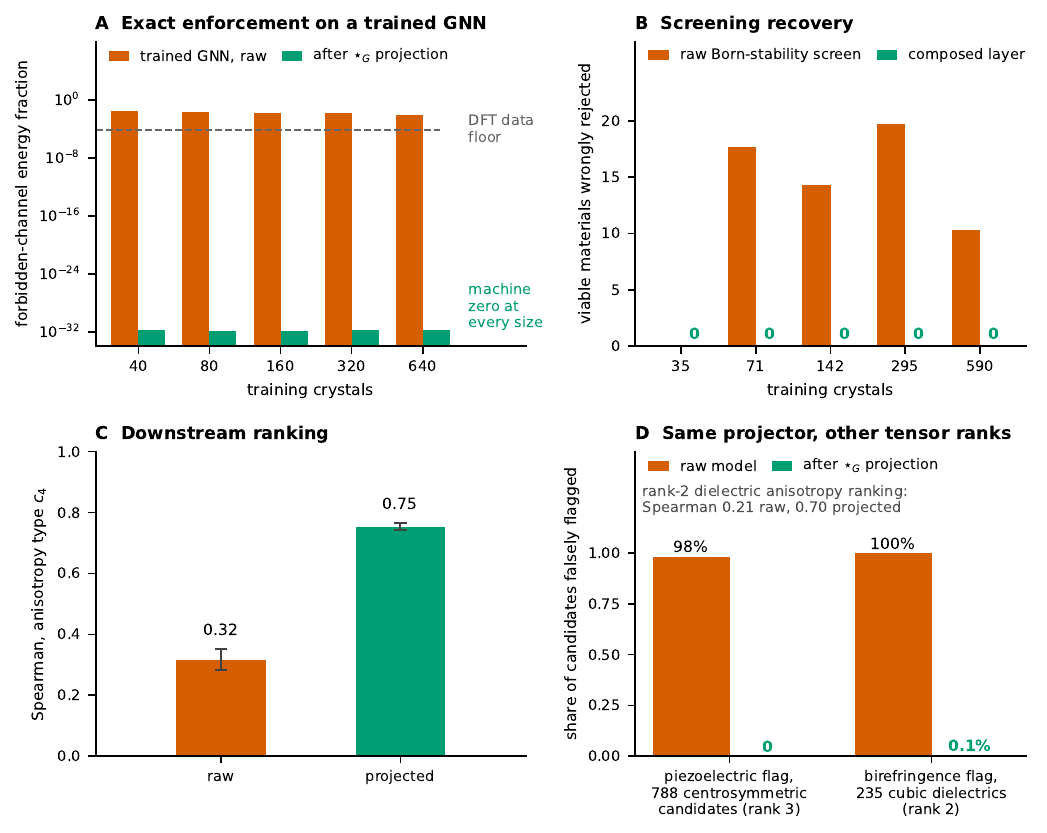}
\caption{\textbf{Projection of crystal tensor predictions.}
(\textbf{A}) Forbidden-channel energy of elastic tensors predicted by a trained crystal graph network across training sizes. The raw model places $0.9$ to $2.6\times10^{-2}$ of its energy in forbidden channels, and the projection reduces this to machine precision (about $1.5\times10^{-32}$).
(\textbf{B}) Viable materials falsely rejected by a raw Born-stability screen. The composed symmetry and nearest-positive-semidefinite projections recover every false rejection, as the positive-semidefinite step guarantees. At the smallest training size the mean-reverted model rejects none.
(\textbf{C}) Spearman correlation of the predicted anisotropy type $c_4$ with the measured one, before and after projection ($3$ seeds, $\pm1$ s.d.).
(\textbf{D}) The same projector on dielectric (rank 2) and piezoelectric (rank 3) tensors: share of candidates falsely flagged by a composition-only regressor before and after projection, with the Spearman correlation of the dielectric anisotropy ranking (SI~Section~\ref{sec:si_tensor_ranks}).}
\label{fig:enforcement}
\end{figure}

The projection changes screening decisions. We composed the symmetry projection with eigenvalue clipping of the Voigt matrix, which removes negative elastic modes, and because the symmetry projection preserves positive semidefiniteness the composed tensor satisfies both constraints, and screened held-out crystals with models trained on few measured tensors. For training sets of $71$ crystals or more, a raw Born-stability screen falsely rejected $10$ to $20$ viable materials per held-out set, and the projected screen recovered every one, as its positive-semidefinite step guarantees (Figure~\ref{fig:enforcement}B). Every measured tensor in the dataset is Born stable, so this screen cannot measure false acceptances (SI~Section~\ref{sec:si_symmetry_informed}). The ranking of elastic anisotropy improved from Spearman $\rho=0.32$ to $0.75$ for the tuned model. Within each crystal system the correlation barely moves (from $0.09$ to $0.11$ for the $449$ cubic crystals), so the gain lies between symmetry classes, where the projection restores the part of the anisotropy that symmetry determines.

The projector is one operator for every tensor rank. We applied it unchanged to the $1{,}056$ dielectric tensors (rank~2)~\citep{petousis2017dielectric} and $941$ piezoelectric tensors (rank~3)~\citep{dejong2015piezoelectric} of the Materials Project (SI~Section~\ref{sec:si_tensor_ranks}). For the odd-rank piezoelectric tensor the improper operations carry a sign, so the group average annihilates the tensor of every centrosymmetric crystal, which is the classical piezoelectric selection rule. A composition-only regressor placed $3.5\%$ (dielectric) and $51\%$ (piezoelectric) of its predicted energy in forbidden channels. Supplying the crystal system, the point group and a centrosymmetry flag as additional features left $3.3\%$ and $43\%$ (SI~Section~\ref{sec:si_symmetry_informed}), so supervision on symmetry labels does not produce the exact zeros and ties that the selection rule requires. After projection none of the $788$ centrosymmetric crystals of the dielectric set remained flagged as a piezoelectric candidate, against $773$ before, and none of the $235$ cubic dielectrics remained birefringent, apart from one crystal in one of three seeds whose relaxed lattice departs from the cubic metric by $2\times10^{-3}$ (Figure~\ref{fig:enforcement}D).

A lookup table of point groups would supply those zeros. To measure what the projection adds beyond such a table, we restricted the evaluation to the $939$ piezoelectric crystals whose point group is non-centrosymmetric and other than $432$, and to the $821$ dielectric and $732$ elastic crystals whose point group is non-cubic, the point groups that allow a nonzero tensor with a pattern of vanishing and tied components (SI~Section~\ref{sec:si_lookup_free}). Projection lowered the error of $100\%$, $99.8\%$ and $99.7\%$ of these crystals for the piezoelectric, dielectric and elastic tensors, and reduced the median relative error by $23\%$, $11\%$ and $2\%$. An ablation that sets to zero only the components the point group forces to vanish removed $3.6\%$, $0.0\%$ and $0.2\%$ of the summed squared error, against $5.7\%$, $1.8\%$ and $4.1\%$ for the projection, so for the dielectric and elastic tensors most of the gain comes from the components that the point group ties together, and for the piezoelectric tensor about a third of it does. A symmetrizer that enforces both the zeros and the ties of the standard tensor tables is the projector itself. Projection shrinks the predicted magnitudes, so a screening threshold must be re-tuned after projection (at $0.1$~C\,m$^{-2}$, $208$ false positives and $12$ false negatives before, $82$ and $159$ after). Every crystal whose error increased had a measured tensor carrying more forbidden energy than the prediction's error carries, the case that Proposition~\ref{prop:repair} identifies.

\subsection{Compression and other symmetries}
\label{sec:beyond}

The same blocks compress trained equivariant networks. Truncating half of the targeted interaction- and product-layer parameters of the pretrained MACE-OFF23-large force field changed its forces by $2.0\%$, against $12\%$ for magnitude pruning at equal parameter count, and kept its equivariance at the level of floating-point rounding. Half of the targeted parameters is $85\%$ of the stored parameters, and the factored layers did not shorten a force evaluation on a GPU (SI~Section~\ref{sec:compression}). Over the wallpaper group of two-dimensional periodic flow, the $\starG$ product defines a Fourier neural operator that stays exactly symmetric during rollout and recovers the group-equivariant Fourier neural operator (SI~Section~\ref{sec:pde}).

\section{Discussion}
A central question for scientific machine learning is whether physical symmetries must be enforced exactly or only approximately. Our results show that approximation can be inadequate when predictions are used repeatedly or when scientific conclusions depend on symmetry itself. In these settings, small equivariance errors can accumulate into qualitatively different outcomes, including orientation-dependent crystal symmetries and forbidden physical responses. By contrast, for an equivariant architecture on a stable trajectory, the residual violations were consistent with numerical roundoff and remained small over all horizons studied. These results suggest that in some scientific workflows exact symmetry is necessary for reproducible and physically consistent predictions.

Whenever the symmetry of a target quantity is known, a trained model can be corrected after training. One determines the symmetry group of the structure together with its action on the model output, applies the Reynolds projection, and records the forbidden component that the projection removes. The removed component provides a reference-free lower bound on the model's error for that structure, while the projected prediction satisfies the appropriate selection rules and cannot have larger squared error than the original prediction when the target is symmetric (Proposition~\ref{prop:repair}). For symmetry-sensitive simulations, including structural relaxation and molecular dynamics, the same projection can be applied at every force evaluation.

Symmetry can reveal information about prediction quality even when the correct answer is unknown. Symmetry is typically viewed as prior knowledge that restricts the space of admissible predictions. Here it plays a second role. The symmetry-breaking components of the error decomposition are computed from the model outputs alone and therefore furnish certified, reference-free lower bounds on prediction error. In this sense, symmetry also becomes a tool for quantifying model reliability. For large-scale screening, molecular dynamics, and foundation-model deployment, where only a small fraction of predictions can be validated by first-principles calculations, these bounds provide a scalable way to identify predictions most likely to require further scrutiny.

\paragraph{What is classical and what is new.} The mathematical ingredients are classical. Group convolution, the Artin--Wedderburn decomposition, Reynolds projection, Schur's lemma, and the Eckart--Young theorem are all well established, and symmetry projection is already used in atomistic simulation and materials software. The contribution of this work is to unify these ideas into a framework for diagnosing, quantifying, correcting, and compressing symmetry-constrained computations.

First, we extend transform-based tensor algebras to noncommutative symmetry groups. The tensor algebras of Kilmer and collaborators, including their low-rank approximation theory, assume commuting entries and therefore cannot represent non-Abelian groups~\citep{dunbar2025tensor}. The $\starG$ algebra carries the same matrix-like structure to arbitrary finite groups and, through exact finite reductions, to band-limited signals on compact groups and crystallographic space groups.

Second, Theorem~\ref{thm:canon} shows that equivariance determines the transform itself. Up to basis choices that preserve ranks and optimal low-rank approximations, the transform is uniquely the group-Fourier transform. The symmetry group therefore determines the block structure used for analysis, approximation, and compression.

Third, the framework places symmetry enforcement and low-rank approximation in the same coordinates. Projection and truncation become a single optimization problem on the representation spaces that arise in equivariant models, including spaces of real type and arbitrary multiplicity, under the invariant metrics of Theorem~\ref{thm:eymod}.

Fourth, Theorem~\ref{thm:three_part} decomposes prediction error into components associated with broken selection rules, inconsistencies across symmetry-related inputs, and residual error. The first two components are removable by symmetry enforcement and can be computed from model outputs alone. They therefore provide certified, reference-free lower bounds on prediction error, allowing a model to identify part of its own error without access to the correct answer. In Section~\ref{sec:audit}, these bounds certified force errors of at least $10$~meV/\AA{} for $48\%$ of $30{,}280$ structures. In Section~\ref{sec:crystals}, the same decomposition correctly identified every crystal whose error increased under projection.

The guarantees also have limitations. Projection enforces the symmetry supplied by the user and therefore preserves a structure at a symmetric saddle point even when the true minimum has lower symmetry. The power-spectrum descriptors considered in the SI are incomplete invariants, although the algebra also expresses the bispectrum, which is complete for generic signals (SI Section~\ref{sec:si_scope}). More broadly, enforcing exact symmetry does not replace the need for accurate models. On pooled molecular-property prediction with ample training data, equivariant graph networks remained more accurate than any $\starG$ model we trained (SI Section~\ref{sec:qm9_accuracy}).

The group algebra supports a second natural operation in addition to convolution: pointwise multiplication of functions on the group. In block coordinates, this operation corresponds to the Clebsch--Gordan couplings that define the nonlinear tensor-product layers of equivariant neural networks. The present work uses only convolution. Extending the framework to incorporate both products could place the nonlinear operations of equivariant networks within the same algebraic setting that here unifies symmetry correction, error decomposition, and optimal compression.

\section{Materials and Methods}

\paragraph{Force fields.} We ran MACE-OFF23-small, MACE-MP-0 and Orb-v2 at their published settings through ASE~\citep{larsen2017ase}, with constant-energy integration in double precision and each model evaluated in the precision stated, so that the force call is the dominant source of symmetry breaking. Forbidden amplitudes exclude rigid translation and rotation, and the projected arm applies the isotypic projector of the site-symmetry group to each force call (SI~Section~\ref{sec:foundation_protocol}). The crystal audit evaluates Orb-v2, Orb-v3 with direct and with gradient forces (trained on OMat) and MACE-MP-0 (medium) on $40$ de Jong crystals, and the screening audit evaluates the two Orb-v3 models and MACE-MP-0 (medium, single precision) on Materials Project structures. Phonons use central differences with a $0.01$~\AA{} step, and relaxations use $200$ FIRE steps to a force tolerance of $10^{-3}$~eV/\AA, with and without projection of each force call, applied through the \texttt{FixSymmetry} constraint of ASE, which implements the same space-group projection (SI~Sections~\ref{sec:si_fm_audit} and~\ref{sec:si_fm_screening}). Unconverged relaxations are classified by their final structure. A symmetric structure counts as unstable under a model when the group-averaged finite-difference force constants have a symmetry-breaking mode below $-5$~cm$^{-1}$. For the direct head these force constants are not the Hessian of an energy, and their median relative asymmetry is $0.049$.

\paragraph{Crystal tensors.} The projector is the point-group Reynolds operator of each crystal's spglib point group, and the stability step composes it with eigenvalue clipping of the Voigt matrix, alternated with the projection and ending with it. Averaging congruent copies of a positive-semidefinite matrix keeps it positive semidefinite, so the result obeys both the symmetry and the sign constraint. Models are a tuned gradient-boosted regressor, a ridge baseline and a crystal graph network over a sweep of training-set sizes, with three seeds wherever a train--test split enters. Elastic tensors are the $1{,}181$ crystals of the de Jong set, and dielectric and piezoelectric tensors are the corresponding matminer datasets (SI~Sections~\ref{sec:si_tensor_ranks} and~\ref{sec:si_symmetry_informed}).

\paragraph{Reproducibility.} Every learning experiment uses at least three random seeds. The trajectory experiments with frozen pretrained models, the projections and the band-limited $\mathrm{SO}(3)$ validation involve no training and are reproduced by a single run of the corresponding script. Each script listed in SI~Table~\ref{tab:codemap} writes its result file directly, and a reproduction guide in the archived deposit maps every table and figure to its script.

\section*{Data Availability}
All data needed to evaluate the conclusions of this study are publicly available. The QM9 dataset~\citep{ramakrishnan2014qm9} main structure file (\texttt{dsgdb9nsd.xyz.tar.bz2}) can be downloaded from \url{https://figshare.com/ndownloader/files/3195389}, and the full collection is archived at \url{https://doi.org/10.6084/m9.figshare.c.978904.v5}. The QM7-X dataset~\citep{hoja2021qm7x} is archived at \url{https://doi.org/10.5281/zenodo.4288677}. The elastic, dielectric and piezoelectric tensor datasets~\citep{dejong2015charting,petousis2017dielectric,dejong2015piezoelectric} are obtained through \texttt{matminer}.

\section*{Code Availability}
Implementations of the $\starG$ algebra in MATLAB and Python, together with scripts that reproduce every experiment, table and figure, are archived on Zenodo at \url{https://doi.org/10.5281/zenodo.21359575}. The Lean~4 formalization (no \texttt{sorry}) is included in the deposit with its per-theorem axiom certificate.

\section*{Acknowledgements}
We wish to acknowledge Tammy Kolda for early feedback and her foundational contributions to the tensor algebra field, and Tess Smidt and Joe Kileel for engaging discussions. P.H. gratefully acknowledges the IBM internship program, which enabled key components of this study. The authors used AI assistants (Claude from Anthropic, and models from OpenAI and Google) for code development, Lean proof engineering and review of the text, and checked all code, results and text.

\section*{Funding}
This work received no external funding. P.H. was supported by the IBM internship program.

\section*{Competing Interests}
The authors declare no competing interests.



\clearpage
\phantomsection\label{si:start}
\setcounter{section}{0}
\setcounter{figure}{0}
\setcounter{table}{0}
\setcounter{theorem}{0}
\renewcommand{\thesection}{S\arabic{section}}
\renewcommand{\thesubsection}{S\arabic{section}.\arabic{subsection}}
\renewcommand{\thefigure}{S\arabic{figure}}
\renewcommand{\thetable}{S\arabic{table}}
\renewcommand{\thetheorem}{S\arabic{theorem}}

\begin{center}
{\LARGE \textbf{Supplementary Information}}\\[10pt]
\end{center}

%

\section{Mathematical Foundations}
\label{sec:si_foundations}

\subsection{Notation}

\begin{itemize}
\item $G$: finite group of order $n = |G|$ with identity $e$.
\item $\hat{G}$: set of equivalence classes of irreducible unitary representations.
\item $d_\rho = \dim(\rho)$ for $\rho \in \hat{G}$.
\item $\calA(:,:,g)$: frontal slice at group element $g$. $\calA_{ij} = \calA(i,j,:)$: entry at indices $i,j$, an element of $\mathbb{K}_G$ (a tube in the order-3 setting).
\end{itemize}

\subsection{The Group Algebra}

\begin{definition}[Group Algebra]
$\bbR[G]$ is the vector space of formal sums $\sum_{g \in G} a_g g$ with convolution product:
\begin{equation}
\left(\sum_g a_g g\right) \cdot \left(\sum_h b_h h\right) = \sum_{c \in G} \left(\sum_{g \in G} a_g b_{g^{-1}c}\right) c,
\end{equation}
 where $a_g$ and $b_h$ are scalar coefficients in $\bbR$.
\end{definition}

\subsection{The Convolutional Ring $\mathbb{K}_G$}

\begin{definition}[Convolutional Ring $\mathbb{K}_G$]\label{def:conv_ring}
The convolutional ring $\mathbb{K}_G$, for $G = G_1 \times \cdots \times G_d$ and $n_i = |G_i|$, is the ring of tensors in $\mathbb{R}^{n_1 \times n_2 \times \cdots \times n_d}$ with pointwise addition and convolution product:
\begin{align}\label{eq: group convolution of tensors}
    (\mathcal{A} \cdot \mathcal{B})(c_1, \ldots, c_d) = \sum_{(a_1 \ldots, a_d) \in G}
\mathcal{A}(a_1, \ldots, a_d) \mathcal{B}(a_1^{-1}c_1, \ldots, a_d^{-1}c_d),
\end{align}
for all $(c_1, \dots, c_d) \in G$. As a ring, $\mathbb{K}_G$ is isomorphic to
the group algebra $\bbR[G]$ above under $a \leftrightarrow (a_1, \ldots, a_d)$,
and it factors as
$\mathbb{K}_G \cong \mathbb{K}_{G_1} \otimes \cdots \otimes \mathbb{K}_{G_d}$.
The entries of a $\starG$ tensor are elements of $\mathbb{K}_G$.
\end{definition}

\section{The Convolution Tensor, Spectral Decomposition, and Generalized Fourier Matrix}
\label{sec:si_convtensor}

\begin{definition}[Convolution Tensor]
$\calT \in \bbR^{n \times n \times n}$ is defined by $\calT(a,b,c) = \delta(g_a g_b = g_c)$ (also known as the structure constants of the group algebra).
\end{definition}

\begin{proposition}[Properties of $\calT$]\label{prop:T_props}
\begin{enumerate}
\item[(i)] Associativity:
\[
\sum_d \calT(a,b,d)\,\calT(d,c,e)
        \;=\; \sum_d \calT(a,d,e)\,\calT(b,c,d).
\]
\item[(ii)] Identity: $\calT(e,b,c) = \delta_{bc}$.
\item[(iii)] Each slice $\calT(a,:,:)$ is a permutation matrix.
\end{enumerate}
\end{proposition}

\begin{proof}
(i) follows from associativity of group multiplication; both sides equal $\delta(g_a g_b g_c = g_e)$. (ii)--(iii) follow from the group axioms.
\end{proof}

\begin{definition}[Generalized Fourier Transform Matrix]
$F_G \in \bbC^{n\times n}$ is defined row-wise: the row $F_G(g, :)$ is given by the concatenation
$[\mathrm{rvec}(\rho_1(g)), \dots, \mathrm{rvec}(\rho_\ell(g))]$,
where $\mathrm{rvec}(\rho_i(g))$ denotes the row-vectorization of the matrix $\rho_i(g)$ for each $\rho_i \in \hat{G}$. For abelian groups, $F_G$ is a generalized Fourier matrix and for cyclic groups, it reduces to the standard DFT matrix. For non-abelian groups $F_G$ is invertible but in general neither unitary nor block-unitary.
\end{definition}

\begin{theorem}[Peter--Weyl Spectral Decomposition]\label{thm:pw_spectral}
\begin{equation}
\calT(a,b,c) = \sum_{i,j,k} \calC(i,j,k) \, F_G(a,i) \, F_G(b,j) \, F_G^{-1}(c,k)
\end{equation}
where $\calC$ is a core tensor  that is typically sparse. For abelian groups, $\calC$ is diagonal.
\end{theorem}

\begin{proof}
By the Peter--Weyl theorem the matrix elements
$\{\rho_{st} : \rho \in \hat{G},\ 1 \le s,t \le d_\rho\}$ form an orthogonal basis
of $L^2(G)$, with $\sum_{g} \rho_{st}(g)\overline{\sigma_{uv}(g)} =
(n/d_\rho)\,\delta_{\rho\sigma}\delta_{su}\delta_{tv}$ by Schur orthogonality.
This is exactly the basis assembled column by column into $F_G$: column
$i \leftrightarrow (\rho, s, t)$ carries $\rho_{st}$, as in the Definition above.
Rescaling each column by $\sqrt{d_\rho/n}$ would make the basis orthonormal and
$F_G$ unitary; we deliberately do \emph{not} do so, which is why $F_G$ is
invertible but not unitary, and why the Plancherel weight $d_\rho/n$ appears
explicitly wherever a norm is taken (Theorem~\ref{thm:ey_full}, Step~1). In
this basis the group-algebra multiplication is the convolution theorem
$\widehat{(f*h)}(\rho) = \hat{f}(\rho)\,\hat{h}(\rho)$, i.e.\ block-diagonal matrix
multiplication: the product of the basis elements $(\rho, s, u)$ and $(\rho', u', t)$
is nonzero only when $\rho = \rho'$ and $u = u'$, and then equals the basis element
$(\rho, s, t)$ (the matrix-unit rule $E^{\rho}_{su} E^{\rho}_{u't} = \delta_{uu'}
E^{\rho}_{st}$). Transporting the convolution tensor $\calT$, with
$\calT(a,b,c) = \mathbf{1}[ab=c]$, to this basis through $F_G$ (contracting mode~3
with $F_G^{-1}$, the exact inverse) therefore gives the stated decomposition with
core $\calC$ supported on the index triples
\begin{equation}
\big((\rho,s,u),\ (\rho,u,t),\ (\rho,s,t)\big),
\qquad \rho \in \hat{G},\ \ 1 \le s,u,t \le d_\rho,
\end{equation}
on which $\calC$ carries the multiplication structure constants of the block
$M_{d_\rho}(\bbC)$, with no scalar factor: because the columns of $F_G$ are the
unrescaled matrix elements, $\widehat{f}(\rho) = \sum_g f(g)\rho(g)$ is a
homomorphism and the convolution theorem holds exactly as stated. Rescaling
column $(\rho,s,t)$ by any $c_\rho$ would introduce compensating factors
$c_\rho^{-1}$ into $\calC$ and replace the Plancherel weight by
$d_\rho/(n|c_\rho|^2)$, so the convention is fixed once here and carried
throughout. This support is the union over $\rho$ of the $M_{d_\rho}$
multiplication tables, so $\calC$ is sparse; when every irrep is one-dimensional
(abelian $G$) each index collapses to $(\rho)$, the support reduces to
$\{(\rho,\rho,\rho)\}$, and $\calC$ is diagonal.
\end{proof}

\begin{corollary}
For $G = \bbZ_n$: $F_G = \mathrm{DFT}_n$ and $\calC$ is diagonal, recovering the circular convolution theorem  (i.e., the $t-$product).
\end{corollary}

\begin{remark}[Implementation convention]
\label{rem:impl_scaling}
The reference implementation builds several transforms in the equivalent
column-normalized scaling $c_\rho = \sqrt{d_\rho/n}$, for which $F_G$ is unitary
and the Plancherel weight is $w_\rho = d_\rho/(n|c_\rho|^2) = 1$, so no explicit
weight is carried in code. Optimality and truncation are unaffected: any
per-irrep rescaling multiplies the singular values of block $\rho$ by a positive
constant and leaves the minimizer of $\sum_\rho w_\rho \|\cdot\|_F^2$ unchanged.
Rank selection is performed on group-domain singular elements, whose norms satisfy
$\|\mathbf{s}_i\|_F^2 = \sum_\rho w_\rho \sum_t \sigma^2$, so the weight is
carried automatically rather than assumed; pooling Fourier-domain singular
values by raw magnitude would not be weight-correct when $w_\rho$ varies.
\end{remark}

\section{The Group Fourier Transform of a Tensor}
\label{sec:si_gft}

\begin{definition}[Group Fourier Transform of a Tensor]\label{def:tensor_ft}
For $\calA \in \bbR^{\ell \times m \times n}$, the Group Fourier transform $\mathcal{F}_G$ assigns to each irrep $\rho \in \hat{G}$ the $\ell d_\rho \times m d_\rho$ block matrix whose $(i,j)$ block is
\begin{equation}
    \hat{\calA}(i,j,\rho) = \sum_{g \in G} \calA(i,j,g)\,\rho(g) \in \bbC^{d_\rho \times d_\rho}, \qquad i \le \ell,\ j \le m.
\end{equation}
The full Fourier representation is the block-diagonal matrix $\oplus_{\rho \in \hat{G}} \hat{\calA}(:,:,\rho)$.
\end{definition}

\begin{proposition}[Group Fourier Inversion Theorem]
Given $\hat{\calA}(:,:,\rho)$ for each $\rho \in \hat{G}$, the inverse Group Fourier transform recovers
\begin{equation}
    \calA(i,j,g) = \sum_{\rho \in \hat{G}} \frac{d_\rho}{n}\,\mathrm{Tr}\!\left[\hat{\calA}(i,j,\rho)\,\rho(g)^H \right].
\end{equation}
\end{proposition}
\begin{proof}
Follows by applying the standard Peter--Weyl inversion theorem for each fixed $i$ and $j$.
\end{proof}

\begin{remark}
The two notions of ``Fourier transform'' used in the paper are related as follows. The contraction of $\calA$ with $F_G$ along its group dimension (as in Section~\ref{sec:si_convtensor}) and the block-diagonal form $\oplus_\rho \hat{\calA}(:,:,\rho)$ (Definition~\ref{def:tensor_ft}) are equivalent via the row-vectorization reshaping $\mathrm{rvec}$ used in the construction of $F_G$. The block-diagonal form is used for the SVD computation (one standard matrix SVD per irrep block); the $F_G$-contraction form is used for the product computation and for establishing the spectral decomposition of $\calT$.
\end{remark}

\begin{algorithm}[!ht]
\caption{$\starG$ product (used by every $\starG$-based method)}
\label{alg:starg_product}
\begin{algorithmic}[1]
\Require $\calA \in \bbR^{\ell \times m \times n}$, $\calB \in \bbR^{m \times p \times n}$, generalized Fourier matrix $F_G \in \bbC^{n \times n}$, irrep block sizes $\{d_\rho\}_{\rho \in \hat{G}}$ \Comment{product groups and order-$(2+d)$ tensors via Remark~\ref{rem:flatten}}
\Ensure $\calC = \calA \starG \calB \in \bbR^{\ell \times p \times n}$
\State $\hat{\calA} \gets \calA \times_3 F_G$ \Comment{Generalized Fourier transform along group dimension}
\State $\hat{\calB} \gets \calB \times_3 F_G$
\For{$\rho \in \hat{G}$ \textbf{in parallel}}
    \State extract $\hat{\calA}_\rho \in \bbC^{\ell d_\rho \times m d_\rho}$, $\hat{\calB}_\rho \in \bbC^{m d_\rho \times p d_\rho}$ from block-diagonal form
    \State $\hat{\calC}_\rho \gets \hat{\calA}_\rho \cdot \hat{\calB}_\rho$ \Comment{ordinary matrix product}
\EndFor
\State assemble $\hat{\calC}$ from $\{\hat{\calC}_\rho\}_\rho$ in block-diagonal form
\State $\calC \gets \mathrm{Re}(\hat{\calC} \times_3 F_G^{-1})$ \Comment{inverse Generalized Fourier transform; imaginary part is zero up to roundoff for real inputs}
\State \Return $\calC$
\end{algorithmic}
\end{algorithm}

\section{The $\starG$ Algebra: Properties and Proofs}
\label{sec:si_algebra_props}

\begin{definition}[$\starG$ Product]\label{def:starg_product}
For $\calA \in \bbR^{\ell \times m \times n}$, $\calB \in \bbR^{m \times p \times n}$:
\begin{equation}
(\calA \starG \calB)_{ij}(c) = \sum_k \sum_{a \in G} \calA_{ik}(a) \calB_{kj}(a^{-1}c).
\end{equation}
Equivalently: $\widehat{(\calA \starG \calB)}(:,:,\rho) = \hat{\calA}(:,:,\rho) \cdot \hat{\calB}(:,:,\rho)$ for each irrep $\rho \in \hat{G}$.
\end{definition}

\begin{remark}[One group or many: the flattening convention]\label{rem:flatten}
Definition~\ref{def:starg_product} is stated for a single group $G$ and
order-3 tensors, whereas the main text (Results, Eq.~(3)) states the product
for $G = G_1 \times \cdots \times G_d$ and order-$(2+d)$ tensors, with entries
in the convolutional ring $\mathbb{K}_G$ of
Definition~\ref{def:conv_ring}. The two are the same definition. The bijection
$a \leftrightarrow (a_1, \ldots, a_d)$ identifies the product group with a
single group of order $n = n_1 \cdots n_d$; under it, entries become
$d$-dimensional arrays (elements of $\mathbb{K}_G$), the single-group
convolution of Definition~\ref{def:starg_product} becomes the multi-index
convolution of Eq.~(3) of the main text, the Fourier matrix factors as
$F_G = F_{G_1} \otimes \cdots \otimes F_{G_d}$, and the irreducible
representations are the Kronecker products
$\rho = \rho_1 \otimes \cdots \otimes \rho_d$ of the factors' irreps
(Section~\ref{sec:si_product_proof}). Consequently every algorithm in this document, stated over
$\bbR^{\ell \times m \times n}$ for readability, applies verbatim to
order-$(2+d)$ tensors over product groups: one runs it on the flattened
group, and each Fourier block is indexed by a tuple of irreps. The order-3
form is a notational flattening, not a restriction.
\end{remark}

\begin{proposition}[Algebraic Properties]
(i)~Associativity. (ii)~Distributivity. (iii)~Identity: $\calI(:,:,e)=I$, $\calI(:,:,g\neq e)=0$.
(iv)~$(\calA \starG \calB)^H = \calB^H \starG \calA^H$.
\end{proposition}

\begin{proof}
All follow from the corresponding matrix properties applied per-irrep in the Fourier domain (Definition~\ref{def:tensor_ft}), plus linearity and invertibility of the Fourier transform. In the Lean~4 development these identities are instead proved directly, by combinatorial reindexing of the group-convolution sum, so they depend on none of the Fourier-analytic axioms.
\end{proof}

\section{The $\starG$-SVD: Full Proof of Optimality}
\label{sec:si_svd_proof}

\begin{theorem}[$\starG$-SVD Existence]\label{thm:svd_exist}
Every $\calA \in \bbR^{\ell \times m \times n}$ admits $\calA = \calU \starG \calS \starG \calV^H$ where $\calU,\calV$ are $\starG$-unitary and $\calS$ is f-diagonal.
\end{theorem}

\begin{proof}
For each irrep $\rho \in \hat{G}$, $\hat{\calA}(:,:,\rho)$ is a standard matrix admitting SVD: $\hat{\calA}(:,:,\rho) = U_\rho \Sigma_\rho V_\rho^H$. Setting $\hat{\calU}(:,:,\rho) = U_\rho$, $\hat{\calS}(:,:,\rho) = \Sigma_\rho$, $\hat{\calV}(:,:,\rho) = V_\rho$ and applying the inverse Fourier transform yields the $\starG$-SVD. Unitarity: $\widehat{\calU^H \starG \calU}(:,:,\rho) = U_\rho^H U_\rho = I$ for all $\rho$, so $\calU^H \starG \calU = \calI$. When the tensor is real and its irreps occur in complex-conjugate pairs, choosing the per-block SVDs conjugate-consistently on paired irreps ($U_{\bar\rho} = \overline{U_\rho}$, and likewise for $\calS$ and $\calV$) makes the inverse transform real, the same convention that yields a real t-SVD; over the real-irrep formalization used in Lean this holds automatically.
\end{proof}

\begin{algorithm}[!ht]
\caption{$\starG$-SVD}
\label{alg:svd}
\begin{algorithmic}[1]
\Require $\calA \in \bbR^{\ell \times m \times n}$ \Comment{product groups and order-$(2+d)$ tensors via Remark~\ref{rem:flatten}}
\Ensure $\calU, \calS, \calV$
\State Compute $\hat{\calA}(:,:,\rho)$ for all $\rho \in \hat{G}$ using $\mathcal{F}_G$ (Definition~\ref{def:tensor_ft})
\For{$\rho \in \hat{G}$ \textbf{in parallel}}
    \State $[U_\rho, \Sigma_\rho, V_\rho] \gets \mathrm{SVD}(\hat{\calA}(:,:,\rho))$
\EndFor
\State Apply $\mathcal{F}_G^{-1}$ to $\{U_\rho\}$, $\{\Sigma_\rho\}$, $\{V_\rho\}$ to obtain $\calU, \calS, \calV$
\end{algorithmic}
\end{algorithm}

\begin{theorem}[Eckart--Young for $\starG$]\label{thm:ey_full}
Let $\calA = \calU \starG \calS \starG \calV^H$ be the $\starG$-SVD returned by
Algorithm~\ref{alg:svd}, whose per-irrep singular values $\sigma_j(\rho)$ are in
non-increasing order.
\emph{(i) Per-block form.} For any per-block rank profile $(k_\rho)_{\rho\in\hat G}$,
the blockwise truncation keeping the top $k_\rho$ singular values of each
$\hat\calA(:,:,\rho)$ satisfies
\begin{equation}
\Big\|\calA-\calA_{(k_\rho)}\Big\|_F^2 = \sum_{\rho\in\hat G}\frac{d_\rho}{n}
\sum_{j>k_\rho}\sigma_j(\rho)^2 \;\le\; \|\calA-\calB\|_F^2
\end{equation}
for any $\calB$ with $\mathrm{rank}\,\hat\calB(:,:,\rho)\le k_\rho$ for every $\rho$.
\emph{(ii) Element form.} Let $\mathbf{s}_1,\ldots,\mathbf{s}_r$ be the singular elements
(with norms automatically ordered, $\|\mathbf{s}_1\|_F\ge\cdots\ge\|\mathbf{s}_r\|_F$),
and $\calA_k$ the truncation retaining the leading $k$. Then
\begin{equation}
\|\calA - \calA_k\|_F^2 = \sum_{i=k+1}^r \|\mathbf{s}_i\|_F^2 \leq \|\calA - \calB\|_F^2
\end{equation}
for any $\calB$ with $\starG$-rank $\leq k$; this is~(i) at the coupled profile
$k_\rho=k\,d_\rho$.
\end{theorem}

\begin{proof}
Part~(i) is the classical matrix Eckart--Young theorem applied independently to each
block $\hat\calA(:,:,\rho)$ at rank $k_\rho$ (Step~2), weighted by the Parseval
isometry (Step~1); part~(ii) follows by specializing to $k_\rho=k\,d_\rho$ and
regrouping the discarded block singular values by entry (Step~3).

\textbf{Step 1 (Parseval).} By the generalized Fourier transform's isometry (Peter--Weyl):
\begin{equation}
\|\calA - \calB\|_F^2 = \sum_{\rho \in \hat{G}} \frac{d_\rho}{n}\|\hat{\calA}(:,:,\rho) - \hat{\calB}(:,:,\rho)\|_F^2.
\end{equation}

\textbf{Step 2 (Per-irrep Eckart--Young).} Because $\calS$ is f-diagonal, each
singular element contributes one $d_\rho \times d_\rho$ diagonal block to
$\hat{\calS}(:,:,\rho)$; a tensor of $\starG$-rank $\leq k$ therefore has at most
$k$ nonzero such blocks, so $\mathrm{rank}(\hat{\calB}(:,:,\rho)) \leq k\,d_\rho$
for each $\rho$. Since Algorithm~\ref{alg:svd} sorts the singular values of every
block in non-increasing order, the leading-$k$-element truncation
$\hat{\calA}_k(:,:,\rho)$ retains exactly the top $k\,d_\rho$ singular values of
$\hat{\calA}(:,:,\rho)$, which by the classical Eckart--Young theorem at rank
$k\,d_\rho$ is the best rank-$(k\,d_\rho)$ approximation of that block:
\begin{equation}
\|\hat{\calA}(:,:,\rho) - \hat{\calA}_k(:,:,\rho)\|_F^2 \leq \|\hat{\calA}(:,:,\rho) - \hat{\calB}(:,:,\rho)\|_F^2.
\end{equation}

\textbf{Step 3 (Summation).} Entry $i$ collects the block singular values
$\sigma_{(i-1)d_\rho + 1}(\rho), \ldots, \sigma_{i d_\rho}(\rho)$, so
$\|\mathbf{s}_i\|_F^2 = \sum_\rho \frac{d_\rho}{n}\sum_{t=1}^{d_\rho}\sigma_{(i-1)d_\rho + t}(\rho)^2$
and the singular values discarded from block $\rho$ (those indexed $j > k\,d_\rho$)
are exactly those carried by the singular elements $\mathbf{s}_{k+1}, \ldots, \mathbf{s}_r$.
Summing Step 2 over $\rho$ and regrouping by entry,
\begin{equation}
\|\calA - \calA_k\|_F^2 = \sum_{\rho \in \hat{G}} \frac{d_\rho}{n}\sum_{j > k d_\rho} \sigma_j(\rho)^2
= \sum_{i=k+1}^r \|\mathbf{s}_i\|_F^2 \leq \|\calA - \calB\|_F^2.
\end{equation}
\end{proof}

\begin{remark}
The optimality is exact within the stated rank class. Optimal CP decomposition is NP-hard to compute, Tucker/HOSVD is quasi-optimal within a factor $\sqrt{d}$ \cite{desilva2008tensor}, and tensor-train truncation has no global optimality guarantee. The $\starG$-SVD achieves both polynomial-time computation and exact optimality because the group structure supplies the block decomposition.
\end{remark}

\begin{remark}[Semisimple generality, and what semisimplicity alone does not give]
\label{rem:semisimple}
The proof of the per-block form uses only that the transform is a
$*$-isomorphism onto a direct sum of matrix algebras carrying a weighted
Frobenius norm. The per-block Eckart--Young statement therefore holds verbatim
for any finite-dimensional semisimple algebra (equivalently, in finite
dimension, any von Neumann regular algebra) equipped with such a
$*$-realization of its Artin--Wedderburn decomposition. The two hypotheses are
not redundant: semisimplicity supplies the block decomposition only up to
\emph{similarity}, and a non-isometric similarity transports the Frobenius
norm, silently voiding ambient-norm optimality even while the Wedderburn
structure is intact. Restoring optimality requires realizing the isomorphism
$*$-compatibly and carrying the multiplicity weights
$w_\rho = d_\rho/(n\,|c_\rho|^2)$ of Remark~\ref{rem:impl_scaling}; in the
commutative (tubal) case this is exactly the transform criterion of
Mor~\cite{mor2025eckart}. The group algebras of this paper are the instance in
which the $*$-realization exists canonically, by unitarity of the suitably
normalized group-Fourier transform, and in which the blocks additionally carry
the physical (irrep) labels that the diagnostics and selection rules exploit;
they are also the instance in which the algebra is \emph{forced} by
equivariance (Theorem~\ref{thm:necessity}) rather than chosen from the
semisimple family.
\end{remark}

\paragraph{Relationship to the tubal Eckart--Young literature.} The transform-domain
matrix-mimetic program originates with Kilmer and
collaborators~\cite{kilmer2021tensor,kilmer2011factorization}, and the precise
conditions under which a tube product admits an Eckart--Young theorem, with an
extension to separable Hilbert-space tubes, have been developed for \emph{tubal}
tensors by Mor and
Avron~\cite{mor2025quasitubaltensoralgebraseparable,mor2026sufficientnecessaryconditionseckartyoung},
whose transform-domain fibers form a \emph{commutative} algebra; that
characterization is established specifically for the commutative case. The present
group-algebraic construction is broader on four axes. (i)~The transform is the
group-Fourier transform of a possibly \emph{non-abelian} group, so a
transform-domain component is a $d_\rho\!\times\!d_\rho$ \emph{matrix} block rather
than a scalar; the fiber algebra is non-commutative and lies outside the tubal
(commutative) setting, and the two coincide only when every irreducible
representation is one-dimensional (the abelian case).
(ii)~We establish the full matrix-mimetic toolbox (QR, symmetric
eigendecomposition, least squares, matrix functions) beyond the SVD.
(iii)~Reading the transform as a group makes equivariance an intrinsic property and
the irrep blocks physically meaningful, which is what produces the Wigner--Eckart
diagnostics and the equivariant-network unification of the main text. (iv)~Several symmetries compose by
$F_{G_1}\!\otimes\!F_{G_2}$ with no redesign (the product-group theorem of the main text). The
Eckart--Young optimality is one shared waypoint; the group-algebraic construction,
its non-abelian reach, and its physical consequences are what is specific here.
The $\starM$ product of Kilmer and collaborators (an arbitrary invertible
transform) and $\starG$ (the group-Fourier transform) are, in this sense,
distinct generalizations of the t-product that agree on the abelian case,
with neither containing the other.
A parallel development in statistical representation learning, arrived at
independently for equivariant conditional-expectation operators, block-diagonalizes
by isotypic component and applies the classical Eckart--Young theorem per block to
reach the same equivariant low-rank optimality (ICML
2026)~\cite{ordonez2025equivariant}; it
targets operator learning with non-asymptotic guarantees rather than a tensor
algebra, and develops neither a $\starG$ product with its $\starG$-SVD algorithm
nor the matrix-mimetic toolbox and formalization given here. Per-isotypic-block optimality predates this work. Any priority claimed here for Eckart--Young optimality is limited to the symmetry-encoding \emph{tensor algebra}, with a machine-checked derivation.

\paragraph{Relation to other transform-based algebras and to standard decompositions.}
\label{sec:si_ey_relations}
This distinction places the $\starG$ construction in relation to recent results on transform-based tensor algebras. Within the commutative tubal setting, the $\starM$ axioms determine the form of the product without uniquely specifying the transform~\cite{avron2025demystifying}. In the Abelian scalar case, group compatibility admits an if-and-only-if characterization, while no tubal $\starM$ construction can represent a non-Abelian group~\cite{dunbar2025tensor}. The $\starG$ construction resolves this limitation by replacing scalar Fourier components with matrix-valued irrep blocks. It further extends the framework to band-limited compact groups (the band-limited reduction of the main text) and establishes that, on the regular and band-limited representation spaces considered here, every group-equivariant linear map is itself a $\starG$ operation (the equivariant-layer corollary of the main text). The rank-$k$ $\starG$-SVD therefore provides a Frobenius-optimal approximation within this equivariant operator class at each prescribed $\starG$-rank.

The optimality guarantee should be interpreted relative to the competitor class specified by the theorem: it is a guarantee over tensors of bounded $\starG$-rank, not over all $G$-equivariant tensors of bounded classical rank. Within their respective native rank notions, this provides a stronger optimality statement than those available for several standard tensor decompositions. Tucker/HOSVD, for example, generally provides a quasi-optimal bound with a dimension-dependent factor~\cite{desilva2008tensor}; optimal CP decomposition is computationally intractable in general; and tensor-train representations do not provide a corresponding global optimality guarantee. The $\starG$ approximation error also has the explicit form
$\|\calA-\calA_k\|_F^2=\sum_{i>k}\|\mathbf{s}_i\|_F^2$
providing a direct criterion for rank selection.
\section{Product Groups: Full Proof}
\label{sec:si_product_proof}

\begin{theorem}[Product Group Ring Isomorphism]\label{thm:prod_full}
For $G = G_1 \times \cdots \times G_d$:
(i)~$\bbK_G \cong \bbK_{G_1} \otimes \cdots \otimes \bbK_{G_d}$.
(ii)~$\calT_G = \calT_{G_1} \otimes \cdots \otimes \calT_{G_d}$.
(iii)~$F_G = F_{G_1} \otimes \cdots \otimes F_{G_d}$.
\end{theorem}

\begin{proof}
\textbf{(i)} Standard: $\bbR[G_1 \times \cdots \times G_d] \cong \bbR[G_1] \otimes \cdots \otimes \bbR[G_d]$ \cite{serre1977linear}.

\textbf{(ii)} The product group multiplication $(a_1,\ldots,a_d)(b_1,\ldots,b_d) = (a_1 b_1,\ldots,a_d b_d)$ gives
\begin{equation}
\calT_G(\mathbf{a},\mathbf{b},\mathbf{c}) = \prod_{i=1}^d \calT_{G_i}(a_i,b_i,c_i) = (\calT_{G_1} \otimes \cdots \otimes \calT_{G_d})(\mathbf{a},\mathbf{b},\mathbf{c}).
\end{equation}

\textbf{(iii)} Over $\bbC$ (the field of $F_G$ used here), the irreps of $G_1 \times \cdots \times G_d$ are exactly the tensor products $\rho_1 \otimes \cdots \otimes \rho_d$. Matrix elements factor: $(\rho_1 \otimes \cdots \otimes \rho_d)(g_1,\ldots,g_d) = \rho_1(g_1) \otimes \cdots \otimes \rho_d(g_d)$. By the $\mathrm{rvec}$ construction of $F_G$ and the mixed-product property of Kronecker products, $F_G = F_{G_1} \otimes \cdots \otimes F_{G_d}$.
\end{proof}

\begin{corollary}[2D Frequency Resolution]
For $G = \bbZ_{n_1} \times \bbZ_{n_2}$, the Fourier transform $F_G = \mathrm{DFT}_{n_1} \otimes \mathrm{DFT}_{n_2}$ computes a 2D DFT, resolving coupled frequencies $(f_1, f_2)$ that are invisible to either factor alone.
\end{corollary}

\section{Invariant Feature Extraction Algorithm}
\label{sec:feat_alg}

The $\starG$-feature extractor used in every linear and Neural-$\starG$
experiment of this paper is given in Algorithm~\ref{alg:features}. The seven
feature blocks (a)--(g) below correspond to the seven concatenated columns
of the output. Block (a) is the DC component; (b) the AC standard
deviation; (c) the global per-frequency power; (d) the per-row Fourier power
restricted to the first $K$ equivariant rows; (e) the singular-element norms
from the $\starG$-SVD; (f) the rows of $\calX$ that are constant under the
group action (recovered by row-variance thresholding); and (g) four spectral
statistics of the unfolded matrix $\calX_{(1)}$. Every feature is exactly
$G$-invariant for every finite group $G$; proofs are in Section~\ref{sec:equiv_proofs}. For the
Fourier-power features (c) and (d) the invariant quantity is the per-irrep
Frobenius power $\|\hat{\calX}(i,:,\rho)\|_F^2$, summed over each $d_\rho \times
d_\rho$ irrep block, which the extractor computes directly (invariance of the
per-irrep Fourier power, proved in Section~\ref{sec:equiv_proofs}): the group action left-multiplies
each block by $\rho(g)$, which is unitary and preserves the block Frobenius norm.
For abelian $G$, where every irrep is one-dimensional, this reduces to the
per-frequency power.

\begin{algorithm}[!tb]
\caption{$\starG$ invariant feature extraction}
\label{alg:features}
\begin{algorithmic}[1]
\Require batch $\calX \in \bbR^{N \times n_f \times n}$, group algebra $G$, optional normalization parameters $\Theta$ from training
\Ensure feature matrix $\Phi \in \bbR^{N \times d_{\mathrm{feat}} + 1}$ ($+1$ for the unregularized intercept) and updated $\Theta$
\If{$\Theta$ is not provided}
    \State $(p, q) \gets \arg\max_{p \cdot q \leq n_f} \min(p, q)$ \Comment{best near-square reshape for $\starG$-SVD}
    \State $n_{\mathrm{svd}} \gets \min(p, q)$
    \State \(\sigma_{\mathrm{row}} \gets \mathrm{var}_g(\calX(1, :, :))\); $\mathrm{inv\_mask} \gets \sigma_{\mathrm{row}} < 10^{-8} \cdot \max(\sigma_{\mathrm{row}})$ \Comment{rows constant under $G$}
    \State $\mathrm{eq\_idx} \gets \{j : \mathrm{inv\_mask}_j = \text{false}\}$;\ $K \gets \min(14, |\mathrm{eq\_idx}|)$
\EndIf
\State $\bar{\calX}_{ij} \gets \tfrac{1}{n} \sum_g \calX(i, j, g)$ \Comment{(a) DC component, $N \times n_f$}
\State $\sigma_{ij} \gets \mathrm{std}_g \calX(i, j, g)$ \Comment{(b) AC energy, $N \times n_f$}
\State $\hat{\calX} \gets \calX \times_3 F_G$ \Comment{Generalized Fourier transform along group dim}
\State $P^{\mathrm{col}}_{i\rho} \gets \sum_j \|\hat{\calX}(i, j, \rho)\|_F^2$ \Comment{(c) per-irrep power, $N \times |\hat{G}|$}
\State $P^{\mathrm{row}}_{i,(r-1)|\hat{G}| + \rho} \gets \|\hat{\calX}(i, \mathrm{eq\_idx}_r, \rho)\|_F^2$ \Comment{(d) per-row per-irrep power, $N \times K|\hat{G}|$}
\For{$i = 1, \ldots, N$}
    \State $X_i \gets$ pad $\calX(i, :, :)$ to $p \times q \times n$
    \State $[\,\cdot\,, \calS_i, \,\cdot\,] \gets \mathrm{starG\_SVD}(X_i)$ \Comment{(e) singular elements}
    \State $T_{i, k} \gets \|\calS_i(k, k, :)\|_F$ for $k = 1, \ldots, n_{\mathrm{svd}}$; sort $T_{i, :}$ descending
    \State $V_{ij} \gets \calX(i, j, 1)$ for $j$ with $\mathrm{inv\_mask}_j$ \Comment{(f) direct invariants}
    \State $\Sigma_i \gets \mathrm{svd}(\calX_{(1)}(i, :, :))$ \Comment{(g) spectral statistics}
    \State $S_i \gets [\sum \Sigma_i,\ \Sigma_{i,1},\ \Sigma_{i,1} / \Sigma_{i,\mathrm{end}},\ -\sum_k \tilde{\Sigma}_{i,k} \log \tilde{\Sigma}_{i,k}]$ where $\tilde{\Sigma}_i = \Sigma_i / \sum \Sigma_i$
\EndFor
\State $\Phi \gets [\bar{\calX}\ |\ \sigma\ |\ P^{\mathrm{col}}\ |\ P^{\mathrm{row}}\ |\ T\ |\ V\ |\ S]$
\State replace NaN/Inf with 0
\If{$\Theta$ not provided}
    \State $\mathrm{keep} \gets \{j : \mathrm{std}(\Phi_{:, j}) \geq 10^{-8}\}$
    \State $\mu \gets \mathrm{mean}(\Phi_{:, \mathrm{keep}})$;\ $s \gets \mathrm{std}(\Phi_{:, \mathrm{keep}})$;\ $s_j \gets \max(s_j, 1)$
    \State store $\Theta = (p, q, n_{\mathrm{svd}}, \mathrm{inv\_mask}, \mathrm{eq\_idx}, K, \mathrm{keep}, \mu, s)$
\EndIf
\State $\Phi \gets (\Phi_{:, \mathrm{keep}} - \mu) / s$
\State $\Phi \gets [\,\mathbf{1}_N\ |\ \Phi\,]$ \Comment{prepend unregularized intercept}
\State \Return $(\Phi, \Theta)$
\end{algorithmic}
\end{algorithm}

\section{Equivariance Proofs}
\label{sec:equiv_proofs}

\begin{proposition}[Equivariance of $\starG$]\label{prop:equivariance}
For every finite group $G$, the $\starG$ product is left-equivariant in its first
argument:
\[(g \cdot \calA) \starG \calB = g \cdot (\calA \starG \calB).\]
The corresponding identity in the second argument,
$\calA \starG (g \cdot \calB) = g \cdot (\calA \starG \calB)$, holds when $G$ is
abelian. For general $G$ the second argument instead intertwines the left action
with right translation: writing $(\calB \triangleleft g)(h) = \calB(hg)$,
\[\calA \starG (\calB \triangleleft g) = (\calA \starG \calB) \triangleleft g.\]
\end{proposition}

\begin{proof}
By definition of the group action and the $\starG$ product,
\begin{align*}
\bigl((g\cdot\calA)\starG\calB\bigr)_{ij}(h)
   &= \sum_k \sum_{a\in G} \calA_{ik}(g^{-1}a)\,\calB_{kj}(a^{-1}h).
\end{align*}
Substituting $a'=g^{-1}a$ (so $a=ga'$ and $a^{-1}=a'^{-1}g^{-1}$),
\begin{align*}
   &= \sum_k \sum_{a'\in G} \calA_{ik}(a')\,\calB_{kj}(a'^{-1}g^{-1}h) \\
   &= (\calA\starG\calB)_{ij}(g^{-1}h)
    = \bigl(g\cdot(\calA\starG\calB)\bigr)_{ij}(h),
\end{align*}
which proves left-equivariance in the first argument for every $G$. For the second
argument the right-translation identity is immediate,
\begin{align*}
\bigl(\calA\starG(\calB\triangleleft g)\bigr)_{ij}(h)
   &= \sum_k \sum_{a\in G} \calA_{ik}(a)\,\calB_{kj}(a^{-1}hg)
    = (\calA\starG\calB)_{ij}(hg)
    = \bigl((\calA\starG\calB)\triangleleft g\bigr)_{ij}(h).
\end{align*}
Replacing right translation by the left action would require
$\calB_{kj}(a^{-1}g^{-1}h) = \calB_{kj}(g^{-1}a^{-1}h)$ for all $a$, i.e.\
$a^{-1}g^{-1} = g^{-1}a^{-1}$, which holds for all $a$ precisely when $g$ is
central; the left-action form of the second identity therefore holds for abelian
$G$, matching the \texttt{CommGroup} hypothesis under which it is discharged in the
Lean~4 development.
\end{proof}

\begin{corollary}[Invariance of Features]
The following are invariant under $g \cdot X$:
(i)~$\|\mathbf{s}_i\|_F$ (singular-element norms);
(ii)~$\|\hat{X}(:,:,\rho)\|_F^2$ (per-irrep Fourier power);
(iii)~$\bar{X}_j = \frac{1}{n}\sum_g X(j,g)$ (DC component).
\end{corollary}

\begin{proof}
(i) The group action permutes frontal slices; the SVD is computed per-irrep where the action multiplies each block by a unitary, preserving singular values. (ii) The group action shifts $g \to g'g$ inside the sum defining $\hat{X}(:,:,\rho)$, multiplying the block by $\rho(g')$, which is unitary, leaving the Frobenius norm unchanged. (iii) $\frac{1}{n}\sum_g X(j,g'g) = \frac{1}{n}\sum_h X(j,h) = \bar{X}_j$.
\end{proof}

\section{Accumulation of a Per-Step Equivariance Defect}
\label{sec:si_depthlaw}

The main text argues that a per-call equivariance error $\varepsilon$ can
accumulate to an $O(1)$ symmetry violation at depth $M\sim1/\varepsilon$, and
that on an unstable orbit the breakdown time grows only like
$\log(1/\varepsilon)$. Both statements are elementary. We record them as propositions so that the experimental sections measure a predicted law rather than an observed trend. Both are machine-checked
(\texttt{StarG.depth\_law} and \texttt{StarG.depth\_law\_growth}), and depend on no axiom beyond Lean's core.

\begin{proposition}[Linear accumulation in the stable regime]
\label{prop:depth_law}
Let $T$ and $T_\varepsilon$ be linear operators on a finite-dimensional normed
space with $\|T\|\le1$, $\|T_\varepsilon\|\le1$ and
$\|T_\varepsilon-T\|\le\varepsilon$. Then for every $M\ge0$,
\begin{equation}\label{eq:depth_law}
\big\|T_\varepsilon^{\,M}-T^{M}\big\|\;\le\;M\varepsilon .
\end{equation}
If in addition $T$ is $G$-equivariant, $\calA$ is a $T$-invariant $G$-submodule,
$P_{\mathrm{forb}}$ is the orthogonal projection onto its complement, and the
initial state satisfies $x\in\calA$, then
\begin{equation}
\big\|P_{\mathrm{forb}}\,T_\varepsilon^{\,M}x\big\|\;\le\;M\varepsilon\|x\| .
\end{equation}
\end{proposition}

\begin{proof}
Induction on $M$. The case $M=0$ is trivial. For the step, write
\begin{equation}\label{eq:telescope}
T_\varepsilon^{\,M+1}-T^{M+1}
= T_\varepsilon\big(T_\varepsilon^{\,M}-T^{M}\big)+\big(T_\varepsilon-T\big)T^{M},
\end{equation}
and take norms: $\|T_\varepsilon\|\le1$ and $\|T^M\|\le1$ give
$\|T_\varepsilon^{\,M+1}-T^{M+1}\|\le M\varepsilon+\varepsilon$. For the second
statement, $T^Mx\in\calA$ because $\calA$ is $T$-invariant, so
$P_{\mathrm{forb}}T^Mx=0$ and
$P_{\mathrm{forb}}T_\varepsilon^{\,M}x=P_{\mathrm{forb}}(T_\varepsilon^{\,M}-T^{M})x$;
now apply \eqref{eq:depth_law} and $\|P_{\mathrm{forb}}\|\le1$.
\end{proof}

\begin{proposition}[Accumulation for nonlinear steps]\label{prop:depth_law_nonlinear}
Let a group $G$ act linearly on a normed space $E$, let $S=\{x\in E: gx=x\text{ for all }g\in G\}$
be the set of invariant states, and let $T:E\to E$ be $G$-equivariant, $T(gx)=gT(x)$, and Lipschitz with constant
$L$. Let $x_{k+1}=T_\varepsilon(x_k)$ with $x_0\in S$, and suppose
$\|T_\varepsilon(x_k)-T(x_k)\|\le\varepsilon$ for $k<M$. Then $d_k=\inf_{y\in S}\|x_k-y\|$
obeys $d_{k+1}\le L\,d_k+\varepsilon$, so $d_M\le\varepsilon\sum_{j<M}L^j$. This is at most
$M\varepsilon$ when $L\le1$ and equals $\varepsilon(L^M-1)/(L-1)\le M\varepsilon L^{M-1}$ when $L>1$.
\end{proposition}

\begin{proof}
For $y\in S$ and $g\in G$, $gT(y)=T(gy)=T(y)$, so $T$ maps $S$ into $S$. For any $y\in S$,
$d_{k+1}\le\|T_\varepsilon(x_k)-T(y)\|\le\|T_\varepsilon(x_k)-T(x_k)\|+\|T(x_k)-T(y)\|\le
\varepsilon+L\|x_k-y\|$, and the infimum over $y\in S$ gives the recursion. Induction from
$d_0=0$ gives the sum, and $\sum_{j<M}L^j\le M\max(1,L)^{M-1}$.
\end{proof}

\noindent The proposition is proved in Lean~4 (\texttt{depth\_law\_nonlinear} and its equivariant form, Section~\ref{sec:lean}) using only Lean's standard axioms. The linear propositions above are the case of linear $T$ and $T_\varepsilon$, with
$\varepsilon$ the operator-norm defect times the state's norm. The nonlinear form covers a
force field integrated by molecular dynamics whenever its exact step has a Lipschitz
constant $L$ on a region that contains the trajectory and the symmetric states nearest to
it, which is the regime of the bound.

\begin{remark}[The bound is attained, and when it is not]
The rate is not merely an upper bound. Taking $T=I$ and
$T_\varepsilon=\exp(\varepsilon B)$ with $B$ a fixed generator, the iterate is
$\exp(M\varepsilon B)$ and the defect is $M\varepsilon\|B\|+O((M\varepsilon)^2)$,
so the linear law is achieved whenever the per-step defects add coherently, which
is the construction of the controlled experiment in the main text. If instead a
composition of $M$ distinct layers carries independent defects, the typical
accumulation is $O(\sqrt{M}\,\varepsilon)$, with \eqref{eq:depth_law} remaining the
worst case. Exact equivariance gives $\varepsilon=0$ and hence no accumulation at
any depth, which is the content of the comparison.
\end{remark}

\begin{proposition}[Geometric accumulation and logarithmic breakdown]
\label{prop:depth_law_growth}
If instead $\|T\|\le L$ and $\|T_\varepsilon\|\le L$ with $L\ge1$, then
\begin{equation}
\big\|T_\varepsilon^{\,M}-T^{M}\big\|\;\le\;M\varepsilon L^{M}.
\end{equation}
Writing $L=e^{\lambda\Delta t}$ for a step of duration $\Delta t$ and
$t=M\Delta t$, the bound is $M\varepsilon e^{\lambda t}$. The time at which the
defect reaches a fixed size therefore satisfies
$t\approx\lambda^{-1}\log(1/\varepsilon)$ up to the slowly varying factor $M$, so
reducing the per-step defect delays breakdown only logarithmically.
\end{proposition}

\begin{proof}
The same induction on \eqref{eq:telescope}, with $\|T^M\|\le L^M$ and
$L^{M}\le L^{M+1}$.
\end{proof}

\noindent This is the law measured on the SF$_6$ orbit in the main text: moving
from single to double precision improves $\varepsilon$ by nine orders of
magnitude and delays breakdown by $\ln(10^{9})/\lambda=11.8$~ps (measured offset
$11.2$~ps), a factor of about three in time rather than a factor of $10^9$.

\section{Controlled Compounding Experiment}
\label{sec:si_compounding}

Section~\ref{sec:si_depthlaw} states the accumulation laws. This section
measures the stable, norm-preserving case on a controlled operator, so that the
growth of the symmetry defect is separated from dynamical instability, and it
supplies panel~C of main-text Figure~1.

Symmetry-preserving computations are rarely applied once. Neural networks
compose layers, dynamical simulations iterate a step map, and self-consistent
solvers repeat an update. For a composition of $M$ distinct layers with
independently distributed errors, cancellation can reduce the typical
accumulation to $O(\sqrt{M}\,\varepsilon)$, while $O(M\varepsilon)$ remains the
worst case~\cite{assaad2023vntransformer}. Iterating one fixed step realizes the
worst case, because its per-step defects add coherently.

\begin{figure}[!t]
\centering
\includegraphics[width=\textwidth]{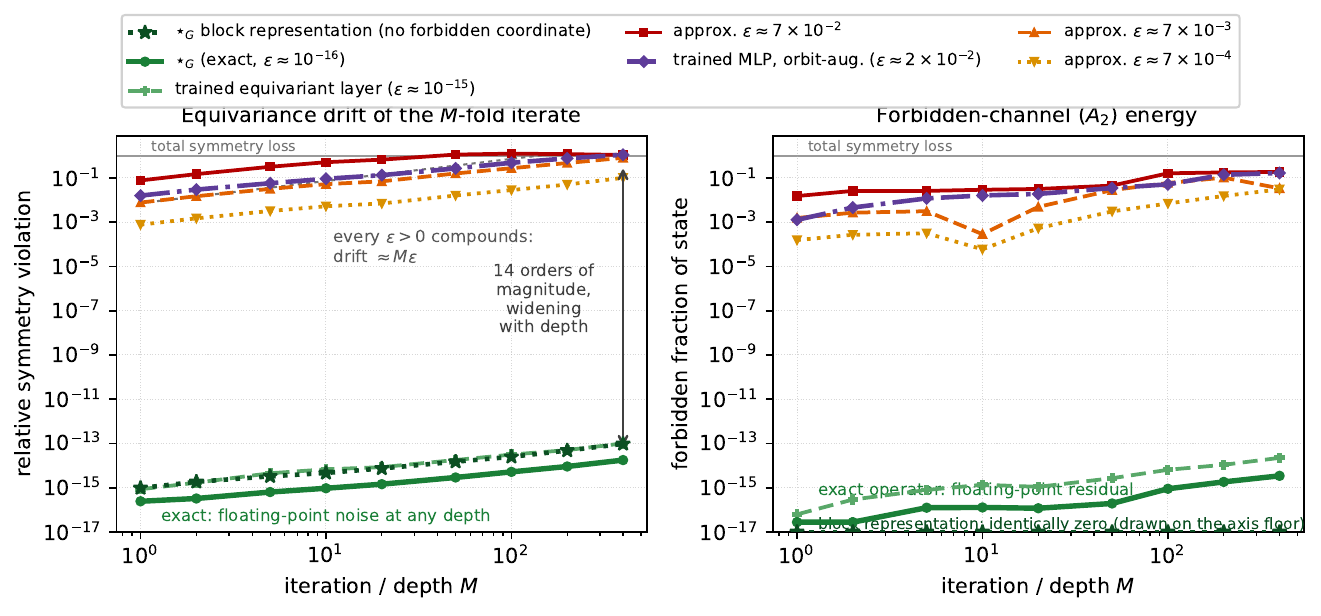}
\caption{\textbf{Accumulation of a per-step equivariance defect on a controlled operator.}
One-step operators on the regular representation of the octahedral group $O$, iterated $M$ times. The exact operator is equivariant to floating-point precision, and the approximate operators have measured single-step errors $\varepsilon$ spanning values observed in trained models~\cite{gruver2023lie,wang2022approximately}. Two operators trained on identical data illustrate the two regimes: an orbit-augmented perceptron ($\varepsilon\approx1.7\times10^{-2}$) and an equivariant linear layer ($\varepsilon\approx10^{-15}$). (\textbf{Left}) Relative equivariance violation of the $M$-fold iterate. Exact operators remain at floating-point noise, and operators with $\varepsilon>0$ grow as $M\varepsilon$ toward $O(1)$. (\textbf{Right}) Energy in the symmetry-forbidden $A_2$ channel. In the block representation the forbidden coordinate is not represented, so this energy is identically zero at every depth and is drawn on the axis floor. The same exact operator carried in ambient coordinates leaves a floating-point residual of $3\times10^{-15}$ at $M=400$, and the approximate iterates leak increasingly. At $M=400$ the exact and approximate cases differ by fourteen orders of magnitude.}
\label{fig:compounding}
\end{figure}

We ran the experiment on the octahedral group $O$
(offline and deterministic, with every reported invariant verified at runtime). Each step is an orthogonal linear map on the
regular representation. The exact operator is $T=\exp(S)$, where $S$ is an
equivariant antisymmetric generator obtained by Reynolds symmetrization, and its
measured single-step equivariance error is $2.5\times10^{-16}$. Because $T$ is
block-diagonal in the irreducible decomposition, it cannot transfer energy
between inequivalent irreducible components. The approximate operators are
$T_\varepsilon=\exp(S+\varepsilon B)$, with $B$ a generic non-equivariant
antisymmetric generator, so they remain orthogonal while their equivariance
errors are measured directly. We swept
$\varepsilon\approx7\times10^{-2},\,7\times10^{-3},\,7\times10^{-4}$, the range
reported for data augmentation and relaxed architectural
constraints~\cite{gruver2023lie,wang2022approximately}.

To connect the controlled operators to learned models, we fitted two operators to
the same exact step from the same $256$ training pairs. An orbit-augmented
multilayer perceptron, with each pair presented in five group frames, reached a
relative fit error of $8\times10^{-3}$ and a measured single-step equivariance
error of $1.7\times10^{-2}$. An equivariant linear layer obtained by least-squares
projection onto the commutant of the regular representation reached an
equivariance error of $9\times10^{-16}$. Identical training data and target
dynamics therefore produced learned operators on the two sides of the divide
between approximate and exact equivariance.

We iterated each operator for $M=1,\ldots,400$ steps (Figure~\ref{fig:compounding}).
The exact operator's equivariance error rose from $2\times10^{-16}$ at $M=1$ to
$1.8\times10^{-14}$ at $M=400$, the growth of rounding error alone. The operators
with $\varepsilon\approx7\times10^{-2}$ and $7\times10^{-3}$ reached $O(1)$
violation at about $M=50$ and $M=400$, and the operator with
$\varepsilon\approx7\times10^{-4}$ reached about $10\%$ by $M=400$. The trained
augmented perceptron followed the same law, and by $M=400$ it had moved $17\%$ of
its state energy into the forbidden channel. The trained equivariant layer stayed
at floating-point error ($10^{-13}$ at $M=400$).

The forbidden channel is the $A_2$ isotypic component of $O$, the
one-dimensional sign representation, which first occurs at angular momentum
$\ell=3$ (Section~\ref{sec:wigner_eckart}). The initial state has no $A_2$
component. Under the exact operator in ambient coordinates the $A_2$ energy
remained below $4\times10^{-15}$ at every depth. Carried in the block
representation, where the $A_2$ coordinate is absent from the state, the same
step left the forbidden energy identically zero, because no coordinate exists in
which a rounding error could appear. A controlled synthetic set with exact
$\bbZ_{12}$ symmetry gives the static counterpart: at matched accuracy ($R^2$ of
$1.000$ versus $0.998$), the $\starG$ singular value decomposition reaches a
rotation variance of $5.8\times10^{-31}$, against $3.7\times10^{-5}$ for an
orbit-augmented perceptron (Section~\ref{sec:si_experiments}).

\paragraph{Protocol.} One fixed draw (seed $0$) generates the generators, and
every reported quantity is a measured property of the resulting operators. The
exact step is $T=\exp(S)$ with
$S=\tfrac{0.6}{\|A-A^{\!\top}\|_2}(A-A^{\!\top})$ and $A$ the Reynolds average
$\frac{1}{|G|}\sum_g P_gA_0P_g^{\!\top}$ of a Gaussian matrix, so $S$ commutes with
the regular representation. The approximate steps use a fixed generic
antisymmetric $B$ normalized in spectral norm. The perceptron is a
two-hidden-layer $\tanh$ network ($24\to256\to256\to24$) trained by full-batch
Adam ($5{,}000$ epochs, learning rate $10^{-3}$), with each pair presented in the
identity frame and four random group translates. The equivariant layer is the
least-squares fit over the $24$-dimensional commutant of the regular
representation. Both trained operators are unrolled exactly as the synthetic
operators are, and the perceptron's state norm stays within $8\%$ of $\sqrt{24}$
at every checkpoint, so its drift is not a norm artifact. The single-step
equivariance error is the mean relative commutator residual over $64$ Gaussian
probe signals and all $24$ group elements.

\begin{table}[!t]
\centering
\caption{Structural comparison of equivariant networks and the $\starG$ algebra.}
\label{tab:comparison}
\small
\begin{tabular}{@{}p{3.4cm}p{5.3cm}p{6.5cm}@{}}
\toprule
\textbf{Feature} & \textbf{Equivariant networks} & \textbf{$\starG$ algebra} \\
\midrule
Symmetry     & Architectural constraint             & Algebraic property \\
Group domain & Fixed when the architecture is designed & An argument of the algebra: finite, or compact through band-limiting (Section~\ref{sec:si_compact}) \\
Composition  & Architecture re-specified per group & $F_{G_1} \otimes F_{G_2}$ (Section~\ref{sec:si_product_proof}) \\
Optimality   & None established for learned layers  & Eckart--Young, on the algebra and on any $G$-module (Theorem~\ref{thm:ey_module}) \\
Forbidden coordinates of a state of known symmetry & Present in the state, kept small by design and rounding & Absent from the block representation \\
Irrep channels & Recovered by a post-hoc probe (Section~\ref{sec:enn_probe}) & Explicit input-side axis (Section~\ref{sec:wigner_eckart}) \\
Output of another model & Not addressed & One projection per output (Proposition~\ref{prop:enforce_guarantee}) \\
\bottomrule
\end{tabular}
\end{table}

\section{Canonicity of the \texorpdfstring{$\starG$}{starG} Algebra}
\label{sec:necessity}

Within the transform-based ($\star_M$/t-product) family, the $\starG$ algebra is the unique algebra compatible with the group action. We make this precise below and draw the consequence that the
product, the block sizes, the rank of each block and the low-rank approximation problem
are determined by $G$ rather than chosen. Metric notions, such as singular values and
orthogonal projectors, are computed in a basis of unitary representations, and the
remaining freedom of the transform is a change of basis within each block.

Fix the left regular action of $G$ on elements of $\mathbb{K}_G$, $(g\cdot x)(h)=x(g^{-1}h)$,
equivalently $g\cdot\delta_a=\delta_{ga}$, where $\delta_g\in\bbC[G]$ is the
indicator of $g$. A \emph{transform-based product} is given by an invertible
$M\in\bbC^{n\times n}$ and a partition of the $n$ coordinates into blocks of
sizes $d_1^2,\dots,d_r^2$ (so $\sum_s d_s^2=n$), which identifies the transform
domain with $\mathcal B=\bigoplus_{s=1}^r M_{d_s}(\bbC)$; one sets
$a\star_M b=M^{-1}(Ma\cdot Mb)$, where $\cdot$ is block-diagonal matrix
multiplication. This is the $\star_M$/t-product template
\cite{kernfeld2015tensor,kilmer2021tensor}, with the abelian case $d_s\equiv1$.
Because $M$ is linear and blockmult is bilinear, associative and unital,
$\star_M$ is automatically an associative unital algebra; its identity is
$M^{-1}\!\big(\bigoplus_s I_{d_s}\big)$.

\begin{theorem}[Necessity and uniqueness of $F_G$]\label{thm:necessity}
Among all transform-based products $\star_M$ on $\bbC[G]$ that
\begin{enumerate}
\item[(i)] are associative with a two-sided identity, normalized so that the
identity is the group-identity element $\delta_e$ (equivalently $M\delta_e=\bigoplus_s I_{d_s}$), and
\item[(ii)] are left-equivariant, $(g\cdot\calA)\star_M\calB=g\cdot(\calA\star_M\calB)$
for all $g\in G$ and all $\calA,\calB$,
\end{enumerate}
the product is exactly group convolution, $\star_M=\starG$; the block sizes are
forced to be the irreducible dimensions with their multiplicities,
$\{d_1,\dots,d_r\}=\{d_\rho\}_{\rho\in\hat G}$ as multisets; and $M=\theta\circ F_G$
for a unique $\theta\in\mathrm{Aut}\big(\bigoplus_\rho M_{d_\rho}(\bbC)\big)$,
i.e.\ up to (a) a per-irrep change of basis
$\hat\calA(\rho)\mapsto P_\rho\,\hat\calA(\rho)\,P_\rho^{-1}$, $P_\rho\in\mathrm{GL}_{d_\rho}(\bbC)$,
and (b) a permutation of blocks of equal dimension. These are precisely the
choices already made in constructing $F_G$ (irrep representatives and their
ordering). Conversely $\starG$ with $M=F_G$ satisfies (i) and (ii).
\end{theorem}

\begin{proof}
\emph{Part I: $\star_M=\starG$.} Let $\iota$ be the two-sided identity; by (i)
$\iota=\delta_e$, and $g\cdot\delta_e=\delta_g$. Setting $\calA=\delta_e$ in (ii),
\begin{equation}\label{eq:transl}
\delta_g\star_M b=(g\cdot\delta_e)\star_M b=g\cdot(\delta_e\star_M b)=g\cdot b,
\qquad \forall\,g\in G,\ b\in\bbC[G],
\end{equation}
so left translation is left $\star_M$-multiplication by $\delta_g$. For any
$x=\sum_g x(g)\delta_g$, bilinearity and \eqref{eq:transl} give
\[
x\star_M b=\sum_g x(g)\,(\delta_g\star_M b)=\sum_g x(g)\,(g\cdot b)
=\Big(\textstyle\sum_g x(g)\,g\cdot\Big)b=x\starG b,
\]
since $\big(\sum_g x(g)\,g\cdot b\big)(h)=\sum_g x(g)\,b(g^{-1}h)=(x\starG b)(h)$.
Hence $\star_M=\starG$.

\emph{Part II: block sizes and transform.} By Part I,
$M(\calA\starG\calB)=M\calA\cdot M\calB$, so $M$ is an isomorphism of
$\bbC$-algebras from $(\bbC[G],\starG)$ onto $\mathcal B=\bigoplus_s M_{d_s}(\bbC)$.
By Maschke's theorem $\bbC[G]$ is semisimple, and by the Artin--Wedderburn
theorem $\bbC[G]\cong\bigoplus_{\rho\in\hat G}M_{d_\rho}(\bbC)$; the multiset of
block dimensions of a semisimple algebra is an isomorphism invariant (uniqueness
of the Wedderburn decomposition~\cite{serre1977linear}), so $\{d_s\}=\{d_\rho\}$.
The Peter--Weyl map $W\colon\calA\mapsto\bigoplus_\rho\hat\calA(\rho)$, whose
coordinate matrix is $F_G$, is also an algebra isomorphism onto
$\bigoplus_\rho M_{d_\rho}(\bbC)$ (the convolution theorem,
Theorem~\ref{thm:pw_spectral}). Hence
$\theta:=M\circ W^{-1}\in\mathrm{Aut}\big(\bigoplus_\rho M_{d_\rho}(\bbC)\big)$.
An automorphism of a direct sum of matrix algebras permutes the simple summands,
mapping each onto one of equal dimension (as $M_d(\bbC)\cong M_{d'}(\bbC)$ iff
$d=d'$), and restricts on each summand to an inner automorphism
$X\mapsto P_\rho X P_\rho^{-1}$ by the Skolem--Noether theorem. This is exactly
the freedom in (a) and (b). The converse is Proposition~\ref{prop:equivariance}
together with Theorem~\ref{thm:pw_spectral}.
\end{proof}

\begin{remark}[Strength of Part I]\label{rem:partI_strength}
Part I is stronger than the theorem's hypotheses suggest: its derivation uses
neither associativity nor the transform-based template. Any \emph{bilinear}
product on $\bbC[G]$ whose two-sided identity is $\delta_e$ and which is
left-equivariant is already group convolution. Associativity and the template
enter only in Part II, to force the block sizes and the transform.
\end{remark}

\begin{remark}[On hypothesis (i)]\label{rem:unit_normalization}
Associativity and the two-sided property of the identity are automatic for every
transform-based product (block-diagonal matrix multiplication is associative and
unital, and $M$ is a bijection); the sole content of (i) is the normalization that
the identity is the group-identity element $\delta_e$, which together with
left-equivariance forces $\star_M=\starG$, so that $M$ is an algebra isomorphism from the
convolution algebra. Without this normalization
the unit-twisted products $a\mapsto a\starG u^{-1}\starG b$ ($u$ any unit) are
associative, left-equivariant, and transform-based but are not convolution.
\end{remark}

\begin{lemma}[Left-equivariant products are twisted convolutions]\label{lem:twisted}
Let $G$ be a finite group and $\star$ a bilinear product on $\bbR[G]$ (or $\bbC[G]$) that is
left-equivariant, $(g\cdot a)\star b=g\cdot(a\star b)$, and has a two-sided
identity $u$. Then $u$ is invertible for convolution, $a\star b=a\starG u^{-1}\starG b$
for all $a,b$. In particular $\star$ is associative, and $\phi(a)=a\starG u^{-1}$ is an algebra isomorphism from
$(\bbR[G],\star)$ onto $(\bbR[G],\starG)$.
\end{lemma}

\begin{proof}
Write $T(b)=\delta_e\star b$, a linear map by bilinearity. Left-equivariance gives $\delta_g\star b=g\cdot T(b)$,
and bilinearity gives $a\star b=\sum_g a(g)\,g\cdot T(b)=a\starG T(b)$. The identity
property $u\star b=b$ reads $u\starG T(b)=b$ for every $b$, so left convolution by $u$
is surjective on the finite-dimensional space $\bbR[G]$, hence bijective, and there is
$v$ with $u\starG v=\delta_e$. Then right convolution by $u$ is injective, because
$x\starG u=0$ implies $x=x\starG u\starG v=0$, hence surjective in finite dimension,
and there is $w$ with $w\starG u=\delta_e$. Associativity of convolution gives
$w=w\starG u\starG v=v$, so $u$ has the two-sided convolution inverse $u^{-1}=v$. Then
$T(b)=u^{-1}\starG b$ and $a\star b=a\starG u^{-1}\starG b$. Finally
$\phi(a\star b)=a\starG u^{-1}\starG b\starG u^{-1}=\phi(a)\starG\phi(b)$, and $\phi$ is
bijective. The twisted product $a\starG u^{-1}\starG b$ is associative because
convolution is. Part~I of Theorem~\ref{thm:necessity} and this lemma are proved in
Lean~4 (Section~\ref{sec:lean}) using only Lean's standard axioms. The formal proof
assumes only that $u$ is a left identity, and it shows that a left identity of such a
product is also a right identity.
\end{proof}

\begin{corollary}[Real transforms]\label{cor:real_canon}
Let $M$ be a real invertible transform from $\bbR[G]$ onto a direct sum
$\bigoplus_s M_{d_s}(\bbR)$ of real matrix algebras, and $a\star_M b=M^{-1}(Ma\cdot Mb)$.
If $\star_M$ is left-equivariant, then every irreducible representation of $G$ is of
real type (Frobenius--Schur indicator $+1$). If in addition its identity is $\delta_e$,
then $M=\theta\circ F_G^{\bbR}$, where $F_G^{\bbR}$ is the unnormalized group-Fourier
transform $x\mapsto\bigoplus_\rho\sum_g x(g)\rho(g)$ built from real matrix forms of the
irreducible representations and $\theta$ permutes blocks of equal
size and changes the basis within each block by a real invertible matrix. Conversely,
if every irreducible representation is of real type, $F_G^{\bbR}$ defines such a
product. For an abelian group, every irreducible representation is of real type if and
only if $g^2=e$ for every $g\in G$, that is, $G\cong(\bbZ_2)^k$ with $k\ge0$, and then a
left-equivariant $\star_M$ with identity $\delta_e$ has as $M$ the unnormalized
Walsh--Hadamard transform up to an ordering of the characters.
For a cyclic group of order $n\ge3$, in particular, no transform-based product whose
transform is real and whose blocks are full real matrix algebras is left-equivariant.
\end{corollary}

\begin{proof}
The product $\star_M$ is associative, and $M^{-1}\big(\bigoplus_s I_{d_s}\big)\in\bbR[G]$ is a two-sided identity.
By Lemma~\ref{lem:twisted} the algebra $(\bbR[G],\star_M)$ is isomorphic to the group
algebra $(\bbR[G],\starG)$, and $M$ is an isomorphism of real algebras onto
$\bigoplus_s M_{d_s}(\bbR)$, so $\bbR[G]\cong\bigoplus_s M_{d_s}(\bbR)$. By Maschke's
theorem $\bbR[G]$ is semisimple, and its Wedderburn decomposition contains one simple
component $M_{d_\rho}(\bbR)$ for each irreducible representation of real type,
$M_{d_\rho}(\bbC)$ for each conjugate pair of complex type and $M_{d_\rho/2}(\mathbb{H})$
for each representation of quaternionic type~\cite{serre1977linear}. The simple
components of a semisimple algebra are unique up to isomorphism. The centre of
$M_{d}(\bbC)$ is $\bbC$, and $M_k(\mathbb{H})$ is not isomorphic to any $M_d(\bbR)$
because its class in the Brauer group of $\bbR$ is nontrivial, so every component is of
the form $M_{d_\rho}(\bbR)$ and every irreducible representation is of real type. With
identity $\delta_e$, Part~I of Theorem~\ref{thm:necessity} applies over $\bbR$ unchanged,
so $\star_M=\starG$ and $\theta=M\circ(F_G^{\bbR})^{-1}$ is an automorphism of
$\bigoplus_\rho M_{d_\rho}(\bbR)$. It permutes the simple summands among those of equal
size and is inner on each, by the Skolem--Noether theorem for the central simple
$\bbR$-algebras $M_{d_\rho}(\bbR)$. On a block of size one the inner automorphism is
the identity. For an abelian group the irreducible representations
are the characters, a character is of real type exactly when it takes values in
$\{\pm1\}$, and all characters do so exactly when every element has order at most two,
because the characters separate the elements of $G$. The real characters of
$(\bbZ_2)^k$ are the rows of the Walsh--Hadamard matrix, and an automorphism of
$\bbR^{2^k}$ with the entrywise product maps the primitive idempotents, the coordinate
vectors, onto each other and fixes the unit $(1,\dots,1)$, so it is a coordinate
permutation.
\end{proof}

\noindent Over $\bbC$, Theorem~\ref{thm:necessity} gives the analogous statement for the
cyclic group $\bbZ_n$: the automorphisms of $\bbC^n$ with the entrywise product are
coordinate permutations, so a left-equivariant transform-based product with identity
$\delta_e$ uses the unnormalized discrete Fourier transform up to an ordering of the frequencies.
The $\star_M$ products built on the discrete cosine transform, or on any other real
transform with real scalar or real matrix blocks, are therefore not left-equivariant
under cyclic shifts of length three or more. A real computation can still realize the
$\starG$ product of a cyclic group by pairing conjugate frequencies into blocks that
represent complex numbers, which the template of full real matrix algebras does not
contain.

\begin{remark}[Equivalent view via Schur]
Equation \eqref{eq:transl} says $M$ conjugates the regular representation into
$\bigoplus_s(A_s\otimes I_{d_s})$ with $A_s(g)=(M\delta_g)_s$. Since each block is
a full matrix algebra, $\{A_s(g)\}_g$ spans $M_{d_s}(\bbC)$, so $A_s$ is
irreducible (Burnside). Matching $\bigoplus_s d_s\cdot A_s$ with the regular
representation $\bigoplus_\rho d_\rho\cdot\rho$ and using uniqueness of the
isotypic decomposition (Schur's lemma) recovers $\{d_s\}=\{d_\rho\}$ and the
stated ambiguity.
\end{remark}

\begin{corollary}[Intrinsic Eckart--Young optimum]\label{cor:intrinsic}
Once the Parseval normalization of Theorem~\ref{thm:ey_full} is fixed, the residual
freedom of Theorem~\ref{thm:necessity} is restricted from the full per-irrep
$\mathrm{GL}_{d_\rho}(\bbC)$ to per-irrep unitary $P_\rho\in U(d_\rho)$ together with
permutations of equal-dimension blocks (a non-unitary change of basis would alter
the block singular values, so the invariance below is a post-normalization
statement). Under this freedom every
$\starG$-SVD singular value, each singular-element norm $\|\mathbf s_i\|_F$, the
$\starG$-rank, and the optimal error $\|\calA-\calA_k\|_F$ are unchanged. Hence
the optimal rank-$k$ approximation is an invariant of $(G,\calA)$, independent of
the choice of irrep basis and block order (up to the usual degenerate
singular-value freedom).
\end{corollary}

\begin{proof}
A per-irrep unitary change of basis acts on each Fourier block by
$\hat\calA(:,:,\rho)\mapsto(I\otimes P_\rho)\hat\calA(:,:,\rho)(I\otimes P_\rho)^H$,
a unitary equivalence preserving singular values and Frobenius norm; block
permutations only relabel summands. The singular elements are assembled from the
block singular values (Theorem~\ref{thm:ey_full}, Step~3), so all listed
quantities are invariant, and truncation commutes with the unitary equivalence.
\end{proof}

\begin{remark}[Eckart--Young optimality selects $F_G$]
Eckart--Young optimality is not special to $F_G$: Mor~\cite{mor2025eckart} shows
it holds for an entire family of transforms ($M=DQ$ with $Q$ unitary and $D$ a
nonzero real diagonal scaling), so best low-rank optimality alone cannot single out
any one transform. The operative content is therefore the selection principle of
Theorem~\ref{thm:necessity}: equivariance uniquely determines the group-Fourier
transform, singling out $F_G$ from that family and making the $\starG$ truncated
SVD the canonical, intrinsic optimum (Corollary~\ref{cor:intrinsic}) rather than
one arbitrary member of an Eckart--Young family.
\end{remark}

\paragraph{Relationship to prior work.}
Dunbar and Newman~\cite{dunbar2025tensor} give a representation-theoretic
\emph{interpretation} of the $\star_M$ product: through Schur's lemma and a
symmetry-adapted basis, a group representation furnishes a matrix $M$ that
block-diagonalizes the product into isotypic components, so the group-Fourier
transform is a natural and sufficient choice. Theorem~\ref{thm:necessity}
establishes the converse and sharper direction: requiring the product to be
equivariant \emph{forces} $M$ to be the non-abelian group-Fourier transform $F_G$,
unique up to a per-irrep change of basis and a permutation of equal-dimension
blocks. They answer which $M$ respects the group, by construction; we answer which
$M$ is compelled by equivariance, by necessity and uniqueness. The abelian case
recovers Jaming's characterization of the Fourier transform as the only linear map
carrying convolution to a pointwise product~\cite{jaming2009characterization}. The
block-diagonalization itself is classical (Peter--Weyl, Artin--Wedderburn); what is
new is the necessity direction and its uniqueness clause.

\paragraph{Numerical corroboration.}
We verify Theorem~\ref{thm:necessity} and Corollary~\ref{cor:intrinsic} numerically
for a panel of groups spanning the abelian and non-abelian cases (cyclic $\bbZ_4$
and $\bbZ_6$; symmetric $S_3$, dimensions $1,1,2$; quaternion $Q_8$, dimensions
$1,1,1,1,2$; octahedral $O$, dimensions $1,1,2,3,3$), in complex double precision. Across $200$ seeded trials per group, every
identity the theorem asserts holds to $\le 10^{-13}$: $F_G$ block-diagonalizes
convolution, the equivariant product is convolution (Part~I), and a per-irrep basis
change or an equal-dimension block permutation preserves it; by
Corollary~\ref{cor:intrinsic}, a per-irrep \emph{unitary} basis change leaves the
$\starG$-SVD singular values unchanged. Each hypothesis is confirmed load-bearing by
an $O(1)$ discriminator: the all-one-dimensional partition cannot reproduce
non-commutative convolution (commutator norm $1.41$), a cross-\emph{dimension} block
swap leaves the automorphism group and breaks both the product and equivariance
(gap $\ge 15$), the unit-twisted product with identity $\neq \delta_e$ differs from
convolution (gap $\ge 5$), and a \emph{non-unitary} per-irrep change moves the
singular values (change $O(1)$, exercised where an irrep has dimension $\ge 2$).

\section{Compact Groups: the Band-Limited Reduction in Detail}
\label{sec:si_compact}

This section carries the detail behind the band-limited reduction of the main
text: the compact-group Peter--Weyl setting, the reduction argument, the numerical
validation on $\mathrm{SO}(3)$, the reach to space groups and to the Euclidean and
Poincar\'e groups, and the bearing on equivariant networks.

The main-text construction used a finite group only to make the convolution tensor
$\calT$ finite-dimensional; the algebraic content is supplied entirely by the
Peter--Weyl theorem, which holds for \emph{every compact group}. For compact $G$
with normalized Haar measure, $L^2(G)$ decomposes as a Hilbert direct sum over
the countable set $\widehat{G}$ of irreducible unitary representations, the group
Fourier transform sends $f \mapsto (\widehat{f}(\rho))_{\rho \in \widehat G}$ with
$\widehat{f}(\rho) = \int_G f(g)\,\rho(g)\,dg \in \mathbb{C}^{d_\rho \times d_\rho}$,
convolution becomes block-diagonal matrix multiplication exactly as in
the spectral decomposition of the main text, and the Plancherel identity reads $\|f\|^2 = \sum_{\rho}
d_\rho \,\|\widehat{f}(\rho)\|_{\mathrm{HS}}^2$. For $G=\mathrm{SO}(3)$ the blocks
are the Wigner $D^{(l)}$ matrices with $d_\rho = 2l+1$, $l=0,1,2,\dots$: the
angular-momentum channels of physics are \emph{literally} the irrep blocks of the
$\starG$ algebra. The $\starG$ product, transpose, norm, and equivariance
therefore carry over verbatim; the only change is that the irrep sum is now
countably infinite. The $\starG$-SVD likewise generalizes to the simultaneous SVD
of the irrep blocks, and optimal low-rank approximation reduces block by block to
the matrix Eckart--Young theorem, with the Plancherel weights $d_\rho$ selecting
which singular values to retain across blocks; for Hilbert--Schmidt operators the
weighted singular-value sequence is summable and the optimum is attained.

The case that matters in practice is finite-dimensional and needs no such
hypothesis. Call a signal \emph{band-limited} to a finite set $S \subset
\widehat{G}$ if $\widehat{f}(\rho)=0$ for $\rho \notin S$ (for $\mathrm{SO}(3)$,
$S = \{l \le l_{\max}\}$).

The proof is immediate: band-limiting truncates the Peter--Weyl sum to the
finitely many blocks in $S$, and block-diagonal multiplication preserves the
truncation, so the band-limited algebra \emph{is} the finite $\starG$ algebra
over $S$. Its reduction target is therefore the same finite-dimensional algebra
the Lean~4 development formalizes over the irrep set $S$; the Peter--Weyl
truncation that effects the reduction is a pen-and-paper argument and is not
itself separately machine-checked.

We validate this reduction numerically for $G=\mathrm{SO}(3)$ band-limited to
$l \le 3$, applying the \emph{same} $\starG$-product and $\starG$-SVD code used
for the finite groups directly to the Wigner-$D$ Fourier blocks. Every finite-group guarantee holds
to machine precision: per-block SVD agreement $2\times10^{-14}$, Eckart--Young
closed-form tail-sum match $2\times10^{-12}$, exact
$\mathrm{SO}(3)$-equivariance of the per-irrep power and singular values to
$2\times10^{-15}$, and a compression advantage on
$\mathrm{SO}(3)$-structured operators that vanishes on unstructured controls.
Band-limiting does not reintroduce the approximate side of
the main-text compounding dichotomy: truncation is an equivariant
projection onto a subalgebra closed under the product, so iterating a
band-limited operator remains exactly equivariant at every depth, and the
$O(L^{-s/2})$ tail incurred on merely smooth signals
(Proposition~\ref{prop:bandlimit_error}) is a resolution error confined to
the symmetry-allowed channels; in practice the band limit is chosen from the
observable spectral tail of the data
(Remark~\ref{rem:truncation_exactness}).

The same mechanism reaches the crystallographic space groups, the actual symmetry
of matter. A space group is an extension of a point group by the lattice
translations, and because the translations are dual to the compact Brillouin zone,
imposing periodic boundary conditions replaces the space group by a finite
group. The algebra, its Eckart--Young optimality, and its selection-rule
enforcement then apply verbatim, resolved per Brillouin-zone point through
the little-group irreducible representations
(Section~\ref{sec:spacegroups}). This is a reduction
to the finite theory rather than new analysis: it extends the point-group treatment
of crystals used above to all $230$ space groups (and, with a projective twist, the
non-symmorphic ones), and its $\Gamma$-point instance is exactly the point-group
selection-rule projection applied to cubic crystals in
the main-text crystal section. Antiunitary time reversal (magnetic and grey groups)
and band topology are out of scope, as detailed there. Nor does the reach stop
at crystals: the Mackey construction extends the same machinery through
compact little-group fibers to the noncompact Euclidean and Poincar\'e
groups, validated numerically on the Poincar\'e mass shell and on an
$\mathrm{SU}(2)$ lattice-gauge plaquette
(Section~\ref{sec:wignermackey}).

This has a direct bearing on equivariant neural networks. Every
$\mathrm{SO}(3)$-equivariant architecture used in practice (e3nn, NequIP,
MACE) represents features in a finite set of angular-momentum channels $l \le
l_{\max}$ (typically $l_{\max}\in\{1,2,3\}$); their feature spaces are precisely
band-limited $\star_{\mathrm{SO}(3)}$ modules. By the band-limited reduction theorem of the main text
the $\starG$-SVD furnishes the \emph{provably optimal} equivariant low-rank
operation on those modules in closed form (the equivariant-layer corollary of the main text), whereas the network fits a map within that same equivariant class by gradient descent, with no optimality characterization.
Architectural and algebraic equivariance thus act on the same band-limited
representation space; what the $\starG$ algebra adds is the Eckart--Young
guarantee and the closed-form per-irrep decomposition
(the per-irrep table of the main text) that standard ENN training does not expose. The finite octahedral
computation of our Wigner--Eckart experiment is one instance of this band-limited
algebra; a continuous $\star_{\mathrm{SO}(3)}$ computation at $l_{\max}=2$ is
another.

\begin{figure}[!ht]
\centering
\includegraphics[width=0.85\textwidth]{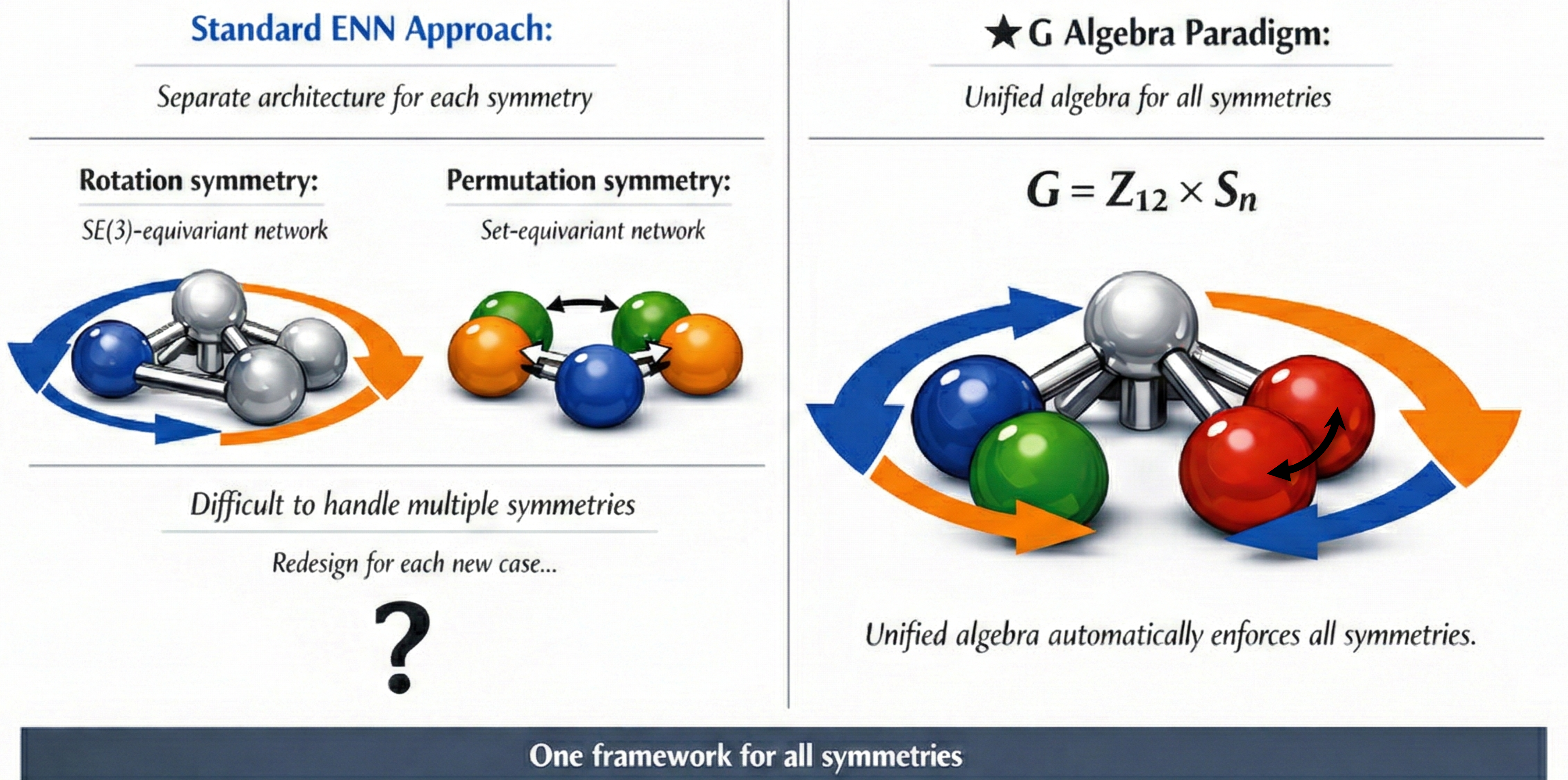}
\caption{\textbf{Two routes to equivariance.}
\textit{Left (architectural equivariance):} the symmetry constraint is built
into the layers, so the set of symmetries is fixed when the architecture is
designed.
\textit{Right (algebraic equivariance):} the same algebra handles any finite
or band-limited compact group; composing symmetries requires only specifying
$G_1 \times G_2$ in the Fourier transform.}
\label{fig:paradigm}
\end{figure}

\subsection{Finite subgroups that reproduce the rotation average}

\begin{proposition}\label{prop:subgroup_si}
Let $H\subset\mathrm{SO}(3)$ be finite, $L\ge1$, and let $V_L=\bigoplus_{\ell\le L}V_\ell\otimes\mathbb{C}^{m_\ell}$
be a band-limited representation space. The averages
$R_H=|H|^{-1}\sum_{h\in H}\rho(h)$ and $R_{\mathrm{SO}(3)}=\int\rho(g)\,dg$ agree on $V_L$
if and only if $\dim V_\ell^{H}=0$ for $1\le\ell\le L$, where
$\dim V_\ell^{H}=|H|^{-1}\sum_{h\in H}\chi_\ell(h)$ and
$\chi_\ell(h)=\sin\big((2\ell+1)\theta_h/2\big)/\sin(\theta_h/2)$ for a rotation by
angle $\theta_h$.
\end{proposition}

\begin{proof}
Both averages commute with the action and preserve each isotypic component, so it
suffices to compare them on $V_\ell\otimes\mathbb{C}^{m_\ell}$. There $R_H$ acts as
$P_\ell^H\otimes\mathrm{id}$, with $P_\ell^H$ the orthogonal projector onto the
$H$-invariant subspace of $V_\ell$, whose dimension is the multiplicity of the trivial
character, $|H|^{-1}\sum_h\chi_\ell(h)$. The rotation average acts as the projector onto
the $\mathrm{SO}(3)$-invariants, which is the identity for $\ell=0$ and zero for
$\ell\ge1$ by Schur's lemma. The two agree on $V_\ell$ if and only if
$P_\ell^H=0$ for every $1\le\ell\le L$, that is, if and only if $\dim V_\ell^H=0$.
\end{proof}

\noindent We built the chiral tetrahedral ($T$, order $12$),
octahedral ($O$, $24$) and icosahedral ($I$, $60$) groups as rotation matrices and counted $\dim V_\ell^H$ for $\ell\le12$. The first $\ell\ge1$ with an invariant is $3$ for $T$ (the
$xyz$ polynomial), $4$ for $O$ and $6$ for $I$, so the groups reproduce the rotation
average up to $L=2$, $3$ and $5$. On random functions on the sphere band-limited to these
$L$, the $H$-average and the rotation average agree to $3\times10^{-16}$, $2\times10^{-16}$
and $2\times10^{-15}$, and at the next $\ell$ they differ by $0.25$, $0.25$ and $2.4$. The counts also give the selection rules used
in the main text: under $O$, a symmetric rank-2 tensor ($\ell=0,2$) has one invariant and is
isotropic, while the elastic tensor ($\ell=0,2,4$) has two, the isotropic part and the
$\ell=4$ cubic anisotropy.

\section{The Algebra on a General \texorpdfstring{$G$}{G}-Module}
\label{sec:gmodule}

The construction of the main text places the algebra on the free module
$\mathbb{K}_G^{\,m}$, whose group-Fourier block at $\rho$ has size
$m d_\rho\times m d_\rho$. The feature space of an equivariant network is a
$G$-module with arbitrary multiplicities, and is free only in special cases.
This section states the algebra, and its Eckart--Young theorem, for an arbitrary
finite-dimensional $G$-module. Everything is a consequence of Schur's lemma and
the per-block argument of Section~\ref{sec:si_svd_proof}; the point is that the
optimality statement then applies verbatim to the layers of e3nn, NequIP and
MACE, which the free-module statement does not cover.

Let $G$ be finite, or compact and band-limited to a finite $S\subset\hat G$, and
let
\begin{equation}\label{eq:gmodule}
V \;=\; \bigoplus_{\rho\in S} V_\rho\otimes\mathbb{C}^{m_\rho},
\end{equation}
where $V_\rho$ carries the irreducible representation $\rho$ of dimension
$d_\rho$, the integer $m_\rho\ge0$ is its multiplicity, and $G$ acts as
$\rho\otimes\mathrm{id}$ on each summand. Equip $V$ with a $G$-invariant inner
product in which each $V_\rho$ is unitary and each multiplicity space
$\mathbb{C}^{m_\rho}$ is Euclidean.

\begin{proposition}[The commutant and its norm]\label{prop:commutant}
The algebra of $G$-equivariant linear maps on $V$ is
\begin{equation}
\mathrm{End}_G(V)\;=\;\Big\{\textstyle\bigoplus_{\rho\in S}\mathrm{id}_{V_\rho}\otimes A_\rho
\;:\;A_\rho\in\mathbb{C}^{m_\rho\times m_\rho}\Big\}\;\cong\;
\bigoplus_{\rho\in S} M_{m_\rho}(\mathbb{C}),
\end{equation}
a semisimple algebra closed under the adjoint, and for
$T=\bigoplus_\rho \mathrm{id}_{V_\rho}\otimes A_\rho$,
\begin{equation}\label{eq:gmodule_norm}
\|T\|_F^2\;=\;\sum_{\rho\in S} d_\rho\,\|A_\rho\|_F^2 .
\end{equation}
\end{proposition}

\begin{proof}
A $G$-equivariant map sends the $\rho$-isotypic component into itself, so it is
block-diagonal over $S$. On $V_\rho\otimes\mathbb{C}^{m_\rho}$, Schur's lemma
gives $\mathrm{Hom}_G(V_\rho,V_\rho)=\mathbb{C}\,\mathrm{id}$, so the restriction
is $\mathrm{id}_{V_\rho}\otimes A_\rho$ for a unique
$A_\rho\in\mathbb{C}^{m_\rho\times m_\rho}$; composition acts blockwise, and the
adjoint sends $A_\rho$ to $A_\rho^H$ because the inner product is invariant and
each $V_\rho$ is unitary. Equation~\eqref{eq:gmodule_norm} is
$\|\mathrm{id}_{V_\rho}\otimes A_\rho\|_F^2=\mathrm{tr}(\mathrm{id}_{V_\rho})\,\|A_\rho\|_F^2
=d_\rho\|A_\rho\|_F^2$, summed over $\rho$.
\end{proof}

\begin{remark}[Field]\label{rem:field}
The statement is over $\mathbb{C}$. For a real representation whose irreducible
constituents are all of real type, which is the case for $\mathrm{SO}(3)$ and for
the octahedral group used throughout, the commutant of each constituent is
$\mathbb{R}$ by Schur and the same statement holds verbatim over $\mathbb{R}$,
with $M_{m_\rho}(\mathbb{R})$ in place of $M_{m_\rho}(\mathbb{C})$. Complex and
quaternionic types enlarge the commutant of a constituent to $\mathbb{C}$ or
$\mathbb{H}$, and the block algebra becomes a sum of matrix algebras over those
division rings; the Eckart--Young argument is unchanged, because it uses only
the $\ast$-structure and the weighted norm. This is the same real-versus-complex
distinction recorded for the Lean development in
Section~\ref{sec:lean}.
\end{remark}

\begin{corollary}[Real tensors and admissible rank profiles]\label{cor:real_profiles}
Let $\calA$ be a real $\starG$ tensor. Choose the irreducible unitary representations so
that each representation of real type is real orthogonal, each representation $\rho$ of
complex type appears together with its entrywise complex conjugate $\bar\rho$, and each
representation $\rho$ of quaternionic type satisfies $\overline{\rho(g)}=S_\rho\,\rho(g)\,S_\rho^{-1}$
for a unitary $S_\rho$ with $S_\rho\overline{S_\rho}=-I$. For an $\ell\times m$ tensor write $R^{(\ell)}_\rho=I_\ell\otimes S_\rho$ and
$R^{(m)}_\rho=I_m\otimes S_\rho$. Then $\widehat\calA(\rho)$ is real on every block of real
type, $\widehat\calA(\bar\rho)=\overline{\widehat\calA(\rho)}$ on every complex-conjugate pair
and $\overline{\widehat\calA(\rho)}=R^{(\ell)}_\rho\widehat\calA(\rho)\big(R^{(m)}_\rho\big)^{-1}$ on
every block of quaternionic type, and every real $\calB$ satisfies the same relations.
For any rank profile $(k_\rho)$, the minimum of $\|\calA-\calB\|_F^2$ over real $\calB$
whose block at $\rho$ has rank at most $k_\rho$ is attained by truncating each block of
$\calA$ to the admissible profile $k'_\rho=\min(k_\rho,k_{\bar\rho})$ on complex-conjugate
pairs, $k'_\rho=2\lfloor k_\rho/2\rfloor$ on blocks of quaternionic type and $k'_\rho=k_\rho$
on blocks of real type, and the truncation can be chosen real. The minimum equals
$\sum_\rho (d_\rho/n)\sum_{j>k'_\rho}\sigma_j(\rho)^2$.
\end{corollary}

\begin{proof}
For a real function $x$, $\overline{\hat x(\rho)}=\sum_g x(g)\overline{\rho(g)}$, which is
$\hat x(\rho)$ itself for a real orthogonal $\rho$, $\hat x(\bar\rho)$ for a complex-conjugate
pair and $S_\rho\hat x(\rho)S_\rho^{-1}$ for quaternionic type, and the relations for the tensor
blocks follow entrywise.
Conversely, by the inverse transform a family of blocks satisfying the relations is the
transform of a real tensor. For a real $\calB$ the relations give
$\operatorname{rank}\widehat\calB(\bar\rho)=\operatorname{rank}\widehat\calB(\rho)$ on a pair.
On a block of quaternionic type the maps $J_m v=(R^{(m)}_\rho)^{-1}\bar v$ on the domain and
$J_\ell w=(R^{(\ell)}_\rho)^{-1}\bar w$ on the codomain are antiunitary with square $-I$, since
$S_\rho\overline{S_\rho}=-I$, and the relation states that $\widehat\calB(\rho)J_m=J_\ell\widehat\calB(\rho)$.
Its kernel is therefore invariant under an antiunitary map whose square is $-I$, so the
kernel and the rank are even-dimensional. For real $\calB$ the constraints
$\operatorname{rank}\widehat\calB(\rho)\le k_\rho$ and $\operatorname{rank}\widehat\calB(\rho)\le k'_\rho$
therefore define the same feasible set. By the Plancherel identity the objective separates over blocks,
and the matrix Eckart--Young theorem bounds each block's error from below by the sum of
its singular values squared beyond $k'_\rho$. The truncation of $\overline{\widehat\calA(\rho)}$
at rank $k'_\rho$ is the conjugate of the truncation of $\widehat\calA(\rho)$ when the
singular vectors are chosen conjugate to each other, so the truncated blocks of a pair
satisfy the relation. On a block of quaternionic type,
$\widehat\calA(\rho)^H\widehat\calA(\rho)$ commutes with $J_m$, so each of its eigenspaces is
$J_m$-invariant. A $J_m$-invariant subspace of any even dimension up to that of the
eigenspace exists, spanned by pairs $v,J_mv$, so the leading right singular subspace of an
even dimension $k'_\rho$ can be chosen $J_m$-invariant also at a tie. The truncation
$\widehat\calA(\rho)\Pi$, with $\Pi$ the orthogonal projector onto that subspace, then satisfies
$\widehat\calA(\rho)\Pi J_m=\widehat\calA(\rho)J_m\Pi=J_\ell\widehat\calA(\rho)\Pi$, which is the
relation. The truncated blocks therefore come from a
real tensor that attains the lower bound.
\end{proof}

\begin{theorem}[Eckart--Young on a general $G$-module]\label{thm:ey_module}
Let $T\in\mathrm{End}_G(V)$ have blocks $A_\rho$ with singular values
$\sigma_1(\rho)\ge\sigma_2(\rho)\ge\cdots$, and let $(k_\rho)_{\rho\in S}$ be any
per-block rank profile. The map $T_{(k_\rho)}$ obtained by truncating each
$A_\rho$ to its leading $k_\rho$ singular values minimizes $\|T-B\|_F^2$ over all
$B\in\mathrm{End}_G(V)$ whose block at $\rho$ has rank at most $k_\rho$, and
\begin{equation}
\big\|T-T_{(k_\rho)}\big\|_F^2=\sum_{\rho\in S} d_\rho\sum_{j>k_\rho}\sigma_j(\rho)^2 .
\end{equation}
\end{theorem}

\begin{proof}
By Proposition~\ref{prop:commutant} the competitor class factorizes over $\rho$
and the squared norm is the $d_\rho$-weighted sum of the per-block squared norms,
so the problem separates into independent matrix problems. Apply the classical
Eckart--Young theorem in each block and sum with the positive weights $d_\rho$;
this is the argument of Theorem~\ref{thm:ey_full}(i) with the Plancherel weights
$d_\rho/n$ replaced by $d_\rho$, and it is the semisimple statement of
Remark~\ref{rem:semisimple} instantiated at $\mathrm{End}_G(V)$.
\end{proof}

\begin{corollary}[Truncation in a weighted invariant inner product]
\label{cor:ey_module_weighted}
Let $\Sigma$ and $\Gamma$ be positive definite operators on $V$ that commute with
the group action, for example the group averages
$\bar\Sigma=\int_G g\,\Sigma_0\,g^{-1}\,dg$ of the second moment $\Sigma_0$ of a
layer's inputs and $\bar\Gamma$ of the second moment of the gradient of a
downstream scalar (such as the network's energy) with respect to the layer's
outputs. Then $\Sigma=\bigoplus_\rho\mathrm{id}_{V_\rho}\otimes\Sigma_\rho$ and
$\Gamma=\bigoplus_\rho\mathrm{id}_{V_\rho}\otimes\Gamma_\rho$ with positive
definite blocks. For $T\in\mathrm{End}_G(V)$ with blocks $A_\rho$ and any rank
profile $(k_\rho)$, the map $B$ with blocks
$B_\rho=\Gamma_\rho^{-1/2}\big[\Gamma_\rho^{1/2}A_\rho\Sigma_\rho^{1/2}\big]_{k_\rho}\Sigma_\rho^{-1/2}$
minimizes $\|\Gamma^{1/2}(T-B)\Sigma^{1/2}\|_F^2$ over $B\in\mathrm{End}_G(V)$
with $\mathrm{rank}\,B_\rho\le k_\rho$, and the minimum is
$\sum_\rho d_\rho\sum_{j>k_\rho}\sigma_j(\Gamma_\rho^{1/2}A_\rho\Sigma_\rho^{1/2})^2$.
If the inputs $x$ have second moment $\Sigma_0$, then
$\mathbb{E}\|(T-B)x\|^2=\|(T-B)\Sigma_0^{1/2}\|_F^2=\|(T-B)\bar\Sigma^{1/2}\|_F^2$
for every equivariant $B$, so with $\Gamma=I$ the minimizer also minimizes the
expected squared output error of the layer on its inputs.
\end{corollary}

\begin{proof}
$\Sigma$ and $\Gamma$ commute with the action, so by
Proposition~\ref{prop:commutant} they lie in $\mathrm{End}_G(V)$ and have the
stated blocks, which are positive definite because $\Sigma$ and $\Gamma$ are. Then
$\|\Gamma^{1/2}(T-B)\Sigma^{1/2}\|_F^2=\sum_\rho d_\rho\|\Gamma_\rho^{1/2}(A_\rho-B_\rho)\Sigma_\rho^{1/2}\|_F^2$
by~\eqref{eq:gmodule_norm}. The substitution
$X_\rho=\Gamma_\rho^{1/2}B_\rho\Sigma_\rho^{1/2}$ preserves rank because both
factors are invertible, so each term is the classical Eckart--Young problem for
$\Gamma_\rho^{1/2}A_\rho\Sigma_\rho^{1/2}$, solved by the truncated singular value
decomposition. For the last statement,
$\mathbb{E}\|(T-B)x\|^2=\mathrm{tr}\big((T-B)\Sigma_0(T-B)^{H}\big)$, and for
equivariant $T-B$ and unitary $g$,
$\mathrm{tr}\big((T-B)g\Sigma_0g^{-1}(T-B)^{H}\big)=\mathrm{tr}\big((T-B)\Sigma_0(T-B)^{H}\big)$,
so averaging over $g$ replaces $\Sigma_0$ by $\bar\Sigma$ without changing the
value. When every block costs the same number of parameters per retained rank,
the optimal profile under a total budget keeps the globally largest values of
$d_\rho\,\sigma_j(\Gamma_\rho^{1/2}A_\rho\Sigma_\rho^{1/2})^2$, because the minimum is a sum of
independent tail sums with equal unit costs.
\end{proof}

\begin{remark}[The free module is the special case $m_\rho=m\,d_\rho$]
The regular representation contains $\rho$ with multiplicity $d_\rho$, so
$\mathbb{K}_G^{\,m}$ has $m_\rho=m\,d_\rho$ and
$\mathrm{End}_G(\mathbb{K}_G^{\,m})\cong\bigoplus_\rho M_{m d_\rho}(\mathbb{C})$,
which is the $\starG$ algebra of the main text; the weights differ by the overall
Plancherel normalization $1/n$. Conversely a module of the
form~\eqref{eq:gmodule} is free precisely when its multiplicities are
proportional to the irrep dimensions.
\end{remark}

\begin{remark}[Why the general form is needed for equivariant networks]
\label{rem:enn_multiplicity}
A typical $\mathrm{SO}(3)$-equivariant feature space carries the same number of
channels in every angular-momentum sector, for instance
$n\!\times\!\ell{=}0$ together with $n\!\times\!\ell{=}1$, so
$m_0=m_1=n$. A free band-limited module of rank $m$ over
$S=\{\ell\le1\}$ has $m_0=m$ and $m_1=3m$. The two agree only for $n=0$, so such
a feature space is not a free $\starG$ module and its equivariant layers are not
$\starG$ operations in the sense of the main text. They are elements of
$\mathrm{End}_G(V)$, and Theorem~\ref{thm:ey_module} is the optimality statement
that applies to them.
\end{remark}

\begin{remark}[Permutations]
For the defining representation of the symmetric group $S_n$ on $\mathbb{C}^n$,
which decomposes as the trivial representation plus the standard representation,
both multiplicities are $1$, so $\mathrm{End}_{S_n}(\mathbb{C}^n)\cong
\mathbb{C}\oplus\mathbb{C}$ is spanned by the identity and the all-ones matrix.
This is the two-parameter permutation-equivariant layer used in practice. The
free module over $S_n$ has dimension $n!$ in its group coordinate, so the
general-module formulation is what brings permutation symmetry within
computational reach, and Theorem~\ref{thm:ey_module} supplies its optimal
low-rank approximation at no extra cost.
\end{remark}

\section{Extension to Space Groups: the Band-Limited \texorpdfstring{$\starSig$}{starSigma} Algebra}
\label{sec:spacegroups}

The finite and band-limited $\starG$ algebra extends to the full symmetry of a
crystal, its space group, after one classical structural step, the Mackey
little-group construction, which fibers the algebra over the Brillouin zone. The
result is a reduction rather than new analysis: on a symmetric $k$-grid the
space-group algebra \emph{is} a finite $\starG$ algebra, so every finite-group
guarantee of this paper transfers unchanged. We treat the unitary space group; the
representation theory below (Mackey induction, the Bloch crossed-product picture,
projective little-group representations) is
classical~\cite{mackey1958unitary,bradley1972,serre1977linear} and has been
developed concurrently for the full space-group setting, including magnetic groups,
by Bekka and Brouder~\cite{bekka2025crystals}; the contribution here is the
reduction and the inherited optimality and machine verification, not the
representation theory. A neural architecture for all space groups, using
symmetry-adapted Fourier bases, was proposed concurrently~\cite{zhang2025singlearchitecture};
it enforces symmetry as an architectural \emph{invariance}, whereas the reduction
here yields an optimality-backed algebra whose selection rules are
\emph{enforceable} on any predictor's output (Corollary~\ref{cor:selection}) and
whose low-rank truncation is provably optimal.

Let $\Sigma$ be a $d$-dimensional space group, an extension
$1\to T\to\Sigma\xrightarrow{\pi}P\to1$ with $T\cong\bbZ^d$ the lattice
translations (abelian, normal, of finite index $|P|$) and $P$ the crystallographic
point group. We develop the symmorphic (split) case $\Sigma=T\rtimes P$; the
non-symmorphic case is Remark~\ref{rem:nonsymm}. Because $T$ is discrete, Pontryagin
duality makes its dual $\Omega:=\hat T$ a \emph{compact} torus, the Brillouin zone;
this compactness is what brings $\Sigma$ within reach of the finite theory. The
point group acts on $\Omega$ by $(p\cdot k)(t)=k(p^{-1}t)$; write
$P_k=\{p\in P:p\cdot k=k\}$ for the little co-group and $\Sigma_k=T\rtimes P_k$ for
the little group.

\begin{theorem}[Mackey little-group construction]\label{thm:mackey}
Every irreducible unitary representation of the symmorphic $\Sigma=T\rtimes P$ is
equivalent to exactly one
$\pi_{[k],\sigma}=\operatorname{Ind}_{\Sigma_k}^{\Sigma}(k\otimes\sigma)$, of
dimension $[P:P_k]\,d_\sigma$, where $[k]$ ranges over $P$-orbits in $\Omega$ and
$\sigma\in\widehat{P_k}$; distinct pairs give inequivalent irreps and these exhaust
$\hat\Sigma$.
\end{theorem}

\begin{proof}
Mackey's theorem for a regular semidirect product with abelian normal
subgroup~\cite{mackey1958unitary}, with the crystallographic specialization
in~\cite{bradley1972} and the finite instance in~\cite[\S8.2]{serre1977linear};
the action is regular because $P$ is finite, so every orbit is finite (closed) and
$\Omega/P$ is standard Borel. As a finite extension of an abelian group, $\Sigma$
is type~I~\cite{thoma1968}, so the direct-integral theory below applies.
\end{proof}

Decompose $a\in\bbC[\Sigma]$ as $a=\sum_{p\in P}a_p\otimes p$ with $a_p\in\bbC[T]$,
and let $\hat a_p(k)=\sum_{t\in T}a_p(t)e^{-ik\cdot t}$ be the lattice (Bloch)
transform.

\begin{theorem}[Bloch = crossed product]\label{thm:bloch}
The Bloch transform is a $*$-isomorphism (at the $C^*$-completion)
$\bbC[\Sigma]\cong C(\Omega)\rtimes P$ under which the $\starSig$ product is the
twisted convolution over $P$,
\begin{equation}\label{eq:twistedconv}
\widehat{(a\starSig b)}_r(k)=\sum_{p\in P}\hat a_p(k)\,
\hat b_{p^{-1}r}\!\big(p^{-1}\!\cdot k\big).
\end{equation}
Over a fundamental domain of $\Omega$ the $\starSig$ product is fiberwise
block-diagonal matrix multiplication, the fiber at $k$ carrying the little-group
irrep blocks $\bigoplus_{\sigma\in\widehat{P_k}}$.
\end{theorem}

\begin{proof}
Equation~\eqref{eq:twistedconv} is the semidirect-product multiplication
$(a_p\otimes p)(b_q\otimes q)=a_p\,(p\cdot b_q)\otimes pq$ transported through the
Bloch transform, using the character shift $\widehat{p\cdot c}(k)=\hat c(p^{-1}\!\cdot
k)$; this is the multiplication of the crossed product $C(\Omega)\rtimes P$, and
$\bbC[T\rtimes P]\cong C(\hat T)\rtimes P$ is the standard identification for an
abelian normal subgroup. Fiberwise diagonalization over each orbit is Mackey
imprimitivity, reproducing Theorem~\ref{thm:mackey}.
\end{proof}

Passing to a finite crystal, Born--von~Karman periodic boundary conditions on a
uniform $N\times\cdots\times N$ supercell replace $T$ by the finite
$T_N=(\bbZ/N)^d$ and $\Omega$ by the $P$-invariant uniform grid
$\Omega_N=\widehat{T_N}$ ($P$-invariant because $P\subset\mathrm{GL}(d,\bbZ)$
preserves $N\bbZ^d$ for every uniform $N$), yielding the \emph{finite} group
$\Sigma_N:=T_N\rtimes P$ of order $N^d|P|$.

\begin{theorem}[Reduction to finite $\starG$]\label{thm:sigmareduction}
Born--von~Karman boundary conditions replace the space group $\Sigma$ by the genuine
finite group $\Sigma_N=T_N\rtimes P$, and the induced product on $N$-periodic
signals (band-limited to $\Omega_N$) is the finite twisted convolution
$\star_{\Sigma_N}$; equivalently, in the Bloch representation the product
\eqref{eq:twistedconv} restricted to the $P$-invariant grid $\Omega_N$ is closed and
its algebra equals $C(\Omega_N)\rtimes P=\bbC[\Sigma_N]$. Because $\Omega_N$ is a
Plancherel-measure-zero finite subset of the continuous $\Omega$, this is a group
replacement, not a subalgebra of the infinite algebra. Since $\Sigma_N$ is an
ordinary finite group, the $\starG$-SVD (Theorem~\ref{thm:svd_exist}),
Eckart--Young optimality (Theorem~\ref{thm:ey_full}), canonicity
(Theorem~\ref{thm:necessity}), and the Lean-verified reductions apply verbatim to
$\bbC[\Sigma_N]$, with each irrep read as a little-group irrep $\pi_{[k],\sigma}$
($k\in\Omega_N$) of size $[P:P_k]\,d_\sigma$ and Plancherel weight
$[P:P_k]\,d_\sigma/(N^d|P|)$.
\end{theorem}

\begin{proof}
$\Omega_N=\widehat{T_N}$ and $P$-invariance make the twisted convolution
\eqref{eq:twistedconv} preserve the grid (vanishing of $\hat a_p$ off $\Omega_N$
forces the product off $\Omega_N$ to vanish), so the restricted algebra is
$C(\Omega_N)\rtimes P=\bbC[\Sigma_N]$; the block statement is
Theorem~\ref{thm:mackey} for the finite group $\Sigma_N$. Note $\Sigma_N$ is a
semidirect product, described by Theorem~\ref{thm:bloch}, not a direct product, so
the product-group ring isomorphism (Theorem~\ref{thm:prod_full}) is available for
composing $\Sigma$ with an internal symmetry but does not itself describe
$\Sigma_N$.
\end{proof}

\begin{corollary}[Selection rules across the Brillouin zone]\label{cor:selection}
A crystal tensor field transforming in a $P$-representation $\tau$ decomposes at
each $k$ into the little-group channels $\widehat{P_k}$; the allowed $\starSig$
channels are the irreps in $\tau{\downarrow}P_k$ (twisted by $k$), branching by the
subduction (compatibility) relations as $k$ moves to a lower-symmetry stratum. The
orthogonal projector onto the allowed channels at each $k$,
$\Pi(k)=\sum_{\sigma\ \mathrm{allowed}}\tfrac{d_\sigma}{|P_k|}\sum_{p\in P_k}
\overline{\chi_\sigma(p)}\,\rho(p)$, is a closed-form $\starSig$ idempotent whose
application sends every symmetry-forbidden channel to zero. Its $k=\Gamma$ instance
(little group $P$) is the point-group projector applied to cubic crystals in the
main text.
\end{corollary}

\begin{proof}
The $\Pi(k)$ are sums of mutually orthogonal isotypic
projectors~\cite[\S2.6]{serre1977linear}, hence idempotent; $P$ fixes
$\Gamma=0\in\Omega$, so $P_\Gamma=P$.
\end{proof}

\begin{proposition}[Enforcement is a guaranteed error reduction]\label{prop:enforce_guarantee}
Let $P$ be the orthogonal projector onto the symmetry-allowed subspace $S$ (the
point-group Reynolds average, Corollary~\ref{cor:selection}), and write
$Q=I-P$.

\emph{(i) Exactly symmetric ground truth.} If the true tensor satisfies
$C^\star\in S$, then for any prediction $C$,
\[
\|C-C^\star\|^2 = \|PC-C^\star\|^2 + \|QC\|^2,
\]
so $\|PC-C^\star\|\le\|C-C^\star\|$, with the error reduction equal to the
symmetry-forbidden energy $\|QC\|^2$. Enforcement therefore never increases the
error and removes exactly the forbidden channel the paper drives to machine zero;
the benefit is largest when the predictor violates symmetry most, that is,
in the low-data or extrapolation regime.

\emph{(ii) Approximately symmetric ground truth.} Without any assumption on
$C^\star$,
\begin{equation}\label{eq:enforce_general}
\|PC-C^\star\|^2+\|Q(C-C^\star)\|^2=\|C-C^\star\|^2+\|QC^\star\|^2 ,
\end{equation}
so projection changes the squared error by exactly
$\|QC^\star\|^2-\|Q(C-C^\star)\|^2$, and it does not increase the error precisely
when $\|QC^\star\|\le\|Q(C-C^\star)\|$: the forbidden content of the truth must
not exceed the forbidden content of the residual. Part~(i) is the case
$QC^\star=0$.
\end{proposition}

\begin{proof}
For~(i), $P$ is an orthogonal projection and $C^\star=PC^\star$, so
$PC-C^\star=P(C-C^\star)$ and $(I-P)C=(I-P)(C-C^\star)$ are orthogonal components
of $C-C^\star$; Pythagoras gives the identity. For~(ii), decompose
$PC-C^\star=P(C-C^\star)-QC^\star$, whose two terms are orthogonal, so
$\|PC-C^\star\|^2=\|P(C-C^\star)\|^2+\|QC^\star\|^2$; adding
$\|Q(C-C^\star)\|^2$ to both sides and using
$\|C-C^\star\|^2=\|P(C-C^\star)\|^2+\|Q(C-C^\star)\|^2$
gives~\eqref{eq:enforce_general}. Both parts are machine-checked
(\texttt{StarG.enforce\_le} and \texttt{StarG.enforce\_identity}), with no axiom beyond Lean's core. Empirically, on the de Jong
tensors, which are symmetry-consistent to a residual forbidden fraction $2\times
10^{-4}$, projection reduces the reconstruction error of $98$--$100\%$ of
individual predictions across a training-fraction sweep, and composing with the
nearest-PSD step recovers every otherwise-viable material a raw Born-stability screen
rejects.
\end{proof}

\subsection{Projection of a model at two levels}
\label{sec:si_two_level}

Proposition~\ref{prop:enforce_guarantee} concerns one prediction for one structure. The same orthogonality governs a model as a whole. Let a finite group $G$ act measurably on a standard Borel input space $X$, written $x\mapsto gx$, and orthogonally through $\rho$ on a finite-dimensional real inner-product space $V$ of outputs. Stabilizers and the map $x\mapsto P_{H_x}$ below are then measurable. Let $\mu$ be a probability measure on $X$ that is invariant under $G$, and let $L^2(\mu;V)$ be the Hilbert space of square-integrable maps $f:X\to V$ with $\langle f,h\rangle=\int\langle f(x),h(x)\rangle\,d\mu(x)$. A map is equivariant when $f(gx)=\rho(g)f(x)$ for all $g$ and almost every $x$. For $x\in X$ let $H_x=\{h\in G:hx=x\}$ be its stabilizer and $P_{H_x}=|H_x|^{-1}\sum_{h\in H_x}\rho(h)$. Define
\[
(\mathcal{E}f)(x)=\frac{1}{|G|}\sum_{g\in G}\rho(g)^{-1}f(gx),\qquad (\mathcal{R}f)(x)=P_{H_x}f(x).
\]
The map $\mathcal{E}$ averages the model over the group and costs $|G|$ evaluations per input. The map $\mathcal{R}$ is the selection-rule projection of the main text and costs one evaluation.

\begin{theorem}[Two levels of projection]\label{thm:two_level}
\begin{enumerate}
\item[(i)] $\mathcal{E}$ is the orthogonal projector of $L^2(\mu;V)$ onto the equivariant maps, and $\mathcal{R}$ is an orthogonal projector whose range contains them, so $\mathcal{R}\mathcal{E}=\mathcal{E}\mathcal{R}=\mathcal{E}$.
\item[(ii)] For every $f$ and every equivariant target $f^\star$,
\begin{equation}\label{eq:three_part}
\|f-f^\star\|^2=\|f-\mathcal{R}f\|^2+\|\mathcal{R}f-\mathcal{E}f\|^2+\|\mathcal{E}f-f^\star\|^2 .
\end{equation}
\item[(iii)] For a representative $f$ defined at every input, a target $f^\star$ equivariant at every input, and every $x$, with $z_g=\rho(g)^{-1}f(gx)$ and $\bar z=(\mathcal{E}f)(x)$ the mean of the $z_g$,
\begin{equation}\label{eq:orbit_variance}
\frac{1}{|G|}\sum_{g\in G}\|f(gx)-f^\star(gx)\|^2=\frac{1}{|G|}\sum_{g\in G}\|z_g-\bar z\|^2+\|\bar z-f^\star(x)\|^2 .
\end{equation}
\end{enumerate}
\end{theorem}

\begin{proof}
(i) Reindexing the sum shows $\mathcal{E}^2=\mathcal{E}$, and $\mathcal{E}f=f$ exactly when $f$ is equivariant. For self-adjointness, substitute $y=gx$, which preserves $\mu$, and use the orthogonality of $\rho(g)$:
\[
\langle\mathcal{E}f,h\rangle=\frac{1}{|G|}\sum_g\int\langle f(y),\rho(g)h(g^{-1}y)\rangle\,d\mu(y)=\langle f,\mathcal{E}h\rangle ,
\]
after replacing $g$ by $g^{-1}$. An idempotent self-adjoint operator is an orthogonal projector. The operator $\mathcal{R}$ acts pointwise by the orthogonal projector $P_{H_x}$, so it is idempotent and self-adjoint for every measure. If $f$ is equivariant and $h\in H_x$, then $\rho(h)f(x)=f(hx)=f(x)$, so $P_{H_x}f(x)=f(x)$. Hence the range of $\mathcal{E}$ lies in the range of $\mathcal{R}$, which gives $\mathcal{R}\mathcal{E}=\mathcal{E}$, and taking adjoints gives $\mathcal{E}\mathcal{R}=\mathcal{E}$.

(ii) Write $f-f^\star=(f-\mathcal{R}f)+(\mathcal{R}f-\mathcal{E}f)+(\mathcal{E}f-f^\star)$. The first term lies in the kernel of $\mathcal{R}$, and the other two lie in its range, because $\mathcal{E}f$ and $f^\star$ are equivariant. The second term equals $(I-\mathcal{E})\mathcal{R}f$ by $\mathcal{E}\mathcal{R}=\mathcal{E}$, so it lies in the kernel of $\mathcal{E}$, and the third lies in its range. The three terms are therefore mutually orthogonal.

(iii) Equivariance of $f^\star$ gives $z_g-f^\star(x)=\rho(g)^{-1}\big(f(gx)-f^\star(gx)\big)$, which has the norm of $f(gx)-f^\star(gx)$. The identity is the decomposition of the mean squared distance of the points $z_g$ from $f^\star(x)$ into their variance about their mean $\bar z$ and the squared distance of the mean.
\end{proof}

The three terms of~\eqref{eq:three_part} separate what each intervention can repair. The first is the selection-rule error, which one projection per output removes. The second is the inconsistency of the model across the orbit of each input that remains after $\mathcal{R}$, which averaging over the group removes. The third is the error of the averaged model, which both projections leave unchanged. At an input with trivial stabilizer the first term vanishes and the second carries the whole violation. Identity~\eqref{eq:orbit_variance} holds at each input and needs no assumption on the distribution of inputs. Its first term, the variance of the back-rotated predictions over the orbit, is computed from the model alone and bounds from below the mean squared error of the model over the orbit, as the forbidden energy of Proposition~\ref{prop:enforce_guarantee} bounds the error at one structure. For a compact group acting continuously, the sums become integrals against the normalized Haar measures of the group and of its closed stabilizers, and the same argument applies, with identity~\eqref{eq:orbit_variance} holding for almost every input unless every orbit restriction is square-integrable.

\begin{corollary}[Projection and truncation are jointly optimal]\label{cor:joint_opt}
Let $V$ be a representation space whose irreducible constituents are of real type, and let $\Sigma$ and $\Gamma$ be positive definite operators on $V$ that commute with the group action. Measure an operator $T$ on $V$ by $\|T\|_{\Gamma,\Sigma}=\|\Gamma^{1/2}T\Sigma^{1/2}\|_F$, and let $\mathcal{P}_{\mathrm{eq}}T=|G|^{-1}\sum_g\rho(g)^{-1}T\rho(g)$, with blocks $\mathrm{id}_{V_\rho}\otimes A_\rho$ in an orthonormal real isotypic decomposition. Write $B_\rho$, $\Sigma_\rho$ and $\Gamma_\rho$ for the multiplicity-space blocks, with $d_\rho=\dim V_\rho$. Then for every profile of nonnegative integers $(k_\rho)$,
\[
\min_{\substack{B\ \mathrm{equivariant}\\ \operatorname{rank}B_\rho\le k_\rho}}\|T-B\|_{\Gamma,\Sigma}^2=\|T-\mathcal{P}_{\mathrm{eq}}T\|_{\Gamma,\Sigma}^2+\sum_\rho d_\rho\sum_{j>k_\rho}\sigma_j\big(\Gamma_\rho^{1/2}A_\rho\Sigma_\rho^{1/2}\big)^2 ,
\]
and a minimizer is $B_\rho=\Gamma_\rho^{-1/2}\big[\Gamma_\rho^{1/2}A_\rho\Sigma_\rho^{1/2}\big]_{k_\rho}\Sigma_\rho^{-1/2}$, the truncation of Corollary~\ref{cor:ey_module_weighted} applied to $\mathcal{P}_{\mathrm{eq}}T$, where $[\cdot]_k$ is a best approximation of rank at most $k$.
\end{corollary}

\begin{proof}
Conjugation by $\rho(g)$ preserves $\|\cdot\|_{\Gamma,\Sigma}$, because $\rho(g)$ is orthogonal and commutes with $\Gamma^{1/2}$ and $\Sigma^{1/2}$. The average $\mathcal{P}_{\mathrm{eq}}$ of these isometries is idempotent and self-adjoint in the inner product of $\|\cdot\|_{\Gamma,\Sigma}$, so it is the orthogonal projector onto the equivariant operators. For equivariant $B$, Pythagoras gives $\|T-B\|^2=\|T-\mathcal{P}_{\mathrm{eq}}T\|^2+\|\mathcal{P}_{\mathrm{eq}}T-B\|^2$, and Corollary~\ref{cor:ey_module_weighted} minimizes the second term.
\end{proof}

\noindent Theorem~\ref{thm:two_level} and Corollary~\ref{cor:joint_opt} apply classical Hilbert-space symmetrization, Pythagoras's theorem for nested orthogonal projectors and the mean--variance identity. Their content here is the interpretation, which assigns each part of the error to an intervention. The corollary is the operator form of the theorem. Projection followed by weighted blockwise truncation attains the joint optimum over equivariant operators of bounded block rank. The empirical weights $\Sigma$ and $\Gamma$ of the compression experiments are accumulated over the components of each irreducible representation, which makes them commute with the group action. When one of them is singular, as for the energy-gradient moments of MACE-MP-0, we raise its eigenvalues to at least $10^{-8}$ times the largest, and to $10^{-38}$ if it vanishes, so the computed truncation is optimal for the regularized weights.

\begin{corollary}[Direct-integral Eckart--Young]\label{cor:directintegral}
For a Hilbert--Schmidt tensor field over the infinite crystal there is a measurable
field of fiberwise SVDs, and for any measurable rank \emph{profile} $r(\cdot)$ the
fiberwise truncation minimizes the Plancherel-weighted Frobenius error over fields
of that profile, the error being the Brillouin-zone integral of the fiberwise
singular-value tails.
\end{corollary}

\begin{proof}
Fiberwise this is the matrix Eckart--Young theorem; measurability of a jointly chosen
SVD field and integrability of the tail are classical measurable direct-integral
facts~\cite{folland2016course,franco2024measurability}, cited not reproved. The only
$\starSig$-specific step is the fiberization, the $N\to\infty$ limit of
Theorem~\ref{thm:sigmareduction}.
\end{proof}

\begin{remark}[Rank profile versus global rank]
Corollary~\ref{cor:directintegral} is stated for a rank \emph{profile}. Converting a
single global $\starG$-rank budget into a per-fiber profile is ill-posed when the
singular-value field has level sets of positive Plancherel measure, so the clean
global-rank optimality of Theorem~\ref{thm:ey_full} is asserted only in the
band-limited Theorem~\ref{thm:sigmareduction}; the continuous-zone statement is
per-profile.
\end{remark}

\begin{remark}[Values, not topology]
These statements concern singular \emph{values} and irrep content, gauge-invariant at
each $k$. Any band-topological content (Chern number, Wannier localizability) resides
in the winding of the singular \emph{vectors} across $\Omega$, to which the value
spectrum is blind: two fields with identical singular-value fields can carry different
Chern number. No topological diagnostic is claimed. Concretely, a $\starSig$
certificate can guarantee the irrep content and the optimal rank of a retained
subspace but cannot certify that the subspace admits an exponentially localized
symmetric Wannier basis: that existence question is precisely the vanishing of the
topological obstruction carried by the vector field, so any Wannier-localizability
guarantee would require singular-vector (bundle) data beyond the value spectrum, and
predictions passed through the enforcement layer inherit selection-rule exactness but
no topological guarantee.
\end{remark}

\begin{remark}[Non-symmorphic groups]\label{rem:nonsymm}
When $\Sigma$ does not split, the little-group blocks become projective
representations of $P_k$ with the $k$-dependent factor system
$\omega_k(R,R')=\exp[-i(R^{-1}k-k)\cdot\tau_{R'}]$ set by the fractional translations
$\tau_R$~\cite{bradley1972}. The twisted group algebra $\bbC_{\omega_k}[P_k]$ is still
finite-dimensional and semisimple over $\bbC$ (twisted Maschke), so
Theorems~\ref{thm:bloch}--\ref{thm:sigmareduction} and
Corollary~\ref{cor:selection} hold with ``irrep'' read as ``projective irrep'': after
Born--von~Karman, $\Sigma_N$ is a genuine finite group whether or not it splits, so
the finite-group theorems and their Lean formalization apply to $\bbC[\Sigma_N]$
directly, the twisted algebra merely naming the little-co-group blocks. The cocycle is
trivial at interior $k$ and nontrivial only on parts of the zone boundary, the
algebraic origin of non-symmorphic band stickings; the same twisted-algebra mechanism,
with a different cocycle, covers spin double groups (spin--orbit coupling without time
reversal). The twist carries no computational penalty. The twisted algebra
$\bbC_{\omega_k}[P_k]$ has the same dimension $|P_k|\le|P|\le 48$ as the untwisted
one, and its projective irreps still satisfy $\sum_\sigma d_\sigma^2=|P_k|$; a
nontrivial cocycle only reallocates that fixed total into fewer, larger blocks (the
band stickings), never into a larger algebra. The $\starSig$-SVD therefore factors,
exactly as in the symmorphic case, into $N^d$ independent little-group fibers whose
block sizes are bounded by the point-group order, so runtime is linear in the
Born--von~Karman volume $N^d$ (embarrassingly parallel over $k$), memory is the
Plancherel-weighted block storage, which equals the size of the signal itself, and no
object of dimension $|\Sigma_N|$ is ever formed.
\end{remark}

\begin{remark}[Antiunitary time reversal, a flagged extension]
Time-reversal symmetry, whether in magnetic (Shubnikov) groups or in ordinary
``grey'' time-reversal-symmetric crystals, adjoins an \emph{antiunitary} operator, so
representations become Wigner corepresentations and the fiber algebras become matrix
algebras over $\bbR$, $\bbC$, or $\bbH$ (the Dyson threefold way; Herring's rules and
Kramers degeneracy). A Frobenius-norm SVD and its Eckart--Young optimum still exist
fiberwise, but the reduction to the complex-linear finite $\starG$ machinery of
Theorem~\ref{thm:sigmareduction}, and its Lean formalization, are not verbatim in the
real or quaternionic case; we record this as a flagged extension.
\end{remark}

\section{Beyond Discrete Translations: the Wigner--Mackey Compact-Fiber Reduction}
\label{sec:wignermackey}

The space-group reduction of Section~\ref{sec:spacegroups} kept the base compact by
discretizing the translations: with $T\cong\bbZ^d$ discrete, its dual $\Omega=\hat T$
is a compact torus, and Born--von~Karman turns the whole algebra finite. We now
remove that discreteness and treat \emph{continuous} translations $N\cong\bbR^n$, the
setting of the relativistic (Poincar\'e) and Euclidean motion groups. The dual
$\widehat N\cong\bbR^n$ (momentum space) is then non-compact, so there is no finite
grid and no global finite $\starG$ algebra. The point of this section is that the
finite/compact algebra nonetheless reaches the physically relevant content: the
non-compactness lives entirely in the \emph{base} (the orbit of momenta), while the
\emph{fiber} (the little-group content) stays compact and finite-dimensional, and the
non-compact covariance is discharged symbolically rather than numerically. The
machinery is the same Mackey construction as for crystals, with $N\rtimes H$ in place
of $T\rtimes P$.

Let $\Sigma=N\rtimes H$ be a regular semidirect product with $N\cong\bbR^n$ abelian
and $H$ acting by automorphisms, hence on $\widehat N$ by
$(h\cdot\chi)(x)=\chi(h^{-1}x)$. For $\chi\in\widehat N$ write $H_\chi=\{h\in
H:h\cdot\chi=\chi\}$ for the little group and $O=H\cdot\chi\cong H/H_\chi$ for its
orbit. Mackey's theorem for a regular semidirect product with abelian normal
subgroup~\cite{mackey1958unitary} gives the irreducibles exactly as in
Theorem~\ref{thm:mackey}, now over a continuous orbit:
$\pi_{[\chi],\sigma}=\operatorname{Ind}_{N\rtimes
H_\chi}^{\Sigma}(\chi\otimes\sigma)$ with $\sigma\in\widehat{H_\chi}$, realized on the
sections $L^2(O,\mu_O)\otimes V_\sigma$ of a homogeneous vector bundle over the orbit
with fiber the little-group irrep $V_\sigma$.

We first record that band-limiting on a compact group is exact on
band-limited signals (the band-limited reduction of the main text) and quantitatively
controlled on smooth ones.

\begin{proposition}[Band-limited approximation bound]\label{prop:bandlimit_error}
Let $K$ be a compact group with Casimir eigenvalue $\lambda_\rho\ge0$ on
$\rho\in\widehat K$, and let $f\in L^2(K)$ have finite Sobolev norm
$\|f\|_{H^s}^2=\sum_{\rho}(1+\lambda_\rho)^s\,d_\rho\,\|\hat f(\rho)\|_{\mathrm{HS}}^2$
for some $s>0$. Let $f_L$ be the truncation of $f$ to $\{\rho:\lambda_\rho\le L\}$.
Then
\[
\|f-f_L\|_{L^2}^2\ \le\ (1+L)^{-s}\,\|f\|_{H^s}^2 ,
\]
so the finite $\star_K$ product and its $\star_K$-SVD approximate the continuous
ones with error $O(L^{-s/2})$ (geometric in $L$ when $f$ is analytic). Consequently
the exact band-limited statement extends to a controlled approximation for smooth
signals, with the fiber truncation below incurring exactly this error.
\end{proposition}

\begin{proof}
$\|f-f_L\|^2=\sum_{\lambda_\rho>L}d_\rho\|\hat f(\rho)\|^2\le(1+L)^{-s}
\sum_{\lambda_\rho>L}(1+\lambda_\rho)^s d_\rho\|\hat f(\rho)\|^2\le(1+L)^{-s}
\|f\|_{H^s}^2$ by Parseval. Submultiplicativity of the $\starG$ product and Weyl's
singular-value perturbation inequality propagate the bound to the product and its
SVD.
\end{proof}

\begin{remark}[Truncation error is approximation error, not symmetry error]
\label{rem:truncation_exactness}
Band-limiting does not move the algebra to the approximate side of the
exact/approximate dichotomy of the main text's compounding section. The
band-limited subspace is invariant under the group action and closed under the
$\star_K$ product (the band-limited reduction), and the truncation
$f \mapsto f_L$ is itself an equivariant projection; a band-limited operator
therefore commutes with $K$ exactly, its iterates are exactly equivariant at
every depth, and no energy ever enters a symmetry-forbidden channel. What
Proposition~\ref{prop:bandlimit_error} bounds is the \emph{resolution} error
against an un-band-limited target; under iteration that error accumulates the
way any discretization error does (one application of the proof's
submultiplicativity argument per step), but it accumulates entirely within the
symmetry-allowed channels. The two axes are orthogonal: $L$ buys resolution,
never equivariance. In practice $L$ is selected from the data's spectral
tail, which is observable: the empirical Fourier energy above $L$ estimates
$\|f-f_L\|^2$ directly, so one takes the smallest $L$ whose tail energy is
below the target tolerance, against a cost that grows polynomially, as
$\sum_{\lambda_\rho\le L} d_\rho^2$ block entries (for $\mathrm{SO}(3)$,
$\sum_{l\le L}(2l+1)^2 = O(L^3)$), while the tail falls as $O(L^{-s})$ for
$H^s$ signals and geometrically for analytic ones. The lattice-gauge
demonstration below makes the trade concrete: the ground-state energy
converges geometrically in the band limit, a wider band being needed only at
stronger coupling.
\end{remark}

We instantiate this band-limited compact reduction on $\mathrm{SU}(2)$, the gauge
group of a lattice-gauge toy, entirely in the finite block domain. Band-limiting to $j\le2$,
including the half-integer blocks that carry the gauge charge, the electric-energy
operator is exactly the $\star_{\mathrm{SU}(2)}$ Casimir (block-diagonal, eigenvalue
$j(j+1)$ to $9\times10^{-16}$); the $\star_{\mathrm{SU}(2)}$-SVD attains the per-block
Eckart--Young error, the discarded singular-element mass, to $4\times10^{-16}$; and the
gauge-singlet isotypic projector annihilates every non-singlet block exactly. For the
single-plaquette Kogut--Susskind Hamiltonian in its class-function sector, the
electric Casimir plus the magnetic Wilson-loop term (multiplication by the
fundamental character, which selects $\Delta j=\pm\tfrac12$ exactly), the
ground-state energy converges geometrically in the band limit: successive corrections
fall by a factor of about $0.04$ at magnetic coupling $\lambda=2$ and about $0.17$ at
$\lambda=6$, the stronger-coupling case needing a wider band, exactly the analytic
regime of Proposition~\ref{prop:bandlimit_error}.

\begin{theorem}[Wigner--Mackey compact-fiber reduction]\label{thm:wignermackey}
Fix an orbit $O\subset\widehat N$ whose little group $H_\chi$ is \emph{compact}, and
band-limit the fiber to a finite set $S\subset\widehat{H_\chi}$. Then:
\begin{enumerate}
\item at each $\chi'\in O$ the $\starSig$ product restricted to the fiber content is
the finite $\star_{H_\chi}$ blockwise matrix multiplication
$\bigoplus_{\sigma\in S}\mathrm{End}(V_\sigma)$, so the $\starG$-SVD
(Theorem~\ref{thm:svd_exist}), Eckart--Young optimality
(Theorem~\ref{thm:ey_full}), and canonicity (Theorem~\ref{thm:necessity}) apply
\emph{verbatim} to the fiber algebra;
\item the fiberwise singular values are constant along the orbit: for
$\chi'=h\cdot\chi$ the fiber operator is conjugated by the unitary $\sigma(h')$
($h'\in H_\chi$ from the Mackey cocycle), so its singular values, and hence the
$\starG$-SVD certificate, are $H$-invariant, a frame-independent scalar field.
\end{enumerate}
The numerical content of the $\starSig$-SVD over a compact-little-group orbit is
therefore the finite $\star_{H_\chi}$-SVD computed once at the base point; the
non-compact directions (the boosts moving $\chi$ along $O$) act by unitary
conjugation and add no numerical content. This is Theorem~\ref{thm:sigmareduction}
with the roles exchanged: a compact fiber over a non-compact base, in place of a
compact base with finite fibers.
\end{theorem}

\begin{proof}
Part~1 is Theorem~\ref{thm:bloch} read fiberwise: the crossed-product multiplication
$C(O)\rtimes H$ is, at each point of the orbit, the little-group convolution algebra
$\bbC[H_\chi]$, whose semisimple block form is
$\bigoplus_{\sigma\in\widehat{H_\chi}}\mathrm{End}(V_\sigma)$ by Artin--Wedderburn;
band-limiting to $S$ keeps finitely many blocks, on which the finite-group theorems
cited are exactly those of the finite $\starG$ algebra. Part~2 is the imprimitivity
conjugation of Theorem~\ref{thm:mackey}: transport along the orbit is implemented by
unitaries, which preserve singular values; it is the continuous-orbit form of the
gauge-invariance-of-values remark of Section~\ref{sec:spacegroups}.
\end{proof}

We validate part~2 numerically on the Poincar\'e massive orbit. A spin fiber,
band-limited to $j\le2$, is transported along $200$ points of the mass shell
$p^2=m^2$ by composed Lorentz boosts whose raw transport has spectral norm up to
$22$ and induces Wigner rotations of up to $0.85$~rad. The induced fiber action
nonetheless lands in the compact little group, the Wigner rotation is orthogonal
to $5\times10^{-14}$ and its spin representation unitary to $4\times10^{-16}$, so
the $\starG$-SVD singular-value field and the Eckart--Young rank-$k$ truncation
error are constant along the orbit to $7\times10^{-16}$. The identical transport
applied to a boost-covariant four-vector observable, moved by the raw non-unitary
boost, changes its singular values by a median factor of $5.8$: the numerical
signature of the residual of Remark~\ref{rem:noncompact_residual}, where
boost-covariant content is not orbit-invariant and needs the direct-integral
$\starSig$. The demonstration thus tests the theorem's
\emph{scope}, separating the compact-fiber-reducible content from the residual,
rather than only the fact that unitaries preserve singular values.

\begin{corollary}[Linearization discharges covariance symbolically]\label{cor:linearization}
Let $F$ be a $\Sigma$-covariant map ($F(g\cdot x)=g\cdot F(x)$) and $b$ a background
with stabilizer $\Sigma_b\supseteq H_\chi$. The differential $DF_b$ intertwines the
isotropy representation of $\Sigma_b$ on the tangent space, so it (and, for an
invariant energy, its Hessian) is $H_\chi$-equivariant and is block-diagonalized
exactly by $\star_{H_\chi}$; when $H_\chi$ is compact the blocks are
finite-dimensional and the $\star_{H_\chi}$-SVD gives the optimal symmetry-preserving
low-rank model of the fluctuation operator. For a background $b'=g\cdot b$ on the same
orbit, $DF_{b'}=\rho(g)\,DF_b\,\rho(g)^{-1}$, so its spectrum and $\starG$-SVD
certificate are identical. Full $\Sigma$-covariance is thus realized by a symbolic
conjugation computed once: the non-compact (boost) covariance is discharged in the
exact symbolic layer, while the numerical $\starG$-SVD acts only on the compact
little-group fiber.
\end{corollary}

\begin{proof}
$g\cdot b=b$ for $g\in\Sigma_b$ gives $F(g\cdot x)=g\cdot F(x)\Rightarrow DF_b\,\rho(g)
=\rho(g)\,DF_b$, i.e.\ $H_\chi$-equivariance; block-diagonalization and per-block
optimality are Theorems~\ref{thm:svd_exist}--\ref{thm:ey_full}. The orbit relation is
covariance evaluated at $b'=g\cdot b$, and conjugation preserves singular values.
\end{proof}

\begin{remark}[The physical spectrum lies on compact-little-group orbits]
For the Poincar\'e group $\Sigma=\bbR^{1,3}\rtimes\mathrm{SO}(3,1)^{\uparrow}$
(or, to include half-integer spin, its universal cover
$\bbR^{1,3}\rtimes\mathrm{SL}(2,\bbC)$, for which every statement below holds
verbatim with $H_\chi=\mathrm{SU}(2)$ and the double cover of
$\mathrm{SO}(2)$ as the compact little groups) the
orbits in $\widehat N$ are the mass shells. The physical one-particle representations
sit on orbits with \emph{compact} little group: massive ($p^2=m^2>0$,
$H_\chi=\mathrm{SO}(3)$, spin) and massless with finite helicity ($p^2=0$, the
physical representations of $H_\chi=\mathrm{ISO}(2)$ factoring through the compact
$\mathrm{SO}(2)$, helicity). Theorem~\ref{thm:wignermackey} covers exactly these. The
non-compact-little-group orbits, Wigner's continuous-spin representations
(infinite-dimensional $\mathrm{ISO}(2)$) and the tachyonic $p^2<0$
($H_\chi=\mathrm{SO}(2,1)$), have never been observed in nature and fall to
the residual below. The same
organization, finite little-group (spin) labels over a continuous spectral fiber, is
the shape of the conformal bootstrap.
\end{remark}

\begin{remark}[The honest residual, and what is excluded]\label{rem:noncompact_residual}
Two regimes lie beyond the compact-fiber reduction. First, when the observable itself
transforms under the non-compact factor (a boost-covariant quantity with no fixed
frame), or when the whole non-compact block is needed numerically (conformal blocks of
$\mathrm{SO}(d,2)$ in the bootstrap; full-covariance nonlinear analysis with no
background to linearize about), no single little group suffices and one needs the
direct-integral $\starSig$ over the non-compact orbit, with Eckart--Young becoming a
generalized-$s$-number (Hilbert--Schmidt) statement. This is the non-compact-base
analogue of Corollary~\ref{cor:directintegral}, and unlike that corollary it is
new analysis rather than a measurable-selection formality; we flag it as the
principal open extension. Second, infinite-dimensional symmetry groups,
$\mathrm{Diff}(M)$, the BMS group, or a gauge group taken as an
infinite-dimensional Lie group, lie outside the tensor-algebra frame entirely (no
local compactness, no Haar or Plancherel measure, as discussed in the main-text scope
analysis) and are kept symbolic or excluded. The reduction therefore covers a compact little group over a continuous fiber and makes no claim about arbitrary non-compact groups.
\end{remark}

\section{Additional Figures for the Synthetic and Product-Group Experiments}

\begin{figure}[!ht]
\centering
\includegraphics[width=0.7\textwidth]{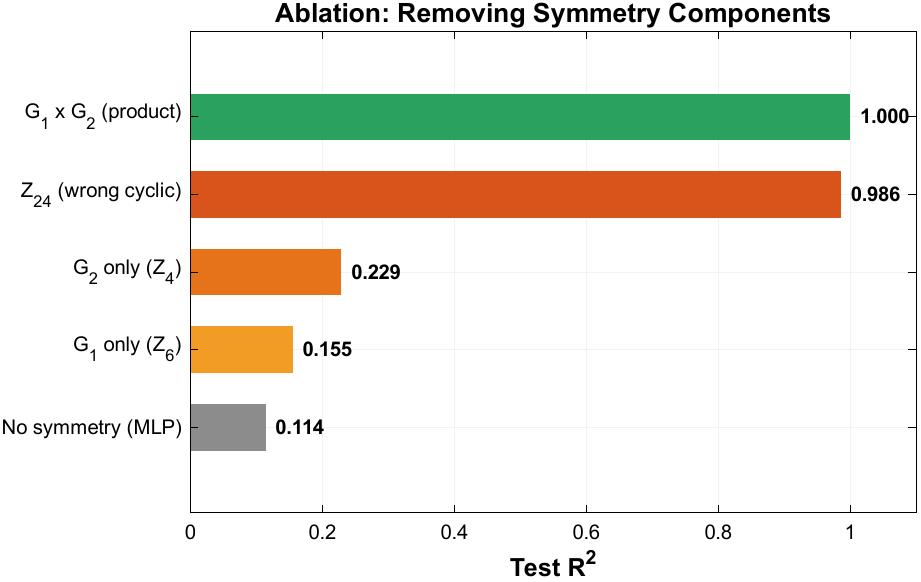}
\caption{\textbf{Ablation of symmetry components.} Removing group structure from the product group experiment causes systematic degradation.}
\label{fig:ext_ablation}
\end{figure}

\begin{figure}[!ht]
\centering
\includegraphics[width=0.85\textwidth]{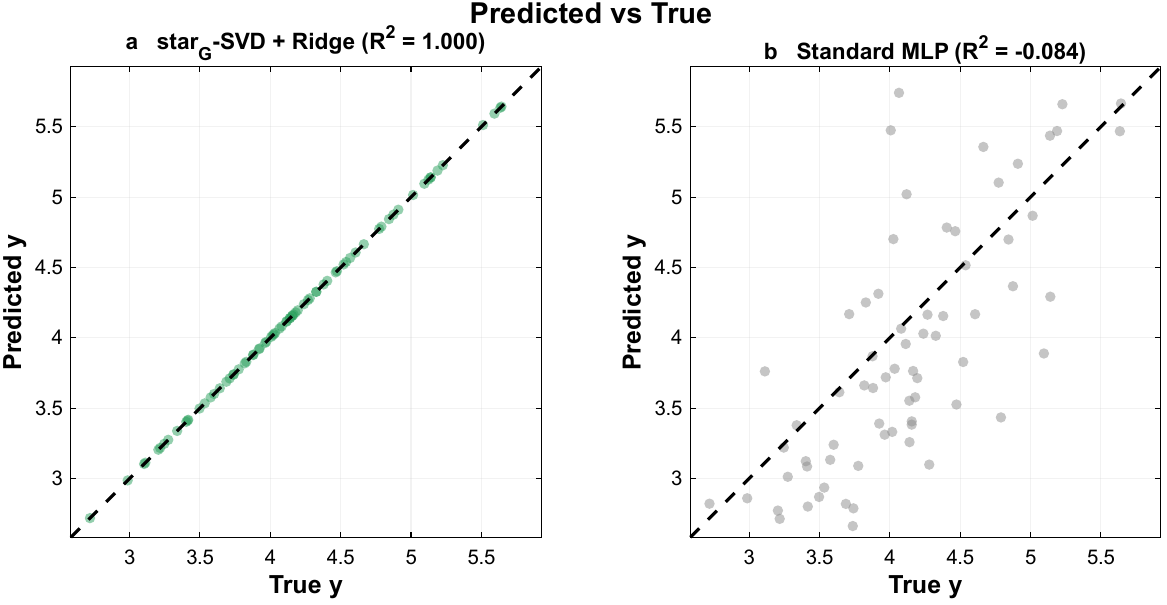}
\caption{\textbf{Predicted versus true values (synthetic).} (a)~$\starG$-SVD: perfect diagonal. (b)~Standard MLP: scattered.}
\label{fig:ext_scatter}
\end{figure}

\begin{table}[!ht]
\centering
\caption{Per-method hyperparameter settings used in all experiments. Hidden
widths $[64, 32]$, ReLU activations, He initialization, and an unregularized
bias per layer are common to all neural baselines. ``Native'' = original
training set of size $n$; ``$|G|$-aug'' = original training set augmented by
applying every $g \in G$ to each input, yielding $|G| \cdot n$ samples.}
\label{tab:hyperparams_full}
\setlength{\tabcolsep}{4pt}
\begin{tabular}{@{}p{2.8cm}p{4.0cm}p{2.6cm}p{1.6cm}p{1.1cm}p{1.4cm}p{1.3cm}@{}}
\toprule
\textbf{Method} & \textbf{Input} & \textbf{Architecture} &
\textbf{Train set} & \textbf{lr} & \textbf{Epochs} & \textbf{Batch} \\
\midrule
$\starG$-SVD + Ridge & $\starG$-features & Linear, ridge $\lambda \in \{10^{-3},\ldots,10^3\}$ (val.) & native & --  & --  & full \\
Standard MLP        & raw $\calX(:, e)$ ($z$-norm)             & $[n_f, 64, 32, 1]$  & native     & 0.003 & 300\textsuperscript{a} & 256\textsuperscript{b} \\
Invariant MLP       & $[\mathrm{mean}, \mathrm{std}, \min, \max]_g$ $\calX$ ($z$-norm) & $[4 n_f, 64, 32, 1]$ & native & 0.003 & 300 & 256 \\
Augmented MLP       & raw $\calX(:, e)$ ($z$-norm on aug.\ set) & $[n_f, 64, 32, 1]$  & $|G|$-aug & 0.003\textsuperscript{c} & 80--300\textsuperscript{d} & 256 \\
Neural $\starG$     & $\starG$-features                        & $\starG$ layers $[\cdot, 64, 32, 1]$ & native & 0.003 & 300 & 32 \\
\bottomrule
\end{tabular}
\medskip

\noindent\footnotesize
\textsuperscript{a}~80 epochs in the synthetic $\bbZ_{12}$ experiment;
300 elsewhere.\quad
\textsuperscript{b}~32 in the synthetic experiment.\quad
\textsuperscript{c}~0.005 in the synthetic experiment.\quad
\textsuperscript{d}~80 in the synthetic experiment, 200 in the
product-group experiment, 300 elsewhere; the smaller epoch budget for
heavily augmented training reflects the $|G|$-fold larger gradient budget per
epoch. Optimizer is Adam ($\beta_1 = 0.9, \beta_2 = 0.999, \varepsilon =
10^{-8}$) with early-stopping patience 20 on validation MSE for every neural
baseline.
\end{table}

\begin{table}[!ht]
\centering
\caption{Computational cost (wall-clock, single seed, 1{,}000 molecules).}
\label{tab:runtimes}
\begin{tabular}{@{}lccc@{}}
\toprule
\textbf{Method} & \textbf{Feature (s)} & \textbf{Train (s)} & \textbf{Total (s)} \\
\midrule
$\starG$-SVD + Ridge & 0.6 & $<$0.1 & 0.7 \\
Neural $\starG$ & 0.6 & 0.5 & 1.1 \\
Standard MLP & -- & 0.4 & 0.4 \\
Invariant MLP & -- & 0.5 & 0.5 \\
Augmented MLP & -- & 4.0 & 4.0 \\
\bottomrule
\end{tabular}
\end{table}

\section{Symmetry Diagnostics: Wigner--Eckart Signatures from Data}
\label{sec:wigner_eckart}

A distinctive capability of the $\starG$ algebra is symmetry diagnosis. Its
irreducible-representation decomposition can reveal, directly from data, which
symmetry channels carry the predictive signal of a physical observable. The
Wigner--Eckart theorem provides the underlying representation-theoretic
framework: matrix elements of an irreducible tensor operator factorize into a
geometric component, determined by the relevant Clebsch--Gordan coefficients,
and a reduced matrix element that is independent of the magnetic quantum
numbers. The resulting selection rules constrain which angular-momentum
channels can couple to a given observable. A generalized Wigner--Eckart
theorem has previously been incorporated into equivariant learning to
parameterize steerable convolution kernels~\cite{lang2021wigner}. Here, we
use the same symmetry structure in the opposite direction, as a diagnostic
tool for identifying symmetry-resolved predictive signal directly from data.

We consider the chiral octahedral group $O$, of order $24$, which is a finite
subgroup of $\mathrm{SO}(3)$. Its five irreducible representations arise in
the restriction of the $\mathrm{SO}(3)$ angular-momentum representations.
The relevant correspondence is

\begin{center}
\begin{tabular}{@{}llcl@{}}
\toprule
\textbf{Irrep} & \textbf{Dim} & \textbf{Lowest $l$} & \textbf{Physical content} \\
\midrule
A$_1$ & 1 & $l=0$ & Scalar (spherically symmetric) \\
A$_2$ & 1 & $l=3$ & Sign representation ($xyz$-type) \\
E     & 2 & $l=2$ & Quadrupole ($d$-wave) \\
T$_1$ & 3 & $l=1$ & Vector ($p$-wave, dipolar) \\
T$_2$ & 3 & $l=2$ & Quadrupole ($d$-wave) \\
\bottomrule
\end{tabular}
\end{center}

\noindent Here, ``lowest $l$'' denotes the smallest angular-momentum sector in
which the corresponding octahedral irrep occurs; an irrep can recur at higher
values of $l$. In particular, A$_1$ is the unique $l=0$ representation, while
A$_2$ first occurs at $l=3$ and should not be identified with the scalar
channel.

\begin{table}[!t]
\centering
\caption{\textbf{Per-irrep \Rsq{} on QM9.}
Test \Rsq{} for individual irreducible representations of the chiral
octahedral group $O$ (130{,}831 molecules; 3 seeds). Scalar properties are
A$_1$-dominated, whereas $|\boldsymbol{\mu}|$ is T$_1$-dominated. The
A$_2$ channel is negligible. Label shuffling and random-basis controls remove
the observed irrep selectivity.
}
\label{tab:wigner}
\begin{tabular}{@{}lcccccr@{}}
\toprule
\textbf{Property} & \textbf{A$_1$} & \textbf{A$_2$} & \textbf{E} &
\textbf{T$_1$} & \textbf{T$_2$} & \textbf{all} \\
\midrule
HOMO--LUMO gap                   & \textbf{0.319} & 0.000 & 0.009 & 0.058 & 0.002 & 0.344 \\
$\alpha$ (isotropic)             & \textbf{0.614} & 0.000 & 0.052 & 0.005 & 0.066 & 0.632 \\
ZPVE                             & \textbf{0.854} & 0.000 & 0.000 & 0.031 & 0.032 & 0.856 \\
$|\boldsymbol{\mu}|$ (magnitude) & 0.130 & 0.000 & 0.007 & \textbf{0.278} & 0.000 & 0.359 \\
\midrule
\;\;gap, shuffled labels         & 0.000 & 0.000 & 0.000 & 0.000 & 0.000 & 0.000 \\
\;\;gap, random basis            & 0.343 & 0.340 & 0.345 & 0.346 & 0.346 & 0.348 \\
\bottomrule
\end{tabular}
\end{table}

We evaluate two datasets. On the full QM9 dataset (130{,}831 molecules, three
random seeds), each molecule is represented under all 24 octahedral rotations.
We compute rotation-invariant Fourier power for each irrep and evaluate the
test $\Rsq{}$ obtained using each irrep channel independently for the scalar
properties HOMO--LUMO gap, isotropic polarizability $\alpha$, and ZPVE, as
well as for the dipole magnitude $|\boldsymbol{\mu}|$. To test the rank-2
selection rule directly, rather than through a charge-derived proxy, we
additionally use QM7-X~\cite{hoja2021qm7x}, which provides the full molecular
polarizability tensor as a real symmetric $3\times3$ matrix.

Table~\ref{tab:wigner}, Figure~\ref{fig:wigner}, and
Table~\ref{tab:qm7x_polarizability} summarize five observations.

\textbf{(i) Scalar observables are concentrated in A$_1$.}
The HOMO--LUMO gap, ZPVE, and isotropic polarizability are dominated by the
A$_1$ ($l=0$) channel, with $\Rsq{}=0.32$, $0.85$, and $0.61$, respectively.
The A$_2$ channel contributes no measurable predictive signal
($|\Rsq{}|<0.001$). This is consistent with the transformation properties of
these observables: true scalars transform according to the trivial
representation A$_1$, whereas A$_2$ is the sign representation of $O\cong S_4$, which first
occurs at $l=3$ and changes sign under the fourfold rotations.

\textbf{(ii) Directional information appears in T$_1$.}
Although the dipole magnitude $|\boldsymbol{\mu}|$ is itself rotationally
invariant, its predictive signal is dominated by the T$_1$ ($l=1$) channel,
with $R^2_{T_1}=0.28$ compared with $R^2_{A_1}=0.13$. This reflects the
fact that the magnitude is derived from an underlying vector quantity, whose
directional information resides in the $l=1$ sector.

\textbf{(iii) Channel selectivity depends on the physical group structure.}
Randomly permuting molecular labels reduces the predictive power of all
channels to approximately zero. Replacing $F_G$ with a random orthonormal
basis having the same block dimensions distributes the signal across the
blocks; for example, each block achieves approximately the full-model
$\Rsq{}\approx0.34$ for the HOMO--LUMO gap. Thus, the observed A$_1$/T$_1$
selectivity is attributable to the octahedral group structure rather than to
generic block decomposition or a coordinate convention.

\textbf{(iv) The rank-2 selection rule is recovered exactly.}
For the QM7-X polarizability tensor, the octahedral decomposition contains
A$_1$ (isotropic), E, and T$_2$ components with relative contributions
$0.93$, $0.03$, and $0.04$, respectively, while the T$_1$ component is
identically zero to machine precision (maximum over 3{,}000 molecules).
This follows from the representation-theoretic decomposition of a symmetric
rank-2 tensor: under $\mathrm{SO}(3)$ it decomposes into $l=0\oplus l=2$,
which restricts under $O$ to A$_1\oplus E\oplus T_2$. Consequently, no T$_1$
($l=1$) component can occur. The $\starG$ decomposition makes this
selection rule explicit at the level of the data representation.

\textbf{(v) Irrep channels provide an explicit feature-selection axis.}
Because the symmetry decomposition is available prior to training, individual
irrep channels can be selected or excluded by construction. On the QM7-X
polarizability components, a two-parameter ridge model using only the E channel
predicts the E-channel magnitude while being structurally unable to predict the
T$_2$ magnitude, and conversely, yielding $96$--$97\%$ cross-selectivity
(Table~\ref{tab:qm7x_polarizability}). This should not be interpreted as a
representational advantage over trained equivariant networks: when given the
full input, $\starG$ with all irreps, symmetry-blind baselines, and three ENN
architectures exhibit near-zero cross-selectivity. The distinction is that
$\starG$ exposes the symmetry channels explicitly at the input level, allowing
them to be restricted without a post-hoc probe or additional training. The
corresponding internal irrep structure of a trained ENN can instead be
recovered using a bespoke post-hoc analysis: a probe of the trained e3nn
baseline finds its scalar block alone reaching $R^2=0.944$ on the gap, with
the $\ell=1,2$ blocks recoverable but weaker (Table~\ref{tab:enn_probe}).

These results demonstrate that the $\starG$ decomposition can serve as a
\emph{symmetry spectroscope}: it identifies which irreducible channels carry
predictive information and exposes selection-rule structure directly in the
data. The $T_1$ absence in the polarizability tensor is not a statistical
discovery but a structural consequence of the symmetric rank-2 representation.
More generally, the observed channel selectivity provides an empirical
signature consistent with the relevant representation-theoretic selection
rules. The value of the decomposition is therefore diagnostic: it provides a
symmetry-resolved representation in which physical constraints become
explicit and testable, without requiring them to be imposed as architectural
constraints. We use the Wigner--Eckart framework as a diagnostic and do not re-derive the theorem.

\begin{figure}[!ht]
\centering
\includegraphics[width=\textwidth]{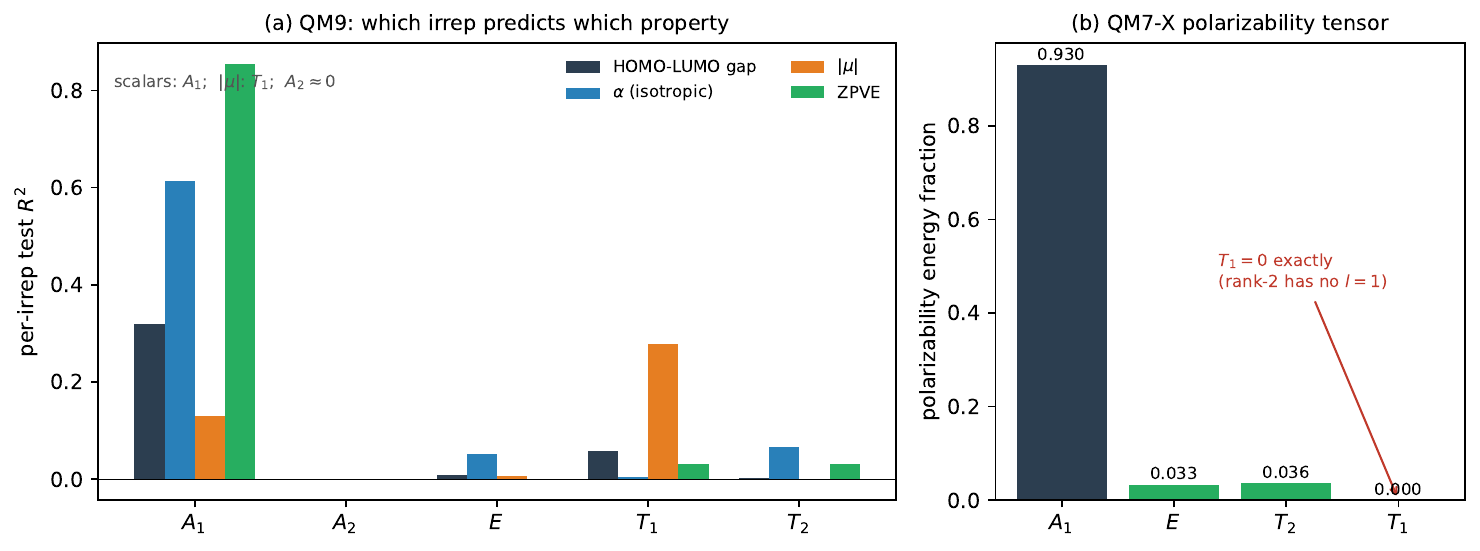}
\caption{\textbf{Symmetry spectroscopy from molecular data.}
(\textbf{a})~Per-irrep test \Rsq{} on QM9 reveals A$_1$ ($l\!=\!0$) dominance
for scalar properties and T$_1$ ($l\!=\!1$) dominance for the dipole magnitude
$|\boldsymbol{\mu}|$. The A$_2$ channel is negligible.
(\textbf{b})~Decomposition of the QM7-X polarizability tensor into octahedral
irreps yields A$_1\oplus$E$\oplus$T$_2$, with an exactly vanishing T$_1$
component across all 3{,}000 molecules, recovering the corresponding
Wigner--Eckart selection rule.}
\label{fig:wigner}
\end{figure}

\begin{figure}[!ht]
\centering
\includegraphics[width=\textwidth]{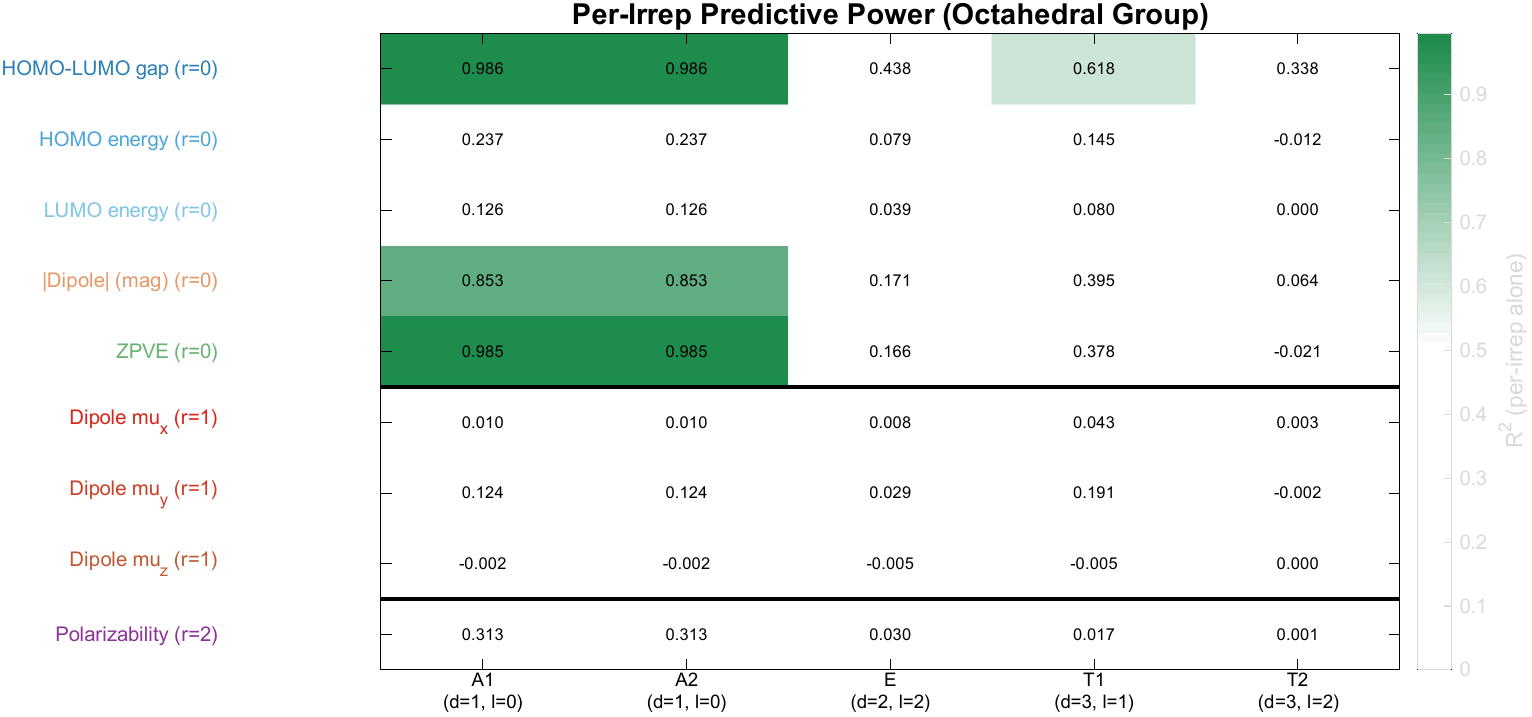}
\caption{\textbf{Irrep decomposition heatmap.} $R^2$ for each (property, irrep) pair, sorted by tensor rank. The block structure separating rank-0 from rank-1 properties is visible as a qualitative change in the A$_1$ and T$_1$ columns across the horizontal separator.}
\end{figure}

The octahedral group $O$ was constructed programmatically from its 24 rotation matrices (6 face, 8 vertex, 6 edge rotations plus identity). The multiplication table was verified to satisfy the group axioms. The five irreducible representations were constructed as: A$_1$ (trivial), A$_2$ (determinant), T$_1$ (the rotation matrices themselves, 3D), and E + T$_2$ (from the rank-2 symmetric traceless tensor representation, decomposed into the 2D and 3D invariant subspaces). All representations were verified to be closed under the group multiplication.

\begin{algorithm}[!ht]
\caption{Per-irrep Fourier decomposition for Wigner--Eckart analysis}
\label{alg:per_irrep}
\begin{algorithmic}[1]
\Require feature batch $\calX \in \bbR^{N \times n_f \times |G|}$, irreps $\hat{G} = \{\rho_1, \ldots, \rho_M\}$ with dimensions $d_\rho$ and matrices $\rho(g)$, target vector $y \in \bbR^N$, ridge grid $\Lambda$
\Ensure per-irrep predictive scores $\{R^2_\rho\}_{\rho \in \hat{G}}$
\For{$\rho \in \hat{G}$}
    \For{$i = 1, \ldots, N$ and $j = 1, \ldots, n_f$ \textbf{in parallel}}
        \State $\hat{X}^\rho_{ij} \gets \sqrt{d_\rho / |G|} \sum_{g} \calX(i, j, g) \, \rho(g) \in \bbC^{d_\rho \times d_\rho}$
        \State $P^\rho_{ij} \gets \|\hat{X}^\rho_{ij}\|_F^2$ \Comment{$G$-invariant power}
    \EndFor
    \State assemble $\Phi^\rho \in \bbR^{N \times n_f}$ from $\{P^\rho_{ij}\}$
    \State split $\Phi^\rho, y$ into train/val/test (70/15/15); standardize $\Phi^\rho$
    \State $\lambda^\star \gets \arg\min_{\lambda \in \Lambda} \mathrm{MSE}_{\mathrm{val}}(\Phi^\rho, y; \lambda)$
    \State $w_\rho \gets (\Phi^{\rho \top}_{\mathrm{train}} \Phi^\rho_{\mathrm{train}} + \lambda^\star I)^{-1} \Phi^{\rho \top}_{\mathrm{train}} y_{\mathrm{train}}$
    \State $R^2_\rho \gets 1 - \mathrm{SS}_{\mathrm{res}}(\Phi^\rho_{\mathrm{test}} w_\rho, y_{\mathrm{test}}) / \mathrm{SS}_{\mathrm{tot}}(y_{\mathrm{test}})$
\EndFor
\State \Return $\{R^2_\rho\}$, $\{T_1 / A_1\text{ ratio}\}$, irrep heatmap data
\end{algorithmic}
\end{algorithm}

\begin{table}[!t]
\centering
\caption{\textbf{QM7-X polarizability components: cross-selectivity.}
For each method, two single-target models are trained to predict the $E_g$
magnitude $|\boldsymbol{\alpha}_E|$ and $T_{2g}$ magnitude
$|\boldsymbol{\alpha}_{T_2}|$ of the molecular polarizability tensor, using
identical 60/20/20 per-molecule splits and three seeds (42, 43, 44).
Cross-selectivity is defined as
$\mathrm{cs}=(R^2_{\mathrm{on}}-R^2_{\mathrm{off}})/
(R^2_{\mathrm{on}}+R^2_{\mathrm{off}}+10^{-12})$, where $R^2_{\mathrm{on}}$
is the performance on the target component and $R^2_{\mathrm{off}}$ is the
performance of the architecturally identical model trained on the other
component. Values near zero indicate comparable performance on the two
components, whereas values near one indicate strong target-specific
selectivity.}
\label{tab:qm7x_polarizability}
\begin{tabular}{@{}lrccc@{}}
\toprule
\textbf{Method} & \textbf{Params} & $R^2(\|\boldsymbol{\alpha}_E\|)$ & $R^2(\|\boldsymbol{\alpha}_{T_2}\|)$ & \textbf{Cross-selectivity} \\
\midrule
$\starG$ E-only + Ridge          & $2$         & $0.591$ & $0.012$ & $\mathbf{96.2\%}$ \\
$\starG$ T$_2$-only + Ridge      & $2$         & $0.008$ & $0.593$ & $\mathbf{97.3\%}$ \\
$\starG$ all irreps + Ridge      & $6$         & $0.622$ & $0.609$ & $1.1\%$ \\
\midrule
Raw features + Ridge             & $37$        & $0.481$ & $0.512$ & $-3.2\%$ \\
Coulomb matrix + Ridge           & $277$       & $0.442$ & $0.446$ & $-0.5\%$ \\
Raw + MLP$(128, 64, 32)$         & $15{,}105$  & $0.586$ & $0.591$ & $-0.5\%$ \\
Coulomb + MLP$(128, 64, 32)$     & $45{,}825$  & $0.495$ & $0.456$ & $4.2\%$ \\
\midrule
SchNet (PyG)                     & $455{,}809$ & $0.664 \!\pm\! 0.026$ & $0.662 \!\pm\! 0.018$ & $0.18\%$ \\
e3nn (SE(3)-equivariant)         & $794{,}032$ & $0.708 \!\pm\! 0.027$ & $0.698 \!\pm\! 0.006$ & $0.7\%$ \\
MACE                             & $1{,}041{,}808$ & $0.687 \!\pm\! 0.032$ & $0.672 \!\pm\! 0.022$ & $1.1\%$ \\
\bottomrule
\end{tabular}
\medskip

\footnotesize
\noindent\emph{The three ENNs achieve comparable predictive performance
($R^2=0.66$--$0.71$) on both polarizability components, matching the
$\starG$ all-irreps + Ridge baseline with six parameters. In contrast,
restricting $\starG$ to a single irrep yields $96.2\%$ and $97.3\%$
cross-selectivity for the $E$-only and $T_2$-only models, respectively,
using only two trainable parameters. With the full input, both $\starG$
and the three ENNs exhibit cross-selectivity near $1\%$. Thus, the observed
selectivity arises from explicit input-side restriction to an irrep, rather
than from a representational advantage over the ENNs. The key distinction is
that $\starG$ exposes the irreducible-representation decomposition as an
explicit, controllable feature axis, enabling such restrictions without
additional training or post-hoc probing.}
\end{table}

\section{Synthetic, Abelian Product-Group, and Permutation Experiments}
\label{sec:si_experiments}

The three experiments in this section use synthetic data with exact
finite-group symmetry, so that each algebraic guarantee (exact invariance,
Kronecker composition of two abelian factors, and permutation equivariance)
is isolated against a known ground truth.

\subsection{Synthetic validation}
\label{sec:synthetic_exp}

\begin{table}[!t]
\centering
\caption{Synthetic validation ($\bbZ_{12}$, 1{,}000 molecules, 3 seeds).}
\label{tab:synthetic}
\begin{tabular}{@{}lccc@{}}
\toprule
\textbf{Method} & \textbf{Test \Rsq} & \textbf{Rot.\ Variance} & \textbf{Params} \\
\midrule
$\starG$-SVD + Ridge & $\mathbf{1.000 \pm 0.000}$ & $5.8 \times 10^{-31}$ & \textbf{101} \\
Augmented MLP       & $0.998 \pm 0.000$           & $3.7 \times 10^{-5}$  & 5{,}249 \\
Neural $\starG$     & $0.697 \pm 0.063$           & $1.0 \times 10^{-30}$ & 8{,}641 \\
Standard MLP        & $0.377 \pm 0.054$           & $1.4 \times 10^{-1}$  & 5{,}249 \\
Invariant MLP       & $0.327 \pm 0.147$           & $\sim 0$              & 14{,}465 \\
\bottomrule
\end{tabular}
\end{table}

\begin{figure}[!t]
\centering
\includegraphics[width=\textwidth]{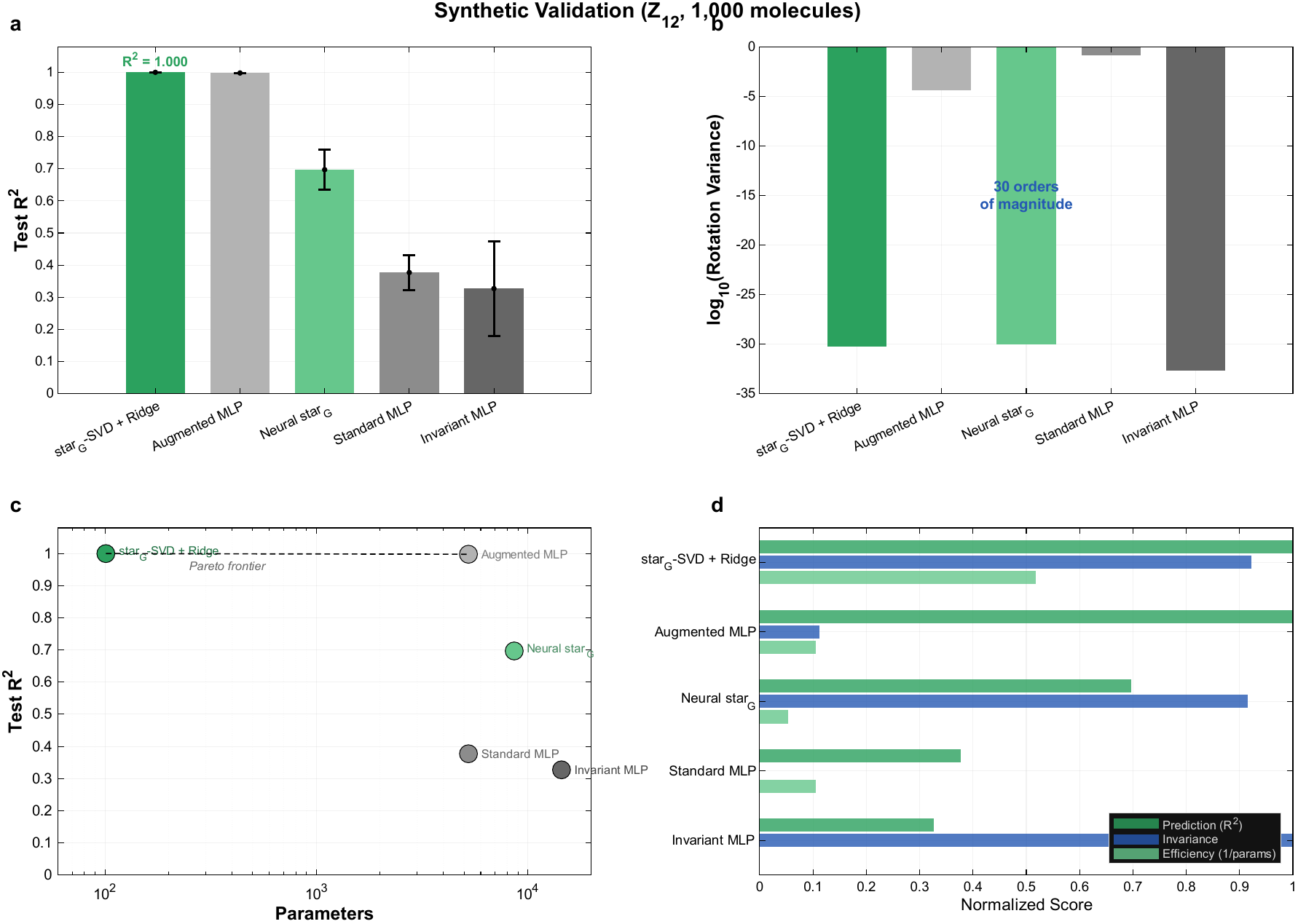}
\caption{\textbf{Synthetic validation ($\bbZ_{12}$).}
(\textbf{a})~Test \Rsq{}.
(\textbf{b})~Rotation variance (log scale): a 26-to-30-orders-of-magnitude gap
between
$\starG$-SVD and the approximately invariant baselines, the static counterpart
of the compounding dichotomy of the main text.
(\textbf{c})~Parameter efficiency (Pareto frontier): on this exactly symmetric
synthetic set, $\starG$-SVD dominates all
baselines on both axes simultaneously.
(\textbf{d})~Multi-metric normalized summary scores.}
\label{fig:synthetic}
\end{figure}

We use a controlled synthetic setting to isolate the exact-versus-approximate
symmetry distinction demonstrated dynamically in the main text.
We generate $1{,}000$ synthetic molecules with exact $\mathbb{Z}_{12}$ rotational
symmetry and compare $\starG$-SVD with ridge regression against four baselines:
Augmented MLP, Neural $\starG$, Standard MLP, and Invariant MLP. We evaluate
predictive accuracy (\Rsq{}), rotational invariance (prediction variance under
unseen rotations), and parameter efficiency. The imposed cyclic symmetry is
verified to machine precision ($4.2\times10^{-16}$). The target $\mathbf{y}$ is
a rotationally invariant scalar combining mean interatomic distance, distance
variance, and atomic-number contributions, analogous to size-dependent molecular
properties such as polarizability.

Results are summarized in Table~\ref{tab:synthetic} and Figure~\ref{fig:synthetic}.
The $\starG$-SVD achieves perfect prediction (\Rsq{} $=1.000\pm0.000$) and exact
invariance (rotation variance $5.8\times10^{-31}$) with only $101$ parameters,
compared with $5{,}249$--$14{,}465$ for the neural baselines. The Standard MLP
achieves \Rsq{} $=0.377$ and rotation variance $0.14$, while the Augmented MLP
achieves near-perfect accuracy (\Rsq{} $=0.998$) but retains residual rotation
variance of $3.7\times10^{-5}$. Thus, augmentation substantially improves
accuracy but provides only approximate invariance, whereas the $\starG$ algebra
enforces invariance by construction. As shown in Figure~\ref{fig:synthetic}b,
the resulting rotation-variance gap between $\starG$-SVD and the approximately
invariant neural baselines spans roughly $26$--$30$ orders of magnitude.

\subsection{Abelian product group $\bbZ_6 \times \bbZ_4$}
\label{sec:si_product_abelian}

\begin{table}[!t]
\centering
\caption{Product group $\bbZ_6 \times \bbZ_4$ (1{,}000 molecules, 3 seeds).}
\label{tab:product}
\begin{tabular}{@{}lcc@{}}
\toprule
\textbf{Method} & \textbf{Test \Rsq} & \textbf{Params} \\
\midrule
$\starG$ over $G_1 \times G_2$ + Ridge & $\mathbf{1.000 \pm 0.000}$ & \textbf{186} \\
$\bbZ_{24}$ cyclic + Ridge             & $0.986 \pm 0.002$           & 157 \\
$\starG$ over $G_1 \times G_2$ + MLP   & $0.826 \pm 0.099$           & 9{,}473 \\
Standard MLP                           & $0.488 \pm 0.116$           & 4{,}481 \\
Invariant MLP                          & $0.324 \pm 0.267$           & 6{,}785 \\
Augmented MLP                          & $0.250 \pm 0.239$           & 4{,}481 \\
$G_2$ only ($\bbZ_4$) + Ridge          & $0.229 \pm 0.191$           & 42 \\
$G_1$ only ($\bbZ_6$) + Ridge          & $0.155 \pm 0.244$           & 55 \\
\bottomrule
\end{tabular}
\end{table}

\begin{figure}[!t]
\centering
\includegraphics[width=\textwidth]{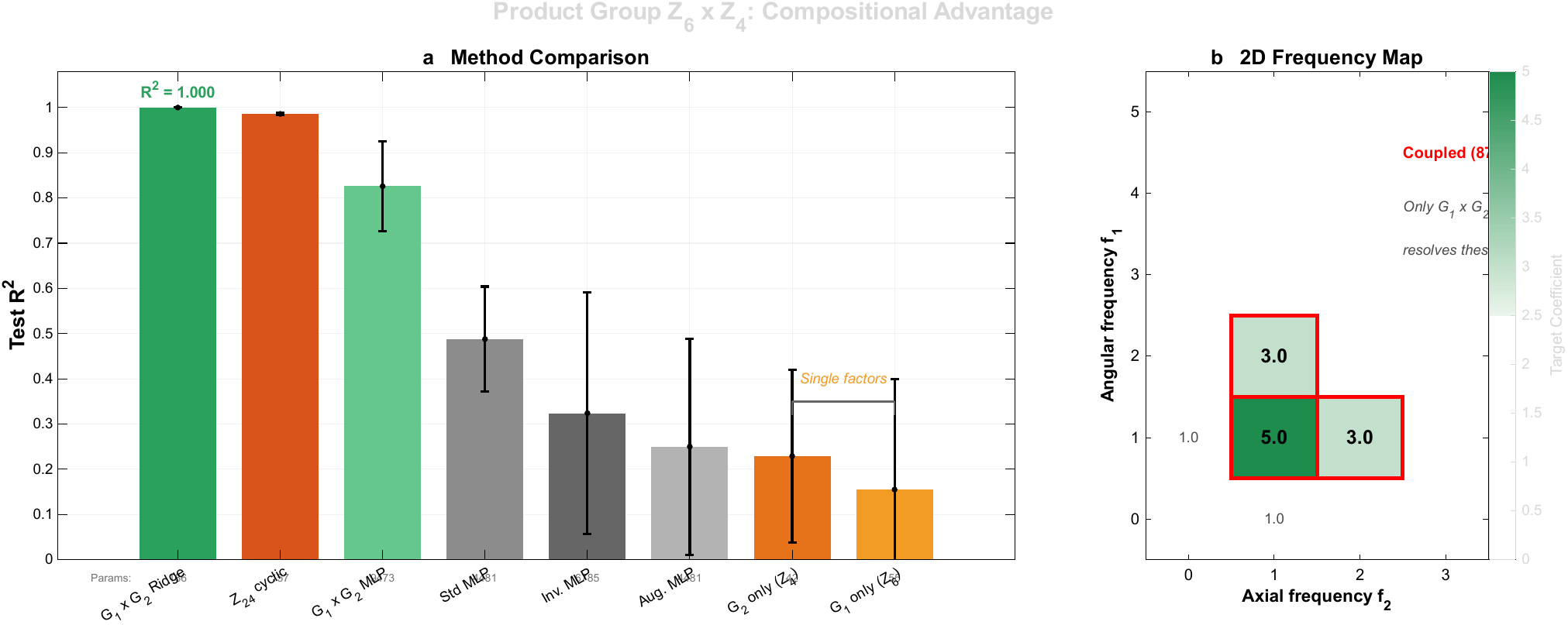}
\caption{\textbf{Product group $\bbZ_6 \times \bbZ_4$: compositional advantage.}
(\textbf{a})~Eight-method comparison. The product group achieves $R^2 = 1.000$;
each factor alone captures $\leq 23$\%.
(\textbf{b})~2D frequency map: coupled cells (red borders) carry 87\% of target
energy and are resolved only by $F_{\bbZ_6} \otimes F_{\bbZ_4}$.}
\label{fig:product}
\end{figure}

\begin{figure}[!ht]
\centering
\includegraphics[width=0.65\textwidth]{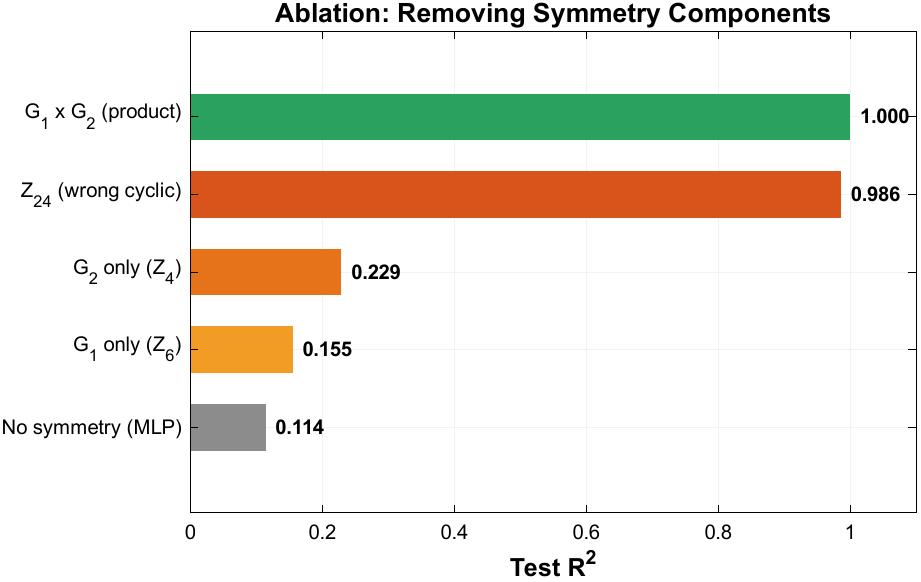}
\caption{\textbf{Ablation cascade.} Progressively removing symmetry components
reveals a strict performance hierarchy: product group $\to$ wrong cyclic
approximation $\to$ single factor $\to$ no symmetry.}
\label{fig:ablation}
\end{figure}

To test whether the algebraic composition theorem
(Theorem~\ref{thm:prod_full}) translates into empirical performance, we
constructed a task with two independent, commuting symmetries:
$G_1 = \bbZ_6$ (discrete rotations in the $xy$-plane) and
$G_2 = \bbZ_4$ (periodic translations along $z$). The target quantity was
designed to be dominated by coupled 2D Fourier frequencies $(f_1, f_2)$ that
require both symmetries simultaneously; neither factor alone can resolve these
modes. This setting models the physical situation in which a molecular property
depends on two structural degrees of freedom simultaneously, as is common in
materials with layered or helical symmetry.

Results are shown in Table~\ref{tab:product} and Figure~\ref{fig:product}. The
$\starG$ model over $G_1 \times G_2$ achieves perfect prediction
(\Rsq{} $= 1.000 \pm 0.000$) with just 186 parameters. The single-factor models
capture at most 23\%: $G_2$ alone (\Rsq{} $= 0.229$) and $G_1$ alone
(\Rsq{} $= 0.155$). The 2D frequency map (Figure~\ref{fig:product}b) visualizes
why: the coupled frequency cells (highlighted in red) carry 87\% of the target
energy and are resolved only by $F_{\bbZ_6} \otimes F_{\bbZ_4}$; the individual
transforms see only the axis-aligned marginals. The cyclic approximation
$\bbZ_{24}$ (treating the same 24-dimensional group dimension as a single cyclic
group) reaches \Rsq{} $= 0.986$, close but not exact, because it cannot
distinguish the tensor-product structure of the irreps. The ablation cascade
(Figure~\ref{fig:ablation}) confirms a strict performance hierarchy: product group
$>$ wrong cyclic $>$ single factor $>$ no symmetry, each step removing algebraic
information and reducing performance monotonically. On this task the implementation realizes
Theorem~\ref{thm:prod_full}: only the full product-group specification
achieves exact recovery, and every coarsening of it degrades performance. Because
the target is constructed to contain coupled product-group frequencies, the
experiment validates that the algebra recovers the intended structure when it is
present, not that such coupled structure is ubiquitous in real materials.

\subsection{Composition with a non-abelian factor}
\label{sec:product_exp}

the product-group theorem (Section~\ref{sec:si_product_proof}) composes symmetries by a Kronecker product of
transforms. We test whether this algebraic composition translates into
empirical performance on two tasks. The first is abelian: two commuting
symmetries, $G_1 = \bbZ_6$ (discrete rotations in the $xy$-plane) and
$G_2 = \bbZ_4$ (periodic translations along $z$), with a target carried by
coupled two-dimensional Fourier modes that neither factor resolves alone. The
product-group algebra reaches $R^2 = 1.000$ with $186$ parameters, the cyclic
approximation $\bbZ_{24}$ reaches $0.986$, and either factor alone captures at
most $23\%$ of the target variance; every coarsening of the group
specification degrades performance monotonically
(Section~\ref{sec:si_product_abelian}). The second task has a non-abelian
factor.

\begin{table}[!t]
\centering
\caption{Composed-symmetry task, chiral octahedral $O$ $\times$ color shift
$\bbZ_4$: test \Rsq{} (mean $\pm$ std, 3 seeds) over a sample-size sweep; the
three symmetry-aware models are matched near $\sim\!50$k parameters, while the
unconstrained MLP baseline has $1.8$M. The product-group readout carries both
symmetries by construction; $\star_{O}$-only is exactly rotation-invariant but must
learn color-shift invariance from data; SchNet is rotation-equivariant with the
label as an input feature. The target is a fixed product-group readout, so the
comparison measures data efficiency, not expressive capacity.}
\label{tab:composed}
\begin{tabular}{@{}lcccc@{}}
\toprule
\textbf{$n_{\mathrm{train}}$} & \textbf{$\star_{O \times \bbZ_4}$} &
\textbf{$\star_{O}$ only} & \textbf{SchNet} & \textbf{MLP} \\
 & \textbf{(59k)} & \textbf{(43k)} & \textbf{(63k)} & \textbf{(1.8M)} \\
\midrule
$100$     & $\mathbf{0.840 \pm 0.043}$ & $0.615 \pm 0.099$ & $0.033 \pm 0.267$ & $-0.003 \pm 0.029$ \\
$300$     & $\mathbf{0.993 \pm 0.000}$ & $0.950 \pm 0.018$ & $0.212 \pm 0.232$ & $0.012 \pm 0.044$ \\
$1{,}000$ & $\mathbf{0.998 \pm 0.001}$ & $0.989 \pm 0.013$ & $0.257 \pm 0.251$ & $0.031 \pm 0.083$ \\
$3{,}000$ & $\mathbf{0.999 \pm 0.000}$ & $0.998 \pm 0.001$ & $0.312 \pm 0.169$ & $0.136 \pm 0.153$ \\
\bottomrule
\end{tabular}
\end{table}

\paragraph{Non-abelian composition and data efficiency against a rotation-only
network.} Composition matters most when one factor is non-abelian, so the second
task has a target that is a joint invariant under the chiral octahedral group $O$
($|O|\!=\!24$, acting on 3D point clouds by rotation) and a global cyclic shift
$\bbZ_C$ ($C\!=\!4$) of a discrete per-point label. The target is a fixed
$\star_{O \times \bbZ_C}$ power-spectrum readout on the joint-nontrivial product
irreps, exactly invariant under both symmetries and dependent on neither alone
(Section~\ref{sec:composed_methods}). Since it is generated by a product-group
readout, the experiment measures data efficiency and inductive bias rather than
expressive capacity. We compare four models over a sample-size sweep (3 seeds;
Table~\ref{tab:composed}), the three symmetry-aware models matched near
$\sim\!50$k parameters and the unconstrained MLP baseline deliberately given
$1.8$M: the product-group
$\star_{O \times \bbZ_C}$ readout, which carries both symmetries by construction; a
single-symmetry $\star_{O}$ readout built from the identical spatial machinery with
the label folded into ordinary channels, exactly rotation-invariant but required to
learn the color-shift invariance from data; a rotation-equivariant SchNet given the
label as an input feature; and a plain MLP. Validation confirms the intended
structure to floating-point precision: the product readout is invariant to both
symmetries ($\leq 5\times 10^{-5}$), whereas $\star_{O}$ is rotation-invariant but
not color-shift-invariant.

The controlled comparison is between the two $\star_G$ readouts, which share a
hypothesis class and differ only in whether the second symmetry is built in. Their
gap is a clean data-efficiency signature: negligible once both saturate ($+0.001$
at $3{,}000$ samples) and growing as data becomes scarce ($+0.043$ at $300$,
$+0.225$ at $100$), where the built-in color-shift invariance spares the product
model from spending data to discover it. The rotation-only baselines sit far below
across the whole sweep: SchNet rises only from $0.03$ to $0.31$ and the MLP barely
leaves zero, so a network that must absorb the second symmetry by rote does not
catch up even at $3{,}000$ samples. We read the SchNet gap cautiously, since it
also reflects SchNet approximating the $\star_G$ environment kernel through
distance messages and carries large seed variance; the matched $\star_{O}$
comparison is the controlled one. The experiment validates
the product-group theorem (Section~\ref{sec:si_product_proof}) for a non-abelian factor and locates the
compositional advantage in the data-scarce regime.


\paragraph{Permutation symmetry.} The same Kronecker rule reaches the symmetric
group $S_n$, the permutation symmetry central to molecular learning, with no
change beyond the group argument. Young's orthogonal form assembles $F_{S_n}$,
and the $\starG[S_n]$ product block-diagonalizes under it to machine precision
(residual $1.8\times10^{-14}$ for $S_4$). By the product-group theorem (Section~\ref{sec:si_product_proof}),
$F_{\bbZ_k}\otimes F_{S_n}$ handles a combined cyclic-and-permutation symmetry
with one transform. On a signal that couples both factors, the product
transform places the full energy in the top $10\%$ of coefficients, where the
cyclic factor alone places $0.565$ and the permutation factor alone $0.445$
(Section~\ref{sec:si_experiments}). These are algebraic demonstrations on a
finite permutation group. Whether $\starG$ matches graph equivariant networks on
full three-dimensional, permutation-invariant molecular prediction is an
architectural question outside this claim.

\subsection{Permutation symmetry and simultaneous composition}
\label{sec:permutation}

A recurring claim for $\starG$ is that a new symmetry costs only a change of the
group argument (the cost is computational, set by the group order, rather than an
architectural redesign). We make this concrete for the
symmetric group $S_n$, the permutation symmetry central to molecular learning,
and for the simultaneous handling of two symmetries at once. The $\starG[S_n]$
algebra is built with exactly the machinery used for the cyclic and octahedral
cases: the irreducible representations of $S_n$ (Young's orthogonal form over
standard Young tableaux) assemble the unitary group-Fourier matrix $F_{S_n}$, and
nothing else changes. Three facts, all verified to machine precision for $S_4$:
(i)~$\starG[S_n]$ convolution block-diagonalizes under $F_{S_n}$, so the product
is one matrix multiplication per Specht-module irrep (residual
$1.8\times10^{-14}$); (ii)~the decomposition is permutation-equivariant: under
left translation by any $\tau \in S_n$ the per-irrep singular values are
invariant ($2.7\times10^{-15}$); and (iii)~by the product-group theorem
(Theorem~\ref{thm:prod_full}), $F_{\bbZ_k \times S_n} = F_{\bbZ_k} \otimes
F_{S_n}$ handles a combined cyclic-and-permutation symmetry with one transform.
On a signal whose structure couples both factors, the
product transform compacts it completely (top-$10\%$ energy fraction $1.000$)
while the cyclic factor alone ($0.565$) and the permutation factor alone
($0.445$) each leave most of the energy spread out. A single algebra thus
captures a joint symmetry that neither factor recovers, and adding the
permutation factor required no new architecture, only the Kronecker product of
Theorem~\ref{thm:prod_full}. These are algebraic and numerical demonstrations on a
finite permutation group: they establish that permutation-equivariance and
composition are intrinsic to the algebra, not that $\starG$ matches graph equivariant
networks for full three-dimensional, permutation-invariant molecular prediction,
which is an architectural question outside this algebraic claim.

\section{Methods Details}
\label{sec:si_methods}

This section holds the methodological detail referenced from the main-text
Materials and Methods.

\subsection{Feature Construction}
\label{sec:methods_featurizer}

For the single-group and product-group experiments, molecular features are inner
products with a rotating measurement basis at angles $\theta_g = 2\pi g/n$. For
the product group, axial features use periodic $z$-embeddings and coupled features
are angular$\times$axial products; the two actions commute because $z$-rotation
modifies $(x,y)$ not $z$, while $z$-translation modifies $z$ not $(x,y)$.

For the Wigner--Eckart experiment, features are computed under the 24 rotations
of the chiral octahedral group $O$ (6 face rotations at
$90^\circ/180^\circ/270^\circ$ about coordinate axes, 8 vertex rotations at
$120^\circ/240^\circ$ about body diagonals, 6 edge rotations at $180^\circ$
about edge midpoints, plus identity). This group is a subgroup of SO(3) whose
irreps (A$_1$, A$_2$, E, T$_1$, T$_2$) correspond to angular momentum channels
$l = 0, 3, 2, 1, 2$ (lowest occurrence).

\subsection{Dipole Vector Computation}

The QM9 .xyz files include Mulliken partial charges $q_i$ as a fifth column.
The dipole vector is computed as $\boldsymbol{\mu} = \sum_i q_i \mathbf{r}_i$,
yielding components $\mu_x, \mu_y, \mu_z$ that transform as a rank-1 tensor
(vector) under rotation. This provides ground-truth targets with known
transformation properties for the Wigner--Eckart test.

\subsection{Per-Irrep Fourier Decomposition}

For each irrep $\rho$ of dimension $d_\rho$ with representation matrices
$\{\rho(g)\}_{g \in G}$, the Fourier transform of feature row $j$ is
$\hat{X}_j^\rho = \sqrt{d_\rho/|G|} \sum_{g} X(j,g) \, \rho(g)$, a
$d_\rho \times d_\rho$ matrix. The per-irrep power is $\|\hat{X}_j^\rho\|_F^2$.
Per-irrep features (one power value per feature row per irrep) are used
independently as predictors for each quantum property via ridge regression,
yielding per-irrep \Rsq{} values. Pseudocode is given in
Algorithm~\ref{alg:per_irrep}.

\subsection{Invariant Feature Extraction}
\label{sec:invariant_features}

Given a sample tensor $\calX \in \bbR^{n_f \times |G|}$, the invariant feature
vector concatenates seven complementary descriptors, all of which are exactly
invariant under the left action of $G$ (proofs in Section~\ref{sec:equiv_proofs}):
(a)~the DC component $\bar{\calX}_j = \tfrac{1}{|G|} \sum_g \calX(j,g)$ for
each feature row $j$;
(b)~the AC energy $\sigma_j = \mathrm{std}_g\,\calX(j,g)$;
(c)~the total per-frequency power $\sum_j |\hat{\calX}(j,k)|^2$ for each
generalized Fourier bin $k$, where $\hat{\calX} = \calX F_G$ is the
generalized Fourier transform of $\calX$ along the group dimension;
(d)~per-row generalized Fourier power $|\hat{\calX}(j,k)|^2$ for the first
$K = 14$ equivariant rows;
(e)~the singular-element norms $\|\mathbf{s}_i\|_F$ for $i = 1, \ldots,
\min(p, q)$ obtained from the $\starG$-SVD of a reshaped $p \times q \times
|G|$ tensor (the reshape $(p, q)$ is chosen to maximize $\min(p, q)$);
(f)~the rows of $\calX$ identified as invariant by row-variance (constant
under the group action); and
(g)~four spectral statistics of the unfolded matrix (nuclear norm, spectral
norm, condition number, and entropy of the singular-value distribution).
Features are $z$-score normalized using statistics computed from training
data, and an unregularized intercept column is appended. Pseudocode is given
in Algorithm~\ref{alg:features}.

\subsection{Ridge Regression and Rank Selection}

The downstream linear model is ridge regression with hyperparameter
$\lambda$ selected from the geometric grid
$\{10^{-3}, 10^{-2}, \ldots, 10^{2}, 10^{3}\}$ by validation MSE. The
intercept is unregularized. The number of singular elements retained in the
$\starG$-SVD feature block is set to $\min(p, q)$ where $(p, q)$ is the
optimal rectangular reshape of $n_f$ (Section~\ref{sec:invariant_features});
no further truncation is applied because the Eckart--Young theorem guarantees
a closed-form bound on the truncation error
(Theorem~\ref{thm:ey_full}). Total parameter count for $\starG$-SVD +
Ridge is $1 + d_{\mathrm{feat}}$, where $d_{\mathrm{feat}}$ is the number of
non-degenerate feature columns retained after $z$-score normalization (for
example $143$ on full QM9 with the $48$-row featurizer and $63$ in the
data-scarce $14$-row setting, and $185$ on the product-group task).

\subsection{Baseline Architectures and Training}
\label{sec:baselines}

All four neural baselines share the same default hidden width
$[64, 32]$, ReLU activations on hidden layers, a linear output, He
initialization, full-batch validation, and the Adam optimizer
($\beta_1 = 0.9, \beta_2 = 0.999, \varepsilon = 10^{-8}$) with early stopping
on validation MSE (patience 20). They differ in three respects: (i)~the
input representation $X_{\mathrm{in}}$, (ii)~the training-set construction,
and (iii)~the parameter count. Table~\ref{tab:baselines_main} summarizes
these differences; the complete per-method specifications and pseudocode are
in Section~\ref{sec:baseline_alg}, and the per-experiment hyperparameter
sheet is Table~\ref{tab:hyperparams_full}.

\textbf{Training recipe (all MLP baselines).} To ensure the baselines are
trained to a modern standard and the comparison is not confounded by an
unfavorable recipe, every MLP uses: (i)~target standardization on the
training split, with predictions de-standardized before scoring, matching the
implicit scaling of the closed-form $\starG$ ridge regressor; (ii)~L2 weight
decay selected on the validation split over the grid
$\{0, 10^{-4}, 10^{-3}, 10^{-2}\}$; and (iii)~a \texttt{ReduceLROnPlateau}
learning-rate schedule on validation loss. On the full-QM9 matched-input
comparison (Section~\ref{sec:qm9_accuracy}) all molecule-level methods, both
$\starG$ variants and the three MLPs, consume the identical $48$-row angular
feature tensor of Section~\ref{sec:methods_featurizer}. The default hidden
width is $[64,32]$.

\begin{table}[!ht]
\centering
\caption{Baseline summary. Hidden $= [64, 32]$, ReLU, Adam, He init,
patience 20. Parameter counts are for QM9.}
\label{tab:baselines_main}
\begin{tabular}{@{}p{3.2cm}p{5.0cm}p{3.2cm}p{1.6cm}p{1.5cm}@{}}
\toprule
\textbf{Method} & \textbf{Input $X_{\mathrm{in}}$} & \textbf{Train data} &
\textbf{Equiv.} & \textbf{Params} \\
\midrule
Standard MLP   & raw slice $\calX(:, e)$, normalized        & native ($n$)        & approx & 5{,}249 \\
Invariant MLP  & $[\mathrm{mean}, \mathrm{std}, \min, \max]_g \calX$ & native ($n$) & exact  & 14{,}465 \\
Augmented MLP  & raw slice $\calX(:, e)$, normalized        & augmented ($|G|\,n$) & approx & 5{,}249 \\
Neural $\starG$ & $\starG$-features (Sec.~\ref{sec:invariant_features}) & native ($n$) & exact  & 62{,}625 \\
$\starG$-SVD + Ridge & $\starG$-features (Sec.~\ref{sec:invariant_features}) & native ($n$) & exact & \textbf{144} \\
\bottomrule
\end{tabular}
\end{table}

\subsection{Composed-symmetry task}
\label{sec:composed_methods}

Each sample is a point cloud of $8$ to $12$ points in $\bbR^3$, each carrying a
discrete label $c_i \in \{0, \ldots, C-1\}$ ($C=4$). For every point we build a
local rotation-covariant descriptor sampled over the $|O|=24$ octahedral frames and
tensor it with the label one-hot $\delta_{c_i}$, giving a per-point $\starG$ feature
on the product group $G = O \times \bbZ_C$; these features are mean-pooled over
points. The target applies a fixed random $\starG$ operator to the pooled feature
and reads out the power spectrum on the joint-nontrivial product irreps (spatial
irrep nontrivial \emph{and} color character nontrivial), weighted by fixed random
coefficients. Each joint power is a rotation-invariant pairwise coupling of point
environments modulated by $\cos\!\big(2\pi k (c_i - c_j)/C\big)$, so the target is
exactly invariant under any octahedral rotation and any global label shift and
vanishes if either the spatial or the label variation is removed. The
$\star_{O}$-only baseline uses the identical spatial descriptor with the label
folded into ordinary feature channels, which a global shift permutes and the
rotation-invariant readout does not undo. All models train for $300$ epochs (Adam,
learning-rate warmup, early stopping on validation loss); the SchNet baseline is
the PyTorch Geometric reference implementation with the label supplied as an
additional node feature. The construction, the four models, the validation suite and the sweep driver are in the code deposit (Section~\ref{sec:si_codemap}).

\section{Selection-Rule Projection Across Tensor Ranks}
\label{sec:si_tensor_ranks}

This section documents the extension of the crystal physicality layer of the
main text from the rank-4 elastic tensor to the rank-2 dielectric and rank-3
piezoelectric tensors (about three minutes on a desktop CPU).

\paragraph{Data.} The matminer datasets \texttt{dielectric\_constant}
($1{,}056$ crystals, total dielectric tensor $\varepsilon$, $3\times3$, in the
Cartesian frame of the structure)~\cite{petousis2017dielectric} and
\texttt{piezoelectric\_tensor} ($941$ crystals, piezoelectric stress tensor $e$
in C\,m$^{-2}$ as a $3\times6$ Voigt matrix)~\cite{dejong2015piezoelectric}.
The Voigt matrix is read as $e_{i,jk} = V_{iJ}$ without shear factors. This
convention reproduces the dataset's own maximal longitudinal modulus
\texttt{eij\_max} (median relative discrepancy $2\times10^{-4}$ over the $941$
entries; the modulus is $\max_{\mathbf n}|\sum_{iJ} n_i V_{iJ} N_J(\mathbf n)|$
with $N=(n_1^2,n_2^2,n_3^2,n_2n_3,n_1n_3,n_1n_2)$ over $6{,}000$ Fibonacci
directions), and it is the convention under which the DFT tensors are
symmetry-consistent.

\paragraph{Projector.} $P_G T = |G|^{-1}\sum_{g\in G} R(g)^{\otimes r}\,T$ over
the full set of Cartesian point-group operations of the structure (spglib via
pymatgen, \texttt{symprec} $=0.01$, the pymatgen default), improper operations
included. On an even-rank tensor an improper operation acts as its rotation
part. On the odd-rank polar piezoelectric tensor it carries the sign, so
$P_G e = 0$ for every centrosymmetric point group and for $432$. Idempotence
$P_G^2 = P_G$ holds to $2\times10^{-15}$ (rank 2) and $3\times10^{-14}$ (rank
3) on random tensors for every point group in the two sets. Orthogonality of
$P_G$ is exact only to the extent that spglib's Cartesian matrices are
orthogonal. For a relaxed lattice that deviates from the ideal metric within
\texttt{symprec} they are orthogonal only to within that deviation ($28$ of
$1{,}056$ and $107$ of $941$ structures above $10^{-6}$, worst $4\times10^{-3}$
and $8\times10^{-3}$), which bounds the residual of the Pythagorean error
identity at $2.5\times10^{-4}$ and $1.4\times10^{-4}$ of the tensor energy. At
the looser \texttt{symprec} $=0.1$ used for the elastic set, spglib assigns
$143$ of the $941$ piezoelectric structures a higher symmetry than their
relaxed lattice supports, with matrices non-orthogonal by up to $5\%$; the
tighter setting is therefore used for both new sets.

\paragraph{Frames.} The dielectric tensors are symmetry-consistent in the
structure frame as shipped. The piezoelectric tensors are reported in the IEEE
standard orientation~\cite{dejong2015piezoelectric}. We conjugate the point
group by pymatgen's IEEE rotation $R$ of the structure ($g \mapsto R g R^{\top}$),
which reconciles $755$ of the $941$ tensors with their structures to a
forbidden fraction below $10^{-6}$. For the classes whose IEEE axis assignment
admits equivalent choices (chiefly $mm2$, $4mm$, $\bar{4}2m$, $\bar{6}m2$) the
remaining ambiguity is resolved by the axis relabeling, a signed permutation
optionally composed with a $30^\circ$, $45^\circ$, or $60^\circ$ rotation about
the principal axis, that makes the DFT tensor symmetry-consistent ($184$
tensors). Two tensors are reconciled by no relabeling and are excluded from the
prediction study. This is a setting reconciliation in the sense of
crystallographic practice, and it uses only the DFT tensor and the point group.

\paragraph{DFT tensors (Part A).} Table~\ref{tab:si_ranks_dft} lists the median
forbidden-energy fraction per crystal system. The projection brings every
tensor to a residual of at most $1.2\times10^{-15}$ (rank 2) and
$1.8\times10^{-6}$ (rank 3, the frame and \texttt{symprec} residual above) of
its norm. Among the $235$ cubic dielectrics (cubic by point group: spglib labels
the four-site antiferromagnetic NiO cell $Fm\bar{3}m$ while returning the $24$
operations of $\bar{3}m$, and the point group decides) the median deviatoric
energy fraction is $1.4\times10^{-9}$; seven exceed a birefringence of $0.01$
through relaxation residuals in high-permittivity compounds. The median
birefringence of the non-cubic dielectrics is $0.51$.

\begin{table}[!ht]
\centering
\caption{DFT tensors: median forbidden-energy fraction $\|T-P_GT\|^2/\|T\|^2$
per crystal system (space-group number), before projection. Dielectric:
$99.5\%$ of $1{,}056$ tensors below $10^{-3}$ and $87\%$ below $10^{-6}$;
piezoelectric (frame-resolved): $100\%$ of $939$ below $10^{-3}$ and $98.8\%$
below $10^{-6}$. After projection every tensor is at machine precision (rank 2)
or at the frame residual $\le 1.8\times10^{-6}$ (rank 3).}
\label{tab:si_ranks_dft}
\begin{tabular}{@{}lrrrr@{}}
\toprule
\textbf{Crystal system} & \multicolumn{2}{c}{\textbf{Dielectric (rank 2)}} & \multicolumn{2}{c}{\textbf{Piezoelectric (rank 3)}} \\
 & $n$ & median & $n$ & median \\
\midrule
cubic        & 236 & $1.3\times10^{-9}$  & 140 & $6.2\times10^{-14}$ \\
hexagonal    & 132 & $1.9\times10^{-9}$  &  94 & $1.4\times10^{-14}$ \\
trigonal     & 253 & $7.6\times10^{-10}$ & 156 & $3.1\times10^{-14}$ \\
tetragonal   & 178 & $2.8\times10^{-10}$ & 200 & $1.9\times10^{-14}$ \\
orthorhombic & 165 & $8.2\times10^{-8}$  & 189 & $1.8\times10^{-16}$ \\
monoclinic   &  77 & $4.7\times10^{-9}$  &  76 & $1.3\times10^{-16}$ \\
triclinic    &  15 & $1.9\times10^{-32}$ &  84 & $2.3\times10^{-31}$ \\
\bottomrule
\end{tabular}
\end{table}

\paragraph{Predicted tensors (Part B).} Composition-only features (the $16$ symmetry-blind descriptors of the enforcement screen), one
\texttt{HistGradientBoostingRegressor} per independent component ($6$ for the
symmetric dielectric tensor, $18$ for the piezoelectric Voigt matrix), five-fold
cross-validation, seeds $0$--$2$; Table~\ref{tab:si_ranks_pred} reports means
over seeds with standard deviations. For the centrosymmetric screen the
piezoelectric model is fit on all $939$ resolved piezoelectric crystals and
applied to the $788$ centrosymmetric crystals of the dielectric set, an
out-of-distribution candidate pool whose true piezoelectric tensor is
identically zero. A candidate counts as flagged when its maximal longitudinal
modulus exceeds $0.1$~C\,m$^{-2}$ ($\alpha$-quartz has $e_{11}=0.17$). A cubic
dielectric counts as falsely birefringent when the principal refractive indices
$n_i=\sqrt{\varepsilon_i}$ of the predicted tensor differ by more than $0.01$
(quartz has $0.009$). The one cubic crystal that remains above the threshold
after projection in one seed is a lattice whose spglib operations are
orthogonal only to $1.7\times10^{-3}$.

\begin{table}[!ht]
\centering
\caption{Predicted tensors, raw versus projected (means over seeds $0$--$2$,
$\pm$ one standard deviation). Errors are Frobenius norms in the tensor's
units (dimensionless $\varepsilon$; C\,m$^{-2}$ for $e$). The centrosymmetric
screen applies the piezoelectric model to the $788$ centrosymmetric crystals of
the dielectric set.}
\label{tab:si_ranks_pred}
\footnotesize
\begin{tabular}{@{}lcccc@{}}
\toprule
\textbf{Quantity} & \multicolumn{2}{c}{\textbf{Dielectric}} & \multicolumn{2}{c}{\textbf{Piezoelectric}} \\
 & raw & projected & raw & projected \\
\midrule
Forbidden-energy fraction (mean)        & $0.035\pm0.001$ & $\le1.4\times10^{-30}$ & $0.506\pm0.003$ & $\le3.2\times10^{-12}$ \\
Reconstruction error (mean)             & $15.2\pm0.3$    & $14.1\pm0.3$          & $2.24\pm0.02$   & $1.90\pm0.02$ \\
Crystals with error reduced             & --              & $99.8\%$              & --              & $100\%$ \\
Cubic dielectrics with $\Delta n>0.01$ (of 235) & $235$   & $0.3\pm0.5$           & --              & -- \\
Spearman $\rho$, dielectric anisotropy  & $0.21\pm0.01$   & $0.70\pm0.01$         & --              & -- \\
Spearman $\rho$, birefringence          & $0.21\pm0.01$   & $0.71\pm0.01$         & --              & -- \\
$e_{\max}$ mean absolute error (C\,m$^{-2}$) & --         & --                    & $0.685\pm0.006$ & $0.598\pm0.007$ \\
Spearman $\rho$, $e_{\max}$             & --              & --                    & $0.43\pm0.00$   & $0.57\pm0.01$ \\
Centrosymmetric candidates flagged (of 788) & --          & --                    & $773$           & $0$ \\
\bottomrule
\end{tabular}
\end{table}

\section{What Symmetry Information Alone Does Not Supply}
\label{sec:si_symmetry_informed}

The enforcement experiments compare a symmetry-blind predictor with the same
predictor followed by the point-group projection. The baseline is
composition-only, so a reader may reasonably ask whether a model simply told the
crystal class would place no energy in forbidden channels to begin with, and
whether the reported downstream gains are anything more than the recovery of
values that symmetry already fixes. This section answers both questions on the
same three datasets, with the same estimator and cross-validation protocol
(three seeds, five folds).

\paragraph{Symmetry as an input feature does not remove the leak.}
Each predictor is trained twice: on the published composition-only features, and
on those features together with explicit symmetry descriptors, namely a one-hot
crystal system, a one-hot point group, a centrosymmetry flag, and the order of
the operation set. Table~\ref{tab:symmetry_informed} reports the mean fraction of
predicted tensor energy lying in symmetry-forbidden channels.

\begin{table}[!ht]
\centering
\caption{\textbf{Forbidden-channel energy of the raw predictions, with and
without symmetry information.} Mean over crystals and three seeds. Handing the
model the crystal system, the point group and a centrosymmetry flag leaves the
violation essentially intact; the projection removes it exactly.}
\label{tab:symmetry_informed}
\begin{tabular}{@{}lrccc@{}}
\toprule
\textbf{Tensor} & \textbf{$n$} & \textbf{Composition only} &
\textbf{{}+ symmetry features} & \textbf{After projection} \\
\midrule
Elastic (rank 4)        & $1{,}181$ & $6.3\times10^{-3}$ & $5.5\times10^{-3}$ & $\le2.6\times10^{-29}$ \\
Dielectric (rank 2)     & $1{,}056$ & $3.5\times10^{-2}$ & $3.3\times10^{-2}$ & $\le1.6\times10^{-30}$ \\
Piezoelectric (rank 3)  & $941$     & $5.1\times10^{-1}$ & $4.3\times10^{-1}$ & $\le3.2\times10^{-12}$ \\
\bottomrule
\end{tabular}
\end{table}

Telling the model the symmetry reduces the forbidden energy by $13\%$ on the
elastic tensors, by $4\%$ on the dielectric tensors and by $15\%$ on the
piezoelectric tensors. On the piezoelectric set, a regressor that is given the
point group and a centrosymmetry flag still places $43\%$ of its output energy in
channels that the crystal class forbids. A learned model has no mechanism for
producing exact zeros, and supervision on symmetry labels does not supply one.
Projection reduces the violation to machine precision in every case and reduces
the reconstruction error of $99.6$ to $99.9\%$ of individual predictions.

A third reference point isolates the other extreme. A predictor that returns the
mean measured tensor of the crystal's own point group, carrying no composition
information at all, has a much smaller forbidden fraction ($8\times10^{-5}$,
$1.8\times10^{-3}$ and $9\times10^{-3}$ for the three ranks, nonzero only because
the averaged crystals sit in different frames), and a much worse reconstruction
error: $0.88$ against $0.27$ for the composition model on the elastic tensors.
Symmetry compliance and predictive accuracy are therefore not exchangeable along
a single axis, which is the practical case for obtaining the first by projection
and the second by learning.

\paragraph{The ranking gain is between symmetry classes rather than within them.}
The main text reports that projection raises the Spearman correlation of
predicted against measured anisotropy type $c_4$, and attributes the gain to the
restoration of the symmetry-fixed value $c_4\approx1$ for cubic crystals.
Table~\ref{tab:stratified_anisotropy} makes that attribution quantitative by
recomputing the correlation within each crystal system.

\begin{table}[!ht]
\centering
\caption{\textbf{Anisotropy-type ranking, overall and within crystal systems.}
Spearman correlation of predicted against measured $c_4$, mean over three seeds,
out-of-fold predictions of the composition-only model. The projection accounts
for the whole of the overall gain and for essentially none of the within-system
ranking.}
\label{tab:stratified_anisotropy}
\begin{tabular}{@{}lrcc@{}}
\toprule
\textbf{Subset} & \textbf{$n$} & \textbf{Raw} & \textbf{Projected} \\
\midrule
All crystals  & $1{,}181$ & $0.279$ & $\mathbf{0.766}$ \\
\midrule
Cubic         & $449$ & $0.089$ & $0.111$ \\
Hexagonal     & $238$ & $0.215$ & $0.258$ \\
Tetragonal    & $196$ & $0.274$ & $0.259$ \\
Orthorhombic  & $193$ & $0.180$ & $0.182$ \\
Trigonal      & $60$  & $0.096$ & $0.030$ \\
Monoclinic    & $45$  & $0.233$ & $0.233$ \\
\bottomrule
\end{tabular}
\end{table}

Within every crystal system the correlation is nearly unchanged, and in two of
the smaller ones it decreases. The interpretation follows from
Proposition~\ref{prop:enforce_guarantee}: projection removes forbidden content
and leaves allowed content untouched, so it can only restore what symmetry
determines. For $c_4$ that part is large, because the measured value is fixed at
$1$ for every cubic crystal and the raw model does not reproduce it. The part
that symmetry leaves free is unchanged, as it must be. We therefore state the
gain as the recovery of symmetry-determined structure, and not as an improvement
in the model's physical discrimination within a class.

\paragraph{Both directions of the stability screen.}
Every ground-truth tensor in the de Jong set is Born-stable: the smallest
eigenvalue of the Voigt matrix over all $1{,}181$ crystals is $0.67$, far above
any numerical threshold. Only one of the two error directions is therefore
measurable on this data. Scoring both arms at a single symmetric threshold
($\lambda_{\min}>10^{-6}$), the raw composition model falsely rejects
$5.7\pm1.7$ viable crystals and the symmetry-plus-stability repair falsely
rejects $2.3\pm0.5$; neither arm admits an unstable crystal, because none exists
in the data. The repair enforces positive semidefiniteness by construction, so
the crystals it still fails are those left exactly at the boundary by eigenvalue
clipping, and raising the clipping floor to the screening threshold removes them
by construction as well. We report the screen as a measurement of false
rejections, which is what this dataset supports, and we do not claim a
false-acceptance rate that it cannot measure.

\section{Equal-Input Crystal Data Efficiency}
\label{sec:si_data_efficiency}

\begin{table}[!t]
\centering
\caption{\textbf{Equal-input crystal data efficiency.} Spearman ranking correlation (mean over 3 seeds) on randomly rotated test crystals versus training-set size for the octahedral-invariant $\starG$ descriptor and symmetry-blind representations of the \emph{same} local-environment moments, evaluated both in a fixed frame and with augmentation over all $24$ octahedral orientations. $c_4$ is symmetry-carried anisotropy; $\log K$ is composition-carried. $\starG$ gives identical scores on rotated and canonical test sets, reflecting exact invariance.}
\label{tab:data_efficiency}
\begin{tabular}{@{}llccc@{}}
\toprule
\textbf{Target} & \textbf{$N_{\mathrm{train}}$} & \textbf{$\starG$} & \textbf{sym-blind} & \textbf{sym-blind\,+\,aug} \\
\midrule
anisotropy $c_4$ & $40$  & \textbf{0.358} & 0.114 & 0.179 \\
                 & $160$ & \textbf{0.483} & 0.140 & 0.237 \\
                 & $640$ & \textbf{0.614} & 0.217 & 0.306 \\
\midrule
$\log K$         & $40$  & \textbf{0.840} & 0.704 & 0.763 \\
                 & $160$ & 0.873 & 0.825 & 0.876 \\
                 & $640$ & 0.894 & 0.870 & 0.910 \\
\bottomrule
\end{tabular}
\end{table}

We isolate the value of symmetry using representations constructed from identical information. Both descriptors use the same chemistry-weighted radial, dipole, and quadrupole moments of each site's neighbors. The $\starG$ descriptor takes their octahedral group-Fourier power spectrum, which is exactly invariant under the octahedral group $O$; the symmetry-blind descriptor retains the same moments in a fixed frame and, because it preserves phase, contains strictly more information. Since the targets are octahedral invariants, randomly reorienting held-out crystals isolates the sample-complexity cost of learning invariance (Table~\ref{tab:data_efficiency}).

For the symmetry-carried anisotropy channel $c_4$, $\starG$ ranks crystals two to three times better than the symmetry-blind representation across all training sizes, including when the latter receives explicit augmentation over all $24$ orientations. No crossover occurs in the sampled regime. For the composition-dominated bulk modulus, $\starG$ is also more sample-efficient at low data volume, achieving Spearman $0.84$ with only $40$ crystals; the symmetry-blind representation approaches it only at the largest training size with explicit augmentation ($0.910$ versus $0.894$ at $640$ crystals). The $\starG$ score is unchanged between canonical and randomly rotated test sets, whereas the symmetry-blind representation degrades under reorientation.

These results isolate the benefit of built-in symmetry: the advantage arises from reduced sample complexity rather than additional input information. A direct comparison with equivariant-network heads using the same structures is a natural extension.

\section{Projection Where No Lookup Applies}
\label{sec:si_lookup_free}

Zeroing the piezoelectric tensor of a centrosymmetric crystal and isotropizing a
cubic dielectric are results that a point-group lookup also delivers. This
section evaluates the projection only on crystals whose point group allows a
nonzero tensor with a structured pattern of vanishing and tied components
(Section~\ref{sec:si_codemap}).
The estimator, features, five-fold out-of-fold splits and seeds $0$ to $2$ are
those of Section~\ref{sec:si_tensor_ranks} and
Section~\ref{sec:si_symmetry_informed}, and point groups come from spglib at
\texttt{symprec} $0.01$.

\paragraph{Crystals.} Piezoelectric: all $939$ crystals of the set with resolved
frames. The set contains no centrosymmetric crystal and no crystal of class $432$,
so every member allows a nonzero polar tensor, and its $140$ cubic members belong
to classes $23$ and $\bar43m$. Dielectric: the $821$ non-cubic crystals of
$1{,}056$. Elastic: the $732$ non-cubic crystals of $1{,}181$.

\paragraph{The zero-pattern arm.} This arm sets to zero the components of the raw
prediction that vanish identically on the allowed subspace, in the crystal's own
frame, and changes nothing else. It is the most that a sparsity table can do. The
projection also enforces the ties between components. Because the zero-pattern
subspace contains the allowed subspace, projection cannot do worse than this arm
when the truth is allowed, and the difference between the two measures the
contribution of the ties.

\begin{table}[!ht]
\centering
\caption{Projection on crystals whose point group allows a structured nonzero
tensor (composition-only regressor, mean over three seeds). Error shares are taken
over crystals with a nontrivial projection. SSE is the summed squared error.}
\label{tab:lookup_free}
\small
\begin{tabular}{@{}lccc@{}}
\toprule
 & \textbf{Piezoelectric} & \textbf{Dielectric} & \textbf{Elastic} \\
 & ($939$) & ($821$ non-cubic) & ($732$ non-cubic) \\
\midrule
Forbidden-energy fraction of the raw prediction (mean) & $0.51$ & $0.029$ & $0.0039$ \\
Median relative error, raw $\rightarrow$ projected & $1.476\rightarrow1.131$ & $0.412\rightarrow0.368$ & $0.222\rightarrow0.217$ \\
SSE removed by the zero-pattern arm & $3.6\%$ & $0.0\%$ & $0.2\%$ \\
SSE removed by the projection & $5.7\%$ & $1.8\%$ & $4.1\%$ \\
Share of crystals whose error decreases & $100\%$ & $99.8\%$ & $99.7\%$ \\
Crystals whose error increases (per seed) & $0$ & $2.0$ & $2.3$ \\
Spearman of the ranked quantity, raw $\rightarrow$ projected & $0.435\rightarrow0.567$ & $0.278\rightarrow0.411$ & $0.363\rightarrow0.384$ \\
\bottomrule
\end{tabular}
\end{table}

The ranked quantities in Table~\ref{tab:lookup_free} are the largest
piezoelectric modulus, the birefringence and the universal elastic anisotropy
$A^U$. The zero-pattern arm reaches Spearman $0.541$ for the piezoelectric modulus
and $0.272$ for the birefringence. With the crystal system, the point group and a
centrosymmetry flag as additional features, the raw piezoelectric prediction still
places $42\%$ of its energy in forbidden channels ($51\%$ without them), the
projection still lowers the error of $100\%$ of the crystals with a nontrivial
projection, and the rankings rise from $0.446$ to $0.553$ (piezoelectric) and from
$0.308$ to $0.436$ (birefringence). The ranking of the elastic anisotropy type
$c_4$ rises from $0.233$ to $0.286$ when pooled over systems and stays flat or
falls slightly inside each crystal system.

\paragraph{Unfavorable results.} Projection removes forbidden energy and therefore
shrinks predicted magnitudes, which moves predictions across a fixed threshold.
For the piezoelectric screen at a largest modulus of $0.1$~C\,m$^{-2}$ ($696$ of
the $939$ crystals are true positives), false positives fall from $207.7$ to
$81.7$ per seed and false negatives rise from $12.3$ to $159.0$, so total errors
grow from $220$ to $241$. The birefringence screen at $\Delta n>0.01$ among
non-cubic dielectrics goes from $12$ to $23.3$ errors. A threshold must therefore
be re-tuned after projection, and the gains claimed here are the error reduction
and the rankings. The uniaxial or biaxial classification from the predicted
eigenvalue pattern rises from $31.4\%$ to $99.4\%$, but the degeneracy of a
projected uniaxial tensor is fixed by the point group alone, so we do not count it
as a gain beyond a lookup. The Born screen of the elastic tensors measures false
rejections only ($2.0$ raw, $1.3$ projected, $0$ with the nearest positive
semidefinite repair), because no measured tensor is unstable.

\paragraph{The guarantee.} Every crystal whose error increased had a measured
tensor carrying more forbidden energy than the prediction's error does, the case
of Proposition~\ref{prop:enforce_guarantee}(ii). The largest relative increase is
$0.6\%$ ($1.3\%$ with symmetry features). The Pythagorean identity of that
proposition holds to $3.6\times10^{-5}$ relative at worst, a residual set by how
far the Cartesian operations that spglib returns depart from orthogonality.

\section{Baseline Implementation Algorithms}
\label{sec:baseline_alg}

This section gives the complete per-method specifications and explicit
pseudocode for the neural baselines used in the main paper.

\paragraph{Standard MLP.} Input is the un-rotated raw feature vector
$X_{\mathrm{in}} = \calX(:, e) \in \bbR^{n_f}$ (frontal slice at the
identity), $z$-score normalized using training-set statistics. The model is
$[n_f \to 64 \to 32 \to 1]$. No symmetry information is used during training.
Trained for up to 300 epochs (synthetic and QM9 / product-group)
at learning rate $0.003$, batch size 32 (synthetic) or 256 (QM9). This is the
``no-symmetry'' control.

\paragraph{Invariant MLP.} Input is the concatenation of four hand-crafted
group-invariant pooling statistics applied along the group dimension:
$X_{\mathrm{in}} = [\,\mathrm{mean}_g \calX,\ \mathrm{std}_g \calX,\
\min_g \calX,\ \max_g \calX\,] \in \bbR^{4 n_f}$, $z$-score normalized. The
model is $[4 n_f \to 64 \to 32 \to 1]$. By construction the input is exactly
$G$-invariant so the model is invariant by composition; this is the ``manual
invariance'' control. Same optimizer schedule as the Standard MLP.

\paragraph{Augmented MLP.} Input is the un-rotated raw feature vector
$X_{\mathrm{in}} = \calX(:, e) \in \bbR^{n_f}$, but the training set is
expanded by including $X_{\mathrm{aug}} = \calX(:, g)$ for every $g \in G$,
with the same target label, yielding a $|G|$-fold augmented training set.
$z$-score statistics are computed on the augmented set. The model is
$[n_f \to 64 \to 32 \to 1]$. Test-time prediction uses the un-rotated
slice. This is the standard data-augmentation strategy and serves as a
proxy for invariance learned from data rather than enforced algebraically;
it is the closest non-equivariant baseline to an ENN. Trained for 80--300
epochs at learning rate $0.003$--$0.005$, batch size 32 (synthetic) or 256
(QM9 / product-group); the lower epoch count for the synthetic experiment
reflects the $|G|$-fold larger effective batch budget. A precise pseudocode
specification is given in Algorithm~\ref{alg:augmented_mlp}.

\paragraph{Neural $\starG$.} A symmetry-aware feed-forward network whose linear
layers are $\starG$ products with weight tensors $W^{(\ell)} \in \bbR^{n_{\ell+1}
\times n_\ell \times |G|}$ rather than ordinary matrix products. Forward
pass: $\calA^{(\ell+1)} = \mathrm{ReLU}\bigl(W^{(\ell)} \starG \calA^{(\ell)}
+ \mathbf{b}^{(\ell)}\bigr)$ for hidden layers, with a linear output and an
invariant pooling $y = \mathrm{mean}_{g, j} \calA^{(L)}(j, g)$. The hidden
widths are $[64, 32]$, matching the MLPs. Input is the invariant
$\starG$-feature vector $X_{\mathrm{in}}$ of Section~\ref{sec:feat_alg}
(Materials and Methods of the main text). Trained for 300 epochs at learning
rate $0.003$, batch size 256, with the same Adam settings and early
stopping. Equivariance is exact by construction (the per-layer rotation
variance is $\sim 10^{-28}$, at floating-point noise) so this baseline
isolates the cost of replacing a closed-form ridge regressor with a non-linear
trainable model on top of the same algebraic representation. Pseudocode is
given in Algorithm~\ref{alg:neural_starg}.

The Standard MLP and Invariant MLP follow the standard MLP training loop
that wraps Algorithm~\ref{alg:augmented_mlp} once augmentation is removed.

\begin{algorithm}[!ht]
\caption{Augmented MLP training}
\label{alg:augmented_mlp}
\begin{algorithmic}[1]
\Require training tensor $\calX^{\mathrm{tr}} \in \bbR^{n \times n_f \times |G|}$, target $y^{\mathrm{tr}} \in \bbR^n$, group $G$, validation $(\calX^{\mathrm{va}}, y^{\mathrm{va}})$, hidden widths $h = [64, 32]$, learning rate $\eta$, max epochs $E$, batch size $B$, patience $P$
\Ensure trained weights $W = \{W^{(1)}, W^{(2)}, W^{(3)}\}$, biases $b$
\State $\tilde{X} \gets \mathrm{reshape}(\mathrm{permute}(\calX^{\mathrm{tr}}, [1, 3, 2]), [\,n |G|, n_f\,])$ \Comment{stack all $|G|$ orbit copies}
\State $\tilde{y} \gets \mathrm{repmat}(y^{\mathrm{tr}}, |G|, 1)$ \Comment{labels are $G$-invariant; replicate}
\State $(\mu, s) \gets (\mathrm{mean}(\tilde{X}), \mathrm{std}(\tilde{X}) + 10^{-8})$ \Comment{$z$-norm on augmented set}
\State $\tilde{X} \gets (\tilde{X} - \mu) / s$
\State initialize $W^{(\ell)} \sim \mathcal{N}(0, 2 / \mathrm{fan\_in}_\ell)$ (He init); $b^{(\ell)} \gets 0$
\State Adam state: $m^{(\ell)}, v^{(\ell)} \gets 0$, step counter $t \gets 0$
\State best-$W \gets W$; $\mathrm{wait} \gets 0$; $\mathrm{best\_val} \gets +\infty$
\For{$\mathrm{epoch} = 1, \ldots, E$}
    \State shuffle $\tilde{X}$
    \For{each minibatch of size $B$}
        \State $t \gets t + 1$
        \State forward: $A^{(0)} \gets X_{\mathrm{batch}}$;\ $A^{(\ell)} \gets \mathrm{ReLU}(W^{(\ell)} A^{(\ell-1)} + b^{(\ell)})$ for $\ell < L$;\ $A^{(L)} \gets W^{(L)} A^{(L-1)} + b^{(L)}$
        \State $\mathcal{L} \gets \tfrac{1}{B} \|A^{(L)} - y_{\mathrm{batch}}\|^2$
        \State backprop $\nabla_W \mathcal{L}, \nabla_b \mathcal{L}$
        \State Adam update: $m, v$ exponential moving averages with $\beta_1 = 0.9, \beta_2 = 0.999, \varepsilon = 10^{-8}$
        \State $W^{(\ell)} \gets W^{(\ell)} - \eta \cdot \hat{m}^{(\ell)} / (\sqrt{\hat{v}^{(\ell)}} + \varepsilon)$
    \EndFor
    \State $\mathcal{L}_{\mathrm{val}} \gets$ MSE on $((\calX^{\mathrm{va}}(:, :, e) - \mu) / s, y^{\mathrm{va}})$
    \If{$\mathcal{L}_{\mathrm{val}} < \mathrm{best\_val}$}
        \State $\mathrm{best\_val} \gets \mathcal{L}_{\mathrm{val}}$;\ best-$W \gets W$;\ $\mathrm{wait} \gets 0$
    \Else
        \State $\mathrm{wait} \gets \mathrm{wait} + 1$;\ \textbf{if} $\mathrm{wait} \geq P$ \textbf{break}
    \EndIf
\EndFor
\State \Return best-$W$, best-$b$
\end{algorithmic}
\end{algorithm}

\begin{algorithm}[!ht]
\caption{Neural $\starG$ forward pass and gradient}
\label{alg:neural_starg}
\begin{algorithmic}[1]
\Require batch $\calX \in \bbR^{N \times n_f \times |G|}$, $\starG$-weights $\{W^{(\ell)}\}_{\ell=1}^{L}$ with $W^{(\ell)} \in \bbR^{n_{\ell+1} \times n_\ell \times |G|}$, biases $\{b^{(\ell)}\} \in \bbR^{n_{\ell+1}}$ (one scalar per output channel, broadcast constant along the group axis), group $G$
\Ensure scalar predictions $\hat{y} \in \bbR^N$
\State $\calA^{(0)} \gets \calX$
\For{$\ell = 1, \ldots, L$}
    \State $\calZ^{(\ell)} \gets W^{(\ell)} \starG \calA^{(\ell-1)} + b^{(\ell)}$ \Comment{Algorithm~\ref{alg:starg_product}}
    \If{$\ell < L$}
        \State $\calA^{(\ell)} \gets \mathrm{ReLU}(\calZ^{(\ell)})$
    \Else
        \State $\calA^{(\ell)} \gets \calZ^{(\ell)}$ \Comment{linear output}
    \EndIf
\EndFor
\State $\hat{y}_i \gets \tfrac{1}{n_L |G|} \sum_{j, g} \calA^{(L)}(i, j, g)$ \Comment{$G$-invariant pooling}
\State \Return $\hat{y}$
\end{algorithmic}
\end{algorithm}

The network is trained for 300 epochs (Adam, $\eta = 0.003$, batch size 32,
early-stopping patience 20), backpropagating through the $\starG$ product
(Algorithm~\ref{alg:starg_product}) layer by layer. Equivariance is preserved
exactly, to floating-point precision, for \emph{any} trained weights: the $\starG$
product factors through the per-irrep block multiplication, and the bias is
constrained to the trivial isotypic component (constant along the group axis), so
it acts only on the invariant channel and commutes with the group action. A free
per-group-element bias, by contrast, would break equivariance as soon as training
moved it off its zero initialization; the group-constant (trivial-isotypic) bias is
what keeps the trained network exactly equivariant.

\subsection{ENN baselines: SchNet, e3nn, MACE}
\label{sec:enn_baselines}

The three ENN baselines used in the main-text pooled-accuracy comparison are not reimplemented from
scratch: each is a published reference implementation, executed on the
same train/val/test splits and the same seeds as the $\starG$ and MLP
baselines. The configuration we used is recorded explicitly so that the
comparison is reproducible.

\begin{itemize}
\item \textbf{SchNet}~\cite{schutt2017schnet}.
  Reference implementation: \texttt{schnetpack} v2.0.4 (pinned in the
  repository's \texttt{requirements.txt}).
  Configuration: $128$ atom-basis features, $6$ interaction blocks,
  $20$-Gaussian radial basis, cosine cutoff at $5.0$\,\AA, MSE loss with
  L1 monitoring, Adam $\eta = 5{\times}10^{-4}$, batch $64$, max $200$
  epochs, early-stop patience $20$. We run the standard
  \texttt{spk.datasets.QM9} loader with \texttt{remove\_uncharacterized
  =True} to match PyG's $130{,}831$-molecule subset, the
  \texttt{ASENeighborList(cutoff=5.0)} transform, and a per-target
  $z$-norm via \texttt{RemoveOffsets}/\texttt{AddOffsets}.
  Scalar targets only ($\mu$, $\alpha$, gap, ZPVE).

\item \textbf{e3nn-based SE(3)-equivariant model}~\cite{thomas2018tensor,geiger2022e3nn}.
  A compact equivariant message-passing network built directly from
  \texttt{e3nn} v0.5.4 primitives. Three layers, hidden irreps
  \texttt{32x0e+16x1o+8x2e}, edge spherical harmonics
  \texttt{1x0e+1x1o+1x2e}, RBF $16$ Gaussians on the $0\!-\!5$\,\AA{}
  cutoff, gated equivariant non-linearities, sum-pool over atoms,
  $\texttt{FullyConnectedTensorProduct}$ head with output irreps matched
  to target rank (\texttt{1x0e} for scalars, \texttt{1x1o} for
  $\boldsymbol{\mu}$ vector, \texttt{1x2e+1x0e} for $\alpha$ tensor).
  Adam $\eta = 5{\times}10^{-4}$, batch $32$, $200$ epochs, patience
  $20$. Used for tensor-rank-matched comparison rather than as the SOTA
  ENN target.

\item \begin{sloppypar}\textbf{MACE}~\cite{batatia2022mace}.
  Reference implementation: \texttt{mace-torch} v0.3.15 (pinned in
  \texttt{requirements.txt}). Configuration:
  \texttt{ScaleShiftMACE} with $r_{\max} = 5.0$\,\AA, $8$ Bessel radial
  features, $5$-th order polynomial cutoff, $\ell_{\max}=3$,
  correlation $3$, two interaction blocks,
  \texttt{RealAgnosticInteractionBlock} followed by
  \texttt{RealAgnosticResidualInteractionBlock}, hidden irreps
  \texttt{128x0e+128x1o}, MLP irreps \texttt{16x0e}, $5$ elements
  (H/C/N/O/F), per-target shift/scale $(\bar y, \mathrm{std}\,y)$.
  Optimizer: Adam (\texttt{amsgrad}), $\eta=10^{-3}$, batch $32$,
  \texttt{ReduceLROnPlateau} (factor $0.5$, patience $15$), max $200$
  epochs, early-stop patience $25$ on validation MSE. Total parameter
  count: $945{,}168$. Scalar targets only.\end{sloppypar}
\end{itemize}

\noindent
\emph{Implementation notes.} Three non-trivial integration adjustments
were required that may be useful to anyone reproducing the comparison.
(i)~In \texttt{schnetpack}'s \texttt{ModelOutput}, every metric must be
an \texttt{nn.Module} (lambda functions are silently rejected by
\texttt{nn.ModuleDict}); we use only \texttt{L1Loss} as the metric and
recompute RMSE/$R^2$ ourselves from saved test predictions.
(ii)~In \texttt{mace-torch} ${\geq}\,0.3.10$, \texttt{interaction\_cls}
and \texttt{interaction\_cls\_first} are mandatory, and
\texttt{hidden\_irreps}/\texttt{MLP\_irreps} must be \texttt{o3.Irreps}
objects, not strings; per-molecule \texttt{AtomicData} construction
requires a single shared \texttt{z\_table} for the full element set
(H/C/N/O/F) so that the per-graph \texttt{node\_attrs} have uniform
width when batched. (iii)~Loading \texttt{e3nn 0.4.4}'s pickled
\texttt{constants.pt} (transitively imported by \texttt{mace-torch})
fails under PyTorch ${\geq}\,2.6$'s default
\texttt{weights\_only=True}; we set
\texttt{TORCH\_FORCE\_NO\_WEIGHTS\_ONLY\_LOAD=1} for the comparison
runs only.

\section{Optimal Equivariant Compression of a Production Force Field}
\label{sec:si_mace_compress}

\subsection{Overview}
\label{sec:compression}

\begin{figure}[!t]
\centering
\includegraphics[width=\textwidth]{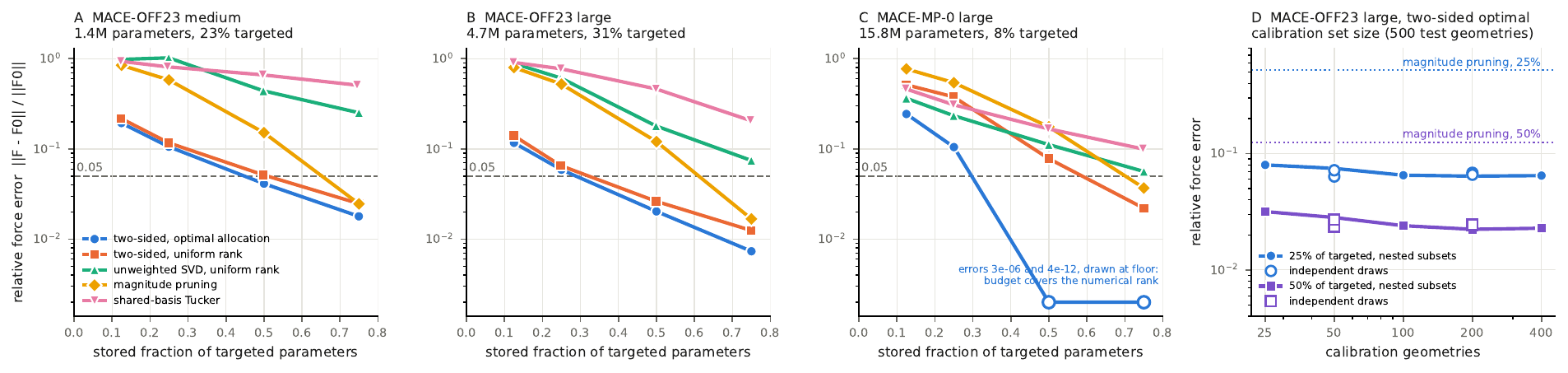}
\caption{\textbf{Compression of pretrained MACE force fields.} Relative force error $\|F-F_0\|/\|F_0\|$ against the uncompressed model on $300$ held-out geometries, plotted against the stored fraction of the targeted parameters, for the two-sided truncation of Corollary~\ref{cor:ey_module_weighted} with optimal and with uniform rank allocation, an unweighted truncation, magnitude pruning and a shared-basis Tucker factorization. (\textbf{A}) MACE-OFF23-medium. (\textbf{B}) MACE-OFF23-large. (\textbf{C}) MACE-MP-0-large, whose rank-deficient blocks make the optimal allocation lossless to within $3\times10^{-6}$ and $4\times10^{-12}$ at half and three quarters of the targeted parameters (drawn at the axis floor). The dashed line marks a relative force error of $0.05$. (\textbf{D}) Calibration set size for the two-sided optimal truncation of MACE-OFF23-large at $25\%$ and $50\%$ of the targeted parameters on $500$ held-out molecules, with magnitude pruning at the same budgets (dotted).}
\label{fig:compression}
\end{figure}

Corollary~\ref{cor:ey_module_weighted} governs the compression of a trained equivariant network layer by layer. Each linear channel-mixing map of MACE is a direct sum of blocks, one per angular momentum $\ell$, so a compression that acts on the blocks separately keeps the network exactly equivariant. When every retained rank costs the same storage, the theorem keeps the largest values of $d_\ell\sigma_j^2$ across all blocks, a weighting by the dimension of the irreducible representation that is absent in the abelian case. The theorem minimizes a weighted error of each layer, so the network errors below are measured. We compressed the interaction and product layers of three released models, which hold $31\%$ of the $4.7$ million parameters of MACE-OFF23-large, $23\%$ of MACE-OFF23-medium and $8\%$ of MACE-MP-0-large, with inner products estimated from $200$ calibration geometries and errors measured against the uncompressed model on $300$ others (Figure~\ref{fig:compression} and Section~\ref{sec:si_mace_compress}). Equivariance errors stayed at the uncompressed level (at most $6\times10^{-15}$ for MACE-OFF23 and $3.4\times10^{-14}$ for MACE-MP-0). For MACE-OFF23-large at half of the targeted parameters, the truncation changed the energy by $1.1$~meV per atom and the forces by $2.0\%$ ($25$~meV/\AA), against $12\%$ for magnitude pruning and $18\%$ for an unweighted truncation of the same blocks at equal parameter count. Setting every $d_\ell$ to $1$ raised the force error at every model and budget, by factors of $1.03$ to $1.36$. The same spectra show how the trained model uses its channels. In the last interaction layer of MACE-MP-0-large, $163$ of the $256$ output channels are zero and the energy gradient at the $\ell=1$, $2$ and $3$ outputs spans $1$, $1$ and $2$ dimensions, so the allocation of Corollary~\ref{cor:ey_module_weighted} reproduced its forces to $3\times10^{-6}$ at half of the targeted parameters, where a uniform rank per block left $7.8\%$. The minimized objective of the theorem is a certificate of each truncation, and at equal budget the truncation with the lower certificate had the lower force error in $32$ of $35$ pairs. Calibrated on QM9, the compressed MACE-OFF23-large had a mean absolute force error against DFT of $114$~meV/\AA{} on molecules with P, S, Cl, Br or I (uncompressed: $15$). With half of the calibration set drawn from molecules containing these elements the error was $28$~meV/\AA{} on those molecules and $22$~meV/\AA{} on molecules of H, C, N, O and F (uncompressed: $13$), so the calibration set must cover the chemistry of the application. In $10$~ps of constant-energy dynamics of four molecules the half-size model kept every bond, and its total-energy drift was at most $1.1\times10^{-3}$~meV per atom per ps (uncompressed: $0.7\times10^{-3}$). Half of the targeted parameters is $85\%$ of the stored parameters of MACE-OFF23-large, and the factored layers did not shorten a force evaluation on a GPU.

\paragraph{Methods.} MACE layers are compressed in place with the released weights. Input and output weights are activation and energy-gradient second moments from $200$ calibration geometries (QM9 molecules for MACE-OFF23, Materials Project crystals for MACE-MP-0), and errors are reported against the uncompressed model on disjoint geometries, with magnitude-pruning baselines at equal parameter count (Section~\ref{sec:si_mace_compress}).

Corollary~\ref{cor:ey_module_weighted} applies to the linear layers of a trained
equivariant network, and this section tests it on MACE-OFF23-small and then on three larger released models
(Section~\ref{sec:si_codemap}). The released
MACE-OFF23-small weights are used unmodified, in double precision, and every
error is measured against the uncompressed model, so no reference labels enter.

\paragraph{The layers are the commutant of a non-free module.} Each of the two
interaction layers \texttt{interactions.$k$.linear} maps
$96\times0e+96\times1o+96\times2e+96\times3o$ to itself. Its $36{,}864$ parameters
are four $96\times96$ matrices $W_\ell$, one per angular momentum, which is
$\mathrm{End}_G(V)\cong\bigoplus_{\ell\le3}M_{96}(\mathbb{R})$ for
$V=\bigoplus_\ell V_\ell\otimes\mathbb{R}^{96}$. The multiplicities are equal
across $\ell$ while the irreducible dimensions are $1,3,5,7$, so $V$ is not a free
$\starG$ module (Remark~\ref{rem:enn_multiplicity}). Any compression that acts on
the $W_\ell$ separately stays in $\mathrm{End}_G(V)$, and the measured
equivariance error of the compressed networks lies between $2.4\times10^{-15}$
and $5.3\times10^{-14}$ for every method and budget, against $5.2\times10^{-14}$
uncompressed.

\paragraph{Inner products and methods.} The input weight $\Sigma_\ell$ is the
Gram matrix of the activations entering block $\ell$, and the output weight
$\Gamma_\ell$ is the Gram matrix of the gradient of the network energy with
respect to the block's output, both accumulated over the $2\ell+1$ components
(which is the group average) on $200$ calibration geometries from QM9.
Evaluation uses $500$ disjoint QM9 geometries, with coordinates only. A rank-$r$
factor of a $96\times96$ block stores $192r$ numbers, and a budget is a fraction
of the $73{,}728$ targeted parameters. The methods are: the input-weighted
truncation of Corollary~\ref{cor:ey_module_weighted} ($\Gamma=I$) with the
globally optimal allocation, with a uniform rank per block and with ranks
proportional to block size; the two-sided truncation ($\Sigma$ and $\Gamma$)
with the optimal and the uniform allocation; the unweighted truncation in the
plain operator norm, whose weights are $2\ell+1$; a plain per-block singular value
decomposition with uniform rank; and magnitude pruning globally, per layer and
per block.

\begin{table}[!ht]
\centering
\caption{Compressing the interaction layers of MACE-OFF23-small: energy error
(meV per molecule) and force error (meV/\AA), mean absolute error against the
uncompressed model over $500$ held-out geometries, at fractions of the targeted
parameters. ``Optimal'' is the allocation of Corollary~\ref{cor:ey_module_weighted}.}
\label{tab:mace_compress}
\small
\begin{tabular}{@{}lccccc@{}}
\toprule
\textbf{Energy (meV)} & $0.125$ & $0.25$ & $0.375$ & $0.50$ & $0.75$ \\
\midrule
Two-sided, optimal allocation  & $642$ & $126$ & $78.6$ & $45.2$ & $22.4$ \\
Two-sided, uniform rank        & $684$ & $162$ & $72.7$ & $45.5$ & $21.0$ \\
Input-weighted, optimal        & $2{,}453$ & $952$ & $302$ & $213$ & $68.6$ \\
Input-weighted, uniform rank   & $1{,}278$ & $544$ & $315$ & $164$ & $56.2$ \\
Operator norm ($2\ell+1$ weights) & $75{,}351$ & $74{,}359$ & $63{,}218$ & $60{,}508$ & $53{,}810$ \\
Magnitude pruning, per block   & $16{,}111$ & $7{,}001$ & $2{,}672$ & $1{,}139$ & $82.2$ \\
\midrule
\textbf{Force (meV/\AA)} & & & & & \\
\midrule
Two-sided, optimal allocation  & $211$ & $115$ & $80.0$ & $58.4$ & $35.4$ \\
Two-sided, uniform rank        & $225$ & $116$ & $78.4$ & $57.6$ & $34.1$ \\
Input-weighted, optimal        & $439$ & $338$ & $169$ & $150$ & $96.0$ \\
Input-weighted, uniform rank   & $296$ & $217$ & $156$ & $139$ & $83.9$ \\
Operator norm ($2\ell+1$ weights) & $255$ & $422$ & $667$ & $1{,}138$ & $1{,}750$ \\
Magnitude pruning, per block   & $744$ & $479$ & $208$ & $116$ & $39.8$ \\
\bottomrule
\end{tabular}
\end{table}

\paragraph{Results.} The choice of invariant inner product decides the outcome.
The operator norm weights block $\ell$ by $2\ell+1$, which is the right weight for
approximating the layer as an operator. It spends the budget on high angular
momenta and deletes the scalar channel that carries the energy readout, and its
energy error stays above $50$~eV at every budget. The two-sided truncation lowers
the energy error $3.1$ to $7.5$-fold and the force error $2.1$ to $2.9$-fold
relative to the input weighting alone, and $25$ to $55$-fold (energy) and $2.0$
to $4.2$-fold (forces) relative to the best magnitude pruning up to half of the
parameters. At $75\%$ of the parameters the margins over the best pruning are $3.5$-fold on
energy and $1.1$-fold on forces. How the budget is divided among angular momenta
matters little. Under the two-sided inner product, the optimal and the uniform
allocation agree to within a factor between $0.93$ and $1.28$ on energy and $0.97$
and $1.07$ on forces, and under the input weighting alone the uniform allocation is
better at the smallest budgets ($0.52$ times the energy error and $0.67$ times the
force error at $12.5\%$). The optimality statement concerns the layer's own output
error in the chosen inner product, and the energy and forces pass through the
nonlinear remainder of the network, which explains why an optimal allocation for
the layer need not be optimal for the forces.

\paragraph{No usable compressed model at this size.} Every budget in
Table~\ref{tab:mace_compress} stores fewer numbers than the dense layers, and at
every one of them the relative force error $\|F-F_0\|/\|F_0\|$ is at least $0.28$.
Usable accuracy appears only at $75\%$ and $87.5\%$ of the full rank (force errors
$6.9$ and $2.5$~meV/\AA{} for the two-sided truncation), where factored storage
exceeds dense storage. The network is small, so this experiment demonstrates the
equivariance and the optimality certificate of block truncation without producing
a useful compressed force field.

\paragraph{Dynamics.} We ran $2$~ps of velocity-Verlet dynamics ($0.5$~fs step,
$300$~K initial velocities, zero linear and angular momentum) for benzene,
ethanol and a QM9 molecule with nine heavy atoms, with the uncompressed model and
with compressed models at $25\%$ and $50\%$ of the parameters and at $75\%$ of the
full rank. All $27$ compressed trajectories stay stable, with total-energy drifts
of $0.54$ to $4.86$~meV against $0.83$ to $4.28$~meV uncompressed. A control run
of the uncompressed model from positions displaced by $10^{-6}$~\AA{} stays within
$1.4\times10^{-5}$~\AA{} on benzene and ethanol and reaches $3.7\times10^{-3}$~\AA{}
on the nine-heavy-atom molecule, so the deviations from the uncompressed trajectory
come from the compression. They are $0.05$ to $0.19$~\AA{} on benzene and ethanol
at $25\%$ and $50\%$ of the parameters and up to $0.94$~\AA{} on the larger
molecule. At $75\%$ of the full rank the two-sided truncations track the reference
to $0.005$ to $0.02$~\AA{} at $0.5$~ps.

\paragraph{Larger models.} We repeated the experiment, in double precision and
with the released weights, on MACE-OFF23-medium ($1.43$ million parameters),
MACE-OFF23-large ($4.71$ million) and MACE-MP-0-large ($15.8$ million). The
targets are every linear layer inside the interaction and product blocks: the
input and output linear layers of both interactions and the linear layer after
each product basis, $14$ to $16$ blocks per model with multiplicities $128$, $224$
and $256$. They hold $327{,}680$ ($22.9\%$), $1{,}455{,}104$ ($30.9\%$) and
$1{,}310{,}720$ ($8.3\%$) of the parameters. The output layer of the second
interaction maps the concatenated coupling paths, so its blocks are rectangular
(for MACE-OFF23-large, $672$ to $1{,}120$ inputs and $224$ outputs). The species
embedding and the readouts are left unchanged. For MACE-OFF23 the calibration and
evaluation sets are $200$ and $300$ disjoint QM9 molecules with $12$ to $29$ atoms,
and for MACE-MP-0 they are $200$ and $300$ crystal structures from the Materials
Project elastic-tensor set with at most $64$ atoms per cell. All coordinates are
displaced by $0.05$~\AA{} with a fixed seed, and errors are against the
uncompressed model. A rank-$r$ factor of an $m_{\mathrm{in}}\times m_{\mathrm{out}}$
block stores $r(m_{\mathrm{in}}+m_{\mathrm{out}})$ numbers, and a budget is a
fraction of the targeted dense parameters. The whitened block of angular momentum
$\ell$ is $\Sigma_\ell^{1/2}A_\ell\Gamma_\ell^{1/2}$, with $A_\ell$ including the e3nn
path weight and with $\Sigma_\ell$ and $\Gamma_\ell$ the second moments per component
of $V_\ell$. Because the blocks differ in shape, the optimal allocation takes the
values $d_\ell\sigma_j^2/(m_{\mathrm{in}}+m_{\mathrm{out}})$ in decreasing order across
all targeted blocks, which is optimal up to the last rank unit. The methods are the two-sided truncation with this
allocation and with a uniform rank per block, an unweighted per-block truncation
with uniform rank, magnitude pruning with the same kept fraction in every block
(index storage not counted), and a shared-basis Tucker factorization in which
blocks of equal shape within one layer share their left and right factors and keep
separate square cores. The input-weighted truncation was not repeated on these
models. The input layer of the first interaction acts on the species embedding and
receives no energy gradient in a force evaluation. During calibration we added a
zero-valued position term to its output, which leaves energies and forces unchanged
and supplies the energy gradient, so this layer is weighted on both sides as well.

\begin{table}[!ht]
\centering
\caption{Compressing the interaction and product layers of three larger MACE
models: relative force error $\|F-F_0\|/\|F_0\|$ and, in parentheses, energy mean
absolute error (meV per atom) against the uncompressed model over $300$ held-out
geometries, at fractions of the targeted parameters. ``Optimal'' is the allocation
of Corollary~\ref{cor:ey_module_weighted} filled in the order of $d_\ell\sigma_j^2/(m_{\mathrm{in}}+m_{\mathrm{out}})$.}
\label{tab:mace_compress_v3}
\footnotesize
\begin{tabular}{@{}lcccc@{}}
\toprule
 & $0.125$ & $0.25$ & $0.50$ & $0.75$ \\
\midrule
\multicolumn{5}{@{}l}{\textit{MACE-OFF23-medium}} \\
Two-sided, optimal & $0.194$ ($23.7$) & $0.106$ ($7.00$) & $0.0412$ ($1.68$) & $0.0179$ ($0.750$) \\
Two-sided, uniform rank & $0.218$ ($26.3$) & $0.117$ ($8.49$) & $0.0516$ ($2.28$) & $0.0248$ ($1.24$) \\
Unweighted SVD, uniform rank & $0.982$ ($4{,}361$) & $1.02$ ($3{,}216$) & $0.441$ ($607$) & $0.252$ ($235$) \\
Magnitude pruning, per block & $0.846$ ($1{,}571$) & $0.582$ ($608$) & $0.151$ ($79$) & $0.0246$ ($16.3$) \\
Shared-basis Tucker & $0.936$ ($4{,}371$) & $0.814$ ($3{,}332$) & $0.662$ ($1{,}361$) & $0.51$ ($637$) \\
\midrule
\multicolumn{5}{@{}l}{\textit{MACE-OFF23-large}} \\
Two-sided, optimal & $0.117$ ($10.8$) & $0.0588$ ($3.71$) & $0.0203$ ($1.09$) & $0.00734$ ($0.274$) \\
Two-sided, uniform rank & $0.141$ ($14.9$) & $0.0653$ ($4.24$) & $0.0263$ ($1.51$) & $0.0125$ ($0.541$) \\
Unweighted SVD, uniform rank & $0.888$ ($2{,}469$) & $0.612$ ($706$) & $0.18$ ($153$) & $0.0745$ ($42.7$) \\
Magnitude pruning, per block & $0.799$ ($931$) & $0.523$ ($450$) & $0.121$ ($94.8$) & $0.0168$ ($10.6$) \\
Shared-basis Tucker & $0.911$ ($3{,}494$) & $0.775$ ($1{,}976$) & $0.463$ ($710$) & $0.207$ ($303$) \\
\midrule
\multicolumn{5}{@{}l}{\textit{MACE-MP-0-large}} \\
Two-sided, optimal & $0.244$ ($14.2$) & $0.105$ ($1.32$) & $3\times10^{-6}$ ($2\times10^{-5}$) & $4\times10^{-12}$ ($4\times10^{-10}$) \\
Two-sided, uniform rank & $0.514$ ($47.6$) & $0.378$ ($22.2$) & $0.0781$ ($3.33$) & $0.0220$ ($0.439$) \\
Unweighted SVD, uniform rank & $0.365$ ($194$) & $0.234$ ($75.1$) & $0.112$ ($29.5$) & $0.0563$ ($14.3$) \\
Magnitude pruning, per block & $0.771$ ($721$) & $0.542$ ($303$) & $0.176$ ($71.3$) & $0.0369$ ($11.1$) \\
Shared-basis Tucker & $0.464$ ($167$) & $0.31$ ($65.2$) & $0.167$ ($39$) & $0.1$ ($18.3$) \\
\bottomrule
\end{tabular}
\end{table}

\paragraph{Results on the larger models.} The equivariance error of every
compressed MACE-OFF23 model lies between $2.4\times10^{-15}$ and $6.0\times10^{-15}$,
against $2.9\times10^{-15}$ and $3.0\times10^{-15}$ uncompressed. The rotation test
of the sweep moves the atoms without the cell, which is not a symmetry of a periodic
crystal, and it gives errors of order $10^4$ for the uncompressed MACE-MP-0 as
well. A separate test that rotates cell and positions together on $10$ evaluation
structures gives at most $2.8\times10^{-14}$ uncompressed and at most
$3.4\times10^{-14}$ for the two-sided optimal truncation at $12.5\%$, $25\%$ and
$50\%$ of the targeted parameters. On both MACE-OFF23 models the two-sided
optimal truncation gives the lowest relative force error of the five methods at every
budget. Its relative force error is $7.6$ to $10.4$-fold below that of the unweighted
truncation on MACE-OFF23-large. Against magnitude pruning it is $39$ to $121$-fold
lower in energy and $2.3$ to $8.9$-fold lower in force on MACE-OFF23-large, and $22$
to $87$-fold and $1.4$ to $5.5$-fold on MACE-OFF23-medium. The uniform rank per block
gives $1.11$ to $1.97$ times the energy error of the optimal allocation on both
MACE-OFF23 models. On MACE-MP-0-large the allocation matters more. At a tolerance of
$10^{-10}$ relative to the largest singular value, $11$ of its $14$ targeted blocks
have numerical rank between $1$ and $183$ out of $256$, and factors of these ranks
store $48.4\%$ of the targeted parameters. A budget above this fraction admits a
truncation that discards only round-off. At $50\%$ and $75\%$ of the targeted
parameters the optimal allocation reproduces the uncompressed forces to relative
errors of $3\times10^{-6}$ and $4\times10^{-12}$. At $50\%$ it keeps fewer ranks
than the numerical rank in $12$ of the $14$ blocks, because the calibration second
moments are themselves rank deficient. A uniform rank per block leaves $7.8\%$ and
$2.2\%$ at the same budgets. At $12.5\%$ and $25\%$ the uniform rank gives $2.1$ and
$3.6$ times the force error of the optimal allocation. The forces on the displaced
crystals are smaller than on the molecules, so equal absolute errors give larger
relative errors on MACE-MP-0 ($10.1$~meV/\AA{} force error at $25\%$, against
$74.9$~meV/\AA{} for MACE-OFF23-large). The shared-basis
Tucker factorization is the least accurate method on MACE-OFF23-large at every
budget. The compression figure of the main text plots the force errors.

\paragraph{Calibration set size.} For MACE-OFF23-large we varied the calibration
set at $25\%$ and $50\%$ of the targeted parameters and evaluated on $500$ QM9
molecules disjoint from every calibration pool (Table~\ref{tab:mace_compress_calib}).
Nested subsets of one pool have $25$ to $400$ molecules, and two further draws of
$50$ and of $200$ molecules come from a disjoint pool. The relative force error at
$25\%$ falls from $0.080$ with $25$ molecules to $0.064$ with $200$ and is $0.065$
with $400$. Over three draws of $200$ molecules it ranges from $0.064$ to $0.068$ at
$25\%$ and from $0.022$ to $0.025$ at $50\%$. Over three draws of $50$ molecules the
ranges are $0.064$ to $0.075$ and $0.024$ to $0.028$. Every calibration set gives a
force error at least $6.5$-fold below magnitude pruning at $25\%$ and $3.9$-fold
below it at $50\%$.

\begin{table}[!ht]
\centering
\caption{Calibration set size for the two-sided optimal truncation of
MACE-OFF23-large: relative force error and energy mean absolute error (meV per
atom) against the uncompressed model over $500$ held-out QM9 molecules, at $25\%$
and $50\%$ of the targeted parameters. Magnitude pruning uses no calibration data.}
\label{tab:mace_compress_calib}
\small
\begin{tabular}{@{}lccccc@{}}
\toprule
 & & \multicolumn{2}{c}{$25\%$} & \multicolumn{2}{c}{$50\%$} \\
Calibration set & $n$ & force & energy & force & energy \\
\midrule
Nested subsets & $25$ & $0.0799$ & $4.12$ & $0.0317$ & $1.56$ \\
Nested subsets & $50$ & $0.0745$ & $3.89$ & $0.0282$ & $1.56$ \\
Nested subsets & $100$ & $0.0651$ & $3.87$ & $0.0240$ & $1.50$ \\
Nested subsets & $200$ & $0.0639$ & $3.61$ & $0.0223$ & $1.39$ \\
Nested subsets & $400$ & $0.0646$ & $3.79$ & $0.0228$ & $1.39$ \\
Independent draw 1 & $50$ & $0.0635$ & $3.76$ & $0.0235$ & $1.27$ \\
Independent draw 2 & $50$ & $0.0718$ & $4.63$ & $0.0271$ & $1.66$ \\
Independent draw 1 & $200$ & $0.0683$ & $3.82$ & $0.0247$ & $1.45$ \\
Independent draw 2 & $200$ & $0.0658$ & $3.90$ & $0.0243$ & $1.41$ \\
\midrule
Magnitude pruning & -- & $0.524$ & $450$ & $0.125$ & $95.7$ \\
\bottomrule
\end{tabular}
\end{table}

\paragraph{Storage and time.} Because the targeted layers hold part of each model,
a budget of $50\%$ of the targeted parameters stores $84.5\%$ of the parameters of
MACE-OFF23-large, $88.5\%$ of MACE-OFF23-medium and $95.9\%$ of MACE-MP-0-large,
and a budget of $25\%$ stores $76.8\%$, $82.8\%$ and $93.8\%$. We replaced the
targeted layers by factored modules, which agree with the dense layers to
$1.1\times10^{-13}$ at full rank, and timed the median of $30$ force evaluations.
The timings and the operation counts below were measured with factored modules of the
same shapes built from the ranks of an earlier allocation of the same truncation.
On a GPU the factored models took $0.44$ to $0.49$~s against $0.46$ to $0.49$~s
dense for a $20$-atom molecule with MACE-OFF23-large, $3.21$ to $3.24$~s against
$3.22$ to $3.25$~s for $89$ atoms, and $2.42$ to $2.43$~s against $2.42$ to
$2.45$~s for a $216$-atom crystal with MACE-MP-0-large. On a CPU the timings varied
in both directions, from $0.61$ times the first dense time (MACE-MP-0-large, $16$
atoms) to $1.49$ times (MACE-OFF23-medium, $89$ atoms), and two dense measurements of
the same system differed by up to $19\%$. A forward pass through the
targeted layers of MACE-OFF23-large costs $2.1\times10^{8}$ floating-point
operations for a $20$-atom molecule, and a full force evaluation was counted at
$4.0\times10^{10}$, so these layers are a small part of the compute.

\paragraph{Scope of the larger-model results.} The errors are static deviations
from the uncompressed model on displaced QM9 molecules and Materials Project
crystals. They add to the error of the model itself against the reference
calculations. The paragraphs below evaluate MACE-OFF23-large against the DFT test
set of MACE-OFF23 and in constant-energy dynamics. The calibration set was varied
only for MACE-OFF23-large, and the MACE-MP-0 test set was not used.

\paragraph{Relation to compression of language models.} Weighted low-rank
truncations of this kind underlie activation-aware and Fisher-weighted compression
of language models~\cite{hsu2022weighted,yuan2023asvd,wang2024svdllm}, and the
two-sided weight is the Kronecker-factored estimate of the expected squared change
in energy~\cite{martens2015kfac}. Corollary~\ref{cor:ey_module_weighted} adds the
statement that such a truncation, applied block by block in the irreducible
decomposition, preserves exact equivariance and is optimal in its inner product.

\paragraph{Molecular dynamics.} We ran constant-energy molecular dynamics with
MACE-OFF23-large and its two-sided compressions at $50\%$ and $25\%$ of the targeted
parameters, rebuilt with the calibration set and seeds of
Table~\ref{tab:mace_compress_v3}. The rebuilt models reproduce the held-out relative
force errors of that table ($0.0203$ and $0.0588$ over the same $300$ geometries). The
molecules were ethanol, benzene, the QM9 molecule with nine heavy atoms used for
MACE-OFF23-small (dsgdb9nsd\_062903, C$_7$H$_9$NO) and aspirin, each relaxed with the
uncompressed model. Every model started from the same positions and from
Maxwell--Boltzmann velocities at $300$~K drawn with fixed seeds, with center-of-mass
motion and rotation removed. We integrated $20{,}000$ velocity Verlet steps of $0.5$~fs
($10$~ps), with positions, velocities and energy sums in double precision and the
network evaluated in single precision. Mean kinetic temperatures were $170$ to $241$~K.
A $1$~ps control of the precision of the integrator, computed with the compressed models
of an earlier allocation of the same truncation, ran the network in double precision and
gave energy drifts within
$3\times10^{-4}$~meV per atom per ps of the single-precision runs and positions within
$1.5\times10^{-3}$~\AA{} after $1$~ps. A fourth run per molecule used the uncompressed
model with different initial velocities at the same total energy, to measure how much two
uncompressed trajectories differ.

All $16$ trajectories kept their bond topology for $10$~ps, and the largest bond stretch
was $5.6$ to $9.9\%$ of the equilibrium length for every model
(Table~\ref{tab:mace_md}). The magnitude of the total-energy drift was at most
$1.1\times10^{-3}$~meV per atom per ps for every model and molecule. Along the
uncompressed trajectories, sampled every $25$~fs, the forces of the $50\%$ model deviated
from the uncompressed forces by a mean absolute $6.8$ to $20.1$~meV/\AA{} (relative
$0.014$ to $0.045$), and those of the $25\%$ model by $26.2$ to $47.5$~meV/\AA{} (relative
$0.050$ to $0.095$). Mean bond lengths differed from the uncompressed run by at most
$1.9$~m\AA{} at $50\%$ and $2.1$~m\AA{} at $25\%$. At $25\%$ the O--H bonds of ethanol,
aspirin and the QM9 molecule were $1.3$ to $2.1$~m\AA{} longer. Harmonic frequencies at
each model's own minimum shifted by at most $16$ to $39$~cm$^{-1}$ at $50\%$ (mean
absolute shift $5$ to $7$~cm$^{-1}$) and by at most $64$ to $115$~cm$^{-1}$ at $25\%$
(mean $14$ to $21$~cm$^{-1}$). In six of the eight molecule and budget pairs the
largest shift was a red shift of a C--H or O--H stretch, up to $115$~cm$^{-1}$ for the
O--H stretch of aspirin at $25\%$. Peaks of the
vibrational density of states computed from the $10$~ps velocities moved by up to
$27$~cm$^{-1}$ between compressed and uncompressed runs and by the same amount between
the two uncompressed runs, so these trajectories do not resolve spectral shifts below
that level. In single trajectories the hydroxyl torsion of ethanol changed wells $27$
times with the uncompressed model and once with each compressed model.

\paragraph{Comparison with DFT.} We evaluated the three models on the released
MACE-OFF23 test set (SPICE subsets computed at the $\omega$B97M-D3(BJ)/def2-TZVPPD level)
with at most $200$ configurations per subset ($1{,}080$ configurations, fixed seed).
Kov\'acs et al.\ report mean absolute errors of MACE-OFF23-large of $0.38$ to
$0.98$~meV per atom and $6.6$ to $19.4$~meV/\AA{} over eight test subsets (their SI
Table S1). On the three subsets we used in full (QMugs, water, solvated amino acids), our
uncompressed evaluation gives the published values to the printed digits. On the $728$
test configurations made only of H, C, N, O and F, the elements of the QM9 calibration
set, the mean absolute error of the $50\%$ model was $1.67$~meV per atom and
$23.1$~meV/\AA{}, against $0.55$~meV per atom and $12.8$~meV/\AA{} for the uncompressed
model. The $25\%$ model gave $5.86$~meV per atom and $52.6$~meV/\AA{} on the same
configurations. On the $352$ configurations that contain P, S, Cl, Br or I, the $50\%$
model gave $72.4$~meV per atom and $114$~meV/\AA{} and the $25\%$ model $79.1$~meV per atom
and $142$~meV/\AA{}, against $0.72$~meV per atom and $15.0$~meV/\AA{} uncompressed. A
calibration set drawn from QM9 alone thus yields compressed models that fail on
molecules with these elements. Every molecule in the dynamics above contains only H, C,
N and O.

\paragraph{Calibration coverage.} We repeated the truncation with two further calibration
sets of $200$ configurations each. The first holds $200$ configurations of the MACE-OFF23
test file that are absent from the $1{,}080$ evaluated ones, allocated to the subsets in
proportion to their size (PubChem $137$, DES370K dimers $56$, dipeptides $4$, DES370K
monomers $3$). QMugs, water and the solvated amino acids have no such configurations
because the evaluation uses them in full. Of these $200$ configurations, $51$ contain S,
$23$ Cl, $13$ Br, $6$ P and $6$ I. The second set combines the first $100$ QM9 molecules of
the original calibration set with $100$ further test-pool configurations. The equivariance
error of all six compressed models stayed between $1.6\times10^{-14}$ and
$1.8\times10^{-14}$ (uncompressed: $1.9\times10^{-14}$). On the $352$ test configurations
that contain P, S, Cl, Br or I, the mean absolute force error of the $50\%$ model against
DFT was $30.2$~meV/\AA{} with the test-pool calibration and $27.8$~meV/\AA{} with the mixed
calibration, against $114$~meV/\AA{} with the QM9 calibration and $15.0$~meV/\AA{}
uncompressed. Its energy error on these configurations was $1.8$ and $2.2$~meV per atom,
against $72.4$~meV per atom with the QM9 calibration. On the $728$ configurations made of H,
C, N, O and F, the $50\%$ model gave $26.1$~meV/\AA{} with the test-pool calibration,
$22.1$~meV/\AA{} with the mixed calibration and $23.1$~meV/\AA{} with the QM9 calibration.
On the $300$ held-out QM9 geometries, the relative force deviation of the $50\%$ model from
the uncompressed model rose from $0.0203$ with the QM9 calibration to $0.0307$ with the
mixed calibration and $0.0579$ with the test-pool calibration. At $25\%$ the two
recalibrations left a mean absolute force error against DFT of $52.6$ to $63.2$~meV/\AA{}
on both groups of configurations. The composition of the calibration set thus shifts accuracy between
chemistries at a fixed budget. We repeated the $10$~ps dynamics for the $50\%$ model with
the mixed calibration. All four trajectories kept their bond topology, the drift magnitude
was at most $1.6\times10^{-3}$~meV per atom per ps, and the largest bond stretch was $6.4$ to
$9.8\%$ of the equilibrium length. Along the uncompressed trajectories its forces deviated
from the uncompressed forces by a mean absolute $7.2$ to $29.7$~meV/\AA{}.

\begin{table}[!ht]
\centering
\caption{Compressed MACE-OFF23-large in molecular dynamics and against DFT. Top:
$10$~ps of constant-energy dynamics from identical initial conditions (velocity Verlet,
$0.5$~fs, network in single precision). Drift is the slope of a linear fit of the total
energy, in $10^{-3}$~meV per atom per ps. Force deviation is the mean absolute
difference from the uncompressed forces on $401$ frames of the uncompressed trajectory
(meV/\AA{}). Frequency shift is the largest absolute change of a harmonic frequency at
each model's own minimum (cm$^{-1}$). Bottom: mean absolute error against the
MACE-OFF23 test set (SPICE level of theory), energy in meV per atom and force in
meV/\AA{}, for three calibration sets. The test-pool configurations come from the
MACE-OFF23 test file and are disjoint from the $1{,}080$ evaluated ones. The dynamics
rows use the QM9 calibration.}
\label{tab:mace_md}
\small
\begin{tabular}{@{}lccccccc@{}}
\toprule
 & & \multicolumn{3}{c}{Drift} & \multicolumn{2}{c}{Force deviation} & Frequency shift \\
Molecule & Atoms & dense & $50\%$ & $25\%$ & $50\%$ & $25\%$ & $50\%$ / $25\%$ \\
\midrule
Ethanol & $9$ & $0.38$ & $0.52$ & $0.55$ & $6.8$ & $27.9$ & $16$ / $64$ \\
Benzene & $12$ & $0.60$ & $1.09$ & $0.60$ & $7.9$ & $26.2$ & $28$ / $74$ \\
QM9 C$_7$H$_9$NO & $18$ & $0.66$ & $0.35$ & $0.13$ & $20.1$ & $47.5$ & $24$ / $76$ \\
Aspirin & $21$ & $-0.13$ & $-0.23$ & $-0.49$ & $19.6$ & $41.8$ & $39$ / $115$ \\
\midrule
 & Configurations & \multicolumn{2}{c}{dense} & \multicolumn{2}{c}{$50\%$} & \multicolumn{2}{c}{$25\%$} \\
 & & energy & force & energy & force & energy & force \\
\midrule
\multicolumn{8}{@{}l}{\textit{Calibration on 200 QM9 molecules}} \\
H, C, N, O, F only & $728$ & $0.55$ & $12.8$ & $1.67$ & $23.1$ & $5.86$ & $52.6$ \\
Contains P, S, Cl, Br, I & $352$ & $0.72$ & $15.0$ & $72.4$ & $114$ & $79.1$ & $142$ \\
\multicolumn{8}{@{}l}{\textit{Calibration on 100 QM9 molecules and 100 test-pool configurations}} \\
H, C, N, O, F only & $728$ & $0.55$ & $12.8$ & $0.80$ & $22.1$ & $3.25$ & $52.6$ \\
Contains P, S, Cl, Br, I & $352$ & $0.72$ & $15.0$ & $2.22$ & $27.8$ & $8.19$ & $58.6$ \\
\multicolumn{8}{@{}l}{\textit{Calibration on 200 test-pool configurations}} \\
H, C, N, O, F only & $728$ & $0.55$ & $12.8$ & $0.85$ & $26.1$ & $3.00$ & $61.5$ \\
Contains P, S, Cl, Br, I & $352$ & $0.72$ & $15.0$ & $1.81$ & $30.2$ & $6.91$ & $63.2$ \\
\bottomrule
\end{tabular}
\end{table}

\paragraph{Irreducible blocks of the trained models.} On MACE-MP-0-large the blocks with
$\ell\geq1$ carry at most $0.09\%$ of their layer's weighted energy $d_\ell\sum_j\sigma_j^2$.
In its last interaction layer, $163$ of the $256$ output channels have zero weights at every
$\ell$, and $1$, $2$ and $3$ channels carry the $\ell=1$, $2$ and $3$ outputs. The energy
gradient at these outputs spans $46$, $1$, $1$ and $2$ dimensions for $\ell=0$ to $3$. The
product layer that follows contracts each channel separately to scalars, so a channel whose
$\ell\geq1$ features vanish receives no energy gradient on them. On MACE-OFF23-large the
$\ell=3$ block carries $27\%$ and $42\%$ of the weighted energy of the two interaction layers.
On both model families the input of the first block has rank equal to the number of elements
in the calibration set ($5$ for QM9 and $55$ for the crystals), and the energy gradient at the
last product layer has rank at most $16$, the width of the nonlinear readout. Setting every
$d_\ell$ to $1$ in the allocation moves stored parameters from $\ell=2$ and $3$ to $\ell=0$
(MACE-OFF23-large at half of the targeted parameters: $\ell=0$ from $16.8\%$ to $23.0\%$ of
the stored numbers and $\ell=3$ from $17.5\%$ to $13.3\%$). It raised the relative force error
by factors of $1.03$ to $1.36$ and the energy error at all $10$ model and budget pairs whose
force error exceeds $10^{-3}$ (Fig.~\ref{fig:si_mace_anatomy}A). The minimized objective of
Corollary~\ref{cor:ey_module_weighted}, summed over the targeted layers and divided by the
total weighted energy, serves as a certificate of each truncation. Over $36$ truncations of
three models (the optimal allocation, the same rule with every $d_\ell=1$ and the uniform rank,
at four budgets), its Spearman correlation with the relative force error is $0.84$, and $0.96$
over the $24$ truncations of the two MACE-OFF23 models (Fig.~\ref{fig:si_mace_anatomy}B). Its
absolute value does not transfer between models: MACE-MP-0-large at a quarter of the targeted
parameters has certificate $3\times10^{-8}$ at a force error of $0.105$, and MACE-OFF23-large
at half has $1\times10^{-6}$ at $0.020$. At equal model and budget, the allocation with the
lower certificate had the lower force error in $32$ of $35$ pairs and the lower energy error in
$30$ of $35$. Every pair misordered in force had certificates within a factor $1.4$ of each
other. The certificate does not compare truncations that weight the first block differently.
On MACE-MP-0-large, weighting that block on both sides lowered the certificate of the uniform
rank by factors of $1.9$ and $2.6$ at $12.5\%$ and $25\%$ and raised its force error by factors
of $1.30$ and $1.50$. Seven of the $300$ held-out crystals contain an element (B, Cl, Ge or Ag)
that is absent from the calibration set, and the certificate is computed on the calibration
geometries. Factors at the numerical ranks of MACE-MP-0-large store $48.4\%$ of the targeted
parameters and leave a certificate of $3\times10^{-31}$. At half of the targeted parameters the
allocation therefore reaches a certificate of $2\times10^{-19}$ and a force error of
$3\times10^{-6}$ on the held-out crystals.

\begin{figure}[!ht]
\centering
\includegraphics[width=\textwidth]{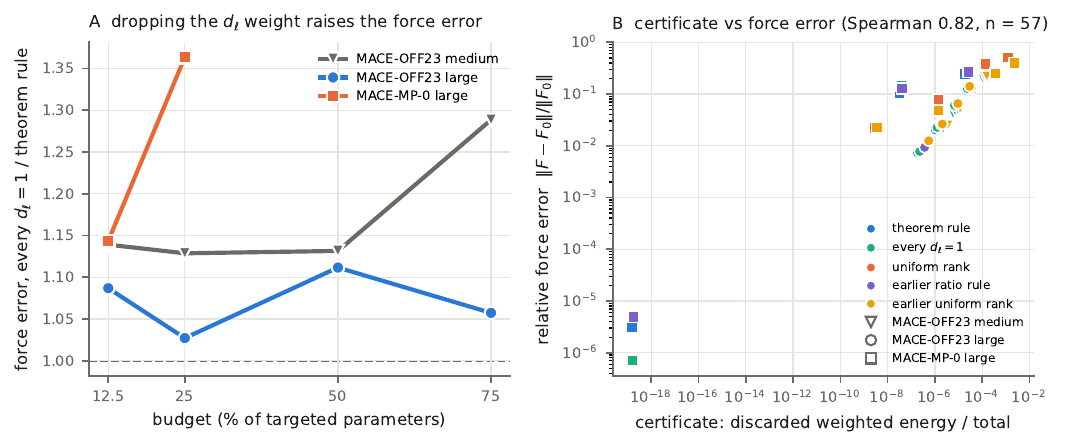}
\caption{Irreducible blocks of three trained MACE models. (A) Relative force error with every
$d_\ell$ set to $1$ in the allocation, divided by the relative force error of the allocation
of Corollary~\ref{cor:ey_module_weighted}, per model and budget. (B) Certificate (discarded weighted
energy relative to the total) against relative force error for every truncation with force
error above $10^{-9}$. Panel B also shows the two allocations of an earlier ranking of the
same truncation, and the Spearman correlation in its title is computed over all $57$ points.}
\label{fig:si_mace_anatomy}
\end{figure}

\section{Angular Resolution: Octahedral versus Icosahedral}
\label{sec:ico_resolution}

The main text notes that the quantitative accuracy ceiling of a finite point
group is set by its angular resolution, and that larger polyhedral groups lift
that ceiling.  We demonstrate this directly in a controlled setting that does
not depend on rare molecular symmetry.
We draw random functions band-limited to angular degree $\ell \le \ell_{\max}$
on the rotation group, observe each function only on the finite orbit of a
generic base direction under a point group $G$ (24 directions for the chiral
octahedral group $O$, 60 for the chiral icosahedral group $I$), fit a
harmonic model on that orbit, and measure the test $R^2$ on a dense set of
$4{,}000$ held-out rotations.  A test $R^2$ below one is aliasing: the finite
orbit cannot resolve the continuous signal.

\begin{table}[!ht]
\centering
\caption{Angular resolution by group: test $R^2$ on $4{,}000$ held-out
rotations for random signals band-limited to degree $\ell_{\max}$ (mean over
300 signals; seed 0).  The chiral octahedral group ($24$ frames) resolves
signals through $\ell=3$ but aliases $\ell=4$; the chiral icosahedral group
($60$ frames) resolves $\ell=4$ exactly.  The $\ell=4$ gap is stable across
seeds (per-seed octahedral $R^2$ = 0.908 / 0.943 / 0.896 for seeds 0/1/2,
each a mean over 300 signals).}
\label{tab:ico_resolution}
\begin{tabular}{@{}ccc@{}}
\toprule
\textbf{Signal degree $\ell_{\max}$} & \textbf{Octahedral $O$ (24 frames)} & \textbf{Icosahedral $I$ (60 frames)} \\
\midrule
1 & $1.000$ & $1.000$ \\
2 & $1.000$ & $1.000$ \\
3 & $1.000$ & $1.000$ \\
4 & $0.908$ & $1.000$ \\
\bottomrule
\end{tabular}
\end{table}

The pattern matches representation theory.  The octahedral group's largest
irreducible representation has dimension $3$ ($T_1$, $T_2$), so a generic
$24$-point orbit cannot linearly separate the nine degree-$4$ harmonics from
lower degrees and aliases them; the icosahedral group has a dimension-$5$
irreducible representation ($H$), and its $60$-point orbit resolves degree $4$
completely.  This confirms, quantitatively, that moving from $O$ to $I$ lifts
the angular-resolution ceiling, and it exercises the machine-validated
icosahedral $\starG$ algebra (five irreps $A, T_1, T_2, G, H$ with $\sum_\rho d_\rho^2 = 60$, group-Fourier
matrix unitary to $2 \times 10^{-15}$, $\starG$-SVD reconstruction to
$8 \times 10^{-6}$).

\section{Error Metrics in Physical Units and Literature Anchoring}
\label{sec:mae_anchor}

The main text reports $R^2$ throughout because the contribution of the
$\starG$ rows is representational (closed form, exact equivariance, per-irrep
attribution) rather than benchmark accuracy, and $R^2$ makes the per-irrep
attribution comparable across targets with different units.  For
cross-referencing with the QM9 literature, which conventionally reports mean
absolute error (MAE) in physical units, Table~\ref{tab:mae_anchor} restates
the full-QM9 results of the main text as MAE, and places our self-trained
SchNet baseline next to the published SchNet
figures~\cite{schutt2018schnet}.  Two facts are visible.  First, our
self-trained baselines are sound: our SchNet is comparable to the published
SchNet, improving on the reported MAE for gap and $\alpha$ and lying within a
small factor for $\mu$ and ZPVE, so the ENN baselines in this paper are
trained to community standard.  Second, the $\starG$ absolute errors sit
far above the message-passing state of the art, as the pooled-$R^2$
concession in the main text states; the $\starG$ contribution is the
closed-form equivariant representation at roughly a hundred parameters, not
benchmark accuracy.

\begin{table}[!ht]
\centering
\caption{Full QM9 test MAE in physical units (mean $\pm$ std; same runs as the
main-text $R^2$ table: 5 seeds for the $\starG$ and SchNet rows, 3 for MACE).
Published SchNet values from
\cite{schutt2018schnet}.}
\label{tab:mae_anchor}
\resizebox{\textwidth}{!}{%
\begin{tabular}{@{}lcccc@{}}
\toprule
\textbf{Method} & \textbf{gap (eV)} & \textbf{$\alpha$ ($a_0^3$)} & \textbf{$\mu$ (D)} & \textbf{ZPVE (eV)} \\
\midrule
$\starG$-SVD + Ridge & $0.746 \pm 0.001$ & $1.746 \pm 0.045$ & $0.782 \pm 0.006$ & $0.029 \pm 0.000$ \\
$\starG$ neural      & $0.671 \pm 0.004$ & $1.333 \pm 0.059$ & $0.690 \pm 0.013$ & $0.027 \pm 0.002$ \\
SchNet (this work)   & $0.055 \pm 0.002$ & $0.127 \pm 0.069$ & $0.060 \pm 0.041$ & $0.0019 \pm 0.0002$ \\
SchNet (published)~\cite{schutt2018schnet} & $0.063$ & $0.235$ & $0.033$ & $0.0017$ \\
MACE (this work)     & $0.116 \pm 0.005$ & $0.317 \pm 0.004$ & $0.060 \pm 0.003$ & $0.0097 \pm 0.0004$ \\
\bottomrule
\end{tabular}}
\end{table}

\section{Molecular Accuracy Benchmarks on QM9}
\label{sec:qm9_accuracy}

The Discussion of the main text concedes pooled predictive accuracy on QM9 to
graph equivariant networks and confines the details to this section. It reports
the full-QM9 comparison against graph equivariant networks and matched-input MLP
baselines, the within-isomer audit that attributes most of the graph-network
margin to an input-representation gap, the invariance error of the
orbit-augmented baseline, and the data-scarce sample-size sweep. None of these
numbers carries a capability claim of the paper. They delimit its scope.

\paragraph{Full-QM9 head-to-head and matched-input comparison.}
Table~\ref{tab:fullqm9} reports test $R^2$ on the four scalar QM9 targets
(130{,}831 characterized molecules; 91{,}581 / 19{,}624 / 19{,}626
train / validation / test split; 5 seeds, 3 for MACE). The five molecule-level
methods (both $\starG$ variants and the three MLP baselines) consume the
\emph{identical} $48$-row angular feature tensor, carrying angular moments,
heavy-atom rows, atom-pair Coulomb couplings, and distance-distribution rows
under cyclic Z$_{12}$, and every MLP is trained with a modern recipe (target
standardization, validation-tuned weight decay, learning-rate schedule;
Methods of the main text), so that comparison is matched-input.
On it, plain standard and invariant MLPs exceed both $\starG$ variants on
gap, $\alpha$, and $\mu$ and tie on ZPVE, so on shared input a well-trained
unconstrained network is the stronger pooled predictor. The orbit-augmented MLP,
which sees each training molecule in all $12$ group frames, is the most accurate
molecule-level model on the gap ($R^2=0.717$), and it remains only approximately
invariant (Section~\ref{sec:aug_invariance}). SchNet and MACE
consume the full atomic graph, a different-input comparison that sets the
pooled ceiling. Parameter counts compare only the trained predictor; the
$\starG$ features are a fixed, untrained preprocessing step.

\begin{table}[!ht]
\centering
\caption{Full QM9, four scalar targets, mean test $R^2 \pm$ std
(130{,}831 molecules; 91{,}581 / 19{,}624 / 19{,}626 train / val
/ test split; $\starG$, MLP, and SchNet rows 5 seeds, MACE row 3 seeds).
The five molecule-level methods consume the identical $48$-row angular
feature tensor (matched input); SchNet and MACE consume the full atomic
graph and set the pooled-$R^2$ ceiling, motivating the within-isomer audit in
Table~\ref{tab:isomer_audit}. The invariance error of the augmented MLP is
measured in Section~\ref{sec:aug_invariance}.
$^{\dagger}$SchNet uses the PyTorch Geometric reference implementation
(\texttt{torch\_geometric.nn.models.SchNet}) with $128$ hidden channels,
$6$ interactions, $50$-Gaussian RBF, cutoff $10$\,\AA, $455{,}809$
trainable parameters. Full configurations for all three ENN baselines are in
Section~\ref{sec:enn_baselines}.}
\label{tab:fullqm9}
\begin{tabular}{@{}lrcccc@{}}
\toprule
\textbf{Method} & \textbf{Params} & \textbf{gap} & \textbf{alpha} & \textbf{mu} & \textbf{ZPVE} \\
\midrule
$\starG$-SVD + Ridge & $144$    & $0.481 \!\pm\! 0.002$ & $0.907 \!\pm\! 0.003$ & $0.472 \!\pm\! 0.006$ & $0.998 \!\pm\! 0.000$ \\
$\starG$ neural       & $62{,}625$ & $0.571 \!\pm\! 0.006$ & $0.945 \!\pm\! 0.005$ & $0.563 \!\pm\! 0.016$ & $0.998 \!\pm\! 0.002$ \\
MLP standard         & 5{,}249  & $0.687 \!\pm\! 0.008$ & $0.973 \!\pm\! 0.001$ & $0.672 \!\pm\! 0.007$ & $0.999 \!\pm\! 0.000$ \\
MLP invariant        & 14{,}465  & $0.686 \!\pm\! 0.004$ & $0.972 \!\pm\! 0.001$ & $0.642 \!\pm\! 0.009$ & $0.999 \!\pm\! 0.000$ \\
MLP augmented        & 5{,}249  & $0.717 \!\pm\! 0.003$ & $0.974 \!\pm\! 0.001$ & $0.675 \!\pm\! 0.011$ & $0.999 \!\pm\! 0.000$ \\
SchNet$^{\dagger}$   & $455{,}809$  & $0.996 \!\pm\! 0.000$ & $0.999 \!\pm\! 0.001$ & $0.995 \!\pm\! 0.005$ & $1.000 \!\pm\! 0.000$ \\
MACE                 & $945{,}168$ & $0.985 \!\pm\! 0.001$ & $0.997 \!\pm\! 0.000$ & $0.995 \!\pm\! 0.001$ & $1.000 \!\pm\! 0.000$ \\
\bottomrule
\end{tabular}
\end{table}

\paragraph{Within-isomer audit.}
On QM9 the HOMO--LUMO gap, the polarizability $\alpha$, and the dipole
magnitude $\mu$ are strongly size-correlated, so a model that captures gross
size and composition already explains the bulk of the pooled variance. The
within-isomer $R^2$ (restricted to molecular formulas with $\geq 5$
constitutional isomers, sample-weighted across formulas, same test
predictions) removes that signal (Table~\ref{tab:isomer_audit}). Every method
on the molecule-level summary caps between within-isomer $R^2 \approx 0.35$
(the MLPs trained in one frame) and $0.42$ (the orbit-augmented MLP), with the
closed-form linear $\starG$-SVD at $0.01$,
whereas SchNet ($0.991$) and MACE ($0.963$) consume the full atomic graph
carrying the bond topology that distinguishes constitutional isomers. Most of
the pooled-$R^2$ advantage of the graph networks is therefore an
input-representation gap rather than a purely algorithmic one. The per-atom
$\starG$ descriptor of Section~\ref{sec:representation} reaches within-isomer
$R^2 = 0.785$ with a symmetry-blind learner, which confirms that the ceiling is a
property of the molecule-level featurization.

\begin{table}[!ht]
\centering
\caption{Pooled vs.\ within-isomer test $R^2$ for the HOMO--LUMO gap
target. Within-isomer $R^2$ is computed over molecular formulas with
$\geq 5$ constitutional isomers and reported as the sample-weighted
mean across formulas. The collapse from pooled to within-isomer
quantifies how much of the headline $R^2$ is size-prediction signal.}
\label{tab:isomer_audit}
\begin{tabular}{@{}lcc@{}}
\toprule
\textbf{Method} & \textbf{Pooled $R^2$} & \textbf{Within-isomer $R^2$} \\
\midrule
$\starG$-SVD + Ridge & $0.481$  & $\phantom{-}0.011$ \\
$\starG$ neural       & $0.571$  & $\phantom{-}0.101$ \\
MLP standard         & $0.687$  & $\phantom{-}0.350$ \\
MLP invariant        & $0.686$  & $\phantom{-}0.345$ \\
MLP augmented        & $0.717$  & $\phantom{-}0.422$ \\
SchNet               & $0.996$ & $\phantom{-}0.991$ \\
MACE                 & $0.985$ & $\phantom{-}0.963$ \\
\bottomrule
\end{tabular}
\medskip

\footnotesize
\noindent\emph{Within-isomer $R^2$ averaged over $\sim$230 molecular
formulas (per seed, $\geq 5$ constitutional isomers each); covered
test molecules range from $\sim$10{,}500 (PyG-split SchNet) to
$\sim$19{,}250 ($\starG$ / MLP / MACE). The sample-weighted mean
across formulas is robust to the per-method test-set partitioning,
so the numerical comparison is unaffected. }
\end{table}

\paragraph{Data-scarce sample-size sweep.}
With $\bbZ_{12}$ rotational featurization and the same fair recipe for every
method, a sweep from $100$ to $1{,}000$ training molecules (5 seeds) shows a
regime crossover (Table~\ref{tab:qm9}, Figure~\ref{fig:learning_curves}). At
$100$ and $178$ molecules the closed-form $\starG$-SVD ridge regressor leads every
trained baseline ($R^2 = 0.415$ and $0.447$ at $64$ parameters, against $0.267$
and $0.421$ for the orbit-augmented MLP and $0.199$ and $0.231$ for the standard
MLP). The orbit-augmented MLP overtakes by $316$ molecules ($0.470$ against
$0.422$), and at $1{,}000$ molecules it leads every method ($0.607$, against
$0.558$ for the standard MLP and $0.493$ for $\starG$-SVD). The crossover
therefore falls between $178$ and $316$ training molecules, and the advantage of
the closed-form regressor is confined to the scarcest data.

\begin{table}[!ht]
\centering
\caption{QM9 HOMO--LUMO gap ($\bbZ_{12}$), test $R^2$ (mean $\pm$ std over 5
seeds) at the smallest and largest sample sizes of the sweep, with the same
fair training recipe for all methods. $\starG$-SVD leads at $n=100$, and the
orbit-augmented MLP leads at $n=1{,}000$.}
\label{tab:qm9}
\begin{tabular}{@{}lccc@{}}
\toprule
\textbf{Method} & \textbf{Params} & \textbf{$R^2$ ($n{=}100$)} & \textbf{$R^2$ ($n{=}1{,}000$)} \\
\midrule
$\starG$-SVD + Ridge & \textbf{64}    & $\mathbf{0.415 \pm 0.187}$ & $0.493 \pm 0.095$ \\
$\starG$ neural       & $36{,}513$ & $0.074 \pm 0.196$          & $0.479 \pm 0.108$ \\
MLP standard         & $5{,}249$  & $0.199 \pm 0.176$          & $0.558 \pm 0.110$ \\
MLP invariant        & $5{,}761$  & $-0.041 \pm 0.367$         & $0.505 \pm 0.094$ \\
MLP augmented        & $5{,}249$  & $0.267 \pm 0.223$          & $\mathbf{0.607 \pm 0.086}$ \\
\bottomrule
\end{tabular}
\end{table}

\begin{figure}[!ht]
\centering
\includegraphics[width=0.72\textwidth]{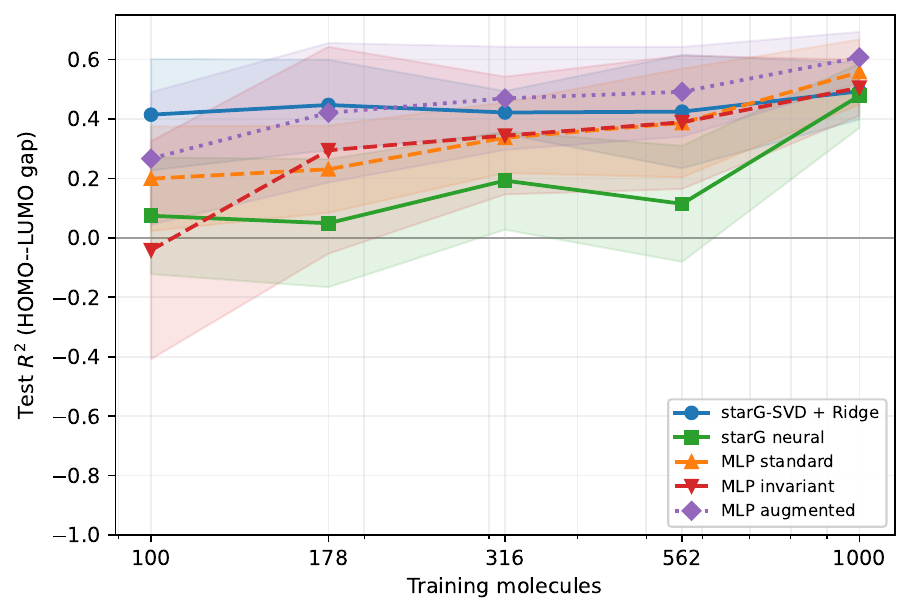}
\caption{\textbf{Data-scarce learning curves on the QM9 HOMO--LUMO gap
($\bbZ_{12}$), all methods trained with the same fair recipe (5 seeds; bands
$\pm 1$ s.d.).} The closed-form $\starG$-SVD ridge regressor ($64$ parameters)
is the most sample-efficient method at $100$ and $178$ molecules, where the
unconstrained MLPs overfit their thousands of parameters. The orbit-augmented MLP
overtakes it by $316$ molecules and leads every method at $1{,}000$, and the
standard and invariant MLPs overtake it by $1{,}000$. The $\starG$ advantage is
confined to the scarcest data.}
\label{fig:learning_curves}
\end{figure}

\begin{figure}[!ht]
\centering
\includegraphics[width=\textwidth]{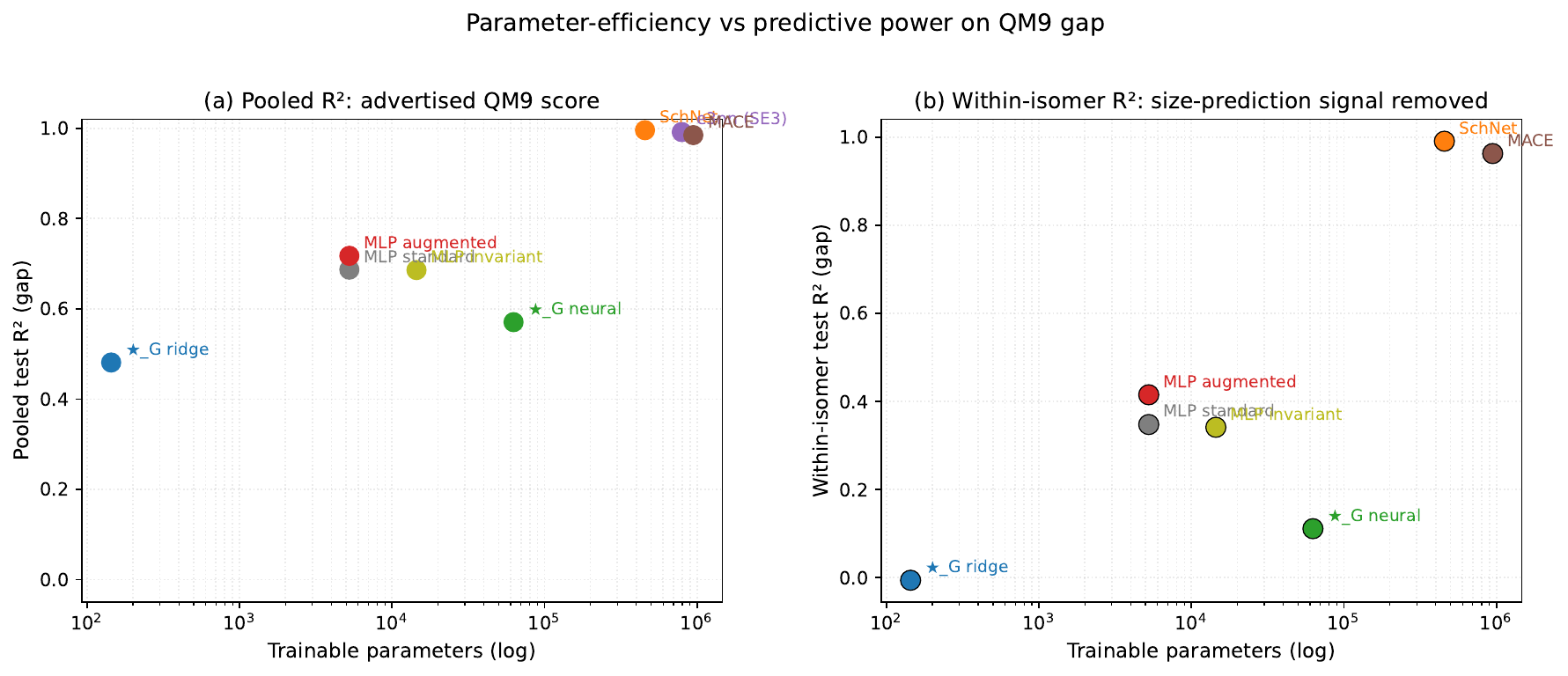}
\caption{\textbf{Parameter-efficiency vs predictive power on QM9 HOMO--LUMO
gap.}
(\textbf{a}) Pooled test $R^2$ vs trainable parameters (error bars $\pm$std).
MACE occupies the upper-right ($R^2 = 0.985$ at $945{,}168$ parameters);
$\starG$-SVD~+~Ridge occupies the far left at $144$ parameters
($R^2 = 0.481$), roughly $6{,}600\times$ fewer parameters than MACE, and the
matched-input MLPs sit above it ($R^2 \approx 0.69$ at a few thousand
parameters, and $0.717$ for the orbit-augmented MLP).
(\textbf{b}) Within-isomer test $R^2$ on the same axes (formulas with
$\geq 5$ constitutional isomers, sample-weighted mean). Every
molecule-level-summary method is capped well below the graph networks:
$\starG$-SVD at $\approx 0.01$, $\starG$ neural at $\approx 0.10$, and the
nonlinear MLPs at $0.35$ to $0.42$, versus SchNet ($\approx 0.99$) and MACE
($\approx 0.96$). The large offset to the ENNs is the input-information gap
(bond topology in the full graph); the spread within the molecule-summary
methods reflects model nonlinearity.}
\label{fig:pareto}
\end{figure}

\subsection{Pooled accuracy, and symmetry in the representation}
\label{sec:representation}

Pooled predictive accuracy on standard molecular benchmarks is not among the
claims of the main text. On full QM9 (130{,}831 molecules, a stress test whose
discriminating signal is largely bond-topological), graph-based equivariant
networks lead on pooled $R^2$ (SchNet $0.996$ against $\starG$-SVD $0.481$ on the
HOMO--LUMO gap). Our SchNet baseline is within $13\%$ of the published SchNet
errors on the gap and ZPVE, better on $\alpha$, and within a factor of two on
$\mu$ (Table~\ref{tab:mae_anchor}). The within-isomer audit
(Table~\ref{tab:isomer_audit}) attributes most of that margin to the
molecule-level input summary, and a per-atom $\starG$ descriptor handed to a
symmetry-blind learner recovers most of it, reaching $0.895$ (below). Where
chemistry-aware pooled accuracy with generous data and compute is the goal, graph
equivariant networks remain the right tool.

These comparisons place $\starG$ at an input disadvantage, because it consumes a
molecule-level summary while the graph networks consume the full atomic graph.
On that summary every tested model stays between a within-isomer $R^2$ of $0.35$
(trained in one frame) and $0.42$ (trained in all twelve frames), so the ceiling
is a property of the featurization. To test the algebra on equal footing we build
a \emph{per-atom} $\starG$ descriptor: for each atom the local environment
(neighbors within a cutoff, resolved by element) is sampled over the $24$
octahedral frames and reduced to its per-irrep power spectrum, a
$\starG$-native readout that is exactly rotation-invariant (to $10^{-16}$).
The per-atom invariants are then pooled over atoms. The symmetry lives entirely in
this representation, and the predictor is a symmetry-\emph{blind} gradient-boosted
regressor with no equivariant architecture. The descriptor belongs to the family of
power-spectrum descriptors of the local neighbor density, of which SOAP is the
standard member~\cite{bartok2013soap}, and to the body-ordered expansions that the
atomic cluster expansion organizes~\cite{drautz2019ace}. Its distinguishing
feature is that the rotation group enters as a finite octahedral subgroup through
the $\starG$ transform, so the same construction composes with further symmetries
through the product-group rule.

On full QM9 this reaches within-isomer $R^2 = 0.785 \pm 0.003$ and pooled
$R^2 = 0.895 \pm 0.001$ (mean over 3 seeds) on the HOMO--LUMO gap, versus
$0.35$ to $0.42$ within isomers for the molecule-level-summary models and $\approx 0.99$
for the graph-based ENNs; even a linear ridge regressor on the same descriptor
exceeds it ($0.46$). The equivariance benefit therefore does not require an equivariant
network: it can be carried by the data representation and handed to an
arbitrary learner. In this role $\starG$ is a composable equivariant
featurizer, one that decouples symmetry from architecture and composes several
symmetries by the Kronecker rule $F_{G_1}\otimes F_{G_2}$ with no redesign.
The remaining gap to the graph ENNs has a clear origin: message passing still resolves
bond topology that a fixed local descriptor does not fully capture, and closing
it is a matter of richer descriptors rather than of the algebra. No algebraic
rigidity prevents combining the two: the descriptor is already computed per
atom, so its per-irrep spectra can serve directly as node features inside any
message-passing network, and message functions could themselves be built from
$\starG$ products, making each layer's equivariance exact by construction. We
do not pursue that hybrid here because it would re-entangle the two axes this
section deliberately separates (representation versus architecture), but it is
an engineering question.
Table~\ref{tab:comparison} summarizes the two routes to equivariance this
result separates: the symmetry can live in the layers of the model, or in the
algebra of the data (illustrated in Figure~\ref{fig:paradigm}).

\section{Invariance Error of the Orbit-Augmented MLP}
\label{sec:aug_invariance}

An earlier version of this comparison reported that the orbit-augmented MLP
collapses at full QM9 scale ($R^2\approx0.002$). That result came from a
label-alignment error in the training script, which paired the rows of the
augmented input, ordered molecule by molecule, with labels ordered frame by
frame, so that all but one twelfth of the rows carried another molecule's label.
With the labels aligned
(the training script now asserts the alignment), the same
recipe gives the results below, on the same splits as every other row of
Table~\ref{tab:fullqm9}.

\paragraph{Accuracy.} On the HOMO--LUMO gap the augmented MLP reaches
$R^2=0.717\pm0.003$ over five seeds, against $0.687$ for the standard MLP on the
identical $48$-row features. The augmented arm takes twelve times as many gradient
steps per epoch as the standard arm, so we trained a control: the standard MLP
with twelve times the epochs and the early-stopping patience, which equalizes the
number of gradient steps. It reaches $R^2=0.685\pm0.005$ on the gap
(\texttt{results\_matched\_equalsteps}). The gain of the augmented arm therefore
comes from seeing each molecule in all twelve frames, and the extra optimization
alone does not supply it.

\paragraph{Invariance.} The targets are group invariants, so an invariant
predictor returns the same value in every frame. We evaluated each trained model
on all $12$ frames of every test molecule and recorded the standard deviation of
its predictions across frames, relative to the standard deviation of the target
over the test set. For the augmented MLP this dispersion is
$8.1\%\pm1.0\%$ of the target spread on the gap, and the largest range of
predictions across frames for a single molecule is $1.4$ to $6.1$ target
standard deviations, depending on the seed. For the equal-step standard MLP,
trained in one frame, the dispersion is $28.2\%\pm2.4\%$, with single-molecule
ranges up to $9.8$ target standard deviations. Augmentation therefore
reduces the invariance error by a factor of about three and a half and leaves a residual of
order $10^{-1}$ of the target spread, while every $\starG$ readout is invariant to
$10^{-16}$ by construction. The augmented MLP is the most accurate molecule-level
predictor and also the one whose answer depends on the orientation in which a
molecule is presented, the per-call error that compounds in the controlled
experiment of Section~\ref{sec:si_compounding}.

\paragraph{Within-isomer discrimination.} On the formulas with at least five
constitutional isomers, the augmented MLP reaches a within-isomer $R^2$ of
$0.422$ (mean over seeds $0$ to $2$, as for the other rows), against $0.35$ for the standard MLP (Table~\ref{tab:isomer_audit}).
Seeing every group frame improves the discrimination of isomers on the
molecule-level summary, and the per-atom descriptor of
Section~\ref{sec:representation} remains well above both ($0.785$).

\paragraph{Other targets.} On $\alpha$, $\mu$ and ZPVE the augmented MLP reaches
$R^2=0.974\pm0.001$, $0.675\pm0.011$ and $0.999\pm0.000$, level with the standard
MLP ($0.973$, $0.672$, $0.999$), so the gain from the frames is specific to the
gap. Its frame dispersion is $2.1\%$, $10.0\%$ and $0.6\%$ of the target spread,
with worst single-molecule ranges of $3.4$, $7.6$ and $1.2$ target standard
deviations. On every target the augmented predictor therefore remains
approximately invariant, with the largest residual on the dipole magnitude.

\section{Probing Trained ENN Internals per Irrep}
\label{sec:enn_probe}

The main text states that standard ENN workflows do not \emph{expose} a
per-irrep predictive decomposition, not that the information is absent from
their learned representations.  To make that distinction concrete we probe
the trained e3nn baseline directly.  The e3nn hidden representation is
\texttt{32x0e + 16x1o + 8x2e}; for each molecule we form a rotation- and
permutation-invariant per-multiplet power (squared multiplet norm summed over
atoms), mirroring the per-irrep Fourier power that $\starG$ uses, then fit
RidgeCV per $\ell$-block and report test $R^2$.

\begin{table}[!ht]
\centering
\caption{Per-$\ell$ probe of the trained e3nn baseline (seed 0): test $R^2$
of RidgeCV on invariant block powers of the hidden representation. The pattern is
reproduced across three seeds (gap and ZPVE): the scalar block stays dominant, with
$\ell=1,2$ recoverable but weaker.}
\label{tab:enn_probe}
\begin{tabular}{@{}lcccc@{}}
\toprule
\textbf{Target} & \textbf{$\ell=0$ (scalar)} & \textbf{$\ell=1$ (vector)} & \textbf{$\ell=2$ (rank-2)} & \textbf{all $\ell$} \\
\midrule
gap   & 0.944 & 0.625 & 0.413 & 0.948 \\
alpha & 0.974 & 0.772 & 0.751 & 0.983 \\
ZPVE  & 0.990 & 0.904 & 0.758 & 0.996 \\
\bottomrule
\end{tabular}
\end{table}

The probe shows that a trained ENN's internals do carry $\ell$-resolved
structure (the scalar block dominates for scalar targets, as expected), and
that it can be recovered by a bespoke post-hoc analysis of a trained model.
The contrast with $\starG$ is therefore precise, and it is threefold. The
probe presupposes the trained network and inherits its contingencies: the
recovered per-$\ell$ attribution is a property of one converged model (it
varies with seed and training budget) and carries no optimality or exactness
certificate. It is read-only: a diagnostic of what the model happened to
learn, with no mechanism for restricting or enforcing a channel. And it is
bespoke to one architecture's internal irrep typing. The $\starG$
decomposition is canonical for the group (the canonicity theorem), costs one
closed-form transform with no training, certifies its selection rules
exactly, and doubles as an input-side control: channels can be restricted,
enforced, or ablated \emph{before} any learner is trained.

\section{End-to-End Workflow}
\label{sec:workflow}

The complete pipeline used to produce every numerical result in this paper
is summarized in Algorithm~\ref{alg:pipeline}. The pipeline is identical
across the synthetic, QM9, product-group, symmetry-discovery, and
Wigner--Eckart experiments; only the group $G$, the featurization
$\phi: \mathrm{molecule} \mapsto \calX$, and the regression head differ.

\begin{algorithm}[!ht]
\caption{End-to-end $\starG$-SVD~+~ridge pipeline}
\label{alg:pipeline}
\begin{algorithmic}[1]
\Require dataset $\{(\mathrm{mol}_i, y_i)\}_{i=1}^{N}$, group $G$, featurizer $\phi$, ridge grid $\Lambda$
\Ensure trained predictor $\hat{f}$, test scores
\State precompute $\calT_G$, $F_G$, irrep dimensions $\{d_\rho\}$
\For{$i = 1, \ldots, N$}
    \State $\calX_i \gets \phi(\mathrm{mol}_i; G)$ \Comment{tensorial featurization}
\EndFor
\State assemble $\calX \in \bbR^{N \times n_f \times |G|}$
\State split $\calX, y$ into train/val/test (70/15/15)
\State $(\Phi_{\mathrm{tr}}, \Theta) \gets$ Algorithm~\ref{alg:features}$(\calX_{\mathrm{tr}}, G)$
\State $\Phi_{\mathrm{va}}, \Phi_{\mathrm{te}} \gets$ Algorithm~\ref{alg:features}$(\cdot, G; \Theta)$
\State $\lambda^\star \gets \arg\min_{\lambda \in \Lambda} \mathrm{MSE}_{\mathrm{val}}(\Phi_{\mathrm{va}}, y_{\mathrm{va}}; \lambda)$
\State $w \gets (\Phi_{\mathrm{tr}}^\top \Phi_{\mathrm{tr}} + \lambda^\star I)^{-1} \Phi_{\mathrm{tr}}^\top y_{\mathrm{tr}}$
\State $R^2_{\mathrm{te}} \gets 1 - \mathrm{SS_{res}}(\Phi_{\mathrm{te}} w, y_{\mathrm{te}}) / \mathrm{SS_{tot}}(y_{\mathrm{te}})$
\State $\nu \gets \mathrm{Var}_{g \in G} \hat{y}(g \cdot \calX_{\mathrm{te}})$ \Comment{rotation-variance audit}
\State \Return $\hat{f}: \calX \mapsto \mathrm{Algorithm~\ref{alg:features}}(\calX; \Theta) \cdot w$, $R^2_{\mathrm{te}}$, $\nu$
\end{algorithmic}
\end{algorithm}

\section{Foundation Force-Field Protocol}
\label{sec:foundation_protocol}

All measurements of the main-text section \emph{Per-call violations compound on production force fields} use the
publicly released pretrained weights, unmodified: MACE-OFF23-small (organic
molecules) and MACE-MP-0-small (materials) from the official ACEsuit releases,
and Orb-v2 from the official orb-models distribution, each evaluated through
its own ASE calculator at the precision stated. Isotypic projectors for the
octahedral group $O$ (order 24, SF$_6$), the tetrahedral rotation group $T$
(order 12, CH$_4$), and the cubic rotations acting on the rocksalt cell are
constructed from characters and verified programmatically (resolution of
identity, idempotence, mutual orthogonality, integer multiplicities) before
use; the displacement representations decompose as
$21 = A_1 \oplus E \oplus 4T_1 \oplus 2T_2$ and
$15 = A \oplus E \oplus 4T$, with the allowed channel one-dimensional
(breathing) in both. Rollouts are velocity Verlet with float64 state and real
masses, $\Delta t = 0.5$~fs (SF$_6$) or $0.25$~fs (CH$_4$), started from the
model's own totally symmetric equilibrium (located by bisection on the radial
force) with a $1\%$ breathing displacement; the force call is the only
symmetry-breaking agent, and potential energy is recorded as the
integrator-health control. The crystal protocol starts the ideal
$8$-atom rocksalt cell at rest, where the exact dynamics is a fixed point. The
enforcement arm applies the allowed-channel Reynolds projector to each force
call; the block arm evolves the allowed coordinate only, so forbidden channels
are not represented. Scripts, per-run records and the audit results are in the code deposit (Section~\ref{sec:si_codemap}).

\section{Symmetry Audit of Foundation Force Fields across Crystals}
\label{sec:si_fm_audit}

This section documents the audit of the main-text section \emph{Selection-rule
violations of foundation force fields across crystals}
(Section~\ref{sec:si_codemap}). Every number below is read from the committed result files.

\paragraph{Crystals and models.} We took $40$ crystals from the de Jong
elastic-tensor dataset, in their density-functional relaxed structures, reduced
to the primitive standard cell by spglib at a tolerance of $10^{-3}$, with at most
ten atoms per cell. The deterministic selection (seed $0$) covers five crystal
systems: $12$ cubic, $8$ hexagonal, $8$ tetragonal, $6$ trigonal and $6$
orthorhombic. Three cubic crystals have one-atom cells (Ba, Cr, Al), in which
every force vanishes, so the force and equivariance statistics use $37$ crystals.
Orthorhombic point groups have only one-dimensional irreducible representations,
so the phonon statistics use the $31$ crystals with protected multiplets. The
models are Orb-v2~\cite{neumann2024orb}, Orb-v3 trained on OMat with a direct force
head and with conservative forces~\cite{rhodes2025orbv3}, both at the package
default precision, and MACE-MP-0 (medium)~\cite{batatia2023foundation} in single
and double precision.

\paragraph{Measurements.} At the density-functional structure, the space-group
Reynolds operator
$(PF)_i=|G|^{-1}\sum_{g}R_gF_{\pi_g^{-1}(i)}$, with the rotations $R_g$, the
fractional translations and the atom permutations $\pi_g$ taken from spglib,
splits the forces into an allowed part $PF$ and a forbidden part $F-PF$. The
allowed part is nonzero because each model's equilibrium differs from the
density-functional one, and only the forbidden part is a symmetry violation. The
point-group projector does the same for the stress. The equivariance error is
$\|F(gx)-gF(x)\|/\|F(x)\|$ over the point-group operations at a configuration
rattled by $0.03$~\AA{} (seed $0$), with positions and cell rotated together. The
$\Gamma$-point Hessian is a central finite difference with a $0.01$~\AA{} step
in the primitive cell, so periodic images move together. Multiplets are
identified in the spectrum of the group-averaged Hessian, which commutes with the
representation, and the reported splitting is the largest within-multiplet spread
of the raw spectrum. The relaxation is $200$ FIRE steps on positions to a force
tolerance of $10^{-3}$~eV/\AA, once unconstrained and once with
\texttt{ase.constraints.FixSymmetry}, followed by spglib symmetry detection at
$10^{-3}$ and $10^{-5}$.

\begin{table}[!ht]
\centering
\caption{Audit of five pretrained force fields (medians over crystals, interquartile
range in brackets for the forbidden force). Space-group changes count crystals
whose spglib space group at tolerance $10^{-3}$ differs after an unconstrained
relaxation. With force projection every model keeps all $40$ space groups.}
\label{tab:fm_audit}
\small
\setlength{\tabcolsep}{3pt}
\begin{tabular}{@{}lccccc@{}}
\toprule
\textbf{Model} & \textbf{Forbidden force (eV/\AA)} & \textbf{Equivariance} & \textbf{Splitting (cm$^{-1}$)} & \textbf{$>1$~cm$^{-1}$} & \textbf{SG changed} \\
 & & \textbf{error} & \textbf{median / max} & \textbf{(of 31)} & \textbf{(of 40)} \\
\midrule
Orb-v2 & $3.8\times10^{-4}$ $[0.9, 14]\times10^{-4}$ & $9.3\times10^{-2}$ & $10.7$ / $28.1$ & $31$ & $20$ \\
Orb-v3, direct & $6.9\times10^{-3}$ $[4.4, 8.2]\times10^{-3}$ & $3.1\times10^{-2}$ & $0.72$ / $3.0$ & $8$ & $34$ \\
Orb-v3, conservative & $1.7\times10^{-3}$ $[0.85, 2.4]\times10^{-3}$ & $8.4\times10^{-3}$ & $0.14$ / $0.46$ & $0$ & $15$ \\
MACE-MP-0, float32 & $6.7\times10^{-7}$ $[1.6, 15]\times10^{-7}$ & $1.8\times10^{-6}$ & $9\times10^{-4}$ / $0.017$ & $0$ & $0$ \\
MACE-MP-0, float64 & $1.6\times10^{-15}$ $[0.25, 3.0]\times10^{-15}$ & $5.3\times10^{-15}$ & $4\times10^{-12}$ / $0.017$ & $0$ & $0$ \\
\bottomrule
\end{tabular}
\end{table}

\paragraph{After projection.} The projected forces of every model have forbidden
parts of at most $4\times10^{-17}$~eV/\AA, and the group-averaged Hessians have
every protected multiplet degenerate to within $2\times10^{-12}$~cm$^{-1}$. With
\texttt{FixSymmetry}, the forbidden displacement of every relaxation is at most
$6\times10^{-15}$~\AA.

\paragraph{Relaxations.} At tolerance $10^{-5}$ the unconstrained relaxations
change the space group in $20$, $36$ and $29$ of the $40$ crystals for Orb-v2,
Orb-v3 direct and Orb-v3 conservative, and in none for MACE-MP-0. The median
forbidden displacement of an unconstrained relaxation is $1.7\times10^{-4}$,
$8.3\times10^{-4}$ and $1.4\times10^{-4}$~\AA{} (root mean square per atom) for
the three Orb models. A space-group change at tolerance $10^{-3}$ therefore
records a symmetry broken at the $10^{-3}$~\AA{} level. In SrSbSe$_2$F
($P4/nmm$) the group-averaged Hessian has four imaginary modes under every model,
so the symmetric structure is a saddle point of each potential. MACE-MP-0 stays
at $P4/nmm$ because its forces carry no forbidden component to start the escape,
and both Orb-v3 models relax to $P1$ with forbidden displacements up to
$0.19$~\AA.

\paragraph{Caveats.} (i) Phonons of direct-force models are defined only up to
the finite-difference step. Orb-v2 forces are not the gradient of an energy. Its
finite-difference Hessian has a median relative asymmetry of $0.30$, its
group-averaged spectrum has imaginary modes in $30$ of $40$ crystals, and on three
test crystals (Be$_2$Nb, NaBr, Al$_2$O$_3$) its largest splitting moves between
$2.9$ and $51$~cm$^{-1}$ as the step goes from $0.005$ to $0.04$~\AA, with the same values in double precision. Compared
with MACE-MP-0, the highest $\Gamma$ frequency of Orb-v2 lies within $20\%$ in
only $9$ of $37$ cells, against $29$ for each Orb-v3 model. The Orb-v2 splittings
are therefore what a user obtains at the default step, and they mix a symmetry
violation with an unsound spectrum. Orb-v3 conservative stays within $0.05$ to
$0.25$~cm$^{-1}$ across the same steps. (ii) The stencil displaces atoms along
Cartesian axes, which a rotation by $60^\circ$ or $120^\circ$ does not permute, so
on hexagonal and trigonal cells the diagnostic has a floor of about
$2\times10^{-2}$~cm$^{-1}$. The largest MACE-MP-0 values ($0.017$~cm$^{-1}$,
Al$_2$O$_3$ and NbSi$_2$) are identical in both precisions and scale as the
square of the step, which identifies them as stencil error. Cubic and tetragonal
cells reach $10^{-12}$~cm$^{-1}$. (iii) Orb models ran with
\texttt{torch.compile} disabled. MACE-MP-0 reports float32 as its recommended
precision for molecular dynamics and float64 for geometry optimization. (iv) One
structure per crystal and one rattle seed.

\section{Symmetry Audit of Foundation Force Fields at Screening Scale}
\label{sec:si_fm_screening}

This section documents the screening-scale audit of the main-text section
\emph{Selection-rule violations of foundation force fields across crystals}
(scripts and result files in Section~\ref{sec:si_codemap}). Every number below is
read from the result files.

\paragraph{Structures.} The source is the formation-energy set of the Matbench
benchmark~\cite{dunn2020matbench}, $132{,}752$ Materials Project structures~\cite{jain2013commentary}, of
which $103{,}611$ have at most $40$ atoms per cell. We symmetrized each of these to
the space group that spglib~\cite{togo2018spglib} detects at a tolerance of
$10^{-3}$~\AA{}, so that the input is symmetric to numerical precision, and took the
space-group operations at $10^{-5}$~\AA. Structures in $P1$ have no forbidden
subspace and were excluded, which leaves $94{,}892$ eligible structures. The
single-point stage evaluates a uniform random subset of $30{,}000$ eligible
structures (seed $0$) together with every relaxed structure, $30{,}280$ in total.
The relaxation stage uses $400$ structures with at most $20$ atoms (median $10$),
drawn from the $65{,}939$ eligible structures of that size with an allocation
proportional to the crystal systems (seed $0$): $14$ triclinic, $65$ monoclinic,
$92$ orthorhombic, $68$ tetragonal, $40$ trigonal, $34$ hexagonal and $87$ cubic.

\paragraph{Models and measurements.} The models are Orb-v3 trained on OMat with a
direct force head and with conservative forces~\cite{rhodes2025orbv3} and
MACE-MP-0 (medium)~\cite{batatia2023foundation} in single precision. The forbidden
force and the forbidden stress are defined as in Section~\ref{sec:si_fm_audit}.
The relaxation is $200$ FIRE steps on positions with the cell fixed, to a force
tolerance of $10^{-3}$~eV/\AA, once unconstrained and once with each force call
projected. We detected the space group of the input and of each relaxed structure
at spglib tolerances of $10^{-5}$, $10^{-4}$, $10^{-3}$, $10^{-2}$ and
$10^{-1}$~\AA. A relaxation lowers the space group at a given tolerance when the
relaxed structure has fewer space-group operations than the input at the same
tolerance. We attribute a lowering to the forbidden force when the unconstrained
relaxation lowers the space group and the projected relaxation of the same
structure does not. At the projected relaxed structure we computed the
central-difference Hessian with a $0.01$~\AA{} step, averaged it over the space
group and recorded the lowest frequency among the symmetry-breaking modes, with
rigid translations removed. A structure with such a mode below $-5$~cm$^{-1}$ is a
saddle point of the model's energy surface. We call an attributed lowering from any
other structure a drift from a stable symmetric minimum.

\begin{table}[!ht]
\centering
\caption{Forbidden force and stress at the symmetrized input of $30{,}280$
Materials Project structures (medians, interquartile range in brackets). The
shares count structures whose root-mean-square forbidden force, or whose largest
per-atom forbidden force, exceeds the stated value in eV/\AA. The relaxation
threshold is $10^{-3}$~eV/\AA.}
\label{tab:fm_screening_single}
\small
\setlength{\tabcolsep}{4pt}
\begin{tabular}{@{}lccccc@{}}
\toprule
\textbf{Model} & \textbf{Forbidden force} & \multicolumn{3}{c}{\textbf{Share above threshold (eV/\AA)}} & \textbf{Forbidden} \\
\cmidrule(lr){3-5}
 & \textbf{(eV/\AA)} & \textbf{rms} & \textbf{largest atom} & \textbf{largest atom} & \textbf{stress} \\
 & & \textbf{$>10^{-3}$} & \textbf{$>10^{-3}$} & \textbf{$>10^{-2}$} & \textbf{(GPa)} \\
\midrule
Orb-v3, conservative & $2.0\times10^{-3}$ $[1.2, 3.2]\times10^{-3}$ & $83.1\%$ & $92.8\%$ & $9.3\%$ & $4.8\times10^{-3}$ \\
Orb-v3, direct & $9.7\times10^{-3}$ $[6.2, 15]\times10^{-3}$ & $99.6\%$ & $99.7\%$ & $75.5\%$ & $3.5\times10^{-2}$ \\
MACE-MP-0, float32 & $1.3\times10^{-6}$ $[0.69, 2.4]\times10^{-6}$ & $0$ & $0$ & $0$ & $3.1\times10^{-6}$ \\
\bottomrule
\end{tabular}
\end{table}

\paragraph{Single points.} The projected forces have forbidden parts of at most
$1.2\times10^{-15}$~eV/\AA{} under every model. Atoms whose site-symmetry group
contains a nontrivial rotation make up $58\%$ of the atoms and carry $52\%$, $50\%$
and $35\%$ of the forbidden energy under Orb-v3 with gradient forces, Orb-v3 with
the direct head and MACE-MP-0. The median forbidden force is lowest for the cubic
structures under every model ($1.4\times10^{-3}$, $6.9\times10^{-3}$ and
$7.5\times10^{-7}$~eV/\AA). The medians agree in order of magnitude with those of
the $40$ de Jong crystals (Table~\ref{tab:fm_audit}).

\paragraph{Label-free error bounds.} Every structure is evaluated at a configuration
that carries its space group, where the forces and stress of any symmetric reference
method lie in the allowed subspace. By Proposition~\ref{prop:enforce_guarantee} the
root-mean-square force error of a model over the atoms is then at least its
root-mean-square forbidden force, and its stress error is at least its forbidden
stress. The certified force error has a median of $2.0$, $9.7$ and
$1.3\times10^{-3}$~meV/\AA{} under Orb-v3 with gradient forces, Orb-v3 with the direct
head and MACE-MP-0, and a 90th percentile of $4.9$, $22.6$ and
$4.5\times10^{-3}$~meV/\AA{}. It is at least $5$~meV/\AA{} in $9.7\%$, $83.7\%$ and
$0\%$ of structures and at least $20$~meV/\AA{} in $0.15\%$, $13.6\%$ and $0\%$. The
certified stress error has a median of $4.8\times10^{-3}$, $3.5\times10^{-2}$ and
$3.1\times10^{-6}$~GPa. The bounds hold for the reference method up to the accuracy with
which its own forces respect the symmetry, the second case of the proposition.

\paragraph{Constraint strength.} A structure in which every atom sits at a site where the symmetry allows no force has a zero reference force by symmetry, so its certificate is its whole force error. Table~\ref{tab:cert_strata} splits the $30{,}280$ structures by how strongly the symmetry constrains their forces. The structures in which no force is allowed are $13.4\%$ of the sample and carry $0.2$ to $1.8\%$ of the forbidden energy $\sum N b^2$, with the lowest median certificate for every model. The certificate falls as the share of atoms on special sites rises (Spearman $\rho=-0.25$, $-0.28$ and $-0.50$). The structures with no atom on a special site carry about $20\%$ of the forbidden energy, which there comes only from the relations between symmetry-equivalent atoms. The script \texttt{cert\_validation/cert\_strata.py} computes the table from the stored screening records. We also checked the certificate against density-functional forces at the same geometries. The D0 single points of the density-functional adjudication give PBE forces for $28$ structures at the symmetrized inputs the models were evaluated at. The median force error was $76$, $98$ and $115$~meV/\AA{} under Orb-v3 with gradient forces, the direct head and MACE-MP-0. The forbidden part carried $0.1$, $2.1$ and $0.04\%$ of the squared error, pooled over atoms. Of the structures with an error above $10$~meV/\AA{} ($25$, $25$ and $24$), the certificate flagged $1$, $14$ and $0$, and every flag was a true error. As a score that ranks the structures by error, the certificate reached an area under the curve of $0.79$, $0.88$ and $0.76$ at $10$~meV/\AA{}. The predicted force magnitude reached $0.92$, $0.95$ and $0.99$, and the symmetry class alone $0.85$, $0.85$ and $0.90$. The reference forces at these geometries are small (median $57$~meV/\AA{}), so the predicted force is close to the error and this baseline is nearly circular. Near equilibrium most of the error lies in the allowed part, and the certificate gives a lower bound that captures little of it. The density-functional forces carry a forbidden part of at most $12$~meV/\AA{}. The script \texttt{cert\_validation/cert\_d0\_label.py} computes these numbers.

\begin{table}[h]
\centering
\caption{Forbidden energy and median certificate $b$ (eV/\AA) by constraint strength over the $30{,}280$ symmetrized structures. Stratum A: no allowed force channel. B: every atom on a special site, some force allowed. C: mixed. D: every atom on a general position.}
\label{tab:cert_strata}
\small
\setlength{\tabcolsep}{4pt}
\begin{tabular}{@{}lrccc@{}}
\toprule
Stratum & $n$ & Orb-v3, conservative & Orb-v3, direct & MACE-MP-0, float32 \\
 & & share / median $b$ & share / median $b$ & share / median $b$ \\
\midrule
A & $4{,}068$ & $1.8\%$ / $1.3\times10^{-3}$ & $1.1\%$ / $5.6\times10^{-3}$ & $0.2\%$ / $3.4\times10^{-7}$ \\
B & $14{,}718$ & $31.1\%$ / $1.9\times10^{-3}$ & $28.9\%$ / $9.1\times10^{-3}$ & $15.5\%$ / $1.1\times10^{-6}$ \\
C & $8{,}601$ & $45.5\%$ / $2.5\times10^{-3}$ & $50.9\%$ / $1.2\times10^{-2}$ & $59.4\%$ / $2.1\times10^{-6}$ \\
D & $2{,}893$ & $21.6\%$ / $2.5\times10^{-3}$ & $19.1\%$ / $1.2\times10^{-2}$ & $24.9\%$ / $2.9\times10^{-6}$ \\
\bottomrule
\end{tabular}
\end{table}

\paragraph{Orbit floor.} Theorem~\ref{thm:two_level}(iii) gives a second floor, the variance $V(x)=|G|^{-1}\sum_g\|R_g^{\top}f(R_gx)-\bar f\|^2$ of the back-rotated forces about their mean $\bar f$, which bounds from below the mean squared force error over the orbit of $x$. We evaluated it for the $24$ proper rotations of the cube, applied to the cell and the positions, on $45$ of the symmetrized inputs of Table~\ref{tab:cert_strata} drawn from all four strata, and on copies of $10$ of them with every atom displaced at random by $0.03$~\AA{}. The displaced copies have no symmetry, so their forbidden force is zero. Table~\ref{tab:cert_orbit} lists the floor $\sqrt{V/N}$. On the symmetric structures it has the size of the forbidden force $b$ and ranks the structures alike (Spearman $\rho=0.83$ under Orb-v3 with the direct head and $0.90$ with gradient forces). The rotations of the cube are in general not symmetries of the crystal, so the two floors are distinct bounds. On the displaced copies the floor stays within $14\%$ of that of the parent structure, so the violation that the forbidden force no longer sees remains certified. Under MACE-MP-0 the floor is at most $1.1\times10^{-3}$~meV/\AA{}, the resolution of single precision. The script \texttt{cert\_validation/cert\_orbit\_models.py} computes the table.

\begin{table}[h]
\centering
\caption{Orbit floor $\sqrt{V/N}$ over the $24$ proper rotations of the cube and forbidden force $b$, in meV/\AA{}, median and interquartile range. Symmetric: $45$ symmetrized Materials Project structures. Displaced: $10$ of them with every atom moved at random by $0.03$~\AA{}, where $b=0$. The last column is the ratio of the floor to that of the undisplaced structure.}
\label{tab:cert_orbit}
\small
\setlength{\tabcolsep}{4pt}
\begin{tabular}{@{}lcccc@{}}
\toprule
Model & floor, symmetric & $b$, symmetric & floor, displaced & ratio to parent \\
\midrule
Orb-v3, direct & $8.0$ ($5.3$--$11.2$) & $8.3$ ($6.0$--$12.0$) & $10.9$ ($9.4$--$16.2$) & $1.02$ ($1.01$--$1.07$) \\
Orb-v3, conservative & $2.1$ ($1.5$--$2.7$) & $1.7$ ($1.4$--$2.6$) & $2.3$ ($1.9$--$4.4$) & $1.01$ ($1.00$--$1.08$) \\
MACE-MP-0, float32 & $4.3\times10^{-4}$ & $9.9\times10^{-4}$ & $5.3\times10^{-4}$ & $0.92$ ($0.70$--$1.27$) \\
\bottomrule
\end{tabular}
\end{table}

\paragraph{Density-functional adjudication.} We asked whether the lowerings under Orb-v3 with the direct head are errors of the model or instabilities of the symmetric structures. The sample was fixed before any reference calculation ran. It stratifies the structures of the orientation audit with at most $16$ atoms by the raw outcome and by the stability that the three models assign, and it replaces structures that contain a lanthanide or an actinide by the next structure of the same stratum. We used Quantum ESPRESSO 7.5 with the PBE functional, SSSP 1.3.0 precision pseudopotentials at $1.2$ times their recommended cutoffs, a $\Gamma$-centered $k$-point grid of reciprocal density $200$, Marzari-Vanderbilt smearing of $0.01$~Ry, spin polarization from the Materials Project initial moments and the Materials Project Hubbard corrections. The cell stays fixed at the symmetrized Materials Project cell, as in the audit. For each structure we relaxed the symmetric positions, computed the $\Gamma$-point Hessian from displacements of $0.01$~\AA{} and projected it onto the symmetry-breaking irreducible representations. Where that Hessian had a mode below $-5$~cm$^{-1}$, we relaxed without symmetry from the softest mode and from two random symmetry-breaking displacements. A script assigns each structure from these calculations alone. A structure is stable if it has no breaking mode below $-5$~cm$^{-1}$, unstable if a relaxation reaches a lower space group at $10^{-2}$~\AA{} at least $1$~meV/atom below the symmetric structure and at least $0.03$~\AA{} from it, and undecided otherwise. In the strata where the direct head lowered the space group and no model found an unstable mode, $14$ of $15$ structures were stable ($95\%$ interval $70$ to $99\%$). The fifteenth, FeAg$_3$(CN)$_6$, is undecided. Four of its eleven displaced calculations converged only after the SCF changed magnetic state or smearing, and its Hessian gave a breaking mode of $-385$~cm$^{-1}$ together with an unstable symmetric mode, which cannot occur at a minimum relaxed under symmetry. All three relaxations from it ended within $0.03$~\AA{} of the symmetric structure and at most $0.3$~meV/atom below it. In the strata where all three models found an unstable mode, $6$ of $10$ structures were unstable, $3$ were undecided and HfBi$_3$ was stable at $+13.6$~cm$^{-1}$, against $-18$ to $-44$~cm$^{-1}$ under the models. One undecided structure, C$_2$NCl, relaxed to $0.88$~eV/atom below its symmetric structure in space group $P1$ but reached the step limit before the force tolerance. Six structures were stopped for compute budget before any result, and the remaining runs were stopped at a fixed deadline and classified from the completed calculations. The calculations test only instabilities commensurate with the cell, start from ferromagnetic order and use one functional. The protocol, its amendments and the per-structure results are in \texttt{dft\_adjudication/} of the code repository.

\begin{table}[!ht]
\centering
\caption{Space-group lowering in $400$ relaxations ($399$ for Orb-v3 with the
direct head) against the spglib tolerance. For each model the first column counts
relaxations attributed to the forbidden force (lowered without projection and kept
with projection), with the share and its $95\%$ Wilson interval in percent, and the
second column counts projected relaxations that lowered the space group.}
\label{tab:fm_screening_tol}
\small
\setlength{\tabcolsep}{3pt}
\begin{tabular}{@{}lcccccc@{}}
\toprule
 & \multicolumn{2}{c}{\textbf{Orb-v3, conservative}} & \multicolumn{2}{c}{\textbf{Orb-v3, direct}} & \multicolumn{2}{c}{\textbf{MACE-MP-0, float32}} \\
\textbf{Tolerance (\AA)} & \textbf{attributed} & \textbf{projected} & \textbf{attributed} & \textbf{projected} & \textbf{attributed} & \textbf{projected} \\
\midrule
$10^{-5}$ & $382$ ($95.5$, $93.0$--$97.1$) & $0$ & $398$ ($99.7$, $98.6$--$100$) & $0$ & $24$ ($6.0$, $4.1$--$8.8$) & $0$ \\
$10^{-4}$ & $382$ ($95.5$, $93.0$--$97.1$) & $0$ & $397$ ($99.5$, $98.2$--$99.9$) & $0$ & $15$ ($3.8$, $2.3$--$6.1$) & $0$ \\
$10^{-3}$ & $277$ ($69.3$, $64.6$--$73.6$) & $0$ & $389$ ($97.5$, $95.4$--$98.6$) & $0$ & $9$ ($2.3$, $1.2$--$4.2$) & $0$ \\
$10^{-2}$ & $75$ ($18.8$, $15.2$--$22.9$) & $1$ & $203$ ($50.9$, $46.0$--$55.8$) & $4$ & $7$ ($1.8$, $0.9$--$3.6$) & $1$ \\
$10^{-1}$ & $66$ ($16.5$, $13.2$--$20.5$) & $4$ & $96$ ($24.1$, $20.1$--$28.5$) & $4$ & $8$ ($2.0$, $1.0$--$3.9$) & $3$ \\
\bottomrule
\end{tabular}
\end{table}

\paragraph{Relaxations.} The projected relaxations lowered no space group at
tolerances up to $10^{-3}$~\AA{} under any model, and their forbidden displacement
was at most $3\times10^{-7}$~\AA. At $10^{-2}$ and $10^{-1}$~\AA{} the projected
relaxations lowered $1$ to $4$ space groups per model. In these structures the
input is detected with a higher group at the loose tolerance, and the allowed
relaxation removes that approximate symmetry. The median forbidden displacement of
an unconstrained relaxation, taken at the atom that moves most, was
$6.7\times10^{-4}$~\AA{} for Orb-v3 with gradient forces, $4.0\times10^{-3}$~\AA{}
for Orb-v3 with the direct head and $1.9\times10^{-7}$~\AA{} for MACE-MP-0. Within
$200$ steps, $87\%$, $63\%$ and $95\%$ of the unconstrained relaxations reached the
force tolerance, against $95\%$ to $96\%$ of the projected ones. Unconverged
relaxations are classified by their final structure.

\begin{table}[!ht]
\centering
\caption{Attributed lowerings split by the stability of the symmetric structure
under the same model: saddle point (a symmetry-breaking mode below
$-5$~cm$^{-1}$) or drift from a stable symmetric minimum.}
\label{tab:fm_screening_saddle}
\small
\begin{tabular}{@{}lccc@{}}
\toprule
\textbf{Tolerance (\AA)} & \textbf{Orb-v3, conservative} & \textbf{Orb-v3, direct} & \textbf{MACE-MP-0, float32} \\
 & \textbf{saddle / drift} & \textbf{saddle / drift} & \textbf{saddle / drift} \\
\midrule
$10^{-5}$ & $107$ / $275$ & $112$ / $286$ & $19$ / $5$ \\
$10^{-4}$ & $107$ / $275$ & $112$ / $285$ & $14$ / $1$ \\
$10^{-3}$ & $96$ / $181$ & $110$ / $279$ & $9$ / $0$ \\
$10^{-2}$ & $70$ / $5$ & $94$ / $109$ & $7$ / $0$ \\
$10^{-1}$ & $66$ / $0$ & $88$ / $8$ & $8$ / $0$ \\
\bottomrule
\end{tabular}
\end{table}

\paragraph{Convergence.} Within $200$ steps, $87\%$, $63\%$ and $95\%$ of the
unconstrained relaxations reached the force tolerance under Orb-v3 with gradient forces,
Orb-v3 with the direct head and MACE-MP-0, against $96\%$, $95\%$ and $96\%$ of the
projected ones. Restricted to the $346$, $250$ and $380$ structures whose two relaxations
both converged, the share attributed to the forbidden force at $10^{-2}$~\AA{} is $12.1\%$
($9.1$ to $16.0\%$, $95\%$ Wilson interval), $47.6\%$ ($41.5$ to $53.8\%$) and $0.3\%$
($0.0$ to $1.5\%$), against $18.8\%$, $50.9\%$ and $1.8\%$ over all relaxations. At
$10^{-3}$~\AA{} the converged shares are $65.3\%$, $97.2\%$ and $0.3\%$.

\paragraph{Materials Project labels.} Materials Project assigns space groups with spglib
at a tolerance of $0.1$~\AA{} on the density-functional structure. We computed that
label from each original structure of the relaxation sample and compared it with the
space group of the relaxed structure at the same tolerance. The symmetrized inputs carry
the label in every case. The unconstrained relaxation kept it in $325$ of $400$
structures under Orb-v3 with gradient forces, $295$ of $399$ under the direct head and
$386$ of $400$ under MACE-MP-0, and the projected relaxation in $392$, $391$ and $394$.
The label changed without projection and was kept with projection in $67$, $97$ and $8$
structures, and $67$, $89$ and $8$ of these start from a structure with an unstable
symmetry-breaking mode under the same model. At this tolerance the drift from a stable
symmetric structure therefore changes the label in $8$ of $399$ relaxations, all under
the direct head, and the larger effect is the escape from structures that are unstable
under the model, which an equivariant model without forbidden force does not start.

\paragraph{Orientation of the input.} A rigid proper rotation applied to the cell
vectors and the Cartesian positions leaves the fractional coordinates and the physics
of a crystal unchanged. We drew $200$ of the $399$ completed relaxation structures with
a fixed seed, stratified by whether the unconstrained relaxation had lowered the space
group at $10^{-2}$~\AA{} under either Orb-v3 model. Each structure was relaxed with the
protocol above in its stored orientation and in three random orientations, with the
same rotations for every model. The analysis covers the first $150$ (direct head),
$120$ (gradient forces) and $80$ (MACE-MP-0) structures of a fixed processing order.
Each structure was also relaxed twice in its stored orientation, which measures the
run-to-run variation of the single-precision GPU computation.
Table~\ref{tab:si_orientation} lists the share of structures whose final space group was
not the same in all four orientations. Two runs in the same orientation gave bitwise
identical final coordinates in $4\%$ (direct head), $11\%$ (gradient forces) and $23\%$
(MACE-MP-0) of structures. A comparison of the stored orientation with one rotation
also involves two runs, and it changed the space group at $10^{-2}$~\AA{} in $30.7\%$,
$9.2\%$ and $1.2\%$ of structures. Under Orb-v3 with gradient forces, $12$ of the $16$
orientation-dependent structures at $10^{-2}$~\AA{} are unstable under the model. Both
orientation-dependent structures under MACE-MP-0 are unstable under that model. Under
the direct head, $52$ of the $71$ are stable under the model. The number of structures
whose space group was lowered at $10^{-2}$~\AA{} changed by at most $11$ of $150$
(direct head) and $1$ of $120$ (gradient forces) between orientations. The shares
reported above therefore estimate a population rate, and the orientation changes which
structures carry a lowered label. With projection, none of the $1{,}400$ runs lost the
input space group at $10^{-5}$~\AA{}. The remaining projected differences come from an
approximate higher symmetry that the loose tolerance detects in some orientations of a
structure and not in others. Projected Orb-v3 runs end at energies that differ across
orientations by a median of $0.8$~meV/atom (direct head) and $0.1$~meV/atom (gradient
forces), close to the spread of the model's energy for the unrelaxed input in the four
orientations.

\begin{table}[h]
\centering
\caption{Share of structures whose final space group differs among four orientations of
the input crystal, without projection and with each force call projected (FixSymmetry),
at two spglib tolerances, with $95\%$ Wilson intervals. The last column is the share for
two runs in the same orientation. Unstable: the model's group-averaged force constants
at the projected-relaxed structure have a symmetry-breaking mode below
$-5$~cm$^{-1}$.}
\label{tab:si_orientation}
\small
\setlength{\tabcolsep}{4pt}
\begin{tabular}{@{}lrccccc@{}}
\toprule
Model & $n$ & \multicolumn{2}{c}{$10^{-2}$~\AA} & \multicolumn{2}{c}{$0.1$~\AA} & Repeat, $10^{-2}$~\AA \\
 & & raw & projected & raw & projected & raw \\
\midrule
Orb-v3, direct head & $150$ & $47.3$ ($39.5$--$55.3$) & $2.0$ ($0.7$--$5.7$) & $22.7$ ($16.7$--$30.0$) & $0.0$ ($0.0$--$2.5$) & $4.7$ ($2.3$--$9.3$) \\
\quad stable & $106$ & $49.1$ & $0.9$ & $5.7$ & $0.0$ & \\
\quad unstable & $44$ & $43.2$ & $4.5$ & $63.6$ & $0.0$ & \\
Orb-v3, gradient forces & $120$ & $13.3$ ($8.4$--$20.6$) & $0.0$ ($0.0$--$3.1$) & $7.5$ ($4.0$--$13.6$) & $0.8$ ($0.1$--$4.6$) & $1.7$ ($0.5$--$5.9$) \\
\quad stable & $86$ & $4.7$ & $0.0$ & $1.2$ & $0.0$ & \\
\quad unstable & $34$ & $35.3$ & $0.0$ & $23.5$ & $2.9$ & \\
MACE-MP-0, float32 & $80$ & $2.5$ ($0.7$--$8.7$) & $0.0$ ($0.0$--$4.6$) & $1.2$ ($0.2$--$6.7$) & $0.0$ ($0.0$--$4.6$) & $1.2$ ($0.2$--$6.7$) \\
\quad stable & $58$ & $0.0$ & $0.0$ & $0.0$ & $0.0$ & \\
\quad unstable & $22$ & $9.1$ & $0.0$ & $4.5$ & $0.0$ & \\
\bottomrule
\end{tabular}
\end{table}

\paragraph{Saddle points.} Of the sampled structures, $108$, $112$ and $114$ are
saddle points of Orb-v3 with gradient forces, Orb-v3 with the direct head and
MACE-MP-0, and $85$ are saddle points of all three models. The saddle counts of
Table~\ref{tab:fm_screening_saddle} change by at most $7$ when the threshold moves
from $-5$ to $-1$ or $-10$~cm$^{-1}$. At a tolerance of $10^{-3}$~\AA{} the
unconstrained relaxation stayed at $12$ of the $108$ saddle points under Orb-v3 with
gradient forces, at $2$ of $112$ under the direct head and at $105$ of $114$ under
MACE-MP-0. The drift column counts lowerings from structures that are stable under
the model, which the forbidden force causes alone. At the $10^{-2}$~\AA{} default
it holds $5$ structures under Orb-v3 with gradient forces and $109$ under the
direct head.

\begin{figure}[!ht]
\centering
\includegraphics[width=\textwidth]{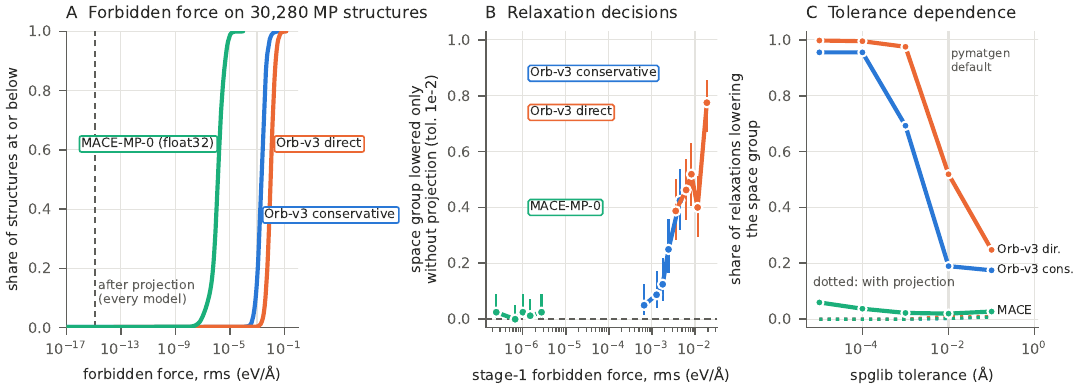}
\caption{\textbf{Symmetry audit of three foundation force fields at screening
scale.} (\textbf{A}) Cumulative distribution of the root-mean-square forbidden force
at the symmetrized input of $30{,}280$ Materials Project structures. The dashed line
marks the largest forbidden force after projection under any model. (\textbf{B})
Share of relaxations whose space group at a tolerance of $10^{-2}$~\AA{} is lowered
without projection and kept with projection, in quintiles of the single-point
forbidden force, with $95\%$ Wilson intervals ($400$ relaxations per model, $399$ for
Orb-v3 with the direct head). (\textbf{C}) Share of relaxations that lower the space
group against the spglib tolerance, without projection (solid) and with each force
call projected (dotted). The vertical line marks the $10^{-2}$~\AA{} default
tolerance of pymatgen~\cite{ong2013pymatgen}.}
\label{fig:fm_screening}
\end{figure}

\paragraph{Caveats.} (i) One Orb-v3 direct-head relaxation terminated the process
in two attempts and is excluded. (ii) The relaxations move positions with the cell
fixed. The inputs are density-functional structures and are not minima of the
models (median allowed force $0.064$ to $0.079$~eV/\AA), so every model also moves
the atoms along allowed coordinates. (iii) The direct-head forces are not the
gradient of an energy. The median relative asymmetry of its finite-difference
Hessian is $0.049$, against $0.013$ for Orb-v3 with gradient forces and
$1\times10^{-4}$ for MACE-MP-0, so its saddle classification is approximate. (iv)
Orb-v2 and MACE-MP-0 in double precision were not run at this scale. (v) Each
$95\%$ Wilson interval of a relaxation share extends at most $5$ percentage points
from the share, and we do not extrapolate the shares to the full dataset.

\section{Symmetric Surrogates of Two-Dimensional Navier--Stokes Flow}
\label{sec:si_pde}

\subsection{Overview}
\label{sec:pde}

Autoregressive neural surrogates of partial differential equations, used in weather and turbulence modeling, apply one learned step map thousands of times. We trained surrogates of two-dimensional incompressible Navier--Stokes flow on a $64\times64$ periodic square (viscosity $10^{-2}$, step $0.1$), whose dynamics commute with the wallpaper group $p4m=(\mathbb{Z}_{64}\times\mathbb{Z}_{64})\rtimes D_4$ of grid translations, rotations by $90^\circ$ and reflections, with reflections reversing the sign of the vorticity, a pseudoscalar (Section~\ref{sec:si_pde}). A Fourier neural operator (FNO)~\cite{li2021fno} is translation-equivariant by construction, because its spectral convolution is a group convolution over the cyclic translations, the abelian case of $\starG$. The $\starG$ product over the full, non-abelian group $p4m$ therefore defines a symmetric FNO, which recovers the group-equivariant FNO of Helwig and collaborators~\cite{helwig2023gfno} from the algebra, so the experiment tests the consequences of exact symmetry for an established architecture. We compared it with an FNO trained without and with augmentation by the eight rotations and reflections~\cite{brandstetter2022lie}, over training sets of $4$ to $256$ trajectories with three seeds each (Figure~\ref{fig:pde}).

The symmetry test uses the fields invariant under the point group $D_4$ of rotations and reflections about the center of the square. Under this flow these symmetric fields are dynamically unstable, with symmetry-breaking perturbations growing at $0.03$ to $0.19$ per unit time, the regime of the exponential accumulation law. The $\starG$-FNO was exact, with a per-call equivariance error of $6\times10^{-16}$ in double precision and a symmetry-breaking part of $6\times10^{-15}$ after $200$ steps. The augmented FNO added a symmetry-breaking amplitude of $0.024$ per call from symmetric initial conditions, which reached $0.84$ after $100$ steps. One projection per step onto the $D_4$-invariant fields reduced its error against the exact symmetric dynamics at time $10$ from $1.14$ to $0.16$, and projecting a state never increased its error against the symmetric reference in $9{,}600$ same-state checks, as Proposition~\ref{prop:enforce_guarantee} requires. Trained on $16$ trajectories, the $\starG$-FNO had a lower one-step error in its worst seed ($0.0035$) than the augmented FNO trained on $256$ in its best ($0.0055$), so it reached a lower error with $16$ times fewer trajectories, and the same bound held for the rollout error at time $1$ from generic initial conditions. At equal training-set size its one-step error was $1.3$ to $3.3$ times lower. Generic flows are chaotic at this viscosity, so every surrogate decorrelates from the truth by time $5$ to $10$ from generic initial conditions, and by time $20$ every run of both FNO variants had blown up, against $2$ of $12$ runs of the matched $\starG$-FNO.

\begin{figure}[!t]
\centering
\includegraphics[width=\textwidth]{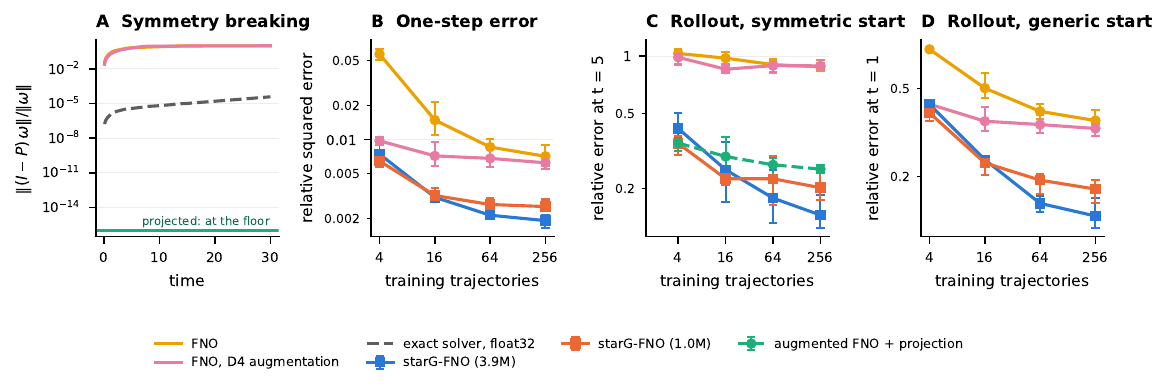}
\caption{\textbf{Surrogates of two-dimensional Navier--Stokes flow on a periodic square.} (\textbf{A}) Symmetry-breaking fraction $\|(I-P)\omega\|/\|\omega\|$ of the vorticity $\omega$ in rollouts from symmetric initial conditions. Both FNO variants lose the symmetry within a few time units, the exact solver in single precision accumulates rounding as the accumulation law predicts, and one projection per step holds the state at the floor. (\textbf{B}) One-step error against the number of training trajectories. (\textbf{C}) Rollout error at time $5$ from symmetric initial conditions, including the augmented FNO with one projection per step. (\textbf{D}) Rollout error at time $1$ from generic initial conditions. The $\starG$-FNO is the $\starG$ product over the wallpaper group $p4m$, shown at $3.9$M and $1.0$M parameters. Points are means over three seeds, and bars span the range over seeds.}
\label{fig:pde}
\end{figure}

This section documents the fluid-flow experiment of the main text (scripts and result files in Section~\ref{sec:si_codemap}).

\paragraph{Dynamics.} Two-dimensional incompressible Navier--Stokes flow in vorticity form on
the periodic square $[0,2\pi)^2$, on a $64\times64$ grid with viscosity $10^{-2}$ and no
forcing, integrated by a pseudo-spectral integrating-factor Runge--Kutta scheme (step
$0.005$, two-thirds dealiasing). Initial vorticity is random and band-limited
($1\le|k|\le8$, kinetic energy $0.5$), and one surrogate step spans $0.1$ time units.
The dynamics commute with the wallpaper group $p4m=(\bbZ_{64}\times\bbZ_{64})\rtimes D_4$,
with reflections acting on the pseudoscalar vorticity with a sign, and the solver is
equivariant to $6\times10^{-16}$. The symmetric subspaces are dynamically unstable: a
symmetry-breaking perturbation of relative size $10^{-10}$ grows at $0.03$ to $0.19$ per unit
time from $D_4$-symmetric initial conditions, the unstable regime of
Proposition~\ref{prop:depth_law_growth}.

\paragraph{Models.} The FNO~\cite{li2021fno} has width $32$, $16$ retained modes and four
layers ($4.20$M parameters). Its spectral convolution is a group convolution over
$\bbZ_N\times\bbZ_N$ with the discrete Fourier transform, the abelian case of the $\starG$
product, so it is translation-equivariant and has no constraint linking rotated or reflected
wavevectors. It is trained without augmentation and with one random element of $D_4$ per
training window. The $\starG$-FNO is the $\starG$ product over $p4m$. The field is lifted to
the regular representation of $D_4$ (eight orientation channels, with the sign $\det g$ of
the pseudoscalar), the spectral convolutions act as $D_4$ group convolutions across
orientation channels with one weight per $D_4$ orbit of wavevector and relative orientation,
the pointwise nonlinearities act on the regular representation, and the output is the
signed $D_4$ average. This recovers the group-equivariant FNOs of~\cite{helwig2023gfno}. We
trained two widths, $3.94$M parameters (matched to the FNO) and $0.98$M. The repair is the
Reynolds projector of $D_4$ applied once per step, and $D_4$ frame averaging of the step
(eight calls) is the comparison.

\paragraph{Exactness.} With trained weights the per-call equivariance residual of the
$\starG$-FNO, relative to the step increment, is $6.3\times10^{-16}$ in double and
$3.0\times10^{-7}$ in single precision. After $200$ steps from rotated or reflected initial
conditions the residual relative to the state is $6.1\times10^{-15}$ and $4.9\times10^{-6}$.

\paragraph{Repair.} An FNO trained with augmentation on $96$ trajectories ($7{,}500$ gradient
steps, four-step unrolled loss) was evaluated in single precision from $16$ $D_4$-symmetric
initial conditions over $300$ steps, against the solver in double precision projected at each
step. Its per-call equivariance error is $1.4\%$ of the step increment, and from a
symmetric state one call creates a symmetry-breaking amplitude of $0.024$ ($0.034$ without
augmentation), which reaches $0.17$ after $10$ steps and $0.84$ after $100$. The relative
error against the symmetric dynamics at times $10$ and $30$ is $1.14$ and $1.51$ raw and
$0.160$ and $0.201$ with one projection per step. The enstrophy error at time $30$ falls from
$0.45$ to $0.086$, and the relative error of the energy spectrum from $4.05$ to $0.093$. On a
symmetric state, frame averaging and the projection coincide, so the eight-call and
one-call repairs give identical rollouts here. Projecting the raw state never increased its
error in $9{,}600$ checks and lowered it by a factor of $0.65$ to $0.68$ on average. From
$C_4$-symmetric initial conditions every arm, the repaired ones included, diverges after time
$10$ to $20$, so those runs are reported to time $10$ only (errors $0.85$ raw, $0.55$
projected, $0.46$ frame-averaged).

\begin{table}[!ht]
\centering
\caption{Training-set-size sweep. One-step error is the mean squared error relative to the
squared step increment. Rollout errors are relative to the exact dynamics, at time $5$ from
$D_4$-symmetric initial conditions and at time $1$ from generic ones. All runs use $2{,}400$
gradient steps. Means over three seeds.}
\label{tab:pde_sweep}
\small
\begin{tabular}{@{}rcccccccc@{}}
\toprule
 & \multicolumn{4}{c}{One-step error} & \multicolumn{4}{c}{Rollout error, symmetric / generic start} \\
$N$ & FNO & FNO aug. & $\starG$ 3.9M & $\starG$ 1.0M & FNO & FNO aug. & $\starG$ 3.9M & $\starG$ 1.0M \\
\midrule
$4$   & $0.057$  & $0.0097$ & $0.0074$ & $0.0064$ & $1.03/0.75$ & $0.99/0.42$ & $0.42/0.42$ & $0.35/0.39$ \\
$16$  & $0.015$  & $0.0072$ & $0.0031$ & $0.0032$ & $0.98/0.50$ & $0.85/0.36$ & $0.25/0.24$ & $0.23/0.23$ \\
$64$  & $0.0086$ & $0.0068$ & $0.0021$ & $0.0027$ & $0.90/0.39$ & $0.89/0.34$ & $0.18/0.15$ & $0.23/0.19$ \\
$256$ & $0.0071$ & $0.0062$ & $0.0019$ & $0.0025$ & $0.88/0.36$ & $0.89/0.33$ & $0.15/0.13$ & $0.20/0.18$ \\
\bottomrule
\end{tabular}
\end{table}

\paragraph{Data efficiency.} Three comparisons hold in every seed. The matched $\starG$-FNO
trained on $16$ trajectories has a lower one-step error in its worst seed ($0.0035$) than the
augmented FNO trained on $256$ in its best ($0.0055$), a data ratio of at least $16$. The same
holds for the generic rollout error at time $1$ ($0.246$ against $0.306$). From symmetric
initial conditions the $\starG$-FNO trained on $4$ trajectories reaches an error of $0.499$ at
time $5$ in its worst seed, against $0.840$ for the augmented FNO at $N=256$ in its best, a
ratio of at least $64$, and neither FNO arm reaches $0.5$ within the sweep. Interpolated
factors are less reliable, because the one-step error of the augmented FNO depends weakly on
$N$ ($0.0097$ at $N=4$, $0.0062$ at $N=256$) compared with its spread over seeds, so its
crossing of a fixed error is ill-conditioned ($13$ to $71$ trajectories across seeds for the
one-step error $0.0062$, which the $\starG$-FNO reaches with $4$ to $7$). Interpolating on
log scales, the factor relative to the augmented FNO has a median of $3.6$ (range $1.8$ to
$17.8$) over the nine pairs of seeds and a bootstrap median of $25$ ($95\%$ interval $2.2$
to $64$) for the one-step error, and $2.1$ (range $1.5$ to $8.0$) and $4.4$ ($1.6$ to $22.8$)
for a generic rollout error of $0.35$ at time $1$. From symmetric initial conditions the
augmented FNO with one projection per step reaches $0.32$ to $0.37$ at $N=4$ and $0.25$ at
$N=256$, ahead of the matched $\starG$-FNO at $N=4$ and behind it from $N=16$ on. After $100$
steps from rotated or reflected generic initial conditions, the two FNO arms differ from the
rotated rollout by $1.08$ to $2.12$ relative to the state, against $1.0\times10^{-6}$ to
$9.7\times10^{-5}$ for the $\starG$-FNOs.

\paragraph{Cost.} Per step for a batch of $16$ fields in single precision: FNO $1.9$~ms, FNO
with projection $2.5$~ms, frame averaging $25.4$~ms, $\starG$-FNO $13.5$~ms ($3.9$M) and
$5.3$~ms ($1.0$M), exact solver $41.4$~ms (double precision, $20$ substeps).

\paragraph{Caveats.} Generic flows at this viscosity are chaotic, so every surrogate decorrelates
from the solver by time $5$ to $10$, which is why the generic comparison is made at time $1$.
We count a run as unstable when its generic rollout error exceeds $2$ at time $20$, since
chaotic decorrelation alone saturates near $1.4$. All $24$ runs of the two FNO arms are
unstable, against $2$ of $12$ for the matched $\starG$-FNO and $5$ of $12$ for the $1.0$M
version, including all three seeds at $N=4$. The instability of the matched $\starG$-FNO at
$N=256$ occurs in seed $0$ only, and seeds $1$ and $2$ stay at $0.49$ and $0.71$. Rollout
stability was tuned for no arm. A convolutional surrogate tried first was unstable in rollout and was replaced by the
FNO, and its runs are not used.

\section{Compression of Symmetric Molecular Tensors}
\label{sec:si_mol_compress}

The Eckart--Young theorem for $\starG$ of the main text predicts that data organized by a symmetry has low rank in the algebra even when it has high rank as a matrix. We tested this on two molecular data sets represented as tensors over the chiral octahedral group: QM9 features sampled in all $24$ octahedral frames~\cite{ramakrishnan2014qm9}, and QM7-X molecular polarizability tensors~\cite{hoja2021qm7x}, which transform as $\alpha\mapsto R_g\alpha R_g^{\top}$ and decompose into the irreducible representations A$_1$, E and T$_2$. The $\starG$ singular value decomposition reconstructed both data sets at rank~$3$ with relative error below $10^{-6}$ (Figure~\ref{fig:si_mol_compress}). At the same rank the matrix singular value decomposition of the lateral-slice matricization left $27\%$ error on QM9 and $21\%$ on QM7-X and needed ranks $11$ and $6$ for exact recovery, while Tucker decomposition performed worse throughout. Each data set is a tensor of $500$ molecules by $14$ features (QM9) or $9$ polarizability components (QM7-X) by $24$ frames, and both decompositions compress the whole tensor, so a rank counts degrees of freedom shared across molecules. At matched rank the $\starG$ factors occupy $2.7\%$ and $1.7\%$ more memory than the matrix factors, so the comparison holds at essentially equal storage. For the polarizability tensor the three retained singular elements are the three irreducible channels the tensor occupies. These tensors are generated by applying the $24$ frame operations to each molecule, so their low rank in the algebra follows from that construction, and the comparison measures how much of it a decomposition that ignores the group recovers. The QM9 tensor holds QM9 angular features over the $24$ octahedral frames.

\begin{figure}[!ht]
\centering
\includegraphics[width=\textwidth]{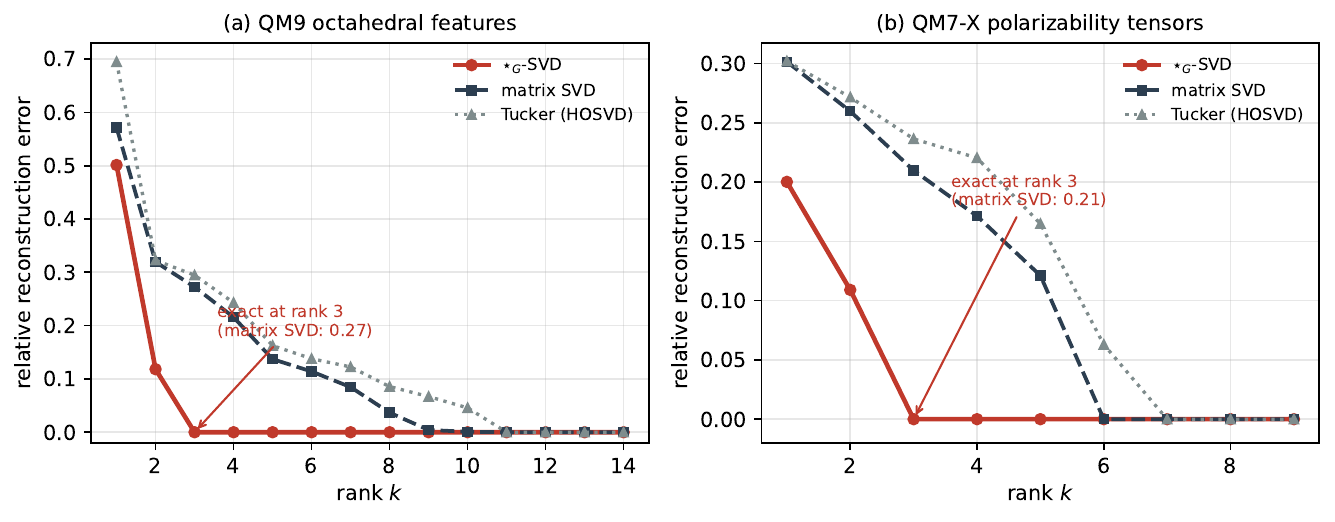}
\caption{\textbf{Compression of symmetric molecular tensors.} Relative reconstruction error against rank for the $\starG$ singular value decomposition, the matrix singular value decomposition of the lateral-slice matricization, and Tucker decomposition. (\textbf{a}) QM9 features over the $24$ octahedral frames. (\textbf{b}) QM7-X polarizability tensors. The $\starG$ decomposition reconstructs both at rank~$3$, where the matrix decomposition leaves $21$ to $27\%$ error.}
\label{fig:si_mol_compress}
\end{figure}

\section{Details of the Main-Text Results}
\label{sec:si_main_details}

This section records details that the main text states in shorter form, in the
order of the main-text sections.

\paragraph{Per-call violations.} A projected relaxation converges to the symmetric
stationary point also when the physical system would leave it, as at the
octahedral tilt instabilities of perovskites, so an unprojected phonon calculation
remains the test of whether the symmetric structure is stable.

\paragraph{Crystal audit.} The largest splitting of a protected multiplet under
Orb-v3 with the direct head exceeded $1$~cm$^{-1}$ in $8$ of the $31$ crystals with
protected multiplets. The Hessian identity uses the linearity of differentiation.
In the screening sample, at a tolerance of $10^{-3}$~\AA{} the attributed shares are
$69\%$, $97\%$ and $2\%$ under Orb-v3 with gradient forces, Orb-v3 with the direct
head and MACE-MP-0, and unconstrained MACE-MP-0 relaxations stayed at $105$ of its
$114$ unstable structures (Section~\ref{sec:si_fm_screening}).

\paragraph{Crystal tensors.} The screen of the main text addresses training sets
with few measured tensors, the regime of materials discovery. Projection reduced the
reconstruction error of $98$ to $100\%$ of individual predictions across the sweep,
and the smallest Voigt eigenvalue of the measured tensors is $0.67$. The
piezoelectric tensors were rotated from the IEEE orientation into the frame of
their structures, and the one cubic dielectric that remained birefringent did so in
one of three seeds. For the lookup-free evaluation the point groups were recomputed
by spglib at a tolerance of $10^{-2}$. The ranking of the largest piezoelectric
modulus improved from Spearman $\rho=0.44$ to $0.57$, and the ranking of
birefringence among non-cubic dielectrics from $0.28$ to $0.41$, with the same
behavior for the regressor given the symmetry class, and no increase of error
exceeded $0.6\%$.

\paragraph{Fluid flow.} The surrogates use viscosity $10^{-2}$ and a step of $0.1$.
Augmentation lowered the one-step error of the FNO and left it approximately
equivariant.

\paragraph{Compression.} Evaluation uses $500$ disjoint geometries for
MACE-OFF23-small and $300$ for the larger models, and $2$~ps of constant-energy
dynamics on three molecules for MACE-OFF23-small.

\section{Scope of the Theory, Group Selection and Further Settings}
\label{sec:si_scope}

\paragraph{Complete invariants.} The per-irreducible-representation power spectrum used by our descriptors is a second-order invariant and therefore cannot distinguish every pair of distinct orbits~\cite{sanborn2023bispectral,mataigne2024selective}. The bispectrum, which is a complete invariant for generic signals on a compact group~\cite{kakarala2012bispectrum}, is a $\starG$ product followed by a Clebsch--Gordan coupling of blocks, so the algebra can also express a complete invariant featurization. Finite point groups resolve angular structure only to their own order, and a larger subgroup reduces the aliasing (Section~\ref{sec:ico_resolution}).

\paragraph{Selecting the group.} Every result of the main text assumes that the
symmetry group is known. The group can in principle be selected by comparison
over a candidate library, and automatic symmetry discovery is outside our scope. Existing
approaches learn continuous~\cite{dehmamy2021lieconv,yang2023gan,hu2024liesd},
finite~\cite{karjol2023discrete}, local~\cite{bhat2025atlasd} or
representation-level~\cite{huh2024symbolic} symmetries, as well as symmetry
breaking relative to an assumed group~\cite{wang2023relaxed} and subgroup
selection within one network through symmetry-breaking
inputs~\cite{goel2026anysubgroup}. Selection from a candidate library gives up
that flexibility in exchange for the guarantee that every candidate algebra is
exactly equivariant.

\paragraph{Realizability of the low-rank operator.} The Eckart--Young certificate
bounds the approximation error. It does not guarantee that the retained
equivariant subspace admits a localized symmetric basis. For finite groups the
transform domain is a finite direct sum of blocks, so the question does not
arise. In the band-limited continuous setting (Section~\ref{sec:si_compact}),
singular values are gauge-invariant while topological information can reside in
the singular vectors, so operators with identical singular-value fields can have
different Chern numbers and different Wannier-localizability properties. The
certificate is therefore an approximation guarantee, and a topological
localization criterion requires further diagnostics such as symmetry indicators
or Berry curvature.

\paragraph{Domain of the theory.} The construction is formulated for finite and
compact groups, where the group-Fourier transform gives a canonical discrete
decomposition into irreducible representations. A formal extension to separable,
unimodular, type~I locally compact groups replaces the sum over irreducible
representations by a Plancherel direct integral and the singular value
decomposition by a measurable field of blockwise
decompositions~\cite{folland2016course,franco2024measurability}. For
Poincar\'e symmetry, the Mackey construction reduces the relevant fibers to
compact little-group representations (Section~\ref{sec:wignermackey}). Beyond
the type~I setting a canonical decomposition is generally
unavailable~\cite{glimm1961type}, and groups that are not locally compact fall
outside the Haar and Plancherel framework. The band-limited compact setting of
Section~\ref{sec:si_compact} is therefore the natural domain of the present
formulation.

\paragraph{Further settings.} The same structure applies to crystallographic
point groups in materials science, to compact gauge groups in particle physics
and to chirality and other discrete symmetries in molecular science. Lattice
gauge theory is a natural extension: after each link is truncated to a finite
band of irreducible representations, the problem lies in the band-limited
compact setting in which the $\starG$ singular value decomposition and its
selection rules apply.

\section{Code and Data Map}
\label{sec:si_codemap}

The archived code deposit holds every script and result file behind the main text
and this Supplementary Information. Paths are relative to the \texttt{python/}
folder of the deposit, and \texttt{large\_scale/REPRODUCE.md} maps every table and
figure to its script, inputs and seeds. Every result file listed is committed, so
the numbers in the text can be checked without rerunning.

\begin{table}[!ht]
\centering
\caption{Scripts and committed result files for each part of the paper.}
\label{tab:codemap}
\scriptsize
\renewcommand{\_}{\textunderscore\allowbreak}
\begin{tabular}{@{}>{\raggedright\arraybackslash}p{3.4cm}>{\raggedright\arraybackslash}p{6.2cm}>{\raggedright\arraybackslash}p{4.9cm}@{}}
\toprule
\textbf{Part} & \textbf{Scripts} & \textbf{Result files} \\
\midrule
Controlled compounding (\ref{sec:si_compounding}) & \texttt{large\_scale/symmetry\_compounding.py} & \texttt{large\_scale/results/symmetry\_compounding.json} \\
Canonicity check (\ref{sec:necessity}) & \texttt{large\_scale/necessity\_fourier.py} & printed by the script \\
Band-limited $\mathrm{SO}(3)$ (\ref{sec:si_compact}) & \texttt{large\_scale/bandlimited\_so3\_experiment.py} & printed by the script \\
Finite subgroups (Proposition~\ref{prop:subgroup_si}) & \texttt{large\_scale/subgroup\_sufficiency.py} & \texttt{large\_scale/results/subgroup\_sufficiency.json} \\
Lattice gauge, Wigner--Mackey (\ref{sec:wignermackey}) & \texttt{large\_scale/su2\_lattice\_gauge\_demo.py}, \texttt{large\_scale/wigner\_mackey\_orbit\_invariance.py} & printed by the scripts \\
Composed symmetry (\ref{sec:composed_methods}) & \texttt{large\_scale/starg\_torch/composed\_symmetry\_knockout.py} & per-seed records in the deposit \\
Feature extraction (Algorithm~\ref{alg:features}) & \texttt{../core/extractStarGFeatures.m} & -- \\
Angular resolution (\ref{sec:ico_resolution}) & \texttt{large\_scale/icosahedral\_resolution.py}, \texttt{large\_scale/starg\_torch/icosahedral.py} & printed by the scripts \\
QM9 accuracy (\ref{sec:qm9_accuracy}, \ref{sec:aug_invariance}) & \texttt{large\_scale/train\_baseline\_mlp.py}, \texttt{large\_scale/train\_starg.py}, \texttt{large\_scale/isomer\_audit.py}, \texttt{large\_scale/build\_summary.py}, \texttt{large\_scale/verify\_tables.py} & \texttt{large\_scale/results/summary.csv} and per-seed JSON files \\
ENN probe (\ref{sec:enn_probe}) & \texttt{large\_scale/enn\_irrep\_probe.py} & per-seed records in the deposit \\
Compression of molecular tensors (main text) & \texttt{large\_scale/compression\_experiment.py}, \texttt{large\_scale/make\_compression\_fig.py} & \texttt{large\_scale/results/compression\_*.json} \\
Crystal enforcement (\ref{sec:si_symmetry_informed}) & \texttt{materials/symmetry\_consistency.py}, \texttt{materials/enforcement\_screening.py}, \texttt{materials/anisotropy\_screening.py}, \texttt{materials/symmetry\_informed\_audit.py}, \texttt{materials/crystal\_gnn\_data\_efficiency.py}, \texttt{materials/make\_enforcement\_fig.py} & \texttt{materials/results/} \\
Tensor ranks (\ref{sec:si_tensor_ranks}) & \texttt{materials/tensor\_rank\_breadth.py} & \texttt{materials/results/tensor\_rank\_breadth.json} \\
Projection without lookups (\ref{sec:si_lookup_free}) & \texttt{materials/lookup\_free\_projection.py}, \texttt{materials/lookup\_free\_tables.py} & \texttt{materials/results/lookup\_free\_projection.json} \\
Crystal data efficiency (\ref{sec:si_data_efficiency}) & \texttt{materials/crystal\_data\_efficiency.py} & \texttt{materials/results/} \\
Force-field trajectories (\ref{sec:foundation_protocol}) & \texttt{foundation\_md/depth\_law\_harness.py}, \texttt{foundation\_md/run\_mace\_rollout.py}, \texttt{foundation\_md/orb\_crystal\_rollout.py}, \texttt{foundation\_md/phonon\_audit.py}, \texttt{foundation\_md/accumulation\_exponent.py}, \texttt{foundation\_md/make\_depth\_law\_fig.py} & \texttt{foundation\_md/*.json} \\
Force-field audit (\ref{sec:si_fm_audit}) & \texttt{foundation\_md/fm\_symmetry\_audit.py}, \texttt{foundation\_md/fm\_audit\_summary.py}, \texttt{foundation\_md/fm\_audit\_delta\_check.py}, \texttt{foundation\_md/make\_fm\_audit\_fig.py} & \texttt{foundation\_md/fm\_symmetry\_audit\_results.json}, \texttt{foundation\_md/fm\_audit\_summary.json} \\
Force-field screening (\ref{sec:si_fm_screening}) & \texttt{foundation\_md/fm\_screening\_audit.py}, \texttt{foundation\_md/fm\_screening\_summary.py}, \texttt{foundation\_md/fm\_screening\_mp\_label.py}, \texttt{foundation\_md/fm\_screening\_certificate.py}, \texttt{foundation\_md/fm\_screening\_converged.py}, \texttt{foundation\_md/fm\_orientation\_repro.py}, \texttt{foundation\_md/make\_fm\_screening\_fig.py} & \texttt{foundation\_md/fm\_screening\_stage1.json}, \texttt{foundation\_md/fm\_screening\_stage2v2.json}, \texttt{foundation\_md/fm\_screening\_summary.json}, \texttt{foundation\_md/fm\_screening\_mp\_label.json}, \texttt{foundation\_md/fm\_screening\_certificate.json}, \texttt{foundation\_md/fm\_screening\_converged.json}, \texttt{foundation\_md/fm\_orientation\_repro.json}, \texttt{foundation\_md/fm\_orientation\_inputs.json} \\
MACE compression (\ref{sec:si_mace_compress}) & \texttt{foundation\_md/starg\_mace\_compress\_v2.py}, \texttt{foundation\_md/make\_mace\_compress\_v2\_tables.py}, \texttt{foundation\_md/starg\_mace\_compress\_v3.py}, \texttt{foundation\_md/starg\_mace\_compress\_v3\_calib.py}, \texttt{foundation\_md/starg\_mace\_compress\_v3\_mp\_equivcheck.py}, \texttt{foundation\_md/mace\_numerical\_rank.py}, \texttt{foundation\_md/make\_mace\_compress\_v3\_tables.py}, \texttt{foundation\_md/make\_mace\_compress\_v3\_fig.py}, \texttt{foundation\_md/starg\_mace\_compress\_v3\_md.py}, \texttt{foundation\_md/merge\_corollary\_v3.py}, \texttt{foundation\_md/summarize\_compress\_v3\_md.py}, \texttt{foundation\_md/run\_corollary\_table.sh}, \texttt{foundation\_md/run\_corollary\_md.sh}, \texttt{foundation\_md/mace\_irrep\_anatomy.py}, \texttt{foundation\_md/make\_mace\_irrep\_anatomy\_fig.py} & \texttt{foundation\_md/starg\_mace\_compress\_v2*.json}, \texttt{foundation\_md/starg\_mace\_compress\_v3\_off\_*.json}, \texttt{foundation\_md/starg\_mace\_compress\_v3\_mp\_large.json}, \texttt{foundation\_md/starg\_mace\_compress\_v3\_calib\_off\_large.json}, \texttt{foundation\_md/starg\_mace\_compress\_v3\_mp\_equivcheck.json}, \texttt{foundation\_md/mace\_numerical\_rank\_*.json}, \texttt{foundation\_md/starg\_mace\_compress\_v3\_md.json}, \texttt{foundation\_md/starg\_mace\_compress\_v3\_*\_corollary.json}, \texttt{foundation\_md/mace\_irrep\_anatomy.json} \\
Fluid-flow surrogates (\ref{sec:si_pde}) & \texttt{pde\_surrogate/ns\_solver.py}, \texttt{pde\_surrogate/model.py}, \texttt{pde\_surrogate/starg\_fno.py}, \texttt{pde\_surrogate/run\_sweep.sh}, \texttt{pde\_surrogate/run\_seeds.sh}, \texttt{pde\_surrogate/summarize\_sweep.py} & \texttt{pde\_surrogate/results/*.json} \\
Figure 1 (main text) & \texttt{large\_scale/make\_figure1.py} & reads the result files above \\
Formal verification (\ref{sec:lean}) & \texttt{../lean/} (\texttt{lake build}; axiom certificate \texttt{StarG/Audit.lean}) & -- \\
\bottomrule
\end{tabular}
\end{table}

\section{Formal Verification in Lean~4}
\label{sec:lean}

The core algebraic results in this paper have been machine-verified in the
Lean~4 proof assistant~\cite{demoura2021lean4} using the Mathlib
library~\cite{mathlib2020}.  The formalization comprises $3{,}608$ lines of Lean~4 across fifteen modules ($3{,}658$ with the root file), with zero unresolved proof obligations
(\texttt{sorry}).  It certifies the $\starG$ algebra laws, equivariance and
the invariance identities, the $\starG$-SVD and its Eckart--Young optimality,
the further matrix-mimetic operations ($\starG$-QR, symmetric
eigendecomposition, and least-squares), the accumulation law for a per-step
equivariance defect and the enforcement guarantee
(Section~\ref{sec:si_depthlaw}), the product-group factorization, the
octahedral selection rules, and the concrete witness systems; the axiom
dependencies of every theorem are
exhibited explicitly, so the certificate is exactly as strong as the declared
axiom budget it rests on: four axioms in the real development, two complex-local analogues confined to the complex module's cyclic witnesses,
and, for the \texttt{native\_decide} selection-rule computations only, Lean's
compiled-evaluation axiom \texttt{Lean.ofReduceBool}.

\subsection{Architecture}

The formalization is organized into fifteen modules mirroring the paper's logical structure (line counts from the repository at submission):

\begin{center}
\resizebox{\textwidth}{!}{%
\begin{tabular}{@{}llrl@{}}
\toprule
\textbf{Module} & \textbf{Paper section} & \textbf{Lines} & \textbf{Content} \\
\midrule
\texttt{Basic.lean}         & SI~\S\S1--4   & 80  & Convolution tensor, $\starG$ product, transpose \\
\texttt{Algebra.lean}       & SI~\S4        & 110 & Associativity, distributivity, identity laws \\
\texttt{Equivariance.lean}  & SI~\S8        & 159 & Equivariance, Frobenius/Fourier-power invariance \\
\texttt{ProductGroup.lean}  & Theorem~2     & 160 & Product-group factorization, Kronecker irreps \\
\texttt{EckartYoung.lean}   & Theorem~1     & 215 & Matrix Eckart--Young, proved from two bridge lemmas \\
\texttt{SVD.lean}           & Theorem~1     & 504 & Irrep systems, $\starG$-SVD, Eckart--Young optimality at a rank profile \\
\texttt{MatrixMimetic.lean} & Methods       & 281 & $\starG$-QR, spectral, least-squares; transpose bridge lemma \\
\texttt{WignerEckart.lean}  & Wigner--Eckart section & 91  & Octahedral characters, derived selection rules \\
\texttt{ConcreteWitness.lean} & SI~\S\ref{sec:lean} & 192 & Axiom-free $C_2$ irrep system; optimality suite instantiated \\
\texttt{OctahedralWitness.lean} & Wigner--Eckart section & 455 & Four of five octahedral irreps, axiom-free, with irreducibility \\
\texttt{ComplexCore.lean}   & SI~\S\ref{sec:lean} & 566 & Complex-typed layer; cyclic (DFT) witnesses; convolution theorem \\
\texttt{DepthLaw.lean}      & \S\ref{sec:si_depthlaw}  & 214 & Accumulation of a per-step defect; enforcement for an approximately symmetric truth \\
\texttt{Canonicity.lean}    & SI~\S\ref{sec:necessity} & 253 & Part~I of Theorem~\ref{thm:necessity}; Lemma~\ref{lem:twisted} \\
\texttt{DepthLawNonlinear.lean} & SI~\S\ref{sec:si_depthlaw} & 217 & Accumulation for Lipschitz equivariant steps (Proposition~\ref{prop:depth_law_nonlinear}) \\
\texttt{Audit.lean}         & (all)         & 111 & \texttt{\#print axioms} certificate for every theorem \\
\bottomrule
\end{tabular}}
\end{center}

\subsection{Axiom Budget}

Four standard results from finite-group harmonic analysis and matrix linear algebra are axiomatized here; most have no Mathlib
counterpart in the form we use. Axiomatizing them is a deliberate trust
boundary rather than a gap: each is a named, citable classical fact, stated in
a faithful form that cannot be satisfied trivially (\texttt{matrix\_qr}, for
instance, requires an upper-triangular factor, so the degenerate
$Q = I$, $R = M$ reading is excluded), and confining the trusted statements to
these classical facts keeps every $\starG$-specific reduction inside the
machine-checked kernel; discharging the Ky~Fan lemma and QR from Mathlib
primitives is mechanical but substantial formalization work, orthogonal to the
claims this paper certifies. Three former axioms are now proved outright: the
group-Fourier convolution theorem (\texttt{fourier\_multiplicative}), directly from
the homomorphism property by reindexing the group sum; the real symmetric spectral
theorem, from Mathlib's \texttt{Matrix.IsHermitian.spectral\_theorem} (giving the
$\starG$ symmetric eigendecomposition); and the least-squares minimizer, from the
normal equations over $\bbR$ (giving $\starG$ least-squares). Discharging these dropped the custom-axiom count from nine to six. Fourier injectivity and surjectivity are now proved as well. Injectivity follows from \texttt{parseval\_group}, because a vanishing sum of squared Fourier blocks forces every entry of the difference to vanish, and surjectivity follows from injectivity by a dimension count using the completeness of the irrep system. The whole Fourier layer therefore rests on the Plancherel identity alone, and the real development declares four axioms. Every other
statement is derived from first principles.  The classical matrix
Eckart--Young theorem itself is \emph{not} among the axioms: it is proved in
\texttt{EckartYoung.lean} (\texttt{matrix\_eckart\_young}) from the two bridge
lemmas listed. The matrix-mimetic bridge lemma
(\texttt{fourierBlock\_starTranspose}: the Fourier transform sends the $\starG$
transpose to the transpose of each irrep block), on which the QR, spectral, and
least-squares $\starG$-orthogonality reductions rest, is likewise proved, not
assumed, and depends only on Lean's core axioms.

\begin{center}
\resizebox{\textwidth}{!}{%
\begin{tabular}{@{}lll@{}}
\toprule
\textbf{Axiom} & \textbf{Content} & \textbf{Reference} \\
\midrule
\texttt{parseval\_group} & Plancherel identity, finite groups & \cite{serre1977linear}, \S2.4 \\

\texttt{rank\_le\_proj\_capture} & rank-$\le k$ matrix is captured by a rank-$\le k$ projection & standard \\
\texttt{ky\_fan\_with\_witness} & Ky Fan (1949): optimal rank-$k$ projection for symmetric $M$ & standard \\
\texttt{matrix\_qr} & Matrix QR factorization ($Q$ orthonormal, $R$ triangular) & standard \\
\bottomrule
\end{tabular}}
\end{center}

The complex-typed module \texttt{ComplexCore.lean}, which exists to build
concrete cyclic-group (DFT) witness systems over $\bbC$ that the real
development cannot express, declares two \emph{complex-local} analogues (\texttt{parseval\_groupC} and \texttt{matrix\_best\_rank\_k\_approxC}), and derives complex Fourier injectivity and surjectivity in the same way as the real ones. They
are used only by that module's complex Eckart--Young corollaries; none of the
real-development theorems above depends on them, as the audit certificate
shows. Counting them, the repository contains six \texttt{axiom} declarations in total: four real, two complex-local. Both Plancherel statements remain assumptions, because discharging them needs the Schur orthogonality relations for matrix coefficients, which Mathlib does not provide.

In addition, the octahedral selection-rule theorems are decided by exhaustive
native computation (\texttt{native\_decide}) and therefore depend on Lean's
compiled-evaluation axiom \texttt{Lean.ofReduceBool}: they trust Lean's
compiler as an oracle for finite rational arithmetic over the $24$ group
elements.  The hypotheses of the optimality theorem are carried by an
\texttt{IrrepSystem} structure requiring completeness
($\sum_\rho d_\rho^2 = |G|$), irreducibility of each member (no proper
invariant subspace), \emph{absolute} irreducibility of each member (Schur
condition: every matrix commuting with the representation is scalar), and
pairwise inequivalence (no invertible intertwiner).  The Schur field is what
makes the Fourier axioms sound over $\mathbb{R}$: without it, a single
real-irreducible representation of complex type can satisfy
$\sum_\rho d_\rho^2 = |G|$ (for example, the two-dimensional rotation
representation of $C_4$, with $d^2 = 4 = |C_4|$) while spanning only part of
the group algebra, which would falsify Parseval and Fourier injectivity.
With absolute irreducibility, completeness, and inequivalence, the
Artin--Wedderburn decomposition forces the system to exhaust
$\mathbb{R}[G]$, so the axioms hold for every admissible instantiation.

\subsection{Key Proof Techniques}

\paragraph{Associativity of $\starG$ (Proposition~4.2(i)).}
Rather than fragile nested calls to \texttt{Finset.sum\_comm},
we define an explicit \texttt{Equiv} on the 4-tuple product type
$\mathrm{Fin}\,p \times (G \times (\mathrm{Fin}\,m \times G))$ that
simultaneously permutes components and applies the bijection
$b \mapsto a^{-1}b$.  A single call to \texttt{Fintype.sum\_equiv} then
completes the proof, with the group-element arithmetic handled by the
\texttt{group} tactic.

\paragraph{Kronecker product of irreps (Theorem~2(iii)).}
The tensor-product representation $\rho_1 \otimes \rho_2$ is defined
entry-wise via \texttt{finProdFinEquiv.symm}, mapping
$\mathrm{Fin}(d_1 d_2)$ indices to pairs $\mathrm{Fin}\,d_1 \times
\mathrm{Fin}\,d_2$.  A \texttt{sum\_split} helper converts sums over
$\mathrm{Fin}(d_1 d_2)$ into double sums, after which the
\texttt{is\_hom} and \texttt{unitary} proofs factor naturally into
products of single sums via $\rho_i.\text{is\_hom}$ and
$\rho_i.\text{unitary}$.

\paragraph{Fourier power invariance (Corollary~7.2(ii)).}
The \texttt{fourierBlock\_leftAction} lemma shows that the group action
multiplies each Fourier block by $(I_\ell \otimes \rho(g))$.  An
\texttt{orthogonal\_preserves\_sum\_sq} lemma proves
$\sum_s (\sum_{s'} R_{s,s'} v_{s'})^2 = \sum_s v_s^2$ when
$R^\top R = I$, by expanding squares, exchanging sums, and applying
orthogonality.

\paragraph{Eckart--Young for $\starG$ (Theorem~1).}
The optimal approximation at a per-irrep rank profile $(k_\rho)$ is defined in the Fourier domain through Fourier surjectivity, itself a theorem: its Fourier block at each irrep $\rho$ is, by
construction, the best rank-$k_\rho$ matrix approximation of
$\hat{\calA}(:,:,\rho)$, supplied by \texttt{matrix\_eckart\_young}.  That
matrix-level theorem is \emph{proved} in \texttt{EckartYoung.lean}: a
Pythagorean split over orthogonal projections reduces best rank-$\le k$
approximation to maximizing the captured trace over rank-$\le k$ projections
(\texttt{rank\_le\_proj\_capture}), and the Ky Fan lemma
(\texttt{ky\_fan\_with\_witness}) identifies the top-$k$ eigenprojection as
the maximizer.  The global bound (\texttt{eckart\_young\_profile}) follows by
applying \texttt{parseval\_group}
to decompose the Frobenius error into per-irrep terms, multiplying each
per-irrep inequality by the positive Plancherel weight $d_\rho / |G|$, and
summing via \texttt{Finset.sum\_le\_sum}.  Two profiles are instantiated: the
uniform profile $k_\rho=k$ (\texttt{eckart\_young\_starG}), in which every
block has rank at most $k$, and the coupled profile $k_\rho=k\,d_\rho$ of the
element form (\texttt{eckart\_young\_element}), whose competitor class
contains every tensor of $\starG$-rank at most $k$ and whose optimum is the
leading-$k$-element truncation computed by the numerical pipeline.  The two
classes are proved to coincide when every irrep is one-dimensional
(\texttt{starGElementRankLE\_iff\_of\_dim\_one}).

\paragraph{Wigner--Eckart selection rules (\S2.2).}
The octahedral rotation group is realized concretely as $O \cong S_4$
(\texttt{Equiv.Perm (Fin 4)}) through its action on the four body diagonals
of the cube.  The five irreducible characters are constructed rather than
hardcoded: $A_1$ trivial, $A_2$ the sign character, $T_2$ the standard
representation ($\chi(\sigma) = \mathrm{fix}(\sigma) - 1$), and $T_1 = T_2
\otimes \mathrm{sign}$ come from explicit representations, while $E$ is
recovered as a class function from the regular character and certified
irreducible by $\langle \chi_E, \chi_E \rangle = 1$, its dimension
$\chi_E(1) = 2$ entering the verified count \texttt{dim\_sum} below.
\texttt{char\_orthonormal} proves
$\langle \chi_\rho, \chi_\sigma \rangle = \delta_{\rho\sigma}$, certifying
that the five are pairwise-inequivalent irreducibles, and \texttt{dim\_sum}
checks $\sum_\rho d_\rho^2 = 24$.  Every selection rule is then
\emph{derived} as a character inner product, never asserted by definition:
a scalar ($A_1$) operator connects an irrep only to itself
(\texttt{scalar\_selection}); $T_1 \otimes T_1$ contains $A_1$ exactly once
(\texttt{vector\_couples\_scalar}); and the symmetric square decomposes as
$\mathrm{Sym}^2(T_1) = A_1 \oplus E \oplus T_2$ with zero $T_1$ multiplicity
(\texttt{polarizability\_no\_T1}, \texttt{polarizability\_decomposition}).
The finite rational arithmetic over the $24$ group elements is discharged by
\texttt{native\_decide}; the build fails if any multiplicity is wrong (for
example, if $T_1$ and $T_2$ are interchanged).

\subsection{Verification Status}

The formalization achieves:
\begin{itemize}
\item Zero \texttt{sorry} (unresolved proof obligations) across all fifteen modules.
\item Four declared axioms in the real development, each a standard textbook
  result of finite-group
  harmonic analysis or matrix linear algebra with no Mathlib counterpart in the
  form we use (three former axioms are now proved outright: the group-Fourier
  convolution theorem, the real symmetric spectral theorem from Mathlib's
  \texttt{Matrix.IsHermitian.spectral\_theorem}, and the least-squares minimizer
  from the normal equations); Fourier injectivity and surjectivity are proved from Plancherel; the complex witness module adds the two complex-local analogues above, six \texttt{axiom} declarations in the repository in total.  Their
  exact usage, as certified by
  \texttt{\#print axioms} in \texttt{StarG/Audit.lean}: the Eckart--Young
  chain (\texttt{eckart\_young\_starG}, \texttt{eckart\_young\_error};
  Theorem~\ref{thm:ey_full}) uses \texttt{parseval\_group}, \texttt{rank\_le\_proj\_capture}, and \texttt{ky\_fan\_with\_witness}; the product-group Fourier factorization
  (Theorem~\ref{thm:prod_full}) now rests on the core axioms, its
  transform-multiplicativity being a theorem; the
  $\starG$-QR theorem uses \texttt{matrix\_qr}, while the $\starG$ symmetric
  eigendecomposition and $\starG$ least-squares minimizer carry no matrix axiom,
  being proved from Mathlib's \texttt{Matrix.IsHermitian.spectral\_theorem} and
  from the normal equations over $\bbR$ respectively; their
  $\starG$-orthogonality reductions rest on \texttt{fourier\_injective}
  together with the axiom-free transpose bridge lemma
  (\texttt{fourierBlock\_starTranspose}).  The octahedral selection-rule theorems
  additionally depend on \texttt{Lean.ofReduceBool} through
  \texttt{native\_decide}, as disclosed above.
\item Coverage of Theorems~\ref{thm:ey_full} and~\ref{thm:prod_full}, the
  algebraic identities of Proposition~\ref{prop:T_props} and the
  $\starG$-product properties, equivariance and Frobenius/Fourier-power
  invariance, the Kronecker product construction for product-group irreps,
  and the Wigner--Eckart selection rules from \S2.2.  Not formalized: the
  band-limited compact-group reduction (Theorem~4 of the main text) and its
  ENN corollary, whose mathematical content is the classical Peter--Weyl
  theorem.
\item All algebraic theorems (associativity, identity, distributivity,
  transpose, equivariance, Frobenius and per-irrep Fourier-power invariance)
  depend solely on Lean's three core axioms
  (\texttt{propext}, \texttt{Classical.choice}, \texttt{Quot.sound}); they
  introduce no project-level axioms.  (The formal \texttt{StarGSVD} structure
  records the unitarity of the factors but does not yet impose
  f-diagonality on $\calS$, so the accompanying existence lemma is a
  packaging statement rather than a verified factorization theorem; the
  optimality theorem above does not rely on it.)
\item A concrete, axiom-free \texttt{IrrepSystem} for the cyclic group $C_2$ (its
  trivial and sign representations) is constructed, and the full optimality suite
  is instantiated at it (\texttt{eckart\_young\_c2}, \texttt{starG\_qr\_c2},
  \texttt{starG\_spectral\_c2}, \texttt{starG\_lstsq\_c2}), each adding no custom
  axiom beyond its parent theorem. The theorems are therefore demonstrably
  non-vacuous rather than conditional over a never-inhabited hypothesis. The
  cyclic groups $\bbZ_n$ with $n\ge3$, whose complex irreducibles the real
  structure cannot express, are witnessed separately by the additive complex
  port described below; a concrete octahedral witness is the one remaining step.
\end{itemize}

\paragraph{Scope of the formalization: real versus complex.}
The formalization is stated over $\mathbb{R}$.  This faithfully covers groups
all of whose irreducible representations are of real type, such as the
octahedral group used in our Wigner--Eckart analysis (where
$\sum_\rho d_\rho^2 = 24$ holds over $\mathbb{R}$).  For cyclic groups the
complex DFT used in our experiments pairs conjugate one-dimensional complex
irreps into two-dimensional real blocks, so $\sum_\rho d_\rho^2 = |G|$ fails
over $\mathbb{R}$ and the cyclic groups $\bbZ_n$ with $n\ge3$ are not witnessed by
a real \texttt{IrrepSystem}.  The one cyclic group whose irreducible
representations are real, $C_2$, is witnessed concretely (above); for the octahedral
group $O\cong S_4$, four of its five irreducibles are formalized as genuine
irreducibles (the two one-dimensional representations and the three-dimensional
standard representation and its sign twist, all realized with rational matrix
entries and proven absolutely irreducible), so only the two-dimensional
representation, which requires irrational basis entries, and the assembled system
remain before the optimality suite instantiates at the flagship group.  To cover
the cyclic instantiations we carried out an additive complex port (module
\texttt{ComplexCore}): a complex-typed $\star_G$ layer whose unitarity is stated
with the conjugate transpose and whose Schur field is ordinary (over $\mathbb{C}$,
absolute irreducibility coincides with irreducibility).  In it we construct a
concrete complex \texttt{IrrepSystem} for every cyclic group directly from the DFT
characters $\rho_k(j)=\omega^{jk}$, instantiated at $\bbZ_3$, and prove the
group-Fourier convolution theorem and the resulting per-character
block-diagonalization; these results are axiom-free beyond Lean's core.  A complex
per-block Eckart--Young corollary is also instantiated, resting on three
complex-typed analogues of the classical Plancherel and matrix Eckart--Young facts (declared as complex-local axioms, disjoint from and leaving unchanged the four real axioms above), with complex Fourier surjectivity proved.  This inhabits the cyclic
$\bbZ_n$ hypotheses that $\mathbb{R}$ cannot; the mathematics of the paper is
unaffected, and the remaining caveat concerns only the octahedral instantiation of
the real suite.

\noindent
To our knowledge, this is the first machine-checked derivation of an
Eckart--Young-type optimality theorem for symmetry-preserving tensor
approximation, conditional on the classical linear-algebra facts declared as
axioms above.

\end{document}